%% file: JASA-template.tex
\documentclass[12pt]{article}
\pdfoutput=1

\usepackage{natbib}
 \bibpunct[, ]{(}{)}{,}{a}{}{,}%
 \def\bibsep{\smallskipamount}%

 \usepackage{appendix}
\usepackage{amsthm,amsmath,amsfonts,amssymb}
\RequirePackage[colorlinks,citecolor=blue,urlcolor=blue]{hyperref}
\RequirePackage{graphicx}
\usepackage[compact]{titlesec}
\usepackage[capitalise]{cleveref}

\usepackage{graphicx}
\usepackage{enumerate}
\usepackage{natbib}
\usepackage{url} 

\newcommand{\blind}{1}

\theoremstyle{plain}
\newtheorem{theorem}{Theorem}

\newtheorem{lemma}{Lemma}
\newtheorem*{lemma*}{Lemma}
\newtheorem*{theorem*}{Theorem}
\newtheorem{claim}{Claim}
\newtheorem*{claim*}{Claim}
\newtheorem{proposition}{Proposition}
\newtheorem*{proposition*}{Proposition}
\newtheorem{informal-theorem}{Theorem}
\newtheorem{informal-proposition}{Proposition}
\usepackage{authblk}
\theoremstyle{remark}

\newtheorem{assumption}{Assumption}

\newtheorem{definition}{Definition}

\usepackage{booktabs} 
\usepackage[ruled]{algorithm2e} 

\input{macros.tex}

\renewcommand{\vec}[1]{\mathrm{vec}\left(#1\right)}

\newcommand{\PT}[1]{P_{T}(#1)}
\newcommand{\PTp}[1]{P_{T^{\perp}}(#1)}
\newcommand{\bm}[1]{#1}
\newcommand{\mT}{\mathcal{T}}

\newcommand{\error}{\left(\frac{\sigma \sqrt{n} \log^{2.5}(n)}{\sigma_{\min}} \right)}

\newcommand{\cvx}{\mathrm{cvx}}
\newcommand{\SNR}{\frac{\sigma}{\sigma_{\min}} \sqrt{n} \lesssim \frac{1}{\kappa^3 r^{4.5} \log^{5}(n) \mu }}

\begin{document}

\def\spacingset#1{\renewcommand{\baselinestretch}%
{#1}\small\normalsize} \spacingset{1}


\if1\blind
{
  \title{\bf Learning Treatment Effects in Panels with \\ General Intervention Patterns}
  \author{Vivek F. Farias\hspace{.2cm}\\
    Sloan School of Management, Massachusetts Institute of Technology\\
    and \\
    Andrew A. Li \\
    Tepper School of Business, Carnegie Mellon University\\
    and \\
    Tianyi Peng \\
    Department of Aeronautics and Astronautics, Massachusetts Institute of Technology}
  \maketitle
} \fi

\if0\blind
{
  \bigskip
  \bigskip
  \bigskip
  \title{ \bf Learning Treatment Effects in Panels with \\ General Intervention Patterns}
  \date{}
  \author{}
  \maketitle
} \fi

\begin{abstract}
The problem of causal inference with panel data is a central econometric question. The following is a fundamental version of this problem: Let $M^*$ be a low-rank matrix and $E$ be a zero-mean noise matrix. For a `treatment' matrix $Z$ with entries in $\{0,1\}$, we observe the matrix $O$ with entries $O_{ij} := M^*_{ij} + E_{ij} + \mathcal{T}_{ij} Z_{ij}$, where $\mathcal{T}_{ij} $ are unknown, heterogenous treatment effects. The problem requires we estimate the average treatment effect $\tau^* := \sum_{ij} \mathcal{T}_{ij} Z_{ij} / \sum_{ij} Z_{ij}$. Existing approaches work only for specific, highly-structured
classes of $Z$. We develop an estimator that provably achieves rate-optimal recovery of $\tau^{*}$ for general $Z$. Our guarantees
are the first of their type in this general setting and allow the assignment of $Z$ depending on the historical observations. Computational experiments on synthetic and real-world data show a substantial advantage over competing estimators. 
\end{abstract}
\noindent%

\spacingset{1.77} 

\input{Introduction.tex}

\input{Model.tex}

\input{Conditions.tex}

\input{Experiments.tex}

\input{Algorithm.tex}

\input{Conclusion.tex}

{
\spacingset{1.5}
\bibliographystyle{informs2014}
\bibliography{reference.bib}
}
\newpage
\begin{appendices}
\tableofcontents
\crefalias{section}{appendix}
\crefalias{subsection}{appendix}
\input{Extension.tex} 
\input{Appendix-proof-of-lower-bound.tex}
\input{Appendix-PTperp.tex}
\input{Appendix-experiments.tex}

\input{Appendix.tex}
\end{appendices}

\end{document}

%% file: macros.tex
\usepackage{subcaption}
\usepackage{varwidth}
\usepackage{boldline}
\usepackage{epsfig}
\usepackage{latexsym}
\usepackage{url}
\usepackage{float}

\usepackage{color}
\usepackage{tikz}
\usepackage{afterpage}
\usepackage{enumerate}
\usepackage[normalem]{ulem}
\usepackage{rotating}
\usepackage{multicol}
\usepackage{booktabs}
\usepackage{ifthen}
\usepackage{enumitem}
\setlist[enumerate]{topsep=0pt,itemsep=-1ex,partopsep=1ex,parsep=1ex}
\usepackage[english]{babel}
\usepackage{thmtools, thm-restate}

\makeatletter
\g@addto@macro{\UrlBreaks}{\UrlOrds}
\makeatother

\newcommand{\R}{\ensuremath{\field{R}}} 

\newcommand{\Bscr}{\ensuremath{\mathcal B}}

\newcommand{\minimize}{\ensuremath{\mathop{\mathrm{minimize}}\limits}}

\usetikzlibrary{arrows,patterns,plotmarks}
\tikzstyle{every picture} += [>=stealth]

\def\draft{1}  

\usepackage{dsfont}
\usepackage{tcolorbox}

\usepackage{mathtools}

\usepackage[noend]{algpseudocode}
\usepackage[normalem]{ulem}

\ifnum\draft=1

\else

\fi

\def\<{\langle}
\def\>{\rangle}

\def\R{\mathbb{R}}

\newcommand{\norm}[1]{\left\lVert\mspace{1mu} #1 \mspace{1mu}\right\rVert}

\newcommand{\tr}{{\rm tr}}

\newcommand{\E}[1] {{\mathbb{E}}\left(#1\right)}

\newcommand {\F } {{\mathrm{F}}}

\newcommand {\1}[1] {{\mathds{1}}\left\{#1\right\}}

\newcommand{\rank}{\mathrm{rank}}

\renewcommand{\O}{\mathcal{O}}

\newcommand{\inner}[2]{\langle #1, #2 \rangle}



%% file: Introduction.tex
\section{Introduction}

Observational data on commerce platforms can, in many settings, be fruitfully viewed as panel data wherein any given unit is exposed to one or more treatments at various points in time. A hallmark of these datasets is that the decision to treat a unit is often the byproduct of some algorithmic process, so that the resulting treatment `patterns' are complex, and adapted to the observed data. A natural inferential task then associated with these datasets is estimating the average effect of the treatment on the treated units (ATT). The present paper considers a parsimonious model of this problem, and designs a rate-optimal algorithm for the task of estimating the ATT. 

As brief concrete motivation, consider a retailer operating a set of stores. The retailer routinely applies a promotion to drive store sales, but the decision of whether or not to apply the promotion at any point in time is store-specific, and dependent on past outcomes. The retailer cares to understand the incremental sales driven by the promotional activity. The panel at hand consists of sales data over time, across stores (i.e.~the units are stores). The binary decision of whether or not to offer the promotion at a store may be viewed as a binary matrix with dimensions conforming to the panel. Estimating the ATT in this setting corresponds precisely to an estimate of the incremental sales driven by the promotions. It is easy to list variants of this example in e-commerce (the units are individual customers), healthcare (where the units are patients), and so on. 

\subsection{An (Informal) Problem Statement}

We consider a panel of $n$ units with outcomes observed over $T$ periods. To that end, let $M^* \in \R^{n \times T}$ be a fixed, unknown matrix and $E\in \R^{n \times T}$ be a zero-mean random matrix; the matrix $M^*+E$ then represents counterfactual outcomes. A known `treatment' matrix $Z\in \{0,1\}^{n \times T}$ encodes the application of a treatment to the $n$ units; specifically $Z_{it} = 1$ if and only if unit $i$ is treated at time $t$. The panel data observed is a matrix $O \in \R^{n \times T}$ with entries $O_{it} := M^*_{it} + E_{it} + \mathcal{T}_{it} Z_{it}$. Here the $\mathcal{T}_{it}$ are unknown, heterogenous treatment effects.
In the spirit of allowing for a parsimonious model, we make the following assumptions: 
\begin{itemize}
\item The matrix $M^*$ is low-rank but otherwise arbitrary, thus allowing us to capture interactive fixed effects. Similarly, the treatment effects $\mathcal{T}_{it}$ are also arbitrary. 
\item The treatments $Z$ are allowed to depend on past observations. In particular, we only require that given all the data available up until time $t-1$ (i.e~the first $t-1$ columns of $O$ and $Z$), the entries in the $t^{\text{th}}$ columns of $Z$ and $E$ are mutually independent. 
\end{itemize}
Given this setup, our goal is to recover the average treatment effect on treated units (or ATT):
\[
\tau^* := \sum_{it} \mathcal{T}_{it}Z_{it}/\sum_{it}Z_{it}.
\]
It is worth placing the problem above in context. In the setting where we were willing to assume that $M^*_{it}$ takes the form $a_i + b_t$, simple difference-in-differences would more or less suffice. On the other hand, this is likely far too simple a counterfactual model for the applications we care about. Now since the ability to impute mean counterfactual outcomes, $M^*_{it}$, on treated entries would immediately yield an approach to estimating the ATT, it is natural to consider approaches that might let us do the same. To that end, the synthetic control method of \cite{abadie2003economic, abadie2010synthetic} -- or related alternatives such as Robust Synthetic Control \citep{amjad2018robust,agarwal2021robustness} and Synthetic Differences-in-Differneces \citep{arkhangelsky2019synthetic} -- are a natural first approach to consider. Unfortunately, the general treatment matrix $Z$ germane to our problem precludes the application of these approaches which typically apply to more structured block-like patterns. Stated more carefully, these approaches typically require $Z$ to be exogenous with its non-zero entries restricted to a block. Another tantalizing possibility, first raised by \cite{athey2021matrix}, is treating entries of $M^*$ in the support of $Z$ as missing, and applying matrix completion techniques to impute these counterfactual values. Here, however, it is unclear that counterfactual recovery -- which is equivalent to matrix recovery with general missingness patterns -- is actually possible, so that recovery guarantees are typically unavailable.

\subsection{Our Contributions}

Succinctly, we develop an estimator that recovers the average treatment effect under provably minimal assumptions on $Z$. This is made possible by a new de-biasing identity and a substantial extension of the entry-wise uncertainty quantification analysis of \cite{chen2019noisy,chen2020bridging} to general non-random missingness patterns that is of independent interest. 

In providing context for our contributions, it is worth asking what one can hope for in this problem. In addition to requiring that $M^*$ have low rank (say $r$), it is clear that we {\em cannot} in general expect to recover $\tau^*$ absent assumptions on the treatment matrix $Z$. For instance, we must rule out the existence of a rank $r$ matrix $M'$, distinct from $M^*$, for which $M' = M^* + \gamma Z$ for some $\gamma \neq 0$, or else identifying $\tau^*$ is impossible even if $E$ is identically zero (we could for example rule this out if $Z$ were not in the tangent space of $M^*$). Separately, unlike  matrix completion, we actually {\em observe} $M^*_{it} + \mathcal{T}_{it} + E_{it}$ on treated entires. If the heterogeneity in treatment effects is too large, however, it is unclear that these observations are of much value. Thus, meaningful results requires assumptions on both (i) the projection of $Z$ onto the tangent space of $M^*$ so as to allow for identification, and (ii) limits on the heterogeneity in treatment effects. 
Against this backdrop, we make the following contributions: 
\begin{itemize}
\item {\em Rate Optimal Estimator: } We construct an estimator that achieves rate optimal guarantees for the recovery of the average treatment effect $\tau^*$ (\cref{thm:error-rate-theorem}) with general treatment patterns. We show under additional assumptions that our estimator is asymptotically normal (\cref{thm:main-theorem}).

\item {\em Minimality of Assumptions: }Should the conditions we place on the projection of $Z$ onto the tangent space of $M^*$ be violated by an amount that grows small with problem size, we show {\em no algorithm} can recover $\tau^*$ even with homogeneous treatment effects (Proposition~\ref{prop:necessity-of-conditions}). Our assumptions on heterogeneity of the treatment effect are also shown to be minimal, and satisfied by extant models in the synthetic control literature. 

\item {\em Easy Ex-Ante Conditions for Identification: }We show that the tangent space conditions required for identification are implied by easy to check conditions on the treatment pattern $Z$. These include requiring that at least a constant (but otherwise arbitrary) fraction of entries are {\em not} treated, or requiring that $Z$ have rank polynomially larger than the rank of $M^*$. We anticipate these conditions are ubiquitously satisfied in many applications of interest. 

\item {\em Empirical Performance: }We show both for synthetic and real data that our estimator provides a material improvement in empirical performance relative to available alternatives, including matrix completion based estimators and, where applicable, state-of-the-art synthetic control estimators.  

\end{itemize}

\subsection{Related Literature}

The synthetic control literature pioneered by \cite{abadie2003economic, abadie2010synthetic} has grown to encompass sophisticated learning and inferential methods; see \cite{abadie2019using} for a review. \cite{doudchenko2016balancing,li2017estimation,ben2021augmented} consider a variety of regularized regression techniques to learn the linear combination of untreated units that yields a synthetic control. \cite{amjad2018robust,amjad2019mrsc,agarwal2021robustness} consider instead the use of principal component regression techniques. \cite{arkhangelsky2019synthetic} proposes alternative approaches to imputing counterfactuals by averaging across both untreated units (rows) and time (untreated columns). 

Matrix completion methods present a means to allow for inference with {\em general} treatment patterns. \cite{athey2021matrix} are among the first to study this, but provides no guarantees on recovering the average treatment effect. Alternatively, methods that do provide guarantees on the recovery of treatment effects via matrix completion tend to make strong assumptions: \cite{xu2017generalized}, \cite{bai2019matrix} effectively assume that $Z$ has support on a block (so that traditional synthetic control techniques could also apply),  \cite{xiong2019large} make stationarity assumptions on $M^*$ and require it to be zero-mean, and \cite{CFMY:19} assumes that $Z$ has i.i.d. entries (wherein a trivial estimator of the average treatment effect is also applicable).
The endogeneity of $Z$ remains largely unexplored. Recently, a line of work allows $Z$ to depende on $M^{*}$ (\cite{athey2021matrix,xiong2019large,agarwal2020synthetic}). We will allow $Z$ to depend on both $M^{*}$ and $E$, thus broadly generalizing the applicability of existing settings.

   
Our estimation procedure begins with first computing a `rough' estimate of the treatment effect via a natural convex estimator;  \cite{xiong2019optimal, gobillon2016regional, gobillon2020local} are empirical studies that use this estimator. Crucially, we provide a new de-biasing technique that allows for recovery guarantees and exhibits a significant performance improvement relative to this rough estimate. It is also worth noting that this convex estimator also finds application in the related problem of panel data regression; see \cite{bai2009panel, moon2017dynamic, moon2018nuclear}. State-of-the-art methods there effectively require that $Z$ is dense.   

Whereas work on matrix completion with non-standard observation patterns \citep{chatterjee2020deterministic,foucart2019weighted,liu2017new,klopp2017robust,agarwal2021causal} exists, this is by and large not obviously useful or applicable to our problem. Instead, we build on a recent program to bridge convex and non-convex formulations for matrix completion \citep{chen2019noisy} and Robust-PCA \citep{chen2020bridging}. That work has provided a pioneering analysis of entry-wise guarantees and uncertainty quantification for convex matrix completion estimators wherein entries remain missing at random. Our work here may be viewed as extending that program to a broad class of non-random missingness patterns, a contribution of important independent interest.

%% file: Model.tex
\section{Model and Algorithm}
\label{sec:model}


We begin by formally defining our problem, which is in fact a generalization of the problem described in the previous section that allows for {\em multiple} treatments.
Let $M^* \in \mathbb{R}^{n \times n}$ be the fixed rank-$r$ \textit{counterfactual} matrix\footnote{Note that in contrast to the previous section, we are now assuming {\em square} matrices (i.e~$n=T$ in the notation of the previous section). This is purely to simplify the notation -- for a rectangular $n$-by-$T$ matrix, all of our theoretical guarantees hold if one swaps $n$ with $\min\{n,T\}$.} which indicates the {\em expected} outcomes absent treatment and noise. 
We denote the singular value decomposition (SVD) of $M^{*}$ by $M^* = U^{*}\Sigma^{*}V^{*\top}$, where $U^{*},V^{*} \in \R^{n\times r}$ have orthonormal columns, and $\Sigma^* \in \R^{r \times r}$ is diagonal with diagonal entries $\sigma_1 \ge \cdots \ge \sigma_r > 0$. Our guarantees will depend specifically on $\sigma_{\max} := \sigma_1$, $\sigma_{\min} := \sigma_{r}$, and the condition number $\kappa := \sigma_{\max} / \sigma_{\min}$. 

We assume there are $k$ treatments, any subset of which can be applied to each entry. For each treatment $m \in \{1,\ldots,k\}$, a {\em treatment matrix} $Z_m \in \{0, 1\}^{n\times n}$  encodes the entries that have received the $m^\text{th}$ treatment (0 encoding no treatment, and 1 encoding being treated).\footnote{Naturally, we assume that at least one entry in each matrix $Z_m$ is non-zero.} We observe (in addition to the treatment matrices) a single matrix of {\em outcomes}:\footnote{The symbol $\circ$ denotes the Hadamard or `entry-wise' product} 
\[ O := M^{*} + E + \sum_{m=1}^k \mathcal{T}_m\circ Z_m, \] 
where each $\mathcal{T}_m \in \R^{n\times n}$ is an unknown matrix of {\em treatment effects}. 
The term $E\in \mathbb{R}^{n \times n}$ is a {\em noise} matrix which encodes the realized deviation from the expected value of each entry.

Finally, let $\tau^{*} \in \mathbb{R}^k$ be the vector of ATT's, whose $m^\text{th}$ value is defined as \[\tau^{*}_m := \frac{\langle \mathcal{T}_m, Z_m \rangle}{\sum_{ij} (Z_m)_{ij}},\] and let $\delta_m = \mathcal{T}_m\circ Z_m - \tau^*_m Z_m$ be the associated `residual' matrices, which represent the extent to which individual treatment effects differ from the average treatment effect.  Our problem is to estimate $\tau^*$, having observed $O$ and $Z_1,\ldots,Z_k$.


Having defined the problem completely, it is worth emphasizing at this point the power of the model above. First, by allowing for general treatment matrices $Z_m$, our model subsumes a number of common settings in causal learning: 
\begin{enumerate}
	\item {\em Synthetic Control}: The family of {\em synthetic control} methods apply to the special case of our model where $k=1$ and $Z_1$ places support on a single row. From here on, we will overload the term `synthetic control' to refer this special case (in addition to the methods themselves). 
	\item {\em Treatment Covariates}: One common setting is where the treatment on any entry is associated with a $\{0,1\}^k$-valued covariate vector, and the treatment effect on that entry is some linear function of this covariate vector. Recovery of $\tau^*$ is then equivalent to recovering covariate-dependent heterogeneous treatment effects.  
	\item {\em Difference-in-Differences (DID)}: DID allows for general treatment matrices $Z_m$, but relies on the assumption that $M^{*}$ be a specialized rank-2 matrix.
\end{enumerate}
Second, note that our model is entirely {\em deterministic}, and in fact our main recovery guarantee (\cref{thm:error-rate-theorem}) will be stated deterministically. This is extremely important -- and is unique in the context of the existing literature -- precisely because it allows us to address stochastic generative models wherein treatments are assigned {\em adaptively}. This is critical to representing real applications (e.g.~as described in the introduction). In particular, we will consider the following generative model. 
\begin{definition}[Adaptive Treatment Model]\label{def:model}
$E$ and $Z_1,\ldots,Z_k$ are random objects. Fix any column $j \in [n]$. \textit{Conditioned on} the previously-observed columns, i.e.~columns $1$ through $j-1$ of $O$ and $Z_1,\ldots,Z_k$, we have the following:
\begin{enumerate}
\item The entries in the $j^\text{th}$ columns of $E$ and $Z_1,\ldots,Z_k$ are jointly independent;
\item The entries in the $j^\text{th}$ column of $E$ are mean-zero and sub-Gaussian\footnote{The {\em sub-Gaussian norm} of a random variable $X$ is defined as $ \norm{X}_{\psi_2} := \inf\{t > 0: \E{\exp(X^2/t^2)}\leq 2\}.$ For $X$ itself to be {\em sub-Gaussian} is equivalent to having finite sub-Gaussian norm.} with sub-Gaussian norm $\|E_{ij}\|_{\psi_2} \le \sigma$.
\end{enumerate}
\end{definition}
\noindent This model admits general dependence of each $Z_m$ on the past outcomes, with a `causal requirement' of conditional independence between $E$ and $Z_m$ which is necessary in general for the problem to be meaningful. This is substantially more general than the typical assumption that $Z$ be entirely exogenous to the noise.

In the coming subsection, we will outline the core contribution of this paper, which is an estimator for $\tau^*$ with a provably rate-optimal guarantee. That discussion will require the notion of the {\em tangent space} of a low-rank matrix, which is critical in characterizing the extent to which the treatment matrices $Z_m$ allow for, or preclude, recovery of $\tau^*$. Loosely speaking, recovery of $\tau^*_m$ is (provably) impossible if $Z_m$  can be `disguised' within $M^*$. The formal version of this statement (\cref{prop:necessity-of-conditions}) relates to a particular decomposition of the linear space of $n \times n$ matrices, $\mathbb{R}^{n \times n} = T^* \oplus  T^{*\perp}$, where $T^*$ is the tangent space of $M^*$ in the manifold consisting of matrices with rank no larger than $r$ (the rank of $M^*$): \[T^* = \{U^*A^\top + BV^{*\top}| \; A,B \in \mathbb{R}^{n \times r}\}.\]
Equivalently, the orthogonal space of $T^*$, denoted $T^{*\perp}$, is the subspace of $\mathbb{R}^{n \times n}$  whose columns and rows are orthogonal, respectively, to the spaces $U^*$ and $V^*$.\footnote{When we refer to a matrix with orthonormal columns as a `space', we mean the subspace spanned its columns.} Let $P_{T^{*\perp}}(\cdot)$ denote the projection operator onto $T^{*\perp}$:
\begin{align*}
P_{T^{*\perp}}(A) = (I - U^{*}U^{*\top}) A (I - V^{*}V^{*\top}).
\end{align*}
We will defer the formal statements to the next section, but suffice to say for now that $Z_m$ can be `disguised'  within $M^*$ when its projection onto the tangent space of $M^*$ is large, or equivalently, when $P_{T^{*\perp}}(Z_m)$ is small. Thus we will require $P_{T^{*\perp}}(Z_m)$ to be sufficiently large (we will show that lower bounds on the size of $P_{T^{*\perp}}(Z_m)$, in the the precise form of our own conditions, are nearly necessary).

\subsection{A De-biased Convex Estimator}
Our estimator for $\tau^*$ is constructed in two steps, stated as \cref{eq:convex-program,eq:construct-taud} below:
\begin{subequations} \label{eq:estimator}
\begin{align} 
   (\hat{M}, \hat{\tau}) \in \underset{M \in \R^{n\times n}, \tau \in \R^k}{\mathrm{argmin}} \quad g(M, \tau) := \frac{1}{2}\norm{O - M - \sum_{m=1}^k\tau_m Z_m}_{\F}^{2} + \lambda \norm{M}_{*},\label{eq:convex-program} 
\end{align}
\begin{align}
    \tau^{d} := \hat{\tau} - D^{-1}\Delta. \label{eq:construct-taud}
\end{align}
\end{subequations}
In \cref{eq:construct-taud}, define $D \in \R^{k \times k}$ as the Gram matrix with entires $D_{lm} = \langle
P_{\hat{T}^{\perp}}(Z_l),P_{\hat{T}^{\perp}}(Z_m)
\rangle
$,
and define $\Delta \in \R^k$ as the `error' vector with components
$\Delta_l = \lambda 
\langle
Z_l,\hat U \hat V^\top
\rangle$, where we have let $\hat{M} = \hat{U}\hat{\Sigma} \hat{V}^{\top}$ be the SVD of $\hat{M}$, and let $\hat{T}$ denote the tangent space of $\hat{M}$.\footnote{Our definition implicitly assumes that $D$ is invertible. We view this as a natural assumption on (the absence of) collinearity in treatments.}

To parse this estimator, note that the first step, \cref{eq:convex-program}, is a natural convex optimization formulation that we use to compute a `rough' estimate of the average treatment effects. The objective function's first term penalizes choices of $M$ and $\tau$ which differ from the observed $O$, and the second term seeks to penalize the rank of $M$ using the nuclear norm as a (convex) proxy. We will take the tuning parameter $\lambda$ to be $\lambda := \Theta(\sigma \sqrt{n} \log^{1.5}(n))$ throughout the paper.


After the first step, having $(\hat{M}, \hat{\tau})$ as a minimizer of \cref{eq:convex-program}, we could simply use $\hat{\tau}$ as our estimator for $\tau^*$. However, a brief analysis of the first-order optimality conditions for \eqref{eq:convex-program} yields a simple, but powerful decomposition of $\hat{\tau}-\tau^*$ that suggests a first-order improvement to $\hat{\tau}$ via {\em de-biasing}: 
\begin{lemma}[Error Decomposition]\label{lem:tau-decomposition}
Suppose $(\hat{M}, \hat{\tau})$ is a minimizer of \eqref{eq:convex-program}. Let $\hat{M} = \hat{U}\hat{\Sigma} \hat{V}^{\top}$ be the SVD of $\hat{M}$, and let $\hat{T}$ denote the tangent space of $\hat{M}$. Denote $\hat E = E +\sum_m \delta_m \circ Z_m$. Then,
\begin{equation} \label{eq:tau-decomposition}
D(\hat \tau - \tau^*)
= \Delta^1 + \Delta^2 + \Delta^3,
\end{equation}
where 
$\Delta^1, \Delta^2, \Delta^3 \in \R^k$ are vectors with components
\[
\Delta^1_m = \lambda 
\langle
Z_m,
\hat U \hat V^\top
\rangle,\;\;\;\;
\Delta^2_m = 
\langle
P_{\hat{T}^{\perp}}(Z_m)
,
\hat E
\rangle,\;\;\;\;
\Delta^3_m = 
\langle
Z_m
,P_{\hat{T}^{\perp}}(M^*)
\rangle.
\]
\end{lemma}

Consider this error decomposition, i.e.~$\hat{\tau} - \tau^* = D^{-1}(\Delta^1 + \Delta^2 + \Delta^3)$ by \cref{eq:tau-decomposition}, and note that $D^{-1}\Delta^1$ is entirely a function of observed quantities. Thus, it is known and {\em removable}. The second step of of our algorithm, \cref{eq:construct-taud}, does exactly this. The resulting de-biased estimator, denoted $\tau^d$, is the subject of this paper. As an aside, it is worth noting that while de-biased estimators for high-dimensional inference have received considerable attention recently, our de-biasing procedure is algorithmically distinct from existing notions of de-biasing, including those for problems closely related to our own, such as matrix completion (\cite{xia2018confidence,CFMY:19}) and  panel data regression (\cite{moon2018nuclear}).

Our main results characterize the error $\tau^d - \tau^*$. The crux of this can be gleaned from the second and third terms of \cref{eq:tau-decomposition}. If $\hat{T}$ is sufficiently `close' to $T^*$, then $\Delta^3$ becomes negligible (because $P_{T^{*\perp}}(M^*)$ = 0). Showing closeness of $\hat{T}$ and $T^{*}$ is the main technical challenge of this work. The remaining error, contributed by $\Delta^2$, can then be characterized as a particular `weighted average' of the entries of $E$ and the residual matrices $\delta_m$ which we show to be min-max optimal.

To conclude this section, we give the proof of \cref{lem:tau-decomposition} for a single treatment ($k=1$); the complete proof is a straightforward generalization, completed in \cref{sec:proof-of-first-order-lemma}.

\begin{proof}[Proof of \cref{lem:tau-decomposition} for $k=1$] 
Since $k=1$, we suppress redundant subscripts. Consider the first-order optimality conditions of \eqref{eq:convex-program}:
\begin{subequations}\label{eq:convex-conditions}
\begin{align}\label{eq:convex-condition-tau}
\langle Z,{O - \hat{M}  - \hat{\tau} Z} \rangle &= 0 \\
       O - \hat{M} - \hat{\tau} Z &= \lambda (\hat{U}\hat{V}^{\top} + W), \label{eq:convex-condition-M}\\
     P_{\hat{T}^\perp}(W) &= W \label{eq:convex-condition-W2}\\
     \norm{W} &\leq 1 \label{eq:convex-condition-W1} 
\end{align}
\end{subequations}
where \cref{eq:convex-condition-tau} corresponds to the condition of $\hat{\tau}$ and \cref{eq:convex-condition-M,eq:convex-condition-W1,eq:convex-condition-W2} correspond to the conditions of $\hat{M}$ ($W$ is called the `dual certificate' in the matrix completion literature: \cite{wright2009robust,recht2011simpler,candes2009exact}).
Combining \cref{eq:convex-condition-M,eq:convex-condition-tau}, we have
\begin{align}
    \inner{Z}{O-\hat{M}-\hat{\tau} Z} = 0 \implies \inner{Z}{\lambda (\hat{U}\hat{V}^{\top} +  W)} = 0 \implies \lambda \inner{Z}{\hat{U}\hat{V}^{\top}} = -\inner{Z}{\lambda W}.\label{eq:ZUVT-ZW}
\end{align}
Next, applying $P_{\hat{T}^{\perp}}(\cdot)$ to both sides of \cref{eq:convex-condition-M} and using \cref{eq:convex-condition-W2} and $P_{\hat{T}^{\perp}}(\hat{U}\hat{V}^{\top}) = 0$:
\begin{align}
    & P_{\hat{T}^{\perp}} (O - \hat{M} - \hat{\tau} Z) = \lambda W \nonumber\\
    \implies & P_{\hat{T}^{\perp}} (M^{*}) + P_{\hat{T}^{\perp}}(E + \delta \circ Z) - (\hat{\tau} - \tau^{*})P_{\hat{T}^{\perp}}(Z) = \lambda W, \label{eq:PTPerpO}
\end{align}
where the implication is by definition: $O = M^{*} + E + \tau^{*}Z + \delta \circ Z$ and $P_{\hat{T}^{\perp}}(\hat{M}) = 0$. 
Finally, substituting $\lambda W$ from \cref{eq:PTPerpO} into \cref{eq:ZUVT-ZW}, we obtain
\begin{align}
& \lambda \inner{Z}{\hat{U}\hat{V}^{\top}} = -\inner{Z}{P_{\hat{T}^{\perp}} (M^{*}) + P_{\hat{T}^{\perp}}(E+\delta \circ Z) - (\hat{\tau} - \tau^{*})P_{\hat{T}^{\perp}}(Z)} \nonumber\\
    \implies& (\hat{\tau} - \tau^{*}) \norm{P_{\hat{T}^{\perp}}(Z)}_{\F}^2 = \lambda \inner{Z}{\hat{U}\hat{V}^{\top}} + \inner{Z}{P_{\hat{T}^{\perp}}(E+\delta \circ Z)} + \inner{Z}{P_{\hat{T}^{\perp}} (M^{*})} \nonumber
\end{align}
This is equivalent to \cref{eq:tau-decomposition}, completing the proof. 
\end{proof}

%% file: Conditions.tex
\section{Theoretical Guarantees}
Summarizing so far, our estimator is constructed in two steps: solve the convex program in \cref{eq:convex-program} to obtain an initial estimate $(\hat{M},\hat{\tau})$, then de-bias according to \cref{eq:construct-taud}. While we have presented this estimator in a setting that allows for {\em multiple} treatments (i.e.~$k \ge 1$), for the sake of simplicity our results here (\cref{thm:error-rate-theorem,thm:main-theorem}) are restricted to the single treatment ($k=1$) setting. 
To ease notation, we will from here on suppress treatment-specific subscripts ($Z_1$, $\tau_1$, etc.)


\subsection{Assumptions}
Before presenting the main result, we formally state two sets of assumptions.
\paragraph{1. Minimal Identification Conditions:}
The first is a set of identification conditions (as discussed earlier) that relate the treatment matrix $Z$ to the tangent space $T^*$ of $M^*$: 

\begin{assumption}[Identification] \label{assum:conditions-Z}There exist positive constants $C_{r_1}, C_{r_2}$ such that
\begin{enumerate}[label={(\alph*)}, ref={\theassumption(\alph*)}]
\item \label[assumption]{cond:Z-condition-nonconvex}
$\|ZV^{*}\|_{\F}^2 + \|{Z^{\top}U^{*}}\|_{\F}^2 \leq \left(1-\frac{C_{r_1}}{\log(n)}\right) \|{Z}\|_{\F}^2,$
\item \label[assumption]{cond:Z-condition-convex}
$
\left|\inner{Z}{U^{*}V^{*\top}}\right| \|P_{T^{*\perp}}(Z)\| \leq \left(1-\frac{C_{r_2}}{\log(n)}\right)  \|P_{T^{*\perp}}(Z)\|_{\F}^2.
$
\end{enumerate}
\end{assumption}
\cref{assum:conditions-Z} broadly expands on the set of treatment `patterns' addressed in the existing literature (such as synthetic control and panel regression; see \cref{sec:treatment-conditions}) and is mild enough to allow for many patterns that occur in practice, such as those generated dynamically (see the experiments in \cref{sec:experiments}). The conditions also can be verified easily ex-ante (see that discussion in depth in \cref{sec:treatment-conditions}). 

The necessity of \cref{assum:conditions-Z} is demonstrated by the following result, which establishes that should either of the two conditions of \cref{assum:conditions-Z} be violated by an amount that grows negligible with $n$, then identification is rendered impossible so that no estimator can recover $\tau^*$. 

\begin{proposition}\label{prop:necessity-of-conditions}
	For any $n$, there exists a matrix $Z \in \{0,1\}^{n \times n}$ and a pair of rank-1 matrices $M_1,M_2 \in \mathbb{R}^{n \times n}$ with SVDs denoted by $M_i = U_i\Sigma_iV_i^\top$ and $T_{i}$ being the tangent space of $M_i$, such that all three of the following statements hold:
	\begin{enumerate}
	\item $\norm{ZV_i}_{\F}^2 + \|{Z^TU_i}\|_{\F}^2 = \norm{Z}_{\F}^2, \;\; i=1,2$
	\vspace{2pt}
	\item $\left| \left\langle Z,U_iV_i^\top \right\rangle \right| \|{P_{T_i^{\perp}}(Z)}\| = \|P_{T_i^{\perp}}(Z)\|_{\F}^2, \;\; i=1,2   $
	\vspace{2pt}
	\item $M_1 + Z= M_2 $
	\end{enumerate}
\end{proposition}

\paragraph{2. Mildly-Bounded Noise:}
In addition to \cref{assum:conditions-Z}, we require mild conditions on the noise matrix $E$.
\begin{assumption}[Bounded Noise]\label{assum:operator-E-delta}
There exists a positive constant $C_{e}$ such that
$$
\max\left\{\|E\|, \frac{|\inner{E}{Z}|}{\|Z\|_{\F}}  \right\} \leq C_{e} \sigma \sqrt{n}.
$$
\end{assumption}
Note that \cref{assum:operator-E-delta} is stated for completely deterministic $E$ and $Z$. To see that it is indeed `mild', consider the following result for the Adaptive Treatment Model (\cref{def:model}):
\begin{lemma}\label{lem:adaptive-noise}
Under the model in \cref{def:model}, with probability $1-O(1/n^{C})$, we have  $\|E\| \lesssim \sigma \sqrt{n}$ and $|\inner{E}{Z}|/\|Z\|_{\F} \lesssim \sigma \log(n)$, for any $C > 0$.
\end{lemma}
\noindent Thus the Adaptive Treatment Model satisfies \cref{assum:operator-E-delta} with high probability. In fact, there is even substantial ‘slack’ (between the $O(\sqrt{n})$ condition and the $\log(n)$ guarantee).

\subsection{A Rate-Optimal Deterministic Guarantee}\label{sec:deterministic}
We can now state our main results.
The first establishes a bound on the error rate of $\tau^d$: 


\begin{theorem}[Optimal Error Rate]\label{thm:error-rate-theorem}
Under \cref{assum:conditions-Z} and \cref{assum:operator-E-delta}, suppose
\begin{equation} \label{eq:error-rate-theorem-conditions}
\frac{\sigma \sqrt{n}}{\sigma_{\min}} \leq C_1\frac{1}{\kappa^2 r^2 \log^{5}(n)}
 \quad\text{ and }\quad
 \norm{\delta} \leq C_{\delta} \sigma \sqrt{n}.
 \end{equation} 
Taking $\tau^d$ as defined in \cref{eq:estimator}, we have that 
\begin{align*}
|\tau^{d} - \tau^{*}|  \leq C\log^{1.5}(n)\max\left(\frac{\sigma}{\norm{Z}_{\F}} \left(\frac{\sigma n\kappa^2 r^{1.5} \log^{4}(n)}{\sigma_{\min}}+1\right), \frac{|\inner{P_{T^{*\perp}}(\delta+E)}{Z}| }{\norm{Z}_{\F}^2}\right),
\end{align*}
where $C$ is a constant depending (polynomially) on $C_1, C_{r_1}, C_{r_2}, C_{\delta}$.
\end{theorem}
Recall that $\delta = \mT \circ Z - \tau^{*}Z$ denotes the matrix of treatment effect `residuals', and so the requirement that $\norm{\delta} \leq C_{\delta} \sigma \sqrt{n}$ is essentially a condition on the heterogeneity of treatment effects. Since $\delta$ is a zero-mean matrix and is zero outside the support of $Z$, the requirement is easily satisfied in practice. For example it is trivially met in synthetic control settings. It is also easily seen as met when $\delta$  has independent, sub-Gaussian entries. Finally as it turns out, the condition can also admit random sub-gaussian matrices with complex correlation patterns; see \cite{moon2015linear}.

To begin parsing \cref{thm:error-rate-theorem}, consider a `typical' scenario in which $\sigma,\kappa,r=O(1)$, and $\sigma_{\min}=\Omega(n)$. Then \cref{thm:error-rate-theorem} implies that
\begin{equation} \label{eq:simplified-rate}
	|\tau^{d}-\tau^{*}| = \tilde{O}\left(\frac{\sigma}{\|Z\|_{\F}} + \frac{|\inner{P_{T^{*\perp}}(\delta+E)}{Z}|}{\|Z\|_{\F}^2}\right).
\end{equation}
This is minimax optimal (up to $\log(n)$ factors), as shown by \cref{prop:necessity-of-conditions-tau} below, the proof of which is deferred to \cref{sec:proof-of-propositions}: 

\begin{proposition}[Minimax Lower Bound]\label{prop:necessity-of-conditions-tau}
	For any estimator $\hat{\tau}$, there exists an instance with $\sigma,\kappa,r=\Theta(1)$, and $\sigma_{\min}=\Theta(n)$, on which, with probability at least $1/3$,
	\begin{align*}
		|\hat{\tau} - \tau^{*}| \geq \max\left(\frac{\sigma}{\norm{Z}_{\F}}, \frac{|\inner{P_{T^{*\perp}}(\delta+E)}{Z}|} { \norm{Z}_{\F}^2}\right).
	\end{align*}
\end{proposition}
To further expound the error bound established in \cref{eq:simplified-rate}, consider the two terms separately. The first error term, 
$\sigma / \|Z\|_{\F},$ is the optimal rate achievable even if $M^{*}$ and $\delta$ were known and the entries of $E$ were generated independently -- thus, it is the statistical lower bound in an `idealized' setting in which the counterfactuals are known and there is zero heterogeneity in the treatment effects. 
 The second error term,  $|\inner{P_{T^{*\perp}}(\delta+E)}{Z}|/\|Z\|_{\F}^2,$ is less natural but precisely characterizes the error introduced by the deterministic disturbances $E$ and $\delta$, as shown by the matching term in \cref{prop:necessity-of-conditions-tau}. 

\paragraph{Achieving the Idealized Rate: Comparison to Existing Methods.}
In the remainder of this subsection, we will highlight some example settings in which the second error term vanishes and \cref{eq:simplified-rate} reduces to the statistical lower bound
\begin{equation} \label{eq:super-simplified-rate} |\tau^{d}-\tau^{*}| = \tilde{O}\left(\sigma/ \|Z\|_{\F} \right).  \end{equation} 
This discussion primarily serves to compare the guarantee made in \cref{thm:error-rate-theorem} against those for existing approaches. As generally assumed in the existing literature, we will assume (only here) that $E$ consists of independent noise, in which case the term $|\inner{P_{T^{*\perp}}(E)}{Z}|/\|Z\|_{\F}^2$ vanishes. 
Any of the following are, alone, sufficient to imply \cref{eq:super-simplified-rate}. See \cref{sec:discussion-of-delta} for details. 
\begin{itemize}
\item {\em Independent $\delta$:} Independent, sub-gaussian $\delta_{ij}$ with $O(1)$ sub-gaussian norm. We will see in the next subsection this also guarantees asymptotic normality of $\tau_d$.

\item {\em Synthetic control and block $Z$:} $\delta_{ij} = O(1)$, and $Z$ consists of an $\ell \times c$ block that is sufficiently sparse: $\sqrt{\ell c}(\ell + c) = O(n)$. For comparison, state-of-the-art synthetic control results (e.g.~\cite{arkhangelsky2019synthetic,agarwal2020synthetic})  require the sparser condition $\ell c(\ell + c) = O(n)$ (though that condition enables asymptotic normality).

\item {\em Panel data regression:} The conditions imposed in \cite{moon2018nuclear}, which are sufficient for `linear panel regression' methods. The most notable condition is that that $Z$ be sufficiently dense: $\|P_{T^{*\perp}}(Z)\|_{\F}^2 = \Theta(n^2)$. In this case, our result recovers their error guarantee (up to $\log$ factors). 

\item {\em Matrix completion:}  the entries of $Z$ are drawn independently, which is the canonical condition for matrix completion at this moment. 
\end{itemize}
As an aside, difference-in-differences would require $M^*$ to be a particular form of rank-2 matrix: $M^* = a_i + b_j$ to achieve the optimal rate. 

In summary, from \cref{thm:error-rate-theorem}, our estimator achieves the mini-max optimal rate for a very general class of $Z$ and $E$, and broadly expands the settings that are addressed in the existing literature (see \cref{sec:treatment-conditions} for more discussion).

\subsection{Asymptotic Normality}
Our second main result (\cref{thm:main-theorem}) establishes asymptotic normality for our estimator. This naturally requires some additional control over the variability of $\delta$ and $E$. We consider the setting in which $E$ and $\delta$ are generated independently. 
\begin{theorem}[Asymptotic Normality]\label{thm:main-theorem}
Suppose the entries of  $E$ and $\delta$ are independent, mean-zero, sub-Gaussian random variables with sub-Gaussian norms $\|E_{ij}\|_{\psi_2}, \|\delta_{ij}\|_{\psi_2} = O(1)$. Assume $\sigma,\kappa,r =O(1)$, $\sigma_{\min}=\Omega(n)$, and $\max(\|U^{*}\|_{2,\infty},\|V^{*}\|_{2,\infty}) =O(\sqrt{r/n})$. Then with probability $1-O(1/n^{3})$, 
\begin{align} \label{eq:main-theorem}
\tau^{d} - \tau^{*} =  \frac{\inner{E+ (\delta \circ Z)}{P_{T^{*\perp}}(Z)}}{\norm{P_{T^{*\perp}}(Z)}_{\F}^2} + O\left(\frac{\log^{8}(n)}{n}\right).
\end{align} 
Consequently, 
\begin{align} \label{eq:main-theorem-2}
\frac{\tau^{d} - \tau^{*}}{V_{\tau}^{1/2}} \rightarrow \mathcal{N}(0, 1), \quad \text{ where } \;\; V_{\tau} = \frac{\sum_{ij} P_{T^{*\perp}}(Z)_{ij}^2 \mathrm{Var}(E_{ij}+\delta_{ij}Z_{ij}) }{ \left(\sum_{ij} P_{T^{*\perp}}(Z)_{ij}^2\right)^2},
\end{align}
provided that $V_{\tau}^{1/2} = \Omega( \log^{9}(n)/n).$
\end{theorem}

Asymptotic normality, e.g.~as established by \cref{thm:main-theorem}, is of econometric interest as it enables {\em inference}. Specifically, inference can be performed using a `plug-in' estimator $\hat{V}_{\tau}$ for $V_{\tau}$, gotten by substituting $\hat{T}$ for $T^{*}$ and $O-M^{d}-\tau^{d}Z$ for $E + (\delta \circ Z) $, where $M^{d}$ is a novel de-biased estimator for $M^{*}$ with entry-wise guarantees (see \cref{section:proof-of-theorem-M}). This `plug-in' estimation for variances is a common procedure in the literature (e.g. \cite{CFMY:19}), and it is straightforward to show that $\hat{V}_{\tau} \sim V_{\tau}.$ 

\vspace{1em}
\textbf{Proof Techniques:} 
The proofs of \cref{thm:error-rate-theorem,thm:main-theorem} are outlined in \cref{sec:Proof-Outline,sec:proof-main-theorem}. They are inspired by recent developments on bridging convex and non-convex formulations for matrix completion \citep{chen2019noisy} and Robust-PCA \citep{chen2020bridging}. Whereas that line of work assumes a random, independent missingness pattern, our proof extends the program to deal with deterministic treatment patterns $Z$. As such, this analysis is likely of interest, in its own right, as a complement to the matrix completion literature  \citep{abbe2017entrywise,ma2019implicit,chen2019noisy,chen2020bridging}. Broadly, we must address the issue that constructing a dual certificate to analyze the quality of our convex estimator directly is hard. Instead  \cite{chen2019noisy} show the existence of such a certificate non-constructively by studying a non-convex estimator and showing that a (fictitious) gradient descent algorithm applied to that estimator recovers a suitable dual certificate. We effectively extend that program to deterministic patterns $Z$, and provide entry-wise recovery guarantees on $M^*$ in this setting that are of independent interest.

\subsection{Applicability of Identification Condition}\label{sec:treatment-conditions}
Having stated our results, we return to our identification conditions for $Z$ (\cref{assum:conditions-Z}) and discuss various treatment patterns that are admissible under \cref{assum:conditions-Z}. These conditions are verifiable since they only depend on $Z$ that is observed. Formal proofs of all claims made here can be found in \cref{sec:proof-discussion-condition-Z}. 

\begin{enumerate}
\item \textit{Rank grows faster than $r$:} \cref{assum:conditions-Z} holds if
$$
\sum_{i=1}^r \sigma_i(Z)^2 \le \left(1 - \frac{C}{\log n} \right)\frac{\|Z\|_{\F}^2}{\sqrt{r}+2}
$$ 
where $\sigma_i(Z)$ denotes the $i$-th largest singular value of $Z$. Loosely speaking, this requires that the rank of $Z$ be strictly higher than $r$, and that less than $1/\sqrt{r}$ of its `mass' lie in its first $r$ components. Put another way, $Z$ must be sufficiently different from any rank-$r$ approximation. One common setting where this occurs is when there is sufficient randomness in generating $Z$, such as the case where the entries of $Z$ are drawn independently, which is the canonical scenario in the matrix completion literature (e.g., \cite{candes2009exact,abbe2017entrywise,ma2019implicit}). 

\item \textit{Maximal number of ones in a row and column:} Let $\ell$ and $c$ denote the maximum number of ones in a row and column, respectively, of $Z$. Let $\mu$ be the incoherent parameter of $M^{*}$, i.e., $\mu := \frac{n}{r}\max(\|U^{*}\|^2_{2,\infty},\|V^{*}\|^2_{2,\infty})$. \cref{assum:conditions-Z} holds if 
$$
\ell + c \le \left(1 - \frac{C}{\log n} \right) \frac{n}{r^2 \mu}.
$$ 
In a typical scenario $r,\mu=O(1)$, this allows $\ell,c = O(n)$ and generalizes the sparse block patterns studied in the literature (e.g. \cite{xu2017generalized,arkhangelsky2019synthetic}), where $Z$ is a two-by-two block matrix with exactly one block equal to one.
\item \textit{Single row or column (Synthetic Control):} Consider the case when $Z$ is supported on a single row (or column, equivalently), as in synthetic control. \cref{assum:conditions-Z} holds if 
$$\|z^\top V^* \|_{\F}^2 \le \left(1 - \frac{C}{\log n} - \frac{\mu r}{n} \right)  \|z\|^2,
$$
where $z^{\top}$ is the non-zero row of $Z$. This will easily hold, if allowing a negligible perturbation to either $z_1$ or the row space of $M^*$. It is also interesting to note that the identification assumption made in the canonical paper \cite{abadie2010synthetic} (i.e., $T_0^{-1}\sigma_{\min}(\sum_{i=1}^{T_0} V^{*}_iV^{*\top}_{i}) > 0$ is bounded away from zero, where $T_0 = \max_{z_i=0} i$), together with $T_0 = \Omega(n/\log(n))$ and $\mu r = O(n/\log(n))$ (also implicitly assumed in \cite{abadie2010synthetic} for an optimal gaurantee), together imply \cref{assum:conditions-Z}. 
\end{enumerate}

We end this section by drawing a connection to the literature on panel data regression with interactive fixed effects \cite{bai2009panel,moon2015linear,moon2018nuclear}. That literature studies estimators similar in spirits to ours (e.g., \cite{moon2018nuclear} analyzed the performance of the convex estimator and a heuristic de-biasing approach). However, those approaches are only known to work if  $\|P_{T^{*\perp}}(Z)\|_{\F}=\Theta(n^2)$. This is of course a substantially stronger assumption than \cref{assum:conditions-Z} and rules out sparse treatment patterns (as in synthetic control). In summary, our approach also has the potential to broaden the scope of problems addressed via panel data regression. See more discussion in \cref{appendix:panel-data-regerssion}.

%% file: Experiments.tex
\section{Experiments}\label{sec:experiments}
We conducted a set of experiments on semi-synthetic datasets (the treatment is introduced artificially and thus ground-truth treatment-effect values are known) and real datasets (the treatment is real and ground-truth treatment-effect values are unknown). The results show that our estimator $\tau^{d}$ is more accurate than existing algorithms and its performance is robust to various treatment patterns, in particular for the treatment that is \textbf{adaptively assigned depending on the historical outcomes}.\footnote{The source code is available in \url{https://github.com/TianyiPeng/Causal-Inference-Code}.}

The following four benchmarks were implemented: (i) Synthetic Difference-in-Difference (SDID) \citep{arkhangelsky2019synthetic}; (ii) Matrix-Completion with Nuclear Norm Minimization (MC-NNM) \citep{athey2021matrix} (iii) Robust Synthetic Control (RSC) \citep{amjad2018robust} (iv) Ordinary Least Square (OLS): Selects $a,b \in \R^{n}, \tau \in \R$ to minimize $\|O-a1^{T} - 1b^{T} - \tau Z\|_{\F}^2$, where $1 \in \R^{n}$ is the vector of ones. This corresponds to the canonical Difference-in-Difference (DID) method with two-way fixed effects. It is also worth noting that SDID and RSC only apply to traditional synthetic control patterns ({\em block} and {\em stagger} below).

\textbf{Warm-Up (block and stagger patterns).} The first dataset consists of the annual tobacco consumption per capita for 38 states during 1970-2001, collected from the prominent synthetic control study \citep{abadie2010synthetic} (the treated unit California is removed). Identical to \cite{athey2021matrix}, we view the collected data as $M^{*}$ and introduce artificial treatments. 
We considered two families of patterns that are common in the economics literature: {\em block}  and {\em stagger}  \citep{athey2021matrix}. Block patterns model simultaneous adoption of the treatment, while stagger patterns model adoption at different times. In both cases, treatment continues forever once adopted. Specifically, given the parameters $(m_1, m_2)$, a set of $m_1$ rows of $Z$ are selected uniformly at random. On these rows, $Z_{ij} = 1$ if and only if $j \geq t_i$, where for block patterns, $t_i = m_2$, and for stagger patterns, $t_i$ is selected uniformly from values greater than $m_2$.

To model heterogenous treatment effects, let $\mT_{ij} = \tau^{*} + \delta_{i}$ where $\delta_{i}$ is i.i.d and $\delta_{i} \sim \mathcal{N}(0, \sigma_{\delta})$ characterizes the unit-specific effect. Then the observation is $O = M^{*} + \mT \circ Z$. We fix $\tau^{*} = \sigma_{\delta} = \bar{M}^{*}/5$ through all experiments, where $\bar{M}^{*}$ is the mean value of $M^{*}$ \footnote{See \cref{appendix:experiments} for estimating row-specific treatment effects.}. The hyperparameters for all algorithms were tuned using rank $r\sim 5$ (estimated via the spectrum of $M^{*}$). 

Next, we compare the performances of the various algorithms on an ensemble of 1,000 instances with $m_1 \sim \mathrm{Uni}[1, n_1), m_2 = \mathrm{Uni}[1,n_2)$ for stagger patterns and $m_1 \sim \mathrm{Uni}[1, 5), m_2 = 18$ for block patterns (matching the year 1988, where California passed its law for tobacco control). The results are reported in the first two rows of \cref{tab:synthetic-average-delta-tau} in terms of the average normalized error $|\tau - \tau^{*}|/\tau^{*}$. 

Note that the treatment patterns here are `home court' for the SDID and RSC synthetic control methods but our approach nonetheless outperforms these benchmarks. One potential reason is that these methods do not leverage all of the available data for learning counterfactuals: MC-NNM and SDID ignore treated observations. RSC ignores even more: it in addition does not leverage some of the {\em untreated} observations in $M^*$ on treated units (i.e. observations $O_{ij}$ for $j <   t_i$ on treated units). 

\begin{table}[h!]
\centering
\caption{Comparison of our algorithm (De-biased Convex) to benchmarks on semi-synthetic datasets (Block and Stagger correspond to Tobacco dataset; Adaptive pattern corresponds to Sales dataset). Average normalized error $|\tau-\tau^{*}|/\tau^{*}$ is reported.}
\begin{tabular}{@{}lccccc@{}}
\toprule
Pattern & De-biased Convex & SDID & MC-NNM & RSC & OLS\\
\midrule
Block  & 0.15 ($\pm 0.13$) &0.23 ($\pm 0.19$)  & 0.27 ($\pm 0.24$) & 0.30 ($\pm 0.26$) & 0.38 ($\pm 0.36$)\\
Stagger  & 0.10 ($\pm 0.20$)  &0.16 ($\pm 0.18$)  & 0.15 ($\pm 0.16$) & 0.20 ($\pm$ 0.27) & 0.18 ($\pm$ 0.19) \\ 
\textbf{Adaptive} & 0.02 ($\pm 0.02$) & -  & 0.13 ($\pm 0.10$) & - &  0.20 ($\pm 0.18$)\\
\bottomrule
\end{tabular}
\label{tab:synthetic-average-delta-tau}
\end{table}

\textbf{Adaptive Treatment Pattern.} The second dataset consists of weekly sales of 167 products over 147 weeks, collected from a Kaggle competition \citep{sales2021}. In this application, treatment corresponds to various `promotions' of a product (e.g. price reductions, advertisements, etc.). 
We introduced an artificial promotion $Z$, used the collected data as $M^{*}$ ($\bar{M}^{*} \approx 12170$), and the goal was to estimate the average treatment effect given $O = M^{*}+\mT \circ Z$ and $Z$ ($\mT$ follows the same generation process as above with $\tau^{*} = \sigma = {\bar{M}^{*}}/{5}$).

Now the challenge in these settings is that these promotions are often decided based on previous sales. Put another way, the treatment matrix $Z$ is constructed {\em adaptively}.
We considered a simple model for generating adaptive patterns for $Z$: Fix parameters $(a, b)$. If the sale of a product reaches its lowest point among the past $a$ weeks, then we added promotions for the following $b$ weeks (this models a common preference for promoting low-sale products). Across our instances, $(a,b)$ was generated according to $a \in \mathrm{Uni}[5, 25], b \in \mathrm{Uni}[5, 25].$ This represents a treatment pattern where it is unclear how typical synthetic control approaches (SDID, RSC) might even be applied. 

The rank of $M^{*}$ is estimated via the spectrum with $r \sim 35$. See \cref{tab:synthetic-average-delta-tau} for the results averaged over 1,000 instances.
 The average of ${|\tau-\tau^{*}|}/{\tau^{*}}$ is $\sim 2\%$ for our algorithm, versus $13\%$ for MC-NNM, indicating a strong improvement. This demonstrates the advantage of our algorithm for \textbf{complex adaptive treatment patterns}, which widely exist in real applications. On the other hand, the performance of matrix-completion algorithms is limited for those structured and adaptive missing-ness patterns. We overcome this limitation by leveraging the treated data\footnote{There is a natural trade-off here: if the heterogeneity in $\delta$ were on the order of the variation in $M^*$ (so that $\|\delta \circ Z\| \gg \sigma_{r}(M^{*})$) then it is unclear that the treated data would help (and it might, in fact, hurt). But for most practical applications, the treatment effects we seek to estimate are typically small relative to the nominal observed values.}.
 
 

\textbf{Real Data.} This dataset consists of daily sales and promotion information of 571 drug stores over 942 days, collected from Rossmann Store Sales dataset \citep{Rossmannsales2021}. The promotion dataset $Z$ is binary (1 indicates a promo is running on that specific day and store). The {\em real} pattern is highly complex (see \cref{fig:rossmann-sales}) and hence synthetic-control type methods (SDID, RSC) again do not apply. Our goal here is to estimate the average increase of sales $\tau^{*}$ brought by the promotion. 

\begin{figure}[h!]
\hspace{-20pt}
\begin{subfigure}[b]{.45\textwidth}
  \centering
    \includegraphics[scale=0.4]{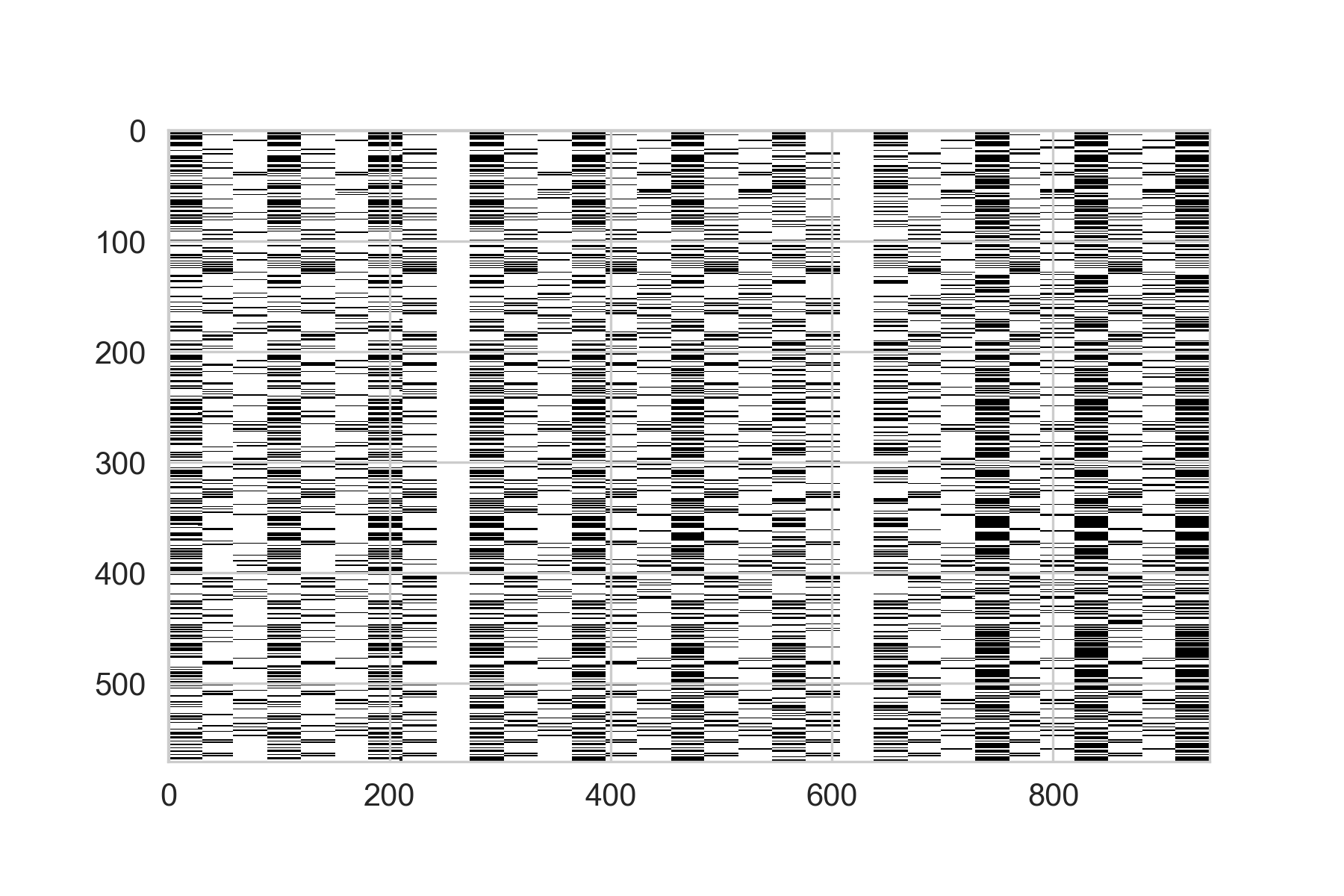}
    \label{fig:rossman-pattern}
\end{subfigure}    
  \begin{subfigure}[b]{0.5\textwidth}
    \centering
    \begin{tabular}{@{}r|rrrr@{}}
	\toprule
   	& $\tau$ & Test Error\\ \hline
  	De-biased Convex & 118.2 ($\pm 2.4$) & 0.04 ($\pm 0.002$)\\
 	MC-NNM & -49.4 ($\pm 0.98$) & 0.07 ($\pm 0.002$) \\
  	OLS & -45.8 ($\pm 1.24$) & 0.18 ($\pm 0.003$)\\
	\bottomrule
	\end{tabular}
    \label{fig:block-synthetic-performance}
    \vspace{30pt}
  \end{subfigure}
    \caption{{\it Left:} The promotion pattern of the real data. {\it Right:} Estimation of $\tau$ and test errors.}
    \label{fig:rossmann-sales}
\end{figure}

The hyperparameters for all algorithms were tuned using rank $r \sim 70$ (estimated via cross validation). A test set $\Omega$ consisting of 20\% of the treated entires is randomly sampled and hidden. The test error is then calculated by $\|P_{\Omega}(M+\tau Z - O)\|_{\F}^2 / \|P_{\Omega}(O-\bar{O})\|_{\F}^2$ where $\bar{O}$ is the mean-value of $O$. \cref{fig:rossmann-sales} shows the results averaged over 100 instances. Our algorithm provides superior test error. This is potentially a conservative measure since it captures error in approximating both $M^{*}$ and $\tau^{*}$; the variation contributed by $M^*$ to observations is substantially larger that that contributed by $\tau^*$. 
Now whereas the ground-truth for $\tau^*$ is not known here, the negative treatment effects estimated by MC-NNM and OLS seem less likely since store-wise promotions are typically associated with positive effects on sales. 

\textbf{Asymptotic Normality.} The normality of our estimator is also verified, where the prediction from \cref{thm:main-theorem} is precise enough to enable inferential tasks such as constructing confidence intervals (CIs) for $\tau^{*}$: our 95\% CIs typically had ``true'' coverage rates in the range of 93-96\% for a synthetic set of instances described in Appendix \ref{appendix:experiments}. See \cref{fig:stagger-distribution}. 

\begin{figure}[h!]
\begin{subfigure}[b]{.45\textwidth}
  \centering
    \includegraphics[scale=0.4]{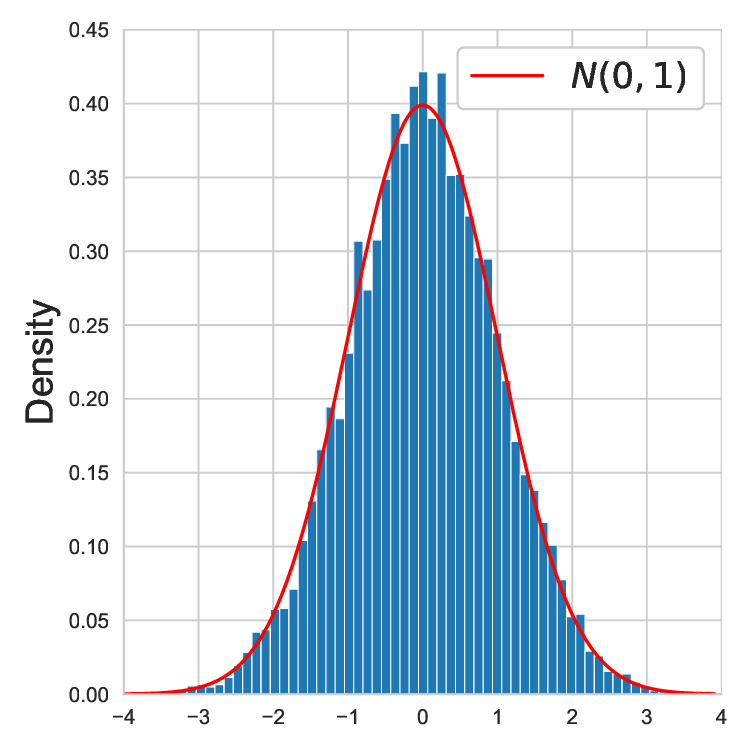}
    \label{fig:stagger-synthetic-performance}
\end{subfigure}    
    \hfill
  \begin{subfigure}[b]{0.5\textwidth}
    \centering
    \begin{tabular}{@{}r|rrrr@{}}
	\toprule
   	& $n_1=50$ & $100$ & $150$ & $200$\\ \hline
  	$n_2/n_1 = 0.5$ & 0.916 & 0.957 & 0.956 & 0.942 \\
 	$1$ & 0.953 & 0.946 & 0.954 & 0.939 \\
  	$2$ & 0.946 & 0.947 & 0.945 & 0.957 \\
  	$4$ & 0.94 & 0.934 & 0.949 & 0.944\\
	\bottomrule
	\end{tabular}
    \label{fig:block-synthetic-performance}
    \vspace{20pt}
  \end{subfigure}
    \caption{ Evaluation of our distributional characterization of $\tau^d$ on a synthetic ensemble where $\delta$ and $E$ follow i.i.d Gaussian distribution. {\it Left:} Empirical Distribution of $(\tau^{d}-\tau^{*})/V_{\tau}$ with $n=100$, overlaid with the $\mathcal{N}(0,1)$ density function as predicted by \cref{thm:main-theorem}. {\it Right:} Coverage rates of $95\%$ confidence intervals (the `correct' coverage rate is 0.95) for different sizes $(n_1, n_2)$ with $r=10$ . See \cref{appendix:experiments} for the data generation processes in details. }
    \label{fig:stagger-distribution}
\end{figure}

%


%

%% file: Algorithm.tex
\section{Proof of \cref{thm:error-rate-theorem}: Overview}\label{sec:Proof-Outline}
We will establish the proof of \cref{thm:error-rate-theorem} by giving an overview of the entire proof in this section, and then proving the `main' lemma (\cref{lem:X-Xstar-Y-Ystar}) in \cref{sec:showing-XXstar-YYstar}.

To begin, recall that in our problem setup, we observe $O = M^{*} + \tau^{*} Z + \hat{E}$, where the newly-defined $\hat{E} := E + \delta$ can be thought of as the total `disturbance' we must contend with. Our algorithm works by first solving the convex program defined in \cref{eq:convex-program}, which yields $(\hat{M},\hat{\tau})$, and then de-biasing $\hat{\tau}$ to produce our estimator $\tau^d$ as defined in \cref{eq:construct-taud}.

As we discussed in \cref{sec:model}, \cref{lem:tau-decomposition} (restricted here to the case $k=1$) allows us to decompose the error $|\tau^d - \tau^*|$. Specifically, letting $T$ denote the tangent space of $\hat{M}$, we have that
\begin{align} \label{eq:lem:tau-decomposition}
\|P_{T^{\perp}}(Z)\|_{\F}^2 (\tau^{d} - \tau^{*}) = \inner{P_{T^{\perp}}(Z)}{\hat{E}} + \inner{Z}{P_{T^{\perp}}(M^{*})},
\end{align}
which implies that
\begin{align}
   |\tau^{d} - \tau^{*}| 
&\leq \underbrace{\left| \frac{\inner{P_{T^{\perp}}(Z)}{\hat{E}}}{\norm{P_{T^{\perp}}(Z)}_{\F}^2}\right|}_{A_1} +\underbrace{\left| \frac{\inner{Z}{P_{T^{\perp}}(M^{*})}}{\norm{P_{T^{\perp}}(Z)}_{\F}^2} \right|}_{A_2}. \label{eq:bound-taud-taustar}
\end{align}
We will eventually bound the terms $A_1$ and $A_2$ separately to complete the proof, but it is worth pausing here to briefly survey the proof ideas from this vantage. In particular, we can see from here that the crux of the proof will be to show that $T \approx T^*$. To see that this is sufficient for bounding $A_1$ and $A_2$, suppose that $T = T^*$. Then the denominator of $A_1$ and $A_2$ is bounded from below using \cref{assum:conditions-Z}. Specifically, recall the statement of \cref{assum:conditions-Z}(a):
\begin{align*}
\norm{ZV^{*}}_{\F}^2 + \norm{Z^{\top}U^{*}}_{\F}^2 &\leq \left(1-\frac{C_{r_1}}{\log(n)}\right) \norm{Z}_{\F}^2.
\end{align*}
A direct implication of this is that\footnote{This can be simply seen by a basis transformation. See \cref{appendix:conditions-illustration} for full details.} \[\norm{P_{T^{*\perp}}(Z)}_{\F}^2 \geq \frac{C_{r_1}}{2\log(n)} \norm{Z}_{\F}^2.\] Furthermore, $A_2$ would in fact be zero because $P_{T^*}(M^*) = 0$, and bounding $A_1$ would then amount to simply controlling $\hat{E}$.

The steps of the proof can be outlined as follows (for convenience, they are numbered to match the coming subsections):
\begin{enumerate}
\item {\em Properties of a Deterministic $\hat{E}$}: We will first state the bounds for $\hat{E}$ for convenience. 
\item {\em The Main Lemma}: We prove \cref{lem:X-Xstar-Y-Ystar}, which states that $\hat{M}$ decomposes as $\hat{M} = XY^\top$ in such a way that $X,Y$ are `close' to $X^*,Y^*$ for which $M^* = X^*Y^{*\top}$. This is arguably the `main' lemma for two reasons. First, it implies that $T \approx T^*$, which is the crux of the proof as we discussed above. Second, its proof (given in \cref{sec:showing-XXstar-YYstar}) requires our key technical machinery -- introducing a non-convex proxy problem and analyzing a (purely hypothetical) gradient descent algorithm.
\item {\em Applying the Main Lemma}: Using \cref{lem:X-Xstar-Y-Ystar}, we prove that $T$ is sufficiently `close' to  $T^*$ (\cref{lem:applying-the-main-lemma}).
\item {\em Bounding $A_1$}: We bound $A_1$ using \cref{lem:applying-the-main-lemma}.
\item {\em Bounding $A_2$}: We bound $A_2$ using \cref{lem:X-Xstar-Y-Ystar,lem:applying-the-main-lemma}.
\item {\em Putting it All Together}: We insert bounds on $\hat{E},A_1,A_2$ (from steps 1,5,6) directly into \cref{eq:bound-taud-taustar}, completing the proof.
\end{enumerate}

\subsection{Properties of a Deterministic $\hat{E}$}
As a preliminary step, we control the size of $\hat{E}$ by \cref{assum:operator-E-delta} and $\|\delta\| \lesssim \sigma\sqrt{n}$. The arguments in the rest of the proof will only require that $\hat{E}$ satisfy the following equations:
\begin{subequations}\label{eq:assumption-Ehat}
\begin{align}
\|\hat{E}\| &\lesssim \sigma \sqrt{n}, \label{eq:op-E}\\
|\inner{Z}{\hat{E}}| &\lesssim  \sigma \sqrt{n} \|Z\|_{\F}, \label{eq:inner-Z-E}
\end{align}
\end{subequations}
For \cref{eq:op-E}, it is simply by $\|E\| \lesssim \sigma \sqrt{n}, \|\delta\| \lesssim \sigma \sqrt{n}$; for \cref{eq:inner-Z-E}, note that $\inner{Z}{\delta}=0$ by the definition of $\delta$ as the residuals, and $|\inner{Z}{E}| \lesssim \sigma \|Z\|_{\F} \sqrt{n}$ (by \cref{assum:operator-E-delta}).

\subsection{The Main Lemma}
Treating $\hat{E}$ as deterministic and satisfying \cref{eq:assumption-Ehat}, we will prove a rate guarantee for $|\tau^{d}-\tau^{*}|$ which has \cref{thm:error-rate-theorem} as a special case (see \cref{eq:taud-taustar-determinisitc-bound} below). 
To begin, recall that $(\hat{M}, \hat{\tau})$ denotes an optimizer of the convex program $g(\cdot,\cdot)$ defined in \cref{eq:convex-program}, and that $M^{*}=U^{*}\Sigma^{*}V^{*\top}$ is the SVD of $M^{*}$. We can then decompose $M^{*}$ as $M^{*}=X^{*}Y^{*\top}$, where
 $$
X^{*} := U^{*}\Sigma^{*1/2}, Y^{*} := V^{*}\Sigma^{*1/2}.
$$

The key step in our proof of \cref{thm:error-rate-theorem} is to show that $\hat{M}$ can be similarly decomposed as 
$
\hat{M} = XY^{\top},
$
where $X, Y \in \R^{n\times r}$ are sufficiently close (in Frobenius norm) to $X^{*}, Y^{*}$ respectively. To be `sufficiently close' will mean that $X,Y$ lie in the following subset of  $\mathbb{R}^{n \times r} \times \mathbb{R}^{n \times r}$, which appears frequently enough in our proof to warrant its own symbol:
\begin{align} \label{eq:X-Xstar-Y-Ystar}
\mathcal{B} = \left\{(X, Y)~\Big|~\|X - X^{*}\|_{\F} + \|Y - Y^{*}\|_{\F} \leq \error \|X^{*}\|_{\F}\right\}.
\end{align}
To parse the definition of $\Bscr$, note that $\|X^{*}\|_{\F}$ in the RHS of \cref{eq:X-Xstar-Y-Ystar} can be replaced by $\|Y^{*}\|_{\F}$ (since $\|X^{*}\|_{\F}=\|Y^{*}\|_{\F} = \|\Sigma^{*1/2}\|_{\F}$), so the definition of $\Bscr$ is only asymmetric in appearance. There is also implicitly a constant on the RHS of \cref{eq:X-Xstar-Y-Ystar}, which we have taken to be 1, but could be any positive value for the sake of this proof. 
Finally, the coefficient in the RHS of \cref{eq:X-Xstar-Y-Ystar} will, by assumption, {\em shrink} with $n$, since a rearrangement of assumption \cref{eq:error-rate-theorem-conditions} in \cref{thm:error-rate-theorem} gives $\error \lesssim \frac{1}{\kappa^2 r^2 \log^{2.5}(n)}$.
Indeed, the particular form of that assumption was chosen to be sufficiently small to enable the remainder of the proof, but in the `typical' regime where $\sigma_{\min}=\Omega(n)$ and  $\sigma = O(1)$, we would have that $\error = \tilde{O}(1/\sqrt{n})$.

The key step then is summarized in the following lemma. Its proof is the primary obstacle in proving \cref{thm:error-rate-theorem}, and is deferred to \cref{sec:showing-XXstar-YYstar}. 
\begin{lemma}[The Main Lemma] \label{lem:X-Xstar-Y-Ystar}
For sufficiently large $n$, there exists $(X,Y) \in \Bscr$ (as defined in \cref{eq:X-Xstar-Y-Ystar}) such that $\hat{M} = XY^{\top}$.
\end{lemma}

\subsection{Applying the Main Lemma} Now take $X,Y$ from \cref{lem:X-Xstar-Y-Ystar}.
Let $XY^{\top}=U \Sigma V^{\top}$ be the SVD of $XY^{\top}$, and let $T$ be the tangent space of $XY^{\top}$. Intuitively, \cref{eq:X-Xstar-Y-Ystar} implies that $X\approx X^{*}$ and $Y\approx Y^{*}$, which will also imply $XY^{\top} \approx M^{*}$ and $T \approx T^{*}.$ Indeed, the following technical lemma makes this explicit. Its proof is by straightforward (yet lengthy) algebra, and thus is deferred to \cref{sec:tech-lemmas-main-theorem}. 

\begin{lemma}\label{lem:applying-the-main-lemma}
For any $(X,Y) \in \Bscr$, let $T$ be the tangent space of $XY^\top$ and $XY^{\top}=U\Sigma V^{\top}$ be the SVD of $XY^{\top}$. Then,
\begin{subequations}
\begin{align}
\norm{ZV}_{\F}^2 + \norm{Z^{\top}U}_{\F}^2 &\leq \left(1-\frac{C_{r_1}}{2\log(n)}\right) \norm{Z}_{\F}^2 \label{eq:assumption3a-mimic}\\
\left|\inner{Z}{U V^{\top}}\right| \norm{P_{T^{\perp}}(Z)} &\leq \left(1-\frac{C_{r_2}}{2\log(n)}\right)  \norm{P_{T^{\perp}}(Z)}_{\F}^2.\label{eq:assumption3b-mimic}
\end{align}
\end{subequations}
As direct implications, the followings hold
\begin{subequations}
\begin{align}
 \norm{P_{T^{\perp}}(Z)}_{\F}^2 &\geq \frac{C_{r_1}}{2\log(n)} \norm{Z}_{\F}^2, \label{lem:assumption-hold-around-ball}\\
    \norm{P_{T^{\perp}}(Z) - P_{T^{*\perp}}(Z)}_{*} &\lesssim \frac{\sigma r^{1.5}\kappa^2\sqrt{n} \log^{2.5}(n)}{\sigma_{\min}} \norm{Z}_{\F}.\label{lem:PTperpZ-near-PTstarperpZ-2}  
\end{align}
\end{subequations}
\end{lemma}
The first part of \cref{lem:applying-the-main-lemma} directly links to \cref{assum:conditions-Z} ($\norm{ZV^{*}}_{\F}^2 + \|{Z^{\top}U^{*}}\|_{\F}^2 \leq \left(1-C_{r_1}/\log(n)\right) \|{Z}\|_{\F}^2$ and $\left|\inner{Z}{U^{*}V^{*\top}}\right| \|P_{T^{*\perp}}(Z)\| \leq \left(1-C_{r_2}/\log(n)\right)  \|P_{T^{*\perp}}(Z)\|_{\F}^2$). Essentially, the \cref{eq:assumption3a-mimic} and \cref{eq:assumption3b-mimic} state that any $(X, Y) \in \mathcal{B}$ enjoys similar conditions as \cref{assum:conditions-Z}, which will be useful as the proof proceeds. 

\cref{lem:assumption-hold-around-ball} establishes a lower bound of $\|P_{T^{\perp}}(Z)\|_{\F}$ and \cref{lem:PTperpZ-near-PTstarperpZ-2} states that $P_{T^{\perp}}(Z) \approx P_{T^{*\perp}}(Z)$ as measured by the nuclear norm. They are direct implications of \cref{eq:assumption3a-mimic,eq:assumption3b-mimic}. The details are deferred to \cref{sec:tech-lemmas-main-theorem}.

\subsection{Bounding $A_1$} 
\begin{align}
A_1 &=  \left| \frac{\inner{P_{T^{\perp}}(Z)}{\hat{E}}}{\norm{P_{T^{\perp}}(Z)}_{\F}^2} \right| \nonumber \\
&\leq \left| \frac{\inner{P_{T^{\perp}}(Z)-P_{T^{*\perp}}(Z)}{\hat{E}}}{\norm{P_{T^{\perp}}(Z)}_{\F}^2}\right| + \left| \frac{\inner{P_{T^{*\perp}}(Z)}{\hat{E}}}{\norm{P_{T^{\perp}}(Z)}_{\F}^2}\right| \nonumber\\
&\overset{(i)}{\leq} \frac{ \|P_{T^{\perp}}(Z)-P_{T^{*\perp}}(Z)\|_{*}\|\hat{E}\|}{\norm{P_{T^{\perp}}(Z)}_{\F}^2} + \left| \frac{\inner{P_{T^{*\perp}}(Z)}{\hat{E}}}{\norm{P_{T^{\perp}}(Z)}_{\F}^2}\right| \nonumber\\
&\overset{(ii)}{\lesssim} \frac{\sigma r^{1.5}\kappa^2 \sqrt{n} \log^{3.5}(n)}{\sigma_{\min} \|Z\|_{\F}} \|\hat{E}\| + \frac{\log(n)}{\|Z\|_{\F}^2} \left|\inner{P_{T^{*\perp}}(Z)}{\hat{E}}\right| \label{eq:bound-A1-main-theorem}
\end{align}
where (i) is due to the trace inequality, and (ii) uses \cref{lem:assumption-hold-around-ball,lem:PTperpZ-near-PTstarperpZ-2}.

\subsection{Bounding $A_2$} 
\begin{align}
\left|\frac{\inner{Z}{P_{T^{\perp}}(M^{*})}}{\norm{P_{T^{\perp}}(Z)}_{\F}^2}\right| &= \left|\frac{\inner{P_{T^{\perp}}(Z)}{P_{T^{\perp}}(M^{*})}}{\norm{P_{T^{\perp}}(Z)}_{\F}^2}\right| \nonumber\\
&\overset{(i)}{\leq} \frac{\|P_{T^{\perp}}(Z)\|_{\F} \|P_{T^{\perp}}(M^{*})\|_{\F}}{\norm{P_{T^{\perp}}(Z)}_{\F}^2} \nonumber\\
&= \frac{\|P_{T^{\perp}}(M^{*})\|_{\F}}{\|P_{T^{\perp}}(Z)\|_{\F}} \overset{(ii)}{\lesssim} \frac{\|P_{T^{\perp}}(M^{*})\|_{\F}\log^{0.5}(n)}{\|Z\|_{\F}} \label{eq:ZPhatTMstar},
\end{align}
where (i) is Cauchy-Schwartz, and (ii) uses  \cref{lem:assumption-hold-around-ball}. For $\|P_{T^{\perp}}(M^{*})\|_{\F}$, note that $P_{T^{\perp}}(XA^{\top}) = P_{T^{\perp}}(AY^{\top})=0$ for any $A \in \R^{n\times r}$ by the definition of the tangent space $T$ and the non-degeneracy of $X, Y$. This implies that
\begin{align}
\|P_{T^{\perp}}(M^{*})\|_{\F} \nonumber
&= \|P_{T^{\perp}}(X^{*}Y^{*\top})\|_{\F} \nonumber\\
&\overset{(i)} = \|P_{T^{\perp}}((X-X^{*})(Y-Y^{*})^{\top})\|_{\F} \nonumber\\
&\leq \|(X-X^{*})(Y-Y^{*})^{\top}\|_{\F} \nonumber\\
&\leq \|X-X^{*}\|_{\F} \|Y-Y^{*}\|_{\F} \overset{(ii)}{\lesssim} \frac{\sigma^2 n \log^{5}(n)}{\sigma_{\min}^2} \|X^{*}\|_{\F}^2 \label{eq:PhatTMstar}
\end{align}
where (i) is due to $P_{T^{\perp}}(X^{*}Y^{\top})=P_{T^{\perp}}(XY^{*\top})=P_{T^{\perp}}(XY^{\top})=0$, and (ii) is by \cref{lem:X-Xstar-Y-Ystar}. Plugging \cref{eq:PhatTMstar} back into \cref{eq:ZPhatTMstar} with $\|X^{*}\|_{\F}^2 \leq \sigma_{\max}r$, we have
\begin{align}
A_2 \leq \frac{\sigma^2 n \log^{5.5}(n) \kappa r}{\sigma_{\min}} \frac{1}{\|Z\|_{\F}}. \label{eq:bound-A2-main-theorem}
\end{align}

\subsection{Putting it All Together} Circling back to \cref{eq:bound-taud-taustar} from the beginning of the proof:
\begin{align*}
   |\tau^{d} - \tau^{*}| 
&\leq \underbrace{\left| \frac{\inner{P_{T^{\perp}}(Z)}{\hat{E}}}{\norm{P_{T^{\perp}}(Z)}_{\F}^2}\right|}_{A_1} +\underbrace{\left| \frac{\inner{Z}{P_{T^{\perp}}(M^{*})}}{\norm{P_{T^{\perp}}(Z)}_{\F}^2} \right|}_{A_2} \\
&\overset{(i)}{\lesssim} \frac{\sigma r^{1.5}\kappa^2 \sqrt{n} \log^{3.5}(n)}{\sigma_{\min} \|Z\|_{\F}} \|\hat{E}\| + \frac{\log(n)}{\|Z\|_{\F}^2} \left|\inner{P_{T^{*\perp}}(Z)}{\hat{E}}\right|  + \frac{\sigma^2 n \log^{5.5}(n) \kappa r}{\sigma_{\min}} \frac{1}{\|Z\|_{\F}}
\end{align*}
where (i) comes directly from the bounds for $A_1$ and $A_2$ in \cref{eq:bound-A2-main-theorem,eq:bound-A1-main-theorem}. Rearranging terms above, we obtain the following bound for $|\tau^{d}-\tau^{*}|$ (with deterministic $\hat{E}$):
\begin{align}
|\tau^{d} - \tau^{*}| = \frac{\sigma r^{1.5} \kappa^2 \log^{3.5}(n) \sqrt{n}}{\sigma_{\min} \|Z\|_{\F}} O\left(\|\hat{E}\| + \sigma\sqrt{n}\log^{2}(n)\right) + O\left(\frac{\log(n)}{\|Z\|_{\F}^2}\right) \left|\inner{P_{T^{*\perp}}(Z)}{\hat{E}}\right|. \label{eq:taud-taustar-determinisitc-bound}
\end{align}
\cref{thm:error-rate-theorem} is then a direct implication of \cref{eq:taud-taustar-determinisitc-bound}: plugging the bounds $\|\hat{E}\| \lesssim \sigma \sqrt{n}$ and $\inner{P_{T^{*\perp}}(Z)}{\hat{E}} = \inner{P_{T^{*\perp}}(E)}{Z} + \inner{P_{T^{*\perp}}(\delta)}{Z}$ into \cref{eq:taud-taustar-determinisitc-bound} completes the proof of \cref{thm:error-rate-theorem}.

\section{Proof Sketch of \cref{thm:main-theorem}} \label{sec:proof-main-theorem}
We conclude by describing the proof ideas of \cref{thm:main-theorem} (the complete proof is contained in \cref{sec:full-details-proof-main}). First, a few simplifying reductions:
\begin{itemize}
\item \cref{thm:main-theorem} contains two main conclusions: \eqref{eq:main-theorem} and \eqref{eq:main-theorem-2}.
It suffices to show \eqref{eq:main-theorem}, as \eqref{eq:main-theorem-2} follows directly from \eqref{eq:main-theorem} by applying Berry-Esseen type inequalities. 
\item In the statement of \cref{thm:main-theorem}, $\delta$ is assumed to consist of independent, mean-zero, sub-Gaussian variables with $O(1)$ sub-Gaussian norm. These same assumptions are made on $E$, and so we can without loss assume that $\delta = 0$ since it can be absorbed into $E$.
\end{itemize}

The overall framework of the proof is similar to that of \cref{thm:error-rate-theorem} above, but requires a substantially more-refined analysis for controlling $\tau^{d}-\tau^{*}$, in order to show asymptotic normality. 
For example, as we will see in a moment, the bounds on Frobenius error $\|X-X^{*}\|_{\F}, \|Y-Y^{*}\|_{\F}$, as used in the proof of \cref{thm:error-rate-theorem} (\cref{sec:Proof-Outline,sec:showing-XXstar-YYstar}), are insufficient for establishing asymptotic normality. Instead we require error bounds on row-norms  $\|X-X^{*}\|_{2,\infty}, \|Y-Y^{*}\|_{2,\infty}$. Similarly, more-refined versions of the statement that $T \approx T^{*}$ (recall that $T$ is the tangent space of $XY^{\top}$) are needed. 

Just as in the proof of \cref{thm:error-rate-theorem}, we start from \eqref{eq:lem:tau-decomposition} (i.e. \cref{lem:tau-decomposition}) and decompose $\tau^d - \tau^*$, but this time more carefully, as follows:
\begin{align*}
   \norm{P_{T^{\perp}}(Z)}_{\F}^2(\tau^{d} - \tau^{*} )
&= \inner{P_{T^{\perp}}(Z)}{E} + \inner{Z}{P_{T^{\perp}}(M^{*})} \\
&= \inner{Z}{P_{T^{\perp}}(E)}
+ \inner{Z}{P_{T^{\perp}}(M^{*})} \\
&= \inner{Z}{P_{T^{*\perp}}(E)}
+ \underbrace{\inner{Z}{P_{T^{\perp}}({E})-P_{T^{*\perp}}({E})}}_{A_1}
+ \underbrace{\inner{Z}{P_{T^{\perp}}(M^{*})}}_{A_2}
\end{align*}
The previous proof had the first two terms in the last line combined. This time, the first term is the `source' of the limiting distribution, and so it remains to bound the other two terms: $A_1$ and $A_2$. For both terms, we will use the fact that the entries of $Z$ are binary. So for $A_1$:
\begin{align*}
\inner{Z}{P_{T^{\perp}}({E})-P_{T^{*\perp}}({E})} \le \|Z\|_\F^2 \|	P_{T^{\perp}}({E})-P_{T^{*\perp}}({E})\|_\infty 
\end{align*}
Similarly for $A_2$, starting again with the fact that the entries of $Z$ are binary:
\begin{align*}
\inner{Z}{P_{T^{\perp}}(M^{*})} 
& \le \|Z\|_\F^2 \|P_{T^{\perp}}(M^{*})\|_\infty.
\end{align*}

Stepping through the proof of \cref{thm:error-rate-theorem}, the terms $\|P_{T^{\perp}}({E})-P_{T^{*\perp}}({E})\|_\infty$, $\|P_{T^{\perp}}(M^{*})\|_{\infty}$ are {\em not} directly controlled, but can be with more careful analysis. For example, the following lemma establishes these guarantees for the iteration sequences $\{(X^{t}, Y^{t}, \tau^{t})\}$ in Algorithm \ref{alg:GD}, as a significant generalization of \cref{lem:induction-Frobenious}.
Incidentally, \cref{thm:non-convex-theorem} also enables entry-wise control for recovering $M^{*}$, and thus may be of independent interest in matrix completion \citep{ma2019implicit,chen2019noisy,chen2020nonconvex}.  The proof is in \cref{sec:proof-non-convex-iteration}.

\begin{restatable}{lemma}{ThmNonConvexTheorem}\label{thm:non-convex-theorem}
Suppose $O = M^{*} + \tau^{*}Z + E$, where the entries of $E$ are independent sub-Gaussian random variables with $\|E_{ij}\|_{\psi_2} \leq \sigma.$ 

Let $\lambda = C_{\lambda} \sigma \sqrt{n} \log^{1.5}(n), \eta = \frac{C_{\eta}}{n^{20}\sigma_{\max}}, t^{\star} = n^{C_t}$. For any $C_3 > 0$, for large enough $n$, with probability $1-O(n^{-C_3})$, the iteration $(X^{t}, Y^{t}, \tau^{t})_{0\leq t\leq t^{\star}}$ given in Algorithm \ref{alg:GD} satisfies the following:
\begin{subequations}\label{eq:induction-main}
\begin{align}
    \norm{F^{t}H^{t} - F^{*}}_{\F}&\leq C_{\F} \error \norm{F^{*}}_{\F} \label{eq:FtHt-Fstar-Fnorm-main}\\
   \norm{F^{t}H^{t} - F^{*}}_{2,\infty} &\leq C_{\infty}\left(\frac{\sigma \mu r^{2.5} \kappa \log^{3.5}(n)}{\sigma_{\min}}\right)\norm{F^{*}}_{\F} \label{eq:FH-Fstar-rownorm-main}\\
    |\tau^{t} - \tau^{*}| &\leq C_{\tau} \left( \frac{\sigma \mu r^{2.5} \kappa \log^{3.5} n}{\sqrt{n}} + \frac{\sigma \log^{1.5} n}{\norm{Z}_{\F}}\right)\label{eq:tau-taustar-main}\\
    \min_{0\leq t\leq t^{\star}} \norm{\nabla f(X^{t}, Y^{t}; \tau^{t})}_{\F} &\leq \frac{\lambda \sqrt{\sigma_{\min}}}{n^{10}} \label{eq:gradient-small-main}.
\end{align}
\end{subequations}
Furthermore, let $T_t$ be the tangent space of $X^{t}Y^{t\top}$. With probability $1-O(n^{-C_3})$, the following hold for all iterations $0\leq t \leq t^{\star}$ simultaneously:
\begin{subequations}\label{eq:PT-bound-main}
\begin{align}
    \norm{P_{T_t^{\perp}}(M^{*})}_{\infty} &\leq  C_{T,1}\frac{\sigma^2 \mu^2 r^{6} \kappa^{3} \log^{7}(n)}{\sigma_{\min}}\label{eq:PTM-bound-main}\\
    \norm{P_{T_t^{\perp}}(E) - P_{T^{*\perp}}(E)}_{\infty} &\leq C_{T,2}\frac{\sigma^2 r^{3.5} \mu^{1.5} \kappa^{2} \log^{4}(n)}{\sigma_{\min}}\label{eq:PTE-bound-main}\\
    \norm{P_{T_t^{\perp}}(Z)-P_{T^{*\perp}}(Z)}_{\F} &\leq C_{T,3}\frac{r\kappa \sigma \sqrt{n}\log^{2.5}(n)}{\sigma_{\min}} \norm{Z} 
    \label{eq:PTZ-bound-main}.
\end{align}
\end{subequations}
Here, $C_{\lambda}$, $C_{\eta}$, $C_{t}$, $C_{\F}$, $C_{\infty}$, $C_{\tau}$, $C_{T,1}$, $C_{T,2}$, $C_{T,3}$ are constants depending (polynomially) on $C_1$, $C_2$, $C_3$, $C_{r_1}$ (where $C_{r_1}$ is the constant in \cref{cond:Z-condition-nonconvex}).
\end{restatable}

In contrast to \cref{lem:induction-Frobenious}, \cref{thm:non-convex-theorem}, involving probabilistic statements, establishes more bounds including $|\tau^{t}-\tau^{*}|$ (\cref{eq:tau-taustar-main}) and $\|F^{t}H^{t}-F^{*}\|_{2,\infty}$ (\cref{eq:FH-Fstar-rownorm-main}), as well as various bounds that show $T^{t} \approx T^{*}$ (\cref{eq:PT-bound-main}). The proof of \cref{thm:non-convex-theorem} relies on mathematical induction and the \textit{leave-one-out} technique developed in \cite{ma2019implicit,abbe2017entrywise,chen2019noisy}, which introduces a set of auxiliary loss functions to facilitate the analysis of the gradient descent algorithm. We also establish a number of technical innovations to minimize additional assumptions on $Z$ while maintaining the desired estimation error bounds. 

Note that \cref{thm:non-convex-theorem} only holds for $t^{\star}$ (a finite number) steps, from which only an approximate critical point of $f$ with $\nabla f \approx 0$ can be found, instead of an exact critical point. This necessitates a generalization of \cref{lem:connection-deterministic}, to connect $f$ and $g$ approximately. The lemma below establishes such a connection: an approximate critical point of $f$ with (extremely) small gradient is (extremely) close to the optimizer of $g$.

\begin{lemma}\label{lem:connection}
Assume $(X,Y) \in \Bscr$ and $\|\nabla f(X, Y; \tau)\|_{\F} \leq \frac{\lambda\sqrt{\sigma_{\min}}}{\kappa n}$ with $\tau = \inner{Z}{O-XY^{\top}}/\|Z\|_{\F}^2$. Let $(\hat{M}, \hat{\tau})$ be an optimal solution of the convex program \cref{eq:convex-program}.

Then for any $C>0$, for large enough $n$, the following hold with probability $1-O(n^{-C}):$
\begin{align}
    \|{XY^{\top} - \hat{M}}\|_{\F} &\lesssim \frac{\kappa \log(n)}{\sqrt{\sigma_{\min}}}  \norm{\nabla f(X,Y; \tau)}_{\F}, \nonumber\\
    |\tau - \hat{\tau}| &\lesssim \frac{\kappa \log(n)}{\sqrt{\sigma_{\min}} \norm{Z}_{\F}}  \norm{\nabla f(X,Y; \tau)}_{\F}. \nonumber
\end{align}
\end{lemma}

Finally, given \cref{lem:tau-decomposition,lem:connection,thm:non-convex-theorem}, similar to the proof of \cref{thm:error-rate-theorem}, we establish the desired results by controlling the error of $P_{\hat{T}^{\perp}}(M^{*})$, $P_{\hat{T}^{\perp}}(E) - P_{T^{*\perp}}(E)$, and showing that $P_{\hat{T}^{\perp}}(Z) \approx P_{T^{*\perp}}(Z)$, which matches the bounds provided in \cref{thm:non-convex-theorem}.

%% file: Conclusion.tex
\section{Conclusion}\label{sec:conclusion}

Motivated by the extremely important econometric problem of estimating treatment effects from panel data, we studied a natural formulation of this problem as one of recovering an unknown quantity that has been added to an unknown low-rank matrix at a subset of its entries. We proposed an estimator based on solving, and de-biasing, a natural regularized least-squares objective. We built on recent techniques for establishing entry-wise guarantees that leverage a connection between convex and non-convex formulations, and proved that our estimator is order-optimal with nearly minimal conditions on the underlying low-rank matrix and the pattern of `treated' entries.

%% file: Extension.tex
\section{Proof of Lemma \ref{lem:tau-decomposition}} \label{sec:proof-of-first-order-lemma}
We begin by stating the first order optimality conditions for \cref{eq:convex-program}:
\begin{subequations}\label{eq:convex-conditions-mdim}
\begin{align}\label{eq:convex-condition-tau-mdim}
\left\langle Z_l, O - \hat M - \sum_{m=1}^k \hat \tau_m Z_m \right\rangle = 0 
\quad {\rm for \ } l = 1,2,\dots,k
\end{align}
\end{subequations}
\addtocounter{equation}{-1}
\begin{subequations}
\setcounter{equation}{1}
\begin{align}\label{eq:convex-condition-M-mdim}
O - \hat{M} - \sum_{m=1}^k \hat \tau_m Z_m &= \lambda (\hat{U}\hat{V}^{\top} + W), \nonumber\\
    \norm{W} &\leq 1, \; P_{\hat{T}^\perp}(W) = W, 
\end{align}
\end{subequations}

Substituting \cref{eq:convex-condition-M-mdim} in \cref{eq:convex-condition-tau-mdim} immediately yields
\begin{equation}
\label{eq:mdim_sub_w}
\inner{Z_l}{\hat U \hat V^\top} 
= - \inner{Z_l}{W} 
\end{equation}
Whereas projecting both sides of \cref{eq:convex-condition-M-mdim} onto $\hat T^\perp$ yields
\[
P_{\hat T^{\perp}}\left(
M^* + \hat{E} + \sum_{m=1}^k (\tau^*_m - \hat \tau_m) Z_M 
\right)
= \lambda W
\]
where we use the fact that $P_{\hat T^{\perp}}(\hat M)$ and $P_{\hat T^{\perp}}(\hat U \hat V^\top)$  are identically zero, and $P_{\hat T^{\perp}}(W) = W$. Together with \cref{eq:mdim_sub_w} this yields
\[
\begin{split}
-\lambda \inner{Z_l}{\hat U \hat V^\top} 
&=
\lambda \inner{Z_l}{W}
\\
& =
\inner{Z_l}{P_{\hat T^{\perp}}(M^*)}
+
\inner{Z_l}{P_{\hat T^{\perp}}(\hat{E})}
+ 
\sum_{m=1}^k (\tau^*_m - \hat \tau_m) \inner{Z_l}{P_{\hat T^{\perp}}(Z_M)} 
\\
& =
\inner{Z_l}{P_{\hat T^{\perp}}(M^*)}
+
\inner{P_{\hat T^{\perp}}(Z_l)}{\hat{E}}
+ 
\sum_{m=1}^k (\tau^*_m - \hat \tau_m) \inner{P_{\hat T^{\perp}}(Z_l)}{P_{\hat T^{\perp}}(Z_M)} 
\end{split}
\]
which is the result. 

%% file: Appendix-proof-of-lower-bound.tex
\section{Proof of Proposition \ref{prop:necessity-of-conditions}}\label{sec:proof-no-go-identification}
Consider the following construction: fix any even-valued $n$, and let  $Z,M_1,M_2$ be the $n$-by-$n$ matrices depicted below:
\begin{equation} \label{eqn:counter-example}
Z =
\begin{bmatrix}
\mathbf{1} &&& \mathbf{0} \\
\mathbf{0} &&& \mathbf{1} 
\end{bmatrix}
\quad\quad M_1 =
\begin{bmatrix}
\mathbf{-1} && \mathbf{0} \\
\mathbf{0} && \mathbf{0} 
\end{bmatrix}
\quad\quad M_2 =
\begin{bmatrix}
\mathbf{0} &&& \mathbf{0} \\
\mathbf{0} &&& \mathbf{1} 
\end{bmatrix},
\end{equation}
where the boldfaced $\mathbf{1}$ represents the $n/2$-by-$n/2$ matrix of all ones (and similarly for $\mathbf{0}$ and $\mathbf{-1}$). 
Then, $Z + M_1 = M_2$ immediately follows. It is also easy to see that both $M_1$ and $M_2$ satisfy $r=1,\kappa=1, \mu=O(1)$, and $\sigma_{\min}=\Theta(n)$.

Next, consider the SVD of $M_1$: $M_1 = \sigma_1 u_1 v_1^{\top}$ where
\begin{align*}
v_1 := \sqrt{2} \Big[\underbrace{\frac{1}{\sqrt{n}}, \dotsc, \frac{1}{\sqrt{n}}}_{\frac{n}{2}}, 0, \dotsc 0\Big]^{\top}, \quad \quad u_1 := \sqrt{2} \Big[\underbrace{\frac{1}{\sqrt{n}}, \dotsc, \frac{1}{\sqrt{n}}}_{\frac{n}{2}}, 0, \dotsc 0\Big]^{\top}.
\end{align*}
Then one can verify that
\begin{align*}
\|Zv_1\|_{\F}^2 + \|Z^{\top} u_1\|_{\F}^2 = n^2 /4 + n^2 /4 = \norm{Z}_{\F}^2.
\end{align*}
Furthermore, let $T_1$ be the tangent space of $M_1$. Note that $P_{T_1^{\perp}}(Z) = \begin{bmatrix}
\mathbf{0} &&& \mathbf{0} \\
\mathbf{0} &&& \mathbf{1} 
\end{bmatrix}$. Then,  
\begin{align*}
\left|\inner{Z}{u_1v_1^{\top}}\right| \|P_{T_1^{\perp}}(Z)\| = \frac{n}{2} \cdot \frac{n}{2} = \frac{n^2}{4} = \|P_{T_1^{\perp}}(Z)\|_{\F}^2.
\end{align*}
Similar equalities hold for $M_2$ by symmetry. This completes the proof.

\section{Proof of Proposition \ref{prop:necessity-of-conditions-tau}}\label{sec:proof-of-propositions}
In this section, we present the proof of \cref{prop:necessity-of-conditions-tau}.

To begin, let $M_1, M_2$ be the $n$-by-$n$ matrices depicted below (for $n=3k, k \in \mathbb{Z}$):
\begin{align} \label{eqn:counter-example}
M_1 =
\begin{bmatrix}
-3t\cdot \mathbf{1}_{1\times \frac{n}{3}} && \mathbf{0}_{1\times \frac{2n}{3}} \\
\mathbf{1}_{(n-1)\times \frac{n}{3}} && \mathbf{0}_{(n-1) \times \frac{2n}{3}} 
\end{bmatrix}&
\quad\quad M_2 =
\begin{bmatrix}
3t\cdot \mathbf{1}_{1\times \frac{n}{3}} && \mathbf{0}_{1\times \frac{2n}{3}} \\
\mathbf{1}_{(n-1)\times \frac{n}{3}} && \mathbf{0}_{(n-1) \times \frac{2n}{3}} 
\end{bmatrix},
\end{align}
where $\mathbf{1}_{a\times b}$ is a $a\times b$ block with all ones (similar for $\mathbf{0}$) and $t \in [0, 1/5]$ is a parameter. Furthermore, let $Z$ be the matrix with all ones in the first row:
\begin{align}
Z &=
\begin{bmatrix}
\mathbf{1}_{1\times n}\\
\mathbf{0}_{(n-1) \times n}
\end{bmatrix}.
\end{align}
Then let $\delta_{1},\delta_{2}$ be the $n$-by-$n$ matrices below:
\begin{align} \label{eqn:counter-example}
\delta_1 =
\begin{bmatrix}
2t\cdot \mathbf{1}_{1\times \frac{n}{3}} && -t\cdot \mathbf{1}_{1\times \frac{2n}{3}} \\
\mathbf{0}_{(n-1)\times \frac{n}{3}} && \mathbf{0}_{(n-1) \times \frac{2n}{3}} 
\end{bmatrix}&
\quad\quad \delta_2 = \begin{bmatrix}
-2t\cdot \mathbf{1}_{1\times \frac{n}{3}} && t\cdot \mathbf{1}_{1\times \frac{2n}{3}} \\
\mathbf{0}_{(n-1)\times \frac{n}{3}} && \mathbf{0}_{(n-1) \times \frac{2n}{3}} .
\end{bmatrix}
\end{align}
Let $\tau_1^{*} = 0$ and $\tau_2^{*} = -2t.$ Then one can verify that $\langle \delta_{1}, Z \rangle = 0, \langle \delta_{2}, Z \rangle = 0$ and 
\begin{align*}
M_1 + \tau_1^{*}Z + \delta_1 \circ Z = M_2 + \tau_2^{*} Z + \delta_2 \circ Z,
\end{align*}
i.e., $(M_1, \tau_1^{*}, \delta_1)$ and $(M_2, \tau_2^{*}, \delta_2)$ will be indistinguishable given the observation. Therefore, for any estimator $\hat{\tau}$, with probability at least $1/2$, either $|\hat{\tau}-\tau_1^{*}| > t$ holds or $|\hat{\tau}-\tau_2^{*}| > t$ holds.

Next, we examine the value of $\langle P_{T_1}(\delta_1), P_{T_1}(Z) \rangle$ where $T_1$ is the tangent space of $M_1$. Easy to see that the SVD of $M_1$ is $M_1 = \sigma_1 u_1 v_1^{\top}$ where
\begin{align*}
v_1 := \sqrt{3} \Big[\underbrace{\frac{1}{\sqrt{n}}, \dotsc, \frac{1}{\sqrt{n}}}_{\frac{n}{3}}, 0, \dotsc 0\Big]^{\top}, \quad \quad u_1 := \sqrt{\frac{n}{n-1+9t^2}}\left[\frac{-3t}{\sqrt{n}}, \frac{1}{\sqrt{n}}, \dotsc, \frac{1}{\sqrt{n}}\right]^{\top} 
\end{align*}
Then, 
\begin{align*}
P_{T_1}(Z) 
&= Z v_1 v_1^{\top} + u_1u_1^{\top} Z (I - v_1 v_1^{\top}) \\
&= \begin{bmatrix}
\mathbf{1}_{1\times \frac{n}{3}} && \mathbf{0}_{1\times \frac{2n}{3}} \\
\mathbf{0}_{(n-1)\times \frac{n}{3}} && \mathbf{0}_{(n-1) \times \frac{2n}{3}}
\end{bmatrix}
+ u_1u_1^{\top} Z (I - v_1 v_1^{\top}).
\end{align*}
This implies
\begin{align*}
\langle P_{T_1}(\delta_1), P_{T_1}(Z) \rangle 
&= \inner{\delta_1}{P_{T_1}(Z)}\\
&= 2t \frac{n}{3} + \inner{\delta_{1}}{u_1u_1^{\top} Z (I - v_1 v_1^{\top})}\\
&\leq \frac{2t}{3} n + \norm{u_1^{\top}\delta_1}\norm{u_1^{\top}Z} \norm{I - v_1 v_1^{\top}}\\
&\overset{(i)}{\leq} tn.
\end{align*}
where (i) is due to $t \leq 1/5.$ Similarly, one can show that $\langle P_{T_2}(\delta_2), P_{T_2}(Z) \rangle \leq tn.$ Note that $\|Z\|_{\F}^2 = n.$ Then, with probability $1/2$, there exists $i \in \{1, 2\}$ such that
\begin{align*}
|\hat{\tau}-\tau^{*}_i| \geq \frac{\inner{P_{T_i}(\delta_i)}{P_{T_i}(Z)}}{\|Z\|_{\F}^2} = \frac{\inner{P_{T_i^{\perp}}(\delta_i)}{Z}}{\|Z\|_{\F}^2}
\end{align*}

After introducing the noise $E_i=\delta_i$, it is also easy to see that $|\hat{\tau}-\tau^{*}_i| \geq \frac{\inner{P_{T_i^{\perp}}(\delta_i+E_i)}{Z}}{\|Z\|_{\F}^2}$. In addition, one shall verify that $r=\kappa=1, \sigma,\mu=O(1), \sigma_{\min}=\Theta(n)$ for both $M_1$ and $M_2$. Furthermore, $\|\delta_1\|, \|\delta_2\|, \|E_{1}\|, \|E_2\| = O(\sqrt{n})$; \cref{assum:conditions-Z} hold for both $M_1, M_2$ and $Z$; \cref{assum:operator-E-delta} hold for both $E_1, E_2$ and $Z$. Note that $\frac{\sigma}{\|Z\|_{\F}} = O(1/\sqrt{n}).$ This then implies that with probability $1/2$, there exists $i \in \{1, 2\}$ such that
\begin{align*}
|\hat{\tau}-\tau^{*}_i| > t \geq O\left(\frac{\sigma}{\|Z\|_{\F}} + \frac{\inner{P_{T_i^{\perp}}(\delta_i+E_i)}{Z}}{\|Z\|_{\F}^2}\right)
\end{align*}
This shows that \cref{thm:error-rate-theorem} achieves the mini-max lower bound.

%% file: Appendix-PTperp.tex
\section{Basis Transformation for Assumption \ref{assum:conditions-Z}}\label{appendix:conditions-illustration}
Here, we present an ``intuitive'' way to interpret \cref{cond:Z-condition-nonconvex} and \cref{cond:Z-condition-convex} by a change of basis for $Z$. 

Let $M^{*}=U^{*}\Sigma^{*}V^{*\top}$ be the SVD of $M^{*}$ where $U^{*}, V^{*} \in \R^{n \times r}$ characterize the column and row space of $M^{*}$ respectively. We consider the expansion of $U^{*}$ and $V^{*}.$ Note that $P_{T^{*\perp}}(Z) = (I - U^{*}U^{*T}) Z (I - V^{*}V^{*\top})$ has the columns and rows orthogonal to the spaces of $U^{*}$ and $V^{*}.$ We take $P_{T^{*\perp}}(Z) = U^{\perp}\Sigma^{\perp}V^{\perp \top}$ be the SVD of $P_{T^{*\perp}}(Z).$ Then $\tilde{U} = [U^{*}, U^{\perp}]$\footnote{Expand $U^{\perp}$ to $n-r$ columns if it is not.} constitutes the basis matrix for the space of $\R^{n}$ (the same for $\tilde{V} = [V^{*}, V^{\perp}]$). 

We then consider a basis transformation for $Z$ (left and right basis respectively) based on $\tilde{U}$ and $\tilde{V}$:
\begin{equation*}
\tilde{U}^{\top}Z \tilde{V}
=\begin{bmatrix}
U^{*\top}ZV^{*} & U^{*\top}ZV^{\perp} \\
U^{\perp \top}ZV^{*} & U^{\perp \top} Z V^{\perp}
\end{bmatrix}
=:
\begin{bmatrix}
Z_{A} & Z_{B} \\
Z_{C} & Z_{D}
\end{bmatrix}.
\end{equation*}
Here, the $2 \times 2$ block representation with $Z_{A} \in \R^{r\times r}, Z_{B} \in \R^{r\times (n-r)}, Z_{C} \in \R^{(n-r) \times r}, Z_{D} \in \R^{(n-r) \times (n-r)}$ is closely connected to $P_{T^{*\perp}}(Z)$ and $P_{T}(Z)$, which will provide nice interpretations for our conditions (described momentarily). To begin, from the definition of $P_{T^{*\perp}}(Z)$, one can see that
\begin{equation*}
\tilde{U}^{\top} P_{T^{*\perp}}(Z) \tilde{V}
=\begin{bmatrix}
U^{*\top} P_{T^{*\perp}}(Z) V^{*} & U^{*\top} P_{T^{*\perp}}(Z)V^{\perp} \\
U^{\perp \top} P_{T^{*\perp}}(Z) V^{*} & U^{\perp \top}  P_{T^{*\perp}}(Z) V^{\perp}
\end{bmatrix}
=:
\begin{bmatrix}
\mathbf{0} & \mathbf{0} \\
\mathbf{0} & Z_{D}
\end{bmatrix}.
\end{equation*}
Similarly, due to $P_{T^{*}}(Z) + P_{T^{*\perp}}(Z) = Z$, we have
\begin{equation*}
\tilde{U}^{\top} P_{T^{*}}(Z) \tilde{V}
=:
\begin{bmatrix}
Z_{A} & Z_{B} \\
Z_{C} & \mathbf{0}
\end{bmatrix}.
\end{equation*}

In addition, this also provides a way to interpret $U^{*\top} Z$ and $Z V^{*}$ by noting that
\begin{equation*}
\left(U^{*\top} Z\right) \tilde{V}
=:
\begin{bmatrix}
Z_{A} & Z_{B}
\end{bmatrix} 
\quad
\quad
\tilde{U}^{\top} \left(Z V^{*}\right)
=:
\begin{bmatrix}
Z_{A}\\
Z_{C}
\end{bmatrix} 
\end{equation*}

Then, we are ready to state the \cref{cond:Z-condition-nonconvex} in an equivalent way (See \cref{lem:condition-non-convex-implication} for a generalized statement and a more algebraic proof).
\begin{claim}
\cref{cond:Z-condition-nonconvex} is equivalent to the following condition: there exists a constant $C > 0$ such that
\begin{align*}
\frac{C}{\log n}\norm{Z}_{\F}^2 \leq \norm{Z_{D}}_{\F}^2 - \norm{Z_{A}}_{\F}^2.
\end{align*}
\end{claim}
\begin{proof}
Note that $\norm{Z}_{\F}^2 = \norm{Z_{A}}_{\F}^2 + \norm{Z_{B}}_{\F}^2 + \norm{Z_{C}}_{\F}^2 + \norm{Z_{D}}_{\F}^2$, $\norm{Z V^{*}}_{\F}^2 = \norm{Z_{A}}_{\F}^2 + \norm{Z_{C}}_{\F}^2$, and $\norm{U^{*\top} Z}_{\F}^2 = \norm{Z_{A}}_{\F}^2 + \norm{Z_{B}}_{\F}^2$ due to the unitary invariance of the Frobenius norm. Then \cref{cond:Z-condition-nonconvex} can be transformed into
\begin{align*}
\norm{Z}_{\F}^2 + \norm{Z_A}_{\F}^2 - \norm{Z_D}_{\F}^2 \leq \left(1-\frac{C}{\log n}\right) \norm{Z}_{\F}^2
\end{align*}
which completes the proof.
\end{proof}
Since $\norm{Z_{D}}_{\F}^2 = \norm{P_{T^{*\perp}}(Z)}_{\F}^2$, the \cref{cond:Z-condition-nonconvex} is effectively a lower bound for requiring $\norm{P_{T^{*\perp}}(Z)}_{\F}^2$ being sufficiently large. 

For \cref{cond:Z-condition-convex}, note that $\inner{Z}{U^{*}V^{*T}} = \tr(U^{*T}ZV^{*}) = \tr(Z_{A})$ and $\norm{P_{T^{*\perp}}(Z)} = \norm{Z_{D}}$ due to the unitary invariance of the operator norm. We have
\begin{claim}
\cref{cond:Z-condition-convex} is equivalent to the following condition: there exists a constant $C > 0$ such that
\begin{align*}
|\tr(Z_{A})| \norm{Z_{D}} \leq \left(1-\frac{C}{\log n}\right) \norm{Z_{D}}_{\F}^2.
\end{align*}
\end{claim}

To see the looseness of these two conditions, note that $Z_{A}$ is a $r\times r$ matrix and $Z_{D}$ is a $(n-r) \times (n-r)$ matrix where $n \gg r.$ Intuitively, some degree of non-overlapping between $Z$ and $M^{*}$ is enough to guarantee a small $Z_{A}$ and large $Z_{D}$ and sufficient for these two conditions. See \cref{sec:treatment-conditions} for a few examples to make this intuition precise.

\textit{Generalization to $k>1$.} To generalize \cref{assum:conditions-Z} to case $k>1$, we need that every linear combination of $Z_{i}, i\in[k]$ ($Z = \sum_{i\in[k]} \alpha_i Z_i$)
 satisfies \cref{assum:conditions-Z} (this could be restrictive in practice when $k$ is large, but one could for example mitigate this by regularization or adding explicit constraints on the $\tau_{i}, i\in[k]$ ). 
%
%

%% file: Appendix-experiments.tex
\section{Discussions of $\delta$ and Assumption \ref{assum:conditions-Z}}

\subsection{Discussions of \texorpdfstring{$\delta$}{delta}}\label{sec:discussion-of-delta}

Suppose $\kappa=\mu=r=O(1).$ We aim to show that $\|\delta\|\lesssim \sqrt{n}$ and $|\inner{P_{T^{*}}(\delta)}{P_{T^{*}}(Z)}|/\|Z\|_{\F}^2 = \tilde{O}(1/\|Z\|_{\F})$ in the following scenarios. 

{\em Independent $\delta$.} Independent, sub-gaussian $\delta_{ij}$ with $O(1)$ sub-gaussian norm. In this scenario, it is easy to see that $\|\delta\| \lesssim \sqrt{n}$. We also have
\begin{align*}
|\inner{P_{T^{*}}(\delta)}{P_{T^{*}}(Z)}| = |\inner{\delta}{P_{T^{*}}(Z)}| \overset{(i)}{\lesssim} \sqrt{\log(n)} \|P_{T^{*}}(Z)\|_{\F}. 
\end{align*}
Here (i) holds with probability $1-O(n^{-C})$ due to the Chernoff bound (the sum of independent sub-Gaussian random variables is still sub-Gaussian). Then
\begin{align*}
\frac{|\inner{P_{T^{*}}(\delta)}{P_{T^{*}}(Z)}|}{\|Z\|_{\F}^2} \lesssim \frac{\sqrt{\log(n)} \|P_{T^{*}}(Z)\|_{\F}}{\|Z\|_{\F}^2} = \tilde{O}\left(\frac{1}{\|Z\|_{\F}}\right).
\end{align*}

{\em Synthetic control and block $Z$.} $\delta_{ij} = O(1)$, $Z$ consists of an $\ell \times c$ block that is sufficiently sparse: $\sqrt{\ell c}(\ell + c) = O(n)$. 

In this scenario, note that $\ell + c \geq 2\sqrt{\ell c}.$ Therefore, $\sqrt{\ell c} \sqrt{\ell c} \leq \sqrt{\ell c}(\ell + c)  = O(n).$ This implies that $\ell c = O(n).$ Note that $\delta$ is vanished outside $Z$. Hence, $\|\delta\| \leq \|\delta\|_{\F} = O(\sqrt{lc} \|\delta\|_{\max}) = O(\sqrt{n}).$

Next, consider $\|P_{T^{*}}(\delta)\|_{\F}^2$. WOLG, assume $Z$ and $\delta$ are vanished outside the first $\ell$ rows and first $c$ columns. To begin, one can verify that
\begin{align*}
\|P_{T^{*}}(\delta)\|_{\F}^2 \leq \|\delta V^{*}\|_{\F}^2 + \|U^{*} \delta^{\top}\|_{\F}^2.
\end{align*}
For $\delta V^{*}$, we have
\begin{align*}
\|\delta V^{*}\|_{\F}^2 = \sum_{i=1}^{n} \left\|\sum_{j=1}^{n} \delta_{ij} V^{*}_{j,\cdot}\right\|^2 
&\leq \sum_{i=1}^{\ell} \left(\sum_{j=1}^{c} \left\| V^{*}_{j,\cdot}\right\|\right)^2 \|\delta\|_{\max} \\
&\overset{(i)}{=} \frac{\ell c^2 \mu r \|\delta\|_{\max}}{n} \overset{(ii)}{=} O(\ell c^2 / n).
\end{align*}
Here (i) uses the incoherence condition and (ii) uses $r=\mu=O(1)$. By symmetry, one can obtain that $\|U^{*}\delta^{\top}\|_{\F}^2 = O(\ell^2 c / n).$ Then $\|P_{T^{*}}(\delta)\|_{\F}^2 = O((\ell+c)\ell c/n).$ Similarly, we also have $\|P_{T^{*}}(Z)\|_{\F}^2 = O((\ell+c)\ell c/n).$

Hence,
\begin{align*}
\frac{|\inner{P_{T^{*}}(\delta)}{P_{T^{*}}(Z)}|}{\|Z\|_{\F}^2} \leq \frac{\|P_{T^{*}}(\delta)\|_{\F}\|P_{T^{*}}(Z)\|_{\F}}{\|Z\|_{\F}^2} &\overset{(i)}{=} O\left(\frac{(\ell+c)\ell c}{n \sqrt{lc} \|Z\|_{\F}}\right)\\
&\overset{(ii)}{=} O\left(\frac{1}{\|Z\|_{\F}}\right).
\end{align*}
Here, (i) uses that $\|Z\|_{\F} = \sqrt{\ell c}$ and (ii) uses that $\sqrt{\ell c}(\ell+c) = O(n).$

\subsubsection{Panel data regression} \label{appendix:panel-data-regerssion}
Finally, we provide a detailed discussion here for comparing assumptions and guarantees between \cite{moon2018nuclear} and our algorithm. To fit \cite{moon2018nuclear} into our model, we consider $k=1$ and view $Z$ as a regressor and $\delta$ as the idiosyncratic noise. 

The key assumptions for the main result in \cite{moon2018nuclear} (which is Theorem 4 in  \cite{moon2018nuclear}) are (i) $\|\delta\| =O(\sqrt{n})$, (ii) $\inner{P_{T^{*\perp}}(Z)}{\delta} = O(n)$, (iii) $\|P_{T^{*\perp}}(Z)\|_{\F}^2 = \Omega(n^2)$, (iv) $|\tau^{0}-\tau^{*}| = o(1)$, and other low-rank and incoherence conditions. Then Theorem 4 in \cite{moon2018nuclear} provides an iterative estimator $\tau^{t}$ such that $|\tau^{t} - \tau^{*}|=O(1/n)$ for large enough $t$. 

Under these assumptions, our result on $\tau^{d}-\tau^{*}$ \textit{recovers their guarantees} up to $\log(n)$ factors. To see this, note that  their assumption (ii) implies $\inner{P_{T^{*}}(Z)}{P_{T^{*}}(\delta)}=O(n)$ (since $\inner{Z}{\delta} = 0$) and (iii) implies $\|Z\|_{\F}^2 \geq \|P_{T^{*\perp}}(Z)\|_{\F}^2 = \Omega(n^2).$ Then our guarantee $|\tau^{d}-\tau^{*}| = \tilde{O}\left(1/\|Z\|_{\F} + |\inner{P_{T^{*}}(\delta)}{P_{T^{*}}(Z)}|/\|Z\|_{\F}^2\right)$ can be simplified to 
$$|\tau^{d} - \tau^{*}| = \tilde{O}(1/n).$$

\textbf{Limitation of their assumption (iii).} Their assumption (iii) $\|P_{T^{*\perp}}(Z)\|_{\F}^2 = \Omega(n^2)$ is probably the most restrictive one, which rules out many interesting scenarios (such as synthetic control). This assumption is also implicitly made in Theorem 1 and 2 in \cite{moon2018nuclear}, where the bound $|\hat{\tau}-\tau^{*}|=O(1/\sqrt{n})$ is provided (recall $\hat{\tau}$ is the convex estimator). In particular, Theorem 1 \cite{moon2018nuclear} requires that $rank(Z)=O(1)$ and $\|P_{T^{*\perp}}(Z)\|_{*} - \norm{U^{*T}ZV^{*}} = \Omega(n)$, hence implying $\sqrt{rank(Z)}\|P_{T^{*\perp}}(Z)\|_{\F} \geq \|P_{T^{*\perp}}(Z)\|_{*} = \Omega(n)$, i.e., $\|P_{T^{*\perp}}(Z)\|_{\F} = \Omega(n).$ Similarly, $\|P_{T^{*\perp}}(Z)\|_{\F} = \Omega(n)$ is implied by their Assumption 1 and Theorem 2(iii) \cite{moon2018nuclear}. 

Technically, the assumption $\|P_{T^{*\perp}}(Z)\|_{\F}=\Theta(n^2)$ greatly simplifies analysis in \cite{moon2018nuclear} since a global bound on $\hat{\tau}-\tau^{*}$ can be easily obtained. One of the main technical innovations in our paper, building on recent advances in the matrix completion literature, is to conduct a refined `local' analysis without the assumption on the density of $Z$.

Other than technical challenges, another reason that panel data regression literature \cite{bai2009panel,moon2017dynamic,moon2018nuclear} did not loose the $\|P_{T^{*\perp}}(Z)\|_{\F}=\Theta(n^2)$ assumption, is probably the regressor is usually viewed as a dense matrix (e.g., GDP, wages). In the contrast, the scenario we are considering is different, where $Z$ characterizes the treatment pattern and $\|P_{T^{*\perp}}(Z)\|_{\F}=\Theta(n^2)$ becomes a more problematic restriction. 

Note that \cite{moon2018nuclear} also provides the results for $k$ being a constant, where we think it is an interesting open question to study what are the minimal identification conditions for $Z_{m}$ in that scenario.

\subsection{Discussions of Assumption \ref{assum:conditions-Z}}\label{sec:proof-discussion-condition-Z}
In the section, we present the proofs for various treatment patterns that are admissible under \cref{assum:conditions-Z}. 

\subsubsection{Rank grows faster than $r$}
 \begin{claim}
 Let $Z \in \R^{n\times n}$. Suppose there exists a constant $C$ such that
 \begin{align}\label{eq:high-rank}  
     \sum_{i=1}^r \sigma_i(Z)^2 \le \left(1 - \frac{C}{\log n} \right)\frac{\|Z\|_{\F}^2}{\sqrt{r}+2}.
 \end{align}
 Then $(Z, M^{*})$ satisfies Assumption \ref{cond:Z-condition-nonconvex} and \ref{cond:Z-condition-convex} for any $M^* \in \R^{n\times n}$ with $\rank(M^{*}) \leq r$:
 \end{claim}
\begin{proof}
Note that 
$$
\norm{ZV^{*}}_{\F}^2 = \tr(Z^{\top}ZV^{*}V^{*\top}) \overset{(i)}{\leq} \sum_{i=1}^{n} \sigma_i(Z^{\top}Z)\sigma_i(V^{*}V^{*\top}) = \sum_{i=1}^{r} \sigma_i^2(Z)
$$
where (i) is by the Von Neumann's trace inequality. Similarly, $\norm{Z^{\top}U^{*}}_{\F}^2 \leq \sum_{i=1}^{r} \sigma_i^2(Z).$ 

Therefore, \cref{eq:high-rank} implies \cref{cond:Z-condition-nonconvex}:
\begin{align*}
    \norm{ZV^{*}}_{\F}^2 + \norm{Z^{\top}U^{*}}_{\F}^2 
    &\leq 2\sum_{i=1}^{r} \sigma_i^2(Z) \\
    &\overset{(i)}{\leq} (1 - C/\log(n)) \frac{2\norm{Z}_{\F}^2}{\sqrt{r}+2} \leq (1 - C/\log(n)) \norm{Z}_{\F}^2
\end{align*}
where (i) is due to \cref{eq:high-rank}. 

To show \cref{cond:Z-condition-convex}, note that
\begin{align*}
    (\sqrt{r}+2)\sum_{i=1}^{r} \sigma_i^2(Z) 
    &\leq (1 - C/\log(n)) \norm{Z}_{\F}^2\\
    &\overset{(i)}{\leq} (1 - C/\log(n)) \left(\norm{P_{T^{*\perp}}(Z)}_{\F}^2 + \norm{ZV^{*}}_{\F}^2 + \norm{Z^{\top}U^{*}}_{\F}^2\right) \\
    &\leq (1 - C/\log(n)) \norm{P_{T^{*\perp}}(Z)}_{\F}^2 + 2\sum_{i=1}^{r} \sigma_i^2(Z)
\end{align*}
where (i) is due to $\norm{Z}_{\F}^2 \leq \norm{P_{T^{*\perp}}(Z)}_{\F}^2 + \norm{ZV^{*}}_{\F}^2 + \norm{Z^{\top}U^{*}}_{\F}^2$ and (ii) is due to $\norm{ZV^{*}}_{\F}^2 + \norm{Z^{\top}U^{*}}_{\F}^2 \leq 2\sum_{i=1}^{r} \sigma_i^2(Z)$. This then implies 
\begin{align}\label{eq:sqrtr}
    \sqrt{r} \sum_{i=1}^{r} \sigma_i^2(Z) \leq (1 - C/\log(n)) \norm{P_{T^{*\perp}}(Z)}_{\F}^2.
\end{align}

Next, we show \cref{cond:Z-condition-convex} by the following
\begin{align*}
    \left|\inner{Z}{U^{*}V^{*\top}}\right|\norm{P_{T^{*\perp}}(Z)} 
    &\overset{(i)}{\leq} \left(\sum_{i=1}^{n}\sigma_i(Z)\sigma_i(U^{*}V^{*\top})\right) \sigma_1(Z)\\
    &\leq \left(\sum_{i=1}^{r} \sigma_i(Z)\right) \sqrt{\sum_{i=1}^{r} \sigma_i^2(Z)}\\
    &\overset{(ii)}{\leq} \sqrt{r} \sum_{i=1}^{r} \sigma_i^2(Z)\\
    &\overset{(iii)}{\leq} (1 - C/\log(n)) \norm{P_{T^{*\perp}}(Z)}_{\F}^2
\end{align*}
where (i) is due to Von Neumann's trace inequality and $\norm{P_{T^{*\perp}}(Z)} \leq \norm{Z} = \sigma_1(Z)$, (ii) is due to the Cauchy-Schwartz inequality $\sum_{i=1}^{r} \sigma_i(Z) \leq \sqrt{r} \sqrt{\sum_{i=1}^{r} \sigma_i^2(Z)}$, and (iii) is due to \cref{eq:sqrtr}. This finishes the proof. 
\end{proof}

\subsubsection{Maximal number of ones in a row and column}
\begin{claim}
Let $k$ and $\ell$ denote the maximum number of ones in a row and column, respectively, of $Z$. Suppose there exists a constant $C$ such that
\begin{align}\label{eq:sparse}
    k + \ell \le \frac{n}{r^2 \mu} \left(1-\frac{C}{\log n}\right).	
\end{align}
Then $(Z, M^{*})$ satisfies Assumption \ref{cond:Z-condition-nonconvex} and \ref{cond:Z-condition-convex} for any $M^*$ 
satisfying $\rank(M^{*}) \leq r$ and $\max(\|U^{*}\|_{2,\infty}, \|V^{*}\|_{2,\infty}) \leq \sqrt{r\mu/n}$ (incoherence assumption).
\end{claim}
\begin{proof}
Let $k_i$ be the number of ones in the $i$-th row of $Z$, $l_j$ be the number of ones in the $j$-th column of $Z$:
\begin{align*}
    k_i := \sum_{j=1}^{n} Z_{ij}, \quad l_j := \sum_{i=1}^{n} Z_{ij}.
\end{align*}
Let $k := \max_{i} \{k_i\}, l := \max_{j} \{l_j\}$ be the maximal number of ones in a row and column, respectively. Consider \cref{cond:Z-condition-nonconvex}. Note that
\begin{align*}
    \norm{ZV^{*}}_{\F}^2 
    &= \sum_{i=1}^{n} \norm{\sum_{j=1}^{n} Z_{ij} V^{*}_{j,\cdot}}^2\\
    &\leq \sum_{i=1}^{n} k_i^2 \norm{V^{*}}_{2,\infty}^2\\
    &\overset{(i)}{\leq} \left(\sum_{i=1}^{n} k_i^2\right) \frac{\mu r}{n}\\
    &\leq \left(\sum_{i=1}^{n} k_i\right)k \frac{\mu r}{n}\\
    &\overset{(ii)}{=} \norm{Z}_{\F}^2 k\frac{\mu r}{n}
\end{align*}
where (i) is due to the incoherence condition of $M^{*}$ and (ii) is due to $\sum_{i=1}^{n} k_i = \norm{Z}_0 = \norm{Z}_{\F}^2.$ Similarly, one can verify that 
\begin{align}\label{eq:ZtopU}
\norm{Z^{\top}U^{*}}_{\F}^2 \leq \norm{Z}_{\F}^2 l\frac{\mu r}{n}.
\end{align} 
Then
\begin{align}
    \norm{ZV^{*}}_{\F}^2 + \norm{Z^{\top}U^{*}}_{\F}^2 \leq \norm{Z}_{\F}^2 (k+l) \frac{\mu r}{n} \leq \left(1-\frac{C}{\log(n)}\right)\frac{1}{r} \norm{Z}_{\F}^2. \label{eq:condition-one-row-sparse-column-sparse}
\end{align}
This verifies the \cref{cond:Z-condition-nonconvex}. Next, consider \cref{cond:Z-condition-convex}. By \cref{eq:condition-one-row-sparse-column-sparse} and \cref{lem:condition-non-convex-implication} (\cref{eq:implication-assumption3a}), we have
\begin{align*}
    \norm{U^{*\top}ZV^{*}}_{\F}^2 
    &\leq \norm{P_{T^{*\perp}}(Z)}_{\F}^2 - \left(1 - \left(1 - \frac{C}{\log(n)}\right)\frac{1}{r}\right) \norm{Z}_{\F}^2\\
    &\leq \norm{P_{T^{*\perp}}(Z)}_{\F}^2 - \left(\frac{r-1}{r} + \frac{C}{\log(n)r}\right) \norm{Z}_{\F}^2.
\end{align*}
Since $\norm{P_{T^{*\perp}}(Z)}_{\F}^2 \leq \norm{Z}_{\F}^2$, this implies 
\begin{align*}
    \norm{U^{*\top}ZV^{*}}_{\F}^2 
    &\leq \norm{P_{T^{*\perp}}(Z)}_{\F}^2 - \left(\frac{r-1}{r} + \frac{C}{\log(n)r}\right) \norm{P_{T^{*\perp}}(Z)}_{\F}^2\\
    &\leq \frac{1}{r}\left(1-\frac{C}{\log(n)}\right)\norm{P_{T^{*\perp}}(Z)}_{\F}^2.
\end{align*}
This further implies that
\begin{align*}
    \sqrt{r}\norm{U^{*\top}ZV^{*}}_{\F} 
    &\leq \sqrt{\left(1-\frac{C}{\log(n)}\right)}\norm{P_{T^{*\perp}}(Z)}_{\F}\\
    &\leq \left(1-\frac{C}{2\log(n)}\right) \norm{P_{T^{*\perp}}(Z)}_{\F}.
\end{align*}

Then, 
\begin{align*}
    \left|\inner{Z}{U^{*}V^{*\top}}\right| \norm{P_{T^{*\perp}}(Z)} 
    &\leq \left|\tr(U^{*\top}ZV^{*})\right|\norm{P_{T^{*\perp}}(Z)}\\
    &\leq \norm{U^{*\top}ZV^{*}}_{*}\norm{P_{T^{*\perp}}(Z)}\\
    &\leq \sqrt{r} \norm{U^{*\top}ZV^{*}}_{\F} \norm{P_{T^{*\perp}}(Z)}\\
    &\leq \left(1-\frac{C}{2\log(n)}\right) \norm{P_{T^{*\perp}}(Z)}_{\F}\norm{P_{T^{*\perp}}(Z)}\\
    &= \left(1-\frac{C}{2\log(n)}\right)\norm{P_{T^{*\perp}}(Z)}_{\F}^2.
\end{align*}
This verifies \cref{cond:Z-condition-convex}.
\end{proof}

\subsubsection{Single row or column (Synthetic Control)}
\begin{claim}
Consider the case when $Z$ is supported on a single row. Suppose 
\begin{align}\label{eq:ZVstar-synthetic-control}
\|Z V^* \|_{\F}^2 \le \left(1 - C / \log n - \mu r/n \right) \|Z\|^2,
\end{align}
Then $(Z, M^{*})$ satisfies Assumption \ref{cond:Z-condition-nonconvex} and \ref{cond:Z-condition-convex} for any $M^*$ 
satisfying $\rank(M^{*}) \leq r$ and  $\max(\|U^{*}\|_{2,\infty}, \|V^{*}\|_{2,\infty}) \leq \sqrt{r\mu/n}$ (incoherence assumption).
\end{claim}
\begin{proof}
By \cref{eq:ZtopU}, we have $\|Z^{\top}U^{*}\|_{\F}^2 \leq \|Z\|_{\F}^2 l \frac{\mu r}{n}$ where $l$ is the maximum number of ones in a column of $Z$. Since $Z$ is only supported in one row, we have $l \leq 1.$ Then, $\|Z^{\top}U^{*}\|_{\F}^2 \leq \|Z\|_{\F}^2 \frac{\mu r}{n}.$ Together this with \cref{eq:ZVstar-synthetic-control},
\begin{align}\label{eq:ZVZtopU-synthetic-control}
\|Z V^* \|_{\F}^2 + \|Z^{\top}U^{*}\|_{\F}^2 \leq (1 - C/\log(n)) \|Z\|^2
\end{align} 
which verifies \cref{cond:Z-condition-nonconvex}. 

By \cref{eq:ZVZtopU-synthetic-control} and \cref{lem:condition-non-convex-implication} (\cref{eq:implication-assumption3a}), we have
\begin{align*}
 \norm{U^{*\top}ZV^{*}}_{\F}^2 
 &\leq \norm{P_{T^{*\perp}}(Z)}_{\F}^2 - \frac{C}{\log(n)} \|Z\|_{\F}^2\\
 &\leq \left(1-\frac{C}{\log(n)}\right) \norm{P_{T^{*\perp}}(Z)}_{\F}^2.
\end{align*}
Therefore
\begin{align*}
\norm{U^{*\top}ZV^{*}}_{\F} 
    &\leq \sqrt{\left(1-\frac{C}{\log(n)}\right)}\norm{P_{T^{*\perp}}(Z)}_{\F}\\
    &\leq \left(1-\frac{C}{2\log(n)}\right) \norm{P_{T^{*\perp}}(Z)}_{\F}.
\end{align*}
Then,
\begin{align*}
\left|\inner{Z}{U^{*}V^{*\top}}\right| \norm{P_{T^{*\perp}}(Z)} 
    &\leq \left|\tr(U^{*\top}ZV^{*})\right|\norm{P_{T^{*\perp}}(Z)}\\
    &\leq \norm{U^{*\top}ZV^{*}}_{*}\norm{P_{T^{*\perp}}(Z)}\\
    &\overset{(i)}{\leq} \norm{U^{*\top}ZV^{*}}_{\F} \norm{P_{T^{*\perp}}(Z)}\\
    &\leq \left(1-\frac{C}{2\log(n)}\right) \norm{P_{T^{*\perp}}(Z)}_{\F}\norm{P_{T^{*\perp}}(Z)}\\
    &= \left(1-\frac{C}{2\log(n)}\right)\norm{P_{T^{*\perp}}(Z)}_{\F}^2.
\end{align*}
In (i), we use the fact that $Z$ is rank-1, and $\|A\|_{*}=\|A\|_{\F}$ for any rank-1 matrix $A$. This verifies \cref{cond:Z-condition-convex} and finishes the proof. 
\end{proof}

\section{Additional Details of Experiments}\label{appendix:experiments}

In this section, we present more details for \cref{sec:experiments}.

\textbf{Computing Infrastructure.} All experiments are done in a personal laptop equipped with 2.6 GHz 6-Core
Intel Core i7 and 16 GB 2667 MHz DDR4. The operating system is macOS Catalina. For each instance, the
running time for our algorithm is within seconds. 

\textbf{Algorithm Implementations.} Recall that we implemented the following four benchmarks: (i) Matrix-Completion with Nuclear Norm Minimization (MC-NNM): \cite{athey2021matrix} that applies matrix completion by viewing the treated entries as missing. (ii) Robust Synthetic Control (RSC): ``Algorithm 1'' in \cite{amjad2018robust} with linear regression used to recover the counterfactual results. This can be viewed as a robust variant of the well-known synthetic control method \cite{abadie2003economic,abadie2010synthetic}. (iii) Ordinary Least Square (OLS): Selects $a,b \in \R^{n}, \tau \in \R$ to minimize $\|O-a1^{T} - 1b^{T} - \tau Z\|_{\F}^2$, where $1 \in \R^{n}$ is the vector of ones. This can be viewed as regression adapted to the difference-in-differences frameworks \cite{imbens2009recent,xiong2019optimal}.  (iv)  Synthetic Difference-in-Difference (SDID): An algorithm proposed in \cite{arkhangelsky2019synthetic} as a generalization of both synthetic control and difference-in-difference.

This section presents more details for algorithm implementations. All algorithms share the same input $(O, Z)$ and $r$, where $r$ is a pre-defined rank. 

For implementing MC-NNM described in \cite{athey2021matrix}, let $\Omega$ be the set of observed control entries ($\Omega$ is the complement of $Z$). We optimize the following problem
\begin{align*}
\min_{M \in \R^{n\times n}, a\in \R^{n}, b \in \R^{n}} \norm{P_{\Omega}(O - a1^{T} - 1b^{T} - M)}_{\F}^2 + \lambda \norm{M}_{*}
\end{align*}
where $a, b$ are for characterizing the fixed effects. Let $(\hat{M}, \hat{a}, \hat{b})$ be the optimizer for the above problem. The estimator of $\tau^{*}$ is then given by $\tau = \inner{Z}{O-\hat{M}-\hat{a}1^{T}-1\hat{b}^{T}} / \norm{Z}_{\F}^2.$ To tune the hyper-parameter $\lambda$, we choose a large enough $\lambda$ and then gradually decrease $\lambda$ until the rank of $\hat{M}$ achieves a pre-defined rank $r$.

For implementing RSC described in \cite{amjad2018robust} for the block and stagger patterns, let $S_1$ be the set of treated units and $S_2$ be set of the control units. For each $i \in S_1$, we use the ``Algorithm 1'' in \cite{amjad2018robust} with the ``linear regression ($\eta = 0$)'' to estimate the counterfactuals of the $i$-th row based on the control units $S_2$. Let $\hat{M}$ be their final estimate of the counterfactuals (combining all estimations across different rows). Let $\tau = \inner{Z}{O-\hat{M}} / \norm{Z}_{\F}^2$ be the estimator of $\tau^{*}.$



For OLS, we obtain the solution by a single linear regression.  For SDID, we implement ``Algorithm 1'' in \cite{arkhangelsky2019synthetic}. The SDID is designed for block patterns, to extend to stagger patterns, we adopt the suggestions in ``Footnote 1'' of \cite{arkhangelsky2019synthetic}. 

For De-biased Convex (our algorithm), we implement an alternating optimization to solve the convex problem. Let $\tau = \tau^{d}$ be the estimator for $\tau^{*}.$ To tune the hyper-parameter $\lambda$, similar to the MC-NNM, we choose a large enough $\lambda$ and then gradually decrease $\lambda$ until the rank of $\hat{M}$ achieves a pre-defined rank $r$.


\textbf{The Choice of $\tau^{*}$.} 
For all algorithms we implemented, given $(M^{*}, Z, E)$, the error $\tau-\tau^{*}$ is in fact invariant from the change of $\tau^{*}. $ For MC-NNM, RSC, this is obvious since they do not use the information from treatment entries. In particular, one can check $\tau^{*}-\tau = \inner{Z}{M^{*}+E-\hat{M}}/\|Z\|_{\F}^2$ where $\hat{M}$ is independent from the choice of $\tau^{*}$ in their algorithms. 

For De-biased Convex, one can verify that if $(\hat{M}, \hat{\tau})$ satisfies the first-order conditions of the instance $O=M^{*}+Z*\tau^{*}+E$, then $(\hat{M}, \hat{\tau}+\Delta)$ satisfies the first order conditions of the instance $O=M^{*}+Z*(\tau^{*}+\Delta)+E.$ This linear response together with the linear debias procedure implies that $\tau^{*}-\tau$ is independent from the choice of $\tau^{*}.$ Similarly, OLS and SDID also shares the same property. See \cref{fig:invariance-tau-test} for the invariance of $\tau - \tau^{*}$ for a particular instance.

 \begin{figure}[h]
     \centering
     \includegraphics[scale=0.5]{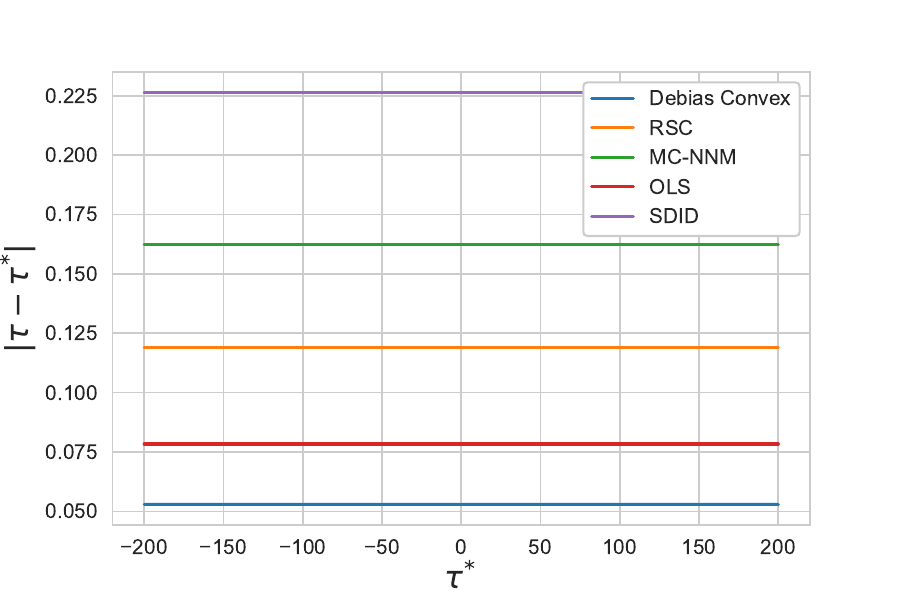}
     \caption{An example of showing that $|\tau-\tau^{*}|$ is invariant for different choice of $\tau^{*}$ for various algorithms.}
      \label{fig:invariance-tau-test}
 \end{figure}
 

\textbf{Real Data.} The data is collected from a Kaggle competition \citep{Rossmannsales2021}. There are two types of promotion information: Promo $Z_1$ and Promo2 $Z_2$, which are both binary matrices. Since $Z_1$ is store-independent (in a specific day,  stores either all have promotions or all have no promotions), we discard $Z_1$ (viewing $Z_1$ being absorbed by latent features of days) and focus on estimating treatment effects of $Z_2$. We also filter out stores that have no promotions for any days.

\textbf{Asymptotic Normality.} We consider a set of synthetic instances. For each instance, we used a typical procedure for generating low-rank non-negative matrices $M^{*} \in \R^{n_1\times n_2}$ (e.g. \cite{cemgil2008bayesian,farias2021near}). We select the rank $r$ and mean-value $\bar{M^{*}}$. Given $(r, \bar{M^{*}})$, we generate $U \in \R^{n_1\times r}$ and $V \in \R^{n_2 \times r}$ with entries drawn i.i.d from $\mathrm{Gamma}(2, 1)$, and set $M^{*} = kUV^{T}$ with $k$ chosen that the mean-value is indeed $\bar{M}^{*}$. For $E$ and $\delta$, its entries were drawn i.i.d from a Gaussian $\mathcal{N}(0,\sigma)$ and $\mathcal{N}(0, \sigma_{\delta})$ respectively. The observation is given by $O = M^{*} + E + \tau^{*}Z + \delta \circ Z.$ 

\cref{fig:stagger-distribution} confirms the asymptotic normality of $\tau^{d}$ using an ensemble of 10,000 instances with $r = 10, \bar{M^{*}}=10, \tau^{*}=\sigma=\sigma_{\delta}=1$, and randomly generated stagger patterns $Z$ where $m_1 \sim \mathrm{Uni}[1, n_1), m_2 \sim \mathrm{Uni}[n_2/2, n_2)$.\footnote{Experiments on other tested patterns (e.g., block patterns) also showed similar performances.} For each instance, we computed $(\tau^{d}-\tau^{*})/(\sigma/\norm{P_{\hat{T}^{\perp}}(Z)}_{\F})$, which is predicted by \cref{thm:main-theorem} to be distributed according to $\mathcal{N}(0,1)$. \cref{fig:stagger-distribution} (Left) shows the empirical distribution of $\delta_{\tau}$ over instances with $n_1=n_2=100$, overlaid with the true density function of $\mathcal{N}(0,1)$.
We also computed the {\em coverage rate}: the proportion of instances where $\tau^{*}$ is within the predicted 95\% confidence interval of $\tau^{d}$, calculated based on the estimated variance. \cref{fig:stagger-distribution} (Right) shows the desired coverage rates for different sizes of matrices. Both results suggest the effectiveness of our estimator for performing inference on $\tau^*$. 

\textbf{Row-specific Treatment Effects ($k>1$).} We also test the performances of various algorithms for multiple treatment effects ($k>1$). In particular, we reuse the setting for semi-synthetic (Tobacco) in \cref{sec:experiments}, where $(O, Z)$ is generated with the same distribution. Instead of estimating the overall average treatment effects, we now consider to estimate the unit-specific treatment effects $\vec{\tau}^{*}$, where $\vec{\tau}^{*}_{i} = \sum_{j} Z_{ij}\mT_{ij} / \sum_{j} Z_{ij} = \tau^{*} + \delta_{i}.$ 

To estimate $\vec{\tau}^{*}$, we use De-biased Convex algorithm for multiple treatment effects specified in \cref{eq:estimator} (each $Z_{l}$ corresponds to one row). We also extend RSC, MC-NNM, and OLS in this case, where the estimator $\vec{\tau}$ can be obtained directly based on their estimation of counterfactuals $\hat{M}$: $\vec{\tau}_i = \sum_{j} Z_{ij}(O_{ij} - \hat{M}_{ij}) / \sum_{j} Z_{ij}$.\footnote{It is not clear how SDID can be extended for multiple treatment effects here.} \cref{tab:synthetic-average-delta-multiple-tau} reports the average $\|\vec{\tau}-\vec{\tau}^{*}\| / \|\vec{\tau}^{*}\|$ over 1000 instances. The results show the compelling performance of De-biased Convex algorithm compared to the state-of-the-arts. 

\begin{table}[h!]
\centering
\caption{Comparison of our algorithm (De-biased Convex) to benchmarks on the semi-synthetic (Tobacco) dataset. Average normalized error $\|\vec{\tau}-\vec{\tau}^{*}\| / \|\vec{\tau}^{*}\|$ is reported.}
\begin{tabular}{@{}lcccc@{}}
\toprule
Pattern & De-biased Convex & MC-NNM & RSC & OLS \\
\midrule
Block  & 0.05 ($\pm 0.06$) &0.08 ($\pm 0.13$)  & 0.10 ($\pm 0.12$) & 0.10 ($\pm 0.16$)\\
Stagger  & 0.03 ($\pm 0.02$)  &0.05 ($\pm 0.03$)  & 0.06 ($\pm 0.07$) & 0.07 ($\pm 0.03$)\\ 
\bottomrule
\end{tabular}
\label{tab:synthetic-average-delta-multiple-tau}
\end{table}

%% file: Appendix.tex
\section{Proof of \cref{thm:error-rate-theorem}\label{sec:proof-of-main-theorem-part2}: Proof of \cref{lem:X-Xstar-Y-Ystar}}\label{sec:showing-XXstar-YYstar}
It remains to prove \cref{lem:X-Xstar-Y-Ystar}. Now a natural idea is to explicitly construct an optimizer $\hat{M}$ (or its dual certificate) of the convex program $g(\cdot,\cdot)$ and then analyze its properties. However, this is challenging (if not infeasible) in this setting. Instead we introduce the following non-convex {\em proxy} problem which plays a critical role in the analysis (note that our algorithm itself does {\em not} involve solving this proxy problem): 
\begin{align}
    \minimize_{X \in \R^{n\times r}, Y \in \R^{n\times r}, \tau \in \R} f(X, Y; \tau) :=  \frac{1}{2}\norm{O - XY^{\top} - \tau Z}_{F}^{2} + \frac{\lambda}{2} \norm{X}_{\F}^2 + \frac{\lambda}{2}\norm{Y}_{\F}^2 \label{eq:non-convex-program}
\end{align}
\cref{eq:non-convex-program} can be viewed as minimizing the Euclidean error with a hard constraint $\rank(XY^{\top}) = r$ and additional regularization terms that serve to maintain an important `balance' between $X$ and $Y$. The hope is that $X \approx X^{*}, Y \approx Y^{*}$ where $M^{*}=X^{*}Y^{*\top}$. Recall that we take $\lambda := \Theta(\sigma \sqrt{nr} \log^{1.5}(n))$ throughout the proof of \cref{thm:error-rate-theorem}.

The proof of \cref{lem:X-Xstar-Y-Ystar} involves two steps, which we will take in \cref{sec:connecting-convex-non-convex,sec:properties-non-convex-estimators}.
\begin{enumerate}
\item {\em Connecting the Convex and Non-convex Optimizers}: The first step is to show that if a critical point $(X, Y)$ of $f$ is in $\Bscr$ (see \cref{eq:X-Xstar-Y-Ystar}), then $\hat{M} = XY^{\top}$ is exactly the unique optimal solution of $g$ (\cref{lem:connection-deterministic}). This converts the problem of directly analyzing $\hat{M}$ to analyzing the critical points of $f$. 
\item {\em Properties of the Non-convex Optimizers}: The second step is then to show the existence of such a critical point of $f$ in $\Bscr$ (\cref{thm:non-convex-theorem-Frobenious-deterministic}). 
\end{enumerate}

As an aside, this larger framework is inspired by recent developments on bridging convex and non-convex problems \cite{chen2019noisy,chen2020bridging} for matrix completion and Robust-PCA. In contrast to those works, which assume a completely random pattern (i.e. independent across entries) of missing or corrupted entries, our deterministic treatment pattern $Z$ necessitates a substantially more-careful analysis.

%
%
\subsection{Step 1: Connecting the Convex and Non-convex Optimizers}\label{sec:connecting-convex-non-convex} 
The lemma below connects the optimizer of the convex program $g$ and a  critical point of the non-convex function $f.$ The result is completely deterministic.
\begin{lemma}\label{lem:connection-deterministic}


Suppose $(X, Y) \in \Bscr$ (as defined in \cref{eq:X-Xstar-Y-Ystar}), and that $\nabla f(X, Y; \tau)=0$ with $\tau$ taken to be $\tau = \inner{Z}{O-XY^{\top}}/\|Z\|_{\F}^2$. Then $(XY^{\top}, \tau)$ is the unique optimal solution to the convex program $g(\cdot,\cdot)$ in \cref{eq:convex-program}, for sufficiently large $n$.
\end{lemma}
The proof of \cref{lem:connection-deterministic} is in two parts, first establishing that $(XY^{\top}, \tau)$ is an optimal solution by direct verification of the first-order conditions, and then showing uniqueness via a ``local strict convexity'' property (omitted for simplicity).
\begin{proof}[Proof of \cref{lem:connection-deterministic}]
To show that $(XY^{\top}, \tau)$ is an optimal point of convex program $g$, it suffices to verify the following first-order conditions:
\begin{subequations}\label{eq:convex-conditions-again}
\begin{align}\label{eq:convex-condition-tau-again}
\inner{Z}{O - XY^{\top}  - \tau Z} &= 0, \\
       O - XY^{\top} - \tau Z &= \lambda (UV^{\top} + W), \label{eq:convex-condition-M-again}\\
     P_{T^\perp}(W) &= W, \label{eq:convex-condition-W2-again}\\
     \norm{W} &\leq 1, \label{eq:convex-condition-W1-again} 
\end{align}
\end{subequations}
where $XY^{\top} = U\Sigma V^{\top}$ is the SVD of $XY^{\top}$ and $T$ is the tangent space of $XY^{\top}.$ We verify these in sequence. First, note that \cref{eq:convex-condition-tau-again} is immediately satisfied given the definition of $\tau$. \cref{eq:convex-condition-M-again} will similarly hold by construction: we select $W$ to be $W:= \frac{1}{\lambda} (O - XY^{\top} - \tau Z - \lambda UV^{\top})$. The remaining two conditions can be verified as follows:

{\bf Verifying  \cref{eq:convex-condition-W2-again}:}
It suffices to show that $WV = 0$ and $U^{\top}W = 0$. By the definition of $W$, this is equivalent to showing that
\begin{subequations}\label{eq:O-XY-UV}
\begin{align}
(O-XY^{\top}-\tau Z)V &= \lambda U \label{eq:O-XY-V},\\
(O-XY^{\top}-\tau Z)^{\top} U &= \lambda V \label{eq:O-XY-U}.
\end{align}
\end{subequations}
To show that \cref{eq:O-XY-UV} holds, we use the fact that $\nabla f(X, Y; \tau)=0$. Note that $\nabla f(X, Y; \tau)=0$ is equivalent to
\begin{subequations}\label{eq:O-XY-XY}
\begin{align}
(O-XY^{\top}-\tau Z)Y &= \lambda X \label{eq:O-XY-Y},\\
(O-XY^{\top}-\tau Z)^{\top} X &= \lambda Y \label{eq:O-XY-X}.
\end{align}
\end{subequations}
Then it boils down to replacing $X$ and $Y$ in \cref{eq:O-XY-XY} by $U$ and $V$. To begin, note that left-multiplying $X^{\top}$ (and $Y^{\top}$) on both sides of \cref{eq:O-XY-Y} (and \cref{eq:O-XY-X}) yields 
\begin{align*}
&\lambda X^{\top} X = X^{\top}(O-XY^{\top}-\tau Z)Y = (\lambda Y^{\top} Y)^{\top} = \lambda Y^{\top}Y.
\end{align*}
Hence $X^{\top}X = Y^{\top}Y$ at any critical point of $f$. This `balance' of $X$ and $Y$ enables connections to $U$ and $V$. In fact, we have the following claim.\footnote{For completeness, we present a simplified proof of this claim in \cref{sec:technical-lemmas}.}
\begin{restatable}[Lemma 20 of \cite{chen2019noisy}]{claim}{XYUVbalance}
\label{claim:XY-UVQ}
Suppose $X, Y \in \R^{n\times r}$ with $\rank(X)=\rank(Y)=r.$ If $X^{\top}X = Y^{\top}Y$, then there exists a rotation matrix $Q \in \R^{r\times r}$ such that
\begin{align*}
X &= U\Sigma^{1/2}Q\\
Y &= V\Sigma^{1/2}Q
\end{align*}
where $XY^{\top}=U\Sigma V^{\top}$ is the SVD of $XY^{\top}$ with $U, V \in \R^{n\times r}, \Sigma \in \R^{r\times r}.$
\end{restatable}

From \cref{claim:XY-UVQ}, we can write $X=U\Sigma^{1/2}Q, Y=V\Sigma^{1/2}Q$, where $Q$ is a rotation matrix. Combining this with \cref{eq:O-XY-XY} and right-multiplying $Q^{-1}\Sigma^{-1/2}$ on both sides of both equations, we obtain  \cref{eq:O-XY-UV}, and hence \cref{eq:convex-condition-W2-again} holds. 

{\bf Verifying  \cref{eq:convex-condition-W1-again}:}
By our selection of $W$, and $W = P_{T^{\perp}}(W)$, we have
\begin{align}
\|W\| &= \|P_{T^{\perp}}(W)\| \nonumber\\
&=\lambda^{-1} \|P_{T^{\perp}} (O - XY^{\top} - \tau Z - \lambda UV^{\top})\| \nonumber\\ 
&\overset{(i)}{=} \lambda^{-1} \|P_{T^{\perp}}(M^{*}) + P_{T^{\perp}}(\hat{E}) + (\tau^{*}-\tau)P_{T^{\perp}}(Z)\| \nonumber\\
&\leq \underbrace{\lambda^{-1}\|P_{T^{\perp}}(M^{*})\|}_{A_1} + \underbrace{\lambda^{-1}\|P_{T^{\perp}}(\hat{E})\|}_{A_2} + \underbrace{\lambda^{-1}|\tau^{*}-\tau|\|P_{T^{\perp}}(Z)\|}_{A_3} \label{eq:bound-W-operator-norm}
\end{align}
where (i) holds because $P_{T^{\perp}}(XY^{\top}) = 0, P_{T^{\perp}}(UV^{\top}) = 0$, and $O = M^{*} + \hat{E} + \tau^{*}Z.$ We will bound $A_1, A_2, A_3$ separately. 

{\it Bounding $A_1$}: By $\lambda = \Theta(\sigma \sqrt{nr} \log^{1.5}(n))$ and $\|P_{T^{\perp}}(M^{*})\|_{\F} \lesssim \frac{\sigma^2 n \log^{5}(n)}{\sigma_{\min}^2} \|X^{*}\|_{\F}^2 $ as shown in \cref{eq:PhatTMstar}, we have that
\begin{align}
A_1:=\lambda^{-1}\|P_{T^{\perp}}(M^{*})\| 
&\leq \lambda^{-1}\|P_{T^{\perp}}(M^{*})\|_{\F} \nonumber\\
&\overset{(i)}{\lesssim} \frac{\sigma^2 \kappa r n \log^{5}(n)}{\sigma_{\min}} \frac{1}{\sigma \sqrt{nr} \log^{1.5}(n)} \nonumber\\
&= \frac{\sigma \kappa r^{0.5} \sqrt{n} \log^{3.5}(n)}{\sigma_{\min}} \overset{(ii)}{=} O(1/\log^{1.5}(n)) \label{eq:PTperpMstar}
\end{align}
where (i) uses $\|X^{*}\|_{\F} \leq \sqrt{r \sigma_{\max}}$ and (ii) uses $\frac{\sigma}{\sigma_{\min}} \sqrt{n} \lesssim \frac{1}{\kappa^2 r^2 \log^{5}(n)}.$ 

{\it Bounding $A_2$}: Using $\|\hat{E}\| \lesssim \sigma\sqrt{n}$ (\cref{eq:op-E}), we have
\begin{align}
A_2 := \lambda^{-1}\|P_{T^{\perp}}(\hat{E})\| \leq \lambda^{-1} \|\hat{E}\| = O(1/\log^{1.5}(n)). \label{eq:PTperpE}
\end{align}

{\it Bounding $A_3$}: We will use the following fact:
\begin{claim}
If $(XY^{\top}, \tau)$ satisfy \cref{eq:convex-condition-tau-again,eq:convex-condition-M-again,eq:convex-condition-W2-again}, then
\begin{align*}
(\tau - \tau^{*})\norm{P_{T^{\perp}}(Z)}_{\F}^2 =  \lambda \inner{Z}{UV^{\top}} + \inner{P_{T^{\perp}}(Z)}{\hat{E}} + \inner{Z}{P_{T^{\perp}}(M^{*})}.
\end{align*}
\end{claim}
\begin{proof}
This decomposition holds due to the same proof in \cref{lem:tau-decomposition}, which in fact only uses the conditions \cref{eq:convex-condition-tau,eq:convex-condition-M,eq:convex-condition-W2} among the first-order conditions. 
\end{proof}

This leads to
\begin{align*}
\lambda^{-1} |\tau - \tau^{*}| \|P_{T^{\perp}}(Z)\| 
&\leq \underbrace{\frac{|\inner{Z}{UV^{\top}}|\|P_{T^{\perp}}(Z)\|}{\|P_{T^{\perp}}(Z)\|_{\F}^2}}_{B_1} + \underbrace{\frac{\left|\inner{P_{T^{\perp}}(Z)}{\hat{E}}\right|\|P_{T^{\perp}}(Z)\|}{\lambda \|P_{T^{\perp}}(Z)\|_{\F}^2}}_{B_2}\\
&\quad +\underbrace{\frac{|\inner{Z}{P_{T^{\perp}}(M^{*})}|\|P_{T^{\perp}}(Z)\|}{\lambda \|P_{T^{\perp}}(Z)\|_{\F}^2}}_{B_3}.
\end{align*}
We will bound $B_1, B_2, B_3$ separately. 

{\it Bounding $B_1$}: By \cref{eq:assumption3b-mimic}, we have $\frac{|\inner{Z}{UV^{\top}}|\|P_{T^{\perp}}(Z)\|}{\|P_{T^{\perp}}(Z)\|_{\F}^2} \leq 1 -\frac{C_{r_2}}{2\log(n)}$, and hence\footnote{This is in fact why \cref{cond:Z-condition-convex} is needed in the first place: a guarantee for $B_1 < 1$ is required for establishing the optimality. To what extent this assumption can be relaxed is an interesting future direction to explore.}
\begin{align}
B_1 \leq 1-\frac{C_{r_2}}{2\log(n)}. \label{eq:B1-bound}
\end{align}

{\it Bounding $B_2$}: 
\begin{align*}
B_2 
&= \frac{\left|\inner{P_{T^{\perp}}(Z)}{\hat{E}}\right|\|P_{T^{\perp}}(Z)\|}{\lambda \|P_{T^{\perp}}(Z)\|_{\F}^2} \\
&\overset{(i)}{\leq} \frac{\left|\inner{P_{T^{\perp}}(Z)}{\hat{E}}\right|}{\lambda \|P_{T^{\perp}}(Z)\|_{\F}}\\
&\overset{(ii)}{\lesssim}  \frac{\sigma r^{1.5}\kappa^2 \sqrt{n} \log^{3}(n)}{\lambda \sigma_{\min}} \|\hat{E}\| + \frac{\log^{0.5}(n)}{\lambda  \|Z\|_{\F}} \left|\inner{P_{T^{*\perp}}(Z)}{\hat{E}}\right|,
\end{align*}
where (i) is due to $\|P_{T^{\perp}}(Z)\| \leq \|P_{T^{\perp}}(Z)\|_{\F}$, and (ii) follows nearly identically the analysis for \cref{eq:bound-A1-main-theorem}. Next, we use the fact that $\left|\inner{P_{T^{*\perp}}(Z)}{\hat{E}}\right| \lesssim \sigma \sqrt{nr} \|Z\|_{\F}$ due to the following
\begin{align*}
\left|\inner{P_{T^{*\perp}}(Z)}{\hat{E}}\right| 
&= \left|\inner{Z-P_{T^{*}}(Z)}{\hat{E}}\right|\\
&\leq \left|\inner{Z}{\hat{E}}\right| +  \left|\inner{P_{T^{*}}(Z)}{\hat{E}}\right|\\
&\overset{(i)}{\leq} \sigma \sqrt{n} \|Z\|_{\F} + \|\hat{E}\| \|P_{T^{*}}(Z)\|_{*}\\
&\overset{(ii)}{\leq} \sigma \sqrt{n} \|Z\|_{\F} + \|\hat{E}\| \sqrt{2r} \|P_{T^{*}}(Z)\|_{\F}\\
&\lesssim \sigma\sqrt{n} \sqrt{r} \|Z\|_{\F} 
\end{align*}
Here in (i) we use that $|\inner{Z}{\hat{E}}| \lesssim \sqrt{n} \|Z\|_{\F}$ (\cref{eq:inner-Z-E}) and the trace inequality, and in (ii) we use that $\|A\|_{*} \leq \rank(A) \|A\|_{\F}$ and $\rank(P_{T^{*}}(Z)) \leq 2r$ by the definition of $T^{*}.$ Plugging in
$\|\hat{E}\| \lesssim \sigma \sqrt{n}$ (\cref{eq:op-E}) and $\left|\inner{P_{T^{*\perp}}(Z)}{\hat{E}}\right| \lesssim \sigma \sqrt{nr} \|Z\|_{\F}$ back to the bound of $B_2$, we obtain
\begin{align}
B_2 \leq C \frac{\sigma^2 r^{1.5}\kappa^2 n \log^{3}(n)}{\lambda \sigma_{\min}} + C \frac{\sigma\log^{0.5}(n)\sqrt{nr}}{\lambda} \overset{(i)}{\leq} \frac{C_{r_1}}{4\log(n)}, \label{eq:B2-bound}
\end{align}
where (i) is by using $\frac{\sigma \sqrt{n}}{\sigma_{\min}} \lesssim \frac{1}{\kappa^2 r^2 \log^{5}(n)}$ and taking $\lambda = C_{\lambda} \sigma \sqrt{nr} \log^{1.5}(n)$ for large enough $C_{\lambda}$.

{\it Bounding $B_3$}: 
\begin{align}
B_3=\frac{|\inner{Z}{P_{T^{\perp}}(M^{*})}|\|P_{T^{\perp}}(Z)\|}{\lambda \|P_{T^{\perp}}(Z)\|_{\F}^2} 
&\leq \frac{|\inner{P_{T^{\perp}}(Z)}{P_{T^{\perp}}(M^{*})}|}{\lambda \|P_{T^{\perp}}(Z)\|_{\F}} \nonumber\\
&\overset{(i)}{\leq} \frac{\|P_{T^{\perp}}(Z)\|_{\F} \|P_{T^{\perp}}(M^{*})\|_{\F}}{\lambda \|P_{T^{\perp}}(Z)\|_{\F}} \nonumber\\
&\leq \lambda^{-1} \|P_{T^{\perp}}(M^{*})\|_{\F} \overset{(ii)}{\lesssim} \frac{1}{\log^{1.5}(n)}, \label{eq:B3-bound}
\end{align}
where (i) is by Cauchy-Schwartz, and (ii) is by \cref{eq:PTperpMstar}. 

Combining \cref{eq:bound-W-operator-norm,eq:PTperpMstar,eq:PTperpE,eq:B1-bound,eq:B2-bound,eq:B3-bound}, we have, for large enough $n$, 
\begin{align*}
\|W\| 
&\leq A_1 + A_2 + B_1 + B_2 + B_3\\
&\leq O\left(\frac{1}{\log^{1.5}(n)}\right) + 1 - \frac{C_{r_1}}{2\log(n)} + \frac{C_{r_1}}{4\log(n)} + O\left(\frac{1}{\log^{1.5}(n)}\right)\\
&< 1.
\end{align*}
This establishes \cref{eq:convex-condition-W1-again}. 
%
\end{proof}

\newcommand{\sSNR}{\frac{\sigma \sqrt{n}}{\sigma_{\min}} \lesssim \frac{1}{r^2\kappa^2 \log^{5}(n)}}

\subsection{Step 2: Properties of the Non-convex Optimizers}\label{sec:properties-non-convex-estimators}
The final part of the proof of \cref{thm:error-rate-theorem} is to show the existence of a point $(X, Y) \in \Bscr$  (see \cref{eq:X-Xstar-Y-Ystar}) satisfying $\nabla f(X, Y; \tau) = 0.$
We achieve this by analyzing an alternating gradient descent algorithm (Algorithm \ref{alg:GD}) on $f$, initialized at $(X^{*}, Y^{*}).$ Note that this gradient descent algorithm exists purely for the purpose of analyzing the properties of the convex optimizer of $g$, so this idealized choice of initialization is not illegitimate. Algorithm \ref{alg:GD} contains a learning rate $\eta$, which we take to be $\eta = \Theta\left(({\kappa^3 n^{20} \sigma_{\max}})^{-1}\right).$
\begin{algorithm}
\caption{Gradient Descent of Non-convex Optimization} \label{alg:GD}
{\bf Input:} the observation $O$ and $Z$
\begin{algorithmic}[1]

\State{\textbf{{Initialization}}: $X^{0} = X^{*}, Y^{0}=Y^{*}, \tau^{0} = \frac{\inner{Z}{O-X^{0}Y^{0\top}}}{\norm{Z}_{\F}^2}$ where  $X^{*} = U^{*}(\Sigma^{*})^{1/2}, Y^{*} = V^{*}(\Sigma^{*})^{1/2}$.}

\State{\textbf{{Gradient updates}}: \textbf{for }$t=0,1,\ldots $
\textbf{do}
 \begin{subequations}\label{subeq:gradient_update_ncvx-loo}
\begin{align}
X^{t+1}= & X^{t}-\eta[(X^{t}Y^{t\top} + \tau^{t} Z - O)Y^{t}+\lambda X^{t}]; \nonumber \\
Y^{t+1}= & Y^{t}-\eta[(X^{t}Y^{t\top} + \tau^{t} Z - O)^{\top}X^{t}+\lambda Y^{t}]; \nonumber \\
\tau^{t+1} = & \frac{\inner{Z}{O - X^{t+1}Y^{(t+1)\top}}}{\norm{Z}_{\F}^2} \nonumber
\end{align}
where $\eta$ determines the learning rate. 
\end{subequations}
}
\end{algorithmic}
\end{algorithm}


Our measure of the accuracy of $(X^t,Y^t)$ should be rotation-invariant (since $X^{t}H^{t}H^{t\top}Y^{t\top} = X^{t}Y^{t\top}$ for any rotation/orthogonal matrix $H \in \O^{r\times r}$). Thus, the coming result is stated for the optimal rotation for aligning $(X^{t}, Y^{t})$ to $(X^{*}, Y^{*})$: 
\begin{align}
\bm{H}^{t} :=\arg\min_{\bm{R}\in\mathcal{O}^{r\times r}} \norm{X^{t}R-X^{*}} _{\mathrm{F}}^2 + \norm{Y^{t}R-Y^{*}} _{\mathrm{F}}^2. \label{eq:rotation-H}
\end{align}

We aim to establish that if $(X^{t}H^{t}, Y^{t}H^{t}) \in \Bscr$, then $(X^{t+1}H^{t+1}, Y^{t+1}H^{t+1}) \in \Bscr$. Then by taking $t\rightarrow \infty$, the iteration sequence converges to a point $(X, Y) \in \Bscr$ with vanishing gradient (by the nature of gradient descent).

The following lemma provides guarantees for the iteration sequence by induction. The proof, which requires only (lengthy) algebra, is deferred to \cref{sec:proof-non-convex-determinisitc}.

%
\begin{lemma}\label{lem:induction-Frobenious}
Suppose the following induction hypothesis (\ref{eq:induction-Frobenious-main}) holds for $q=t$ with (pre-determined) constant $C_{B}$. \begin{subequations}\label{eq:induction-Frobenious-main}
\begin{align}    
	(X^qH^q,Y^qH^q) &\in \Bscr \label{eq:FtHt-Fstar-Fnorm-Frobenious-main}\\
       \|{X^{q \top} X^{q} - Y^{q \top}Y^{q}}\|_{\F} &\leq C_{B} \frac{\sigma}{\kappa n^{15}} \label{eq:XX-YY-Frobenious-main}.
 \end{align}  
 \end{subequations}
 Then (\ref{eq:induction-Frobenious-main}) holds for $q=t+1$, and the following inequality holds:
\begin{align}
         f(X^{q+1}, Y^{q+1}; \tau^{q+1}) &\leq f(X^{q}, Y^{q}; \tau^{q}) - \frac{\eta}{2} \norm{\nabla f(X^{q}, Y^{q}; \tau^{q})}_{\F}^2. \label{eq:f-decrease-Frobenious-main}
\end{align}
\end{lemma}

By \cref{lem:induction-Frobenious}, we can establish the following lemma. Combined with \cref{lem:connection-deterministic}, this completes the proof of \cref{lem:X-Xstar-Y-Ystar} (and thus \cref{thm:error-rate-theorem}).
\begin{lemma}\label{thm:non-convex-theorem-Frobenious-deterministic}
There exists  $(X, Y) \in \Bscr$ (as defined in \cref{eq:X-Xstar-Y-Ystar}) and $\nabla f(X, Y; \tau)=0$ with $\tau = \inner{Z}{O-XY^{\top}}/\|Z\|_{\F}^2$.
\end{lemma}
\begin{proof}[Proof of \cref{thm:non-convex-theorem-Frobenious-deterministic}]
Consider minimization of the objective function $\|\nabla f(X, Y; \tau)\|_{\F}$ with  $\tau = \inner{Z}{O-XY^{\top}}/\|Z\|_{\F}^2$ on (compact set) $\mathcal{B}$. Suppose the minimum is obtained at $(X', Y') \in \mathcal{B}$ with $\|\nabla f(X', Y'; \tau)\|_{\F} = \beta$. 
It suffices to prove that $\beta=0$. 

Suppose $\beta > 0.$ To arrive at a contradiction, consider the sequence $\{(X^{t}, Y^{t}, \tau^{t})\}$ defined in Algorithm \ref{alg:GD}.
Note that (\ref{eq:induction-Frobenious-main}) trivially holds for $t=0$ since $X^{0}=X^{*}, Y^{0}=Y^{*}.$ Thus \cref{eq:induction-Frobenious-main,eq:f-decrease-Frobenious-main} hold for all $t$ by \cref{lem:induction-Frobenious}.

Now by \cref{eq:FtHt-Fstar-Fnorm-Frobenious-main}, we have $(X^{t}H^{t}, Y^{t}H^{t}) \in \mathcal{B}$ for $t \in \mathbb{N}.$ Next, by \cref{eq:f-decrease-Frobenious-main} and the Monotone Convergence theorem (note that $f$ is non-negative), we have
\begin{align*}
\lim_{t\rightarrow \infty} \|\nabla f(X^{t}, Y^{t}; \tau^{t}) \|_{\F} = 0.
\end{align*}
Also note that the Frobenius norm of the gradient of $f$ is rotation-invariant, i.e., for any rotation $H \in \O^{r\times r}$, 
\begin{align*}
\|\nabla f(X^{t}H, Y^{t}H; \tau^{t})\|_{\F} = \|\nabla f(X^{t}, Y^{t}; \tau^{t}) H\|_{\F} = \|\nabla f(X^{t}, Y^{t}; \tau^{t})\|_{\F}.
\end{align*}
Therefore, 
\begin{align*}
\lim_{t\rightarrow \infty} \|\nabla f(X^{t}H^{t}, Y^{t}H^{t}; \tau^{t}) \|_{\F} = 0.
\end{align*}
A contradiction arises since $(X^{t}H^{t}, Y^{t}H^{t}) \in \mathcal{B}$ and $\|\nabla f(X^{t}H^{t}, Y^{t}H^{t}; \tau^{t}) \|_{\F} \geq \beta > 0$. 
\end{proof}

\section{Proof of Lemma \ref{lem:induction-Frobenious}} \label{sec:proof-non-convex-determinisitc}
Defining the quantities $F^t,F^*,\nabla f(X, Y; \tau) \in \R^{2n\times r}$ greatly simplifies the presentation:  
\begin{equation}
F^{t}:=\left[\begin{array}{c}
X^{t}\\
Y^{t}
\end{array}\right],\qquad\bm{F}^{*}:=\left[\begin{array}{c} 
X^{*}\\
Y^{*}
\end{array}\right],
\qquad\text{and}\qquad 
    \nabla f(X, Y; \tau) := \left[\begin{array}{c}
\nabla_{X} f(X, Y; \tau)\\
\nabla_{Y} f(X, Y; \tau)
\end{array}\right].
\nonumber
\end{equation}

We prove the following result, which is a generalized version of \cref{lem:induction-Frobenious}.

\begin{restatable}{lemma}{LemmaInductionFrobenious}\label{lem:induction-Frobenious-appendix}
Assume \cref{cond:Z-condition-nonconvex} holds and $\frac{\sigma}{\sigma_{\min}} \sqrt{n} \leq C_1\frac{1}{\kappa^2 r^{2}\log^{5}(n)}.$ 
Suppose $O = M^{*} + \tau^{*}Z + E$ for a deterministic $E$. Assume $\norm{E} \leq C_2 \sigma \sqrt{n}$ and $|\inner{Z}{E}| \leq C_3 \sigma\sqrt{n}\norm{Z}_{\F}.$

Let $\lambda = C_{\lambda} \sigma \sqrt{n} \log^{1.5}(n), \eta = \frac{C_{\eta}}{\kappa^3 n^{20} \sigma_{\max}}$.  Suppose the induction hypothesis (\ref{eq:induction-Frobenious}) holds for $q=t$ with (pre-determined) constants $C_{\F}, C_{B}$. Then (\ref{eq:induction-Frobenious}) holds for $q=t+1$.
\begin{subequations}\label{eq:induction-Frobenious}
\begin{align}    
	\norm{F^{q}H^{q} - F^{*}}_{\F} &\leq C_{\F} \error \norm{F^{*}}_{\F} \label{eq:FtHt-Fstar-Fnorm-Frobenious}\\
       \norm{X^{q \top} X^{q} - Y^{q \top}Y^{q}}_{\F} &\leq C_{B} \frac{\sigma}{\kappa n^{15}} \label{eq:XX-YY-Frobenious}.
 \end{align}  
 \end{subequations}
Furthermore, if \cref{eq:induction-Frobenious} holds for $q = t$, the following inequality holds.
\begin{align}
         f(X^{q+1}, Y^{q+1}; \tau^{q+1}) &\leq f(X^{q}, Y^{q}; \tau^{q}) - \frac{\eta}{2} \norm{\nabla f(X^{q}, Y^{q}; \tau^{q})}_{\F}^2. \label{eq:f-decrease-Frobenious}
\end{align}
\end{restatable}

\begin{proof}[Proof of \cref{lem:induction-Frobenious-appendix}]
To establish the proof, below we present the mathematical induction for \cref{eq:FtHt-Fstar-Fnorm-Frobenious}, \cref{eq:XX-YY-Frobenious}, and \cref{eq:f-decrease-Frobenious} separately. 

We will use the following fact, as a direct implication of \cref{eq:FtHt-Fstar-Fnorm-Frobenious}. See the proof in \cref{lem:singular-values}.
\begin{align*}
\sqrt{\sigma_{\min}}/2 &\leq \sigma_{i}(X^{t}) \leq 2\sqrt{\sigma_{\max}}\\
\sqrt{\sigma_{\min}}/2 &\leq \sigma_{i}(Y^{t}) \leq 2\sqrt{\sigma_{\max}}\\
\sigma_{\min}/2 &\leq \sigma_{i}(X^{t}Y^{t \top}) \leq 2\sigma_{\max}
\end{align*}

\textit{Proof of \cref{eq:FtHt-Fstar-Fnorm-Frobenious}.} Consider $M^{*} = X^{*}Y^{*\top}.$ Let the iteration sequences for the non-convex program be $(X^{t}, Y^{t}, \tau^{t})$. Then $F^{*} = [X^{*}; Y^{*}] \in \R^{2n \times r}, F^{t} = [X^{t}; Y^{t}] \in \R^{2n \times r}.$ Based on the induction hypothesis \cref{eq:induction-Frobenious} for $q=t$, we aim to show that \cref{eq:FtHt-Fstar-Fnorm-Frobenious} holds for $q=t+1.$  

Write $\nabla f(X^{t}, Y^{t}; \tau^{t})$ as $\nabla f(F^{t}; \tau^{t})$ for simplification. Consider
\begin{align}
    \norm{F^{t+1}H^{t+1} - F^{*}}_{\F}^2 
    &\overset{(i)}{\leq} \norm{F^{t+1}H^{t} - F^{*}}_{\F}^2\label{eq:FHt+1-F-FHt-F}\\
    &\overset{(ii)}{=} \inner{F^{t}H^{t}-F^{*}-\eta \nabla f(F^{t};\tau^{t})H^{t}}{F^{t}H^{t}-F^{*}-\eta \nabla f(F^{t};\tau^{t})H^{t}}.\nonumber
\end{align}
Here, (i) is due to that $H^{t+1}$ is the optimal rotation to align $F^{t+1}$ and $F^{*}$ (and $H^{t}$ is also a rotation), and (ii) is due to $F^{t+1} = F^{t} - \eta \nabla f(F^{t}; \tau^{t}).$ Write $F^{t}H^{t}, X^{t}H^{t}, Y^{t}H^{t}, \tau^{t}$ as $F, X, Y, \tau$ if there is no ambiguity. Let $\Delta_X = X - X^{*}, \Delta_{Y} = Y-Y^{*}, \Delta_{F} = F-F^{*}$. Then we have
\begin{align}
    \norm{F^{t+1}H^{t}- F^{*}}_{\F}^2
    &=\norm{\Delta_{F}}_{\F}^2 - 2\eta \inner{\Delta_{F}}{\nabla f(F^{t}; \tau^{t})H^{t}} + \underbrace{\eta^2 \norm{\nabla f(F^{t}; \tau^{t})}_{\F}^2}_{A_0} \nonumber\\
    &= \norm{\Delta_{F}}_{\F}^2 - 2\eta \inner{\Delta_{X}}{(XY^{\top}-M^{*}-E+(\tau-\tau^{*})Z)Y + \lambda X} \nonumber \\
    &\quad - 2\eta  \inner{\Delta_Y}{(XY^{\top}-M^{*}-E+(\tau-\tau^{*})Z)^{\top}X + \lambda Y} + A_0 \label{eq:FH-Fstar}
\end{align}

We first state that $A_0$ is negligible, since we take $\eta$ to be sufficiently small. 
\begin{claim}
$$
A_0 \lesssim  \eta \frac{\sigma^2}{n^{15}}. 
$$
\end{claim}
\begin{proof}
Consider 
\begin{align*}
\norm{\nabla_{X}f(F^{t}; \tau^{t})}_{\F}  
&=\|(X^{t}Y^{t\top}+\tau^{t}Z-O)Y^{t}H^{t} + \lambda X^{t}\|_{\F} \\
&\leq \|(X^{t}Y^{t\top}+\tau^{t}Z-O)Y^{t}H^{t}\|_{\F} + \lambda \|X^{t}\|_{\F}\\
&\overset{(i)}{\leq} \|(X^{t}Y^{t\top}+\tau^{t}Z-O)Y^{t}H^{t}\|_{\F} + O(\sigma \sqrt{n} \log^{1.5}(n) \sqrt{r\sigma_{\max}})
\end{align*} 
where (i) is by $\|X^{t}\|_{\F} \lesssim \sqrt{r\sigma_{\max}}$. For $\|(X^{t}Y^{t\top}+\tau^{t}Z-O)Y^{t}H^{t}\|_{\F}$, we have
\begin{align*}
\|(X^{t}Y^{t\top}+\tau^{t}Z-O)Y^{t}H^{t}\|_{\F} 
&\overset{(i)}{=} \|(X^{t}Y^{t\top}+\tau^{t}Z-O)Y^{t}\|_{\F}\\
&\leq \|X^{t}Y^{t\top}+\tau^{t}Z-O\|\|Y^{t}\|_{\F} \\
&\overset{(ii)}{\leq} \left(\|X^{t}Y^{t\top} - X^{*}Y^{*\top}\| + |\tau^{t} - \tau^{*}|\|Z\| +  \|E\|\right) \|Y^{t}\|_{\F}.
\end{align*}
Here, in (i) we use that $\|\cdot\|_{\F}$ is rotation-invariant, in (ii) we use that $O=X^{*}Y^{*\top} + E + \tau^{*}Z.$ Note that
\begin{align*}
\|X^{t}Y^{t\top} - X^{*}Y^{*\top}\| 
&\leq   \|X^{t}Y^{t\top} - X^{*}Y^{*\top}\|_{\F}\\
&\leq \|X^{t} - X^{*}\|_{\F} \|Y^{t}\| + \|X^{*}\| \|Y^{t} - Y^{*}\|_{\F}\\
&\overset{(i)}{\lesssim} \sqrt{\sigma_{\max}} \frac{\sigma \sqrt{n} \log^{2.5}(n) }{\sigma_{\min}} \sqrt{\sigma_{\max} r}\\
&\lesssim \sigma \sqrt{n} \kappa r^{0.5} \log^{2.5}(n).
\end{align*}
In (i) we use that $\max(\|X^{*}\|, \|Y^{t}\|) \lesssim \sqrt{\sigma_{\max}}$ and $\|X-X^{*}\|_{\F}+\|Y-Y^{*}\|_{\F}\lesssim \frac{\sigma \sqrt{n} \log^{2.5}(n) }{\sigma_{\min}} \sqrt{\sigma_{\max} r}.$ Next, consider $|\tau^{t} - \tau^{*}| \|Z\|.$ Note that
\begin{align*}
|\tau^{t} - \tau^{*}| 
&= \left|\frac{\inner{Z}{O-X^{t}Y^{t\top}}}{\|Z\|_{\F}^2} - \tau^{*}\right| \\
&\leq \frac{\inner{Z}{E}}{\|Z\|_{\F}^2} + \frac{\inner{Z}{X^{*}Y^{*\top} - X^{t}Y^{t\top}}}{\|Z\|_{\F}^2}\\
&\lesssim \frac{\sigma \sqrt{n} + \sigma \sqrt{n} \kappa r^{0.5} \log^{2.5}(n) }{\|Z\|_{\F}}
\end{align*} 
Combining above, we have
\begin{align}
 \|X^{t}Y^{t\top}+\tau^{t}Z-O\| &\lesssim \sigma \sqrt{n} \kappa r^{0.5} \log^{2.5}(n). \label{eq:D-bound}\\
  \|(X^{t}Y^{t\top}+\tau^{t}Z-O)Y\|_{\F} &\lesssim \sigma \sqrt{n} \kappa r \log^{2.5}(n)\sqrt{\sigma_{\max}}. \label{eq:DY-bound}
\end{align}
Furthermore, 
\begin{align}
\norm{\nabla f(F; \tau)}_{\F}^2 \lesssim \sigma \sqrt{n} r \kappa \log^{2.5}(n) \sqrt{\sigma_{\max}}. \label{eq:DX-F-bound}
\end{align}
This implies
\begin{align}
|A_0| 
&\lesssim \eta^2 (\sigma \sqrt{n} r \kappa \log^{2.5}(n) \sqrt{\sigma_{\max}})^2 \nonumber\\
&\lesssim \eta \frac{1}{\kappa^3 n^{20}\sigma_{\max}} \sigma^2 r^{2} \kappa^2 \log^{5}(n) n \sigma_{\max} \nonumber\\
&\lesssim \eta \frac{\sigma^2}{n^{15}}. \label{eq:FH-Fstar-A0-bound}
\end{align}
\end{proof}

Continue the analysis on \cref{eq:FH-Fstar}, by some algebra, we have
\begin{align}
    &\norm{F^{t+1}H^{t}- F^{*}}_{\F}^2\nonumber\\
    &= \norm{\Delta_{F}}_{\F}^2 -2\eta \inner{\Delta_X}{(XY^{\top}-M^{*}+(\tau-\tau^{*})Z)Y} \nonumber\\
    &\quad - 2\eta \inner{\Delta_Y}{(XY^{\top}-M^{*}+(\tau-\tau^{*})Z)^{\top}X} \nonumber \\
    &\quad + \underbrace{2\eta \inner{\Delta_X}{EY} + 2\eta \inner{\Delta_Y}{E^{\top}X}}_{A_1} +\underbrace{2\eta \inner{\Delta_X}{-\lambda X} + 2\eta \inner{\Delta_Y}{-\lambda Y}}_{A_2} + A_0 \nonumber\\
    &\leq \norm{\Delta_{F}}_{\F}^2 -2\eta \underbrace{\inner{XY^{\top}-M^{*}+(\tau-\tau^{*})Z}{\Delta_{X}Y^{\top}+X\Delta_{Y}^{\top}}}_{A_3} + A_0 + A_1 + A_2. \label{eq:FH-Fstar-2}
\end{align}
For $A_1$ and $A_2$, we have the following claim.

\begin{claim}\label{claim:FtHt-Fstar-Control-A}
\begin{align}
|A_1| &\lesssim \eta \sigma \sqrt{n} \norm{F^{*}}_{\F} \norm{\Delta_{F}}_{\F}\nonumber\\
|A_2| &\lesssim \eta \sigma \sqrt{n}\log^{1.5}(n) \norm{F^{*}}_{\F} \norm{\Delta_{F}}_{\F}\nonumber
\end{align}
\end{claim}

\begin{proof}[Proof of \cref{claim:FtHt-Fstar-Control-A}]
Recall that $A_1 = 2\eta \inner{\Delta_X}{EY} + 2\eta \inner{\Delta_Y}{E^{\top}X}$ and we want to show 
$$
|A_1| \lesssim \eta \sigma \sqrt{n} \norm{F^{*}}_{\F} \norm{\Delta_{F}}_{\F}.
$$ 
This can be verified by the following.
\begin{align}
    |A_1| 
    &\lesssim \eta \norm{\Delta_{X}}_{\F}\norm{EY}_{\F} + \eta \norm{\Delta_{Y}}_{\F} \norm{E^{\top}X}_{\F} \nonumber\\
    &\lesssim \eta \norm{\Delta_{F}}_{\F} \norm{E} \norm{F}_{\F} \nonumber\\
    &\overset{(i)}{\lesssim} \eta \sigma \sqrt{n} \norm{\Delta_{F}}_{\F}\norm{F^{*}}_{\F}\nonumber
\end{align}
where (i) is due to $\norm{E} \lesssim \sigma \sqrt{n}$.

Recall that $A_2 = 2\eta \inner{\Delta_X}{-\lambda X} + 2\eta \inner{\Delta_Y}{-\lambda Y}$ and we want to show 
$$
|A_2| \lesssim \eta \sigma \sqrt{n}\log^{1.5}(n) \norm{F^{*}}_{\F} \norm{\Delta_{F}}_{\F}.
$$
This can be verified by the following. 
\begin{align}
    |A_2| &\lesssim \eta \lambda \norm{\Delta_{X}}_{\F}\norm{Y}_{\F} + \eta \lambda \norm{\Delta_{Y}}_{\F} \norm{X}_{\F}\nonumber\\
    &\lesssim \eta \lambda \norm{\Delta_{F}}_{\F}  \norm{F}_{\F}\nonumber\\
    &\lesssim \eta \sigma \sqrt{n} \log^{1.5}(n)\norm{\Delta_{F}}_{\F}\norm{F^{*}}_{\F}.\nonumber
\end{align}
\end{proof}

For $A_3$, note that $XY^{\top} - M^{*} = \Delta_{X}Y^{\top}+X\Delta_{Y}^{\top}-\Delta_{X}\Delta_{Y}^{\top}, \tau - \tau^{*} = \frac{\inner{Z}{M^{*}-XY^{\top}+E}}{\norm{Z}_{\F}^2}$, then we have
\begin{align}
    A_3 
    &= \inner{XY^{\top}-M^{*}+(\tau-\tau^{*})Z}{\Delta_{X}Y^{\top}+X\Delta_{Y}^{\top}} \nonumber\\
    &= \norm{\Delta_{X}Y^{*\top}+X^{*}\Delta_{Y}^{\top}}_{\F}^2 - \underbrace{\inner{\Delta_X\Delta_Y^{\top}}{\Delta_{X}Y^{\top}+X\Delta_{Y}^{\top}}}_{B_0}\nonumber\\
    & \quad + \frac{\inner{Z}{M^{*}-XY^{\top}}\inner{Z}{\Delta_{X}Y^{\top}+X\Delta_{Y}^{\top}}}{\norm{Z}_{\F}^2} + \underbrace{\frac{\inner{Z}{E}\inner{Z}{\Delta_{X}Y^{\top}+X\Delta_{Y}^{\top}}}{\norm{Z}_{\F}^2}}_{B_1}\nonumber\\
    &\overset{(i)}{=} \norm{\Delta_{X}Y^{\top}+X\Delta_{Y}^{\top}}_{\F}^2 - \frac{\inner{Z}{\Delta_{X}Y^{\top}+X\Delta_{Y}^{\top}}\inner{Z}{\Delta_{X}Y^{\top}+X\Delta_{Y}^{\top}}}{\norm{Z}_{\F}^2}\nonumber\\
    &\quad + \underbrace{\frac{\inner{Z}{\Delta_X\Delta_{Y}^{\top}}\inner{Z}{\Delta_{X}Y^{\top}+X\Delta_{Y}^{\top}}}{\norm{Z}_{\F}^2}}_{B_2} - B_0 + B_1\nonumber
\end{align}
In (i), we use again $XY^{\top}-M^{*} = \Delta_{X}Y^{\top}+X\Delta_{Y}^{\top} - \Delta_{X}\Delta_{Y}^{\top}$.

We then have the following claim to control $B_0, B_1, B_2$.
\begin{claim}\label{claim:FtHt-Fstar-Control-B}
\begin{align}
    |B_0| &\lesssim \sigma \sqrt{n} \norm{\Delta_{F}} \norm{F^{*}}_{\F}\nonumber\\
    |B_1| &\lesssim \sigma \sqrt{n} \norm{\Delta_{F}}_{\F} \norm{F^{*}}_{\F}\nonumber\\
    |B_2| &\lesssim \sigma \sqrt{n} \norm{\Delta_{F}} \norm{F^{*}}_{\F}\nonumber
\end{align}
\end{claim}

\begin{proof}
For $B_0$, recall that $B_0 = \inner{\Delta_X\Delta_Y^{\top}}{\Delta_{X}Y^{\top}+X\Delta_{Y}^{\top}}$ and we want to show $ |B_0| \lesssim \sigma \sqrt{n} \norm{\Delta_{F}} \norm{F^{*}}_{\F}$. This can be verified by
\begin{align}
    |B_0| 
    &\lesssim \norm{\Delta_{X}}_{\F} \norm{\Delta_Y}_{\F} \left( \norm{\Delta_{X}}_{\F} \norm{Y} + \norm{\Delta_Y}_{\F} \norm{X} \right) \nonumber\\
    &\lesssim \norm{\Delta_{F}}_{\F}^3 \norm{F^{*}}\nonumber\\
    &\lesssim \norm{\Delta_{F}} \norm{F^{*}}_{\F} \frac{\sigma \sqrt{n} \log^{2.5}(n)}{\sigma_{\min}} \frac{\sigma \sqrt{n} \log^{2.5}(n)}{\sigma_{\min}} \sqrt{\sigma_{\max} r} \sqrt{\sigma_{\max}} \nonumber\\
    &\lesssim \norm{\Delta_{F}} \norm{F^{*}}_{\F} \frac{\sigma \sqrt{n} \log^{2.5}(n)}{\sigma_{\min}}\sigma \sqrt{n} \log^{2.5}(n) \kappa r^{0.5} \nonumber\\
    &\lesssim \sigma \sqrt{n} \norm{\Delta_{F}} \norm{F^{*}}_{\F}\nonumber
\end{align}
providing that $\sSNR.$ 

For $B_1$, recall that $B_1 = \frac{\inner{Z}{E}\inner{Z}{\Delta_{X}Y^{\top}+X\Delta_{Y}^{\top}}}{\norm{Z}_{\F}^2}$ and we want to show $$
|B_1| \lesssim \sigma \sqrt{n} \norm{\Delta_{F}}_{\F} \norm{F^{*}}_{\F}.
$$ 
Note that $|\inner{Z}{E}|/\norm{Z}_{\F}^2 \lesssim \sigma \sqrt{n}/\norm{Z}_{\F}$. Then
\begin{align}
    |B_1| 
    &\lesssim \frac{\sigma\sqrt{n}}{\norm{Z}_{\F}} \norm{Z}_{\F} \norm{\Delta_{F}}_{\F} \norm{F^{*}} \nonumber\\
    &\lesssim \sigma \sqrt{n} \norm{\Delta_{F}}_{\F} \norm{F^{*}}_{\F}.\nonumber
\end{align}

For $B_2$, recall that $B_2 = \frac{\inner{Z}{\Delta_X\Delta_{Y}^{\top}}\inner{Z}{\Delta_{X}Y^{\top}+X\Delta_{Y}^{\top}}}{\norm{Z}_{\F}^2}$ and we want to show $|B_2| \lesssim \sigma \sqrt{n} \norm{\Delta_{F}} \norm{F^{*}}_{\F}.$ Note that
\begin{align}
    |B_2| 
    &\lesssim \frac{(\norm{Z}_{\F} \norm{\Delta_{F}}_{\F} \norm{\Delta_{F}}_{\F})(\norm{Z}_{\F} \norm{\Delta_{F}}_{\F} \norm{F^{*}})}{\norm{Z}_{\F}^2} \nonumber\\
    &\lesssim \norm{\Delta_{F}}_{\F}^3 \norm{F^{*}} \nonumber\\
    &\overset{(i)}{\lesssim} \sigma \sqrt{n} \norm{\Delta_{F}} \norm{F^{*}}_{\F}\nonumber
\end{align}
where (i) is due to the similar analysis for $B_0$. This completes the proof for the bounds on $B_0, B_1, B_2.$
\end{proof}

Next, note the fact that $P_{T^{\perp}}(\Delta_{X}Y^{\top}+X\Delta_Y^{\top}) = 0$ (The column space of $Y$ is the same as $V$, and the column space of $X$ is the same as $U$). Then $P_{T}(\Delta_{X}Y^{\top}+X\Delta_Y^{\top}) = \Delta_{X}Y^{\top}+X\Delta_Y^{\top}$ and 
\begin{align}
    \left|\inner{Z}{\Delta_{X}Y^{\top}+X\Delta_Y^{\top}}\right|^2 
    &= \left|\inner{P_{T}(Z)}{\Delta_{X}Y^{\top}+X\Delta_Y^{\top}}\right|^2 \nonumber\\
    &\leq \norm{P_{T}(Z)}_{\F}^2\norm{\Delta_{X}Y^{\top}+X^{\top}\Delta_Y^{\top}}_{\F}^2\nonumber
\end{align}
This implies
\begin{align}
    A_3 &\geq \norm{\Delta_{X}Y^{\top}+X\Delta_{Y}^{\top}}_{\F}^2 - \frac{\norm{P_{T}(Z)}_{\F}^2}{\norm{Z}_{\F}^2} \norm{\Delta_{X}Y^{\top}+X\Delta_Y^{\top}}_{\F}^2 + B_2 - B_0 + B_1\nonumber\\
    &\overset{(i)}{\geq} \frac{C_{r_1}}{2\log(n)}\underbrace{\norm{\Delta_{X}Y^{\top}+X\Delta_{Y}^{\top}}_{\F}^2}_{B_3} + B_2 - B_0 + B_1. \label{eq:A3}
\end{align}
where (i) is due to $\|Z\|_{\F}^2 = \|P_{T^{\perp}}(Z)\|_{\F}^2 + \|P_{T}(Z)\|_{\F}^2$ and $\|P_{T^{\perp}}(Z)\|_{\F}^2 \geq \frac{C_{r_1}}{2\log(n)} \|Z\|_{\F}^2$ (by \cref{lem:assumption-hold-around-ball}). To proceed, the following claim connects $\norm{\Delta_{X}Y^{\top}+X\Delta_{Y}^{\top}}_{\F}^2$ and $\norm{\Delta_{F}}_{\F}^2$, whose proof requires some efforts.  
\begin{claim}\label{claim:XDeltaY-lower-bound}
\begin{align}
    \norm{\Delta_{X}Y^{\top}+X\Delta_{Y}^{\top}}_{\F}^2 \geq \frac{\sigma_{\min}}{4}\norm{\Delta_{F}}_{\F}^2  - \frac{\sigma^2}{n^{13}}. \label{eq:DeltaXY}
\end{align}
\end{claim}
\begin{proof}
Note that
\begin{align}
    \norm{\Delta_{X}Y^{\top} + X\Delta_{Y}^{\top}}_{\F}^2 
    &= \tr\left((\Delta_{X}Y^{\top} + X\Delta_{Y}^{\top})(\Delta_{X}Y^{\top} + X\Delta_{Y}^{\top})^{\top}\right) \nonumber\\
    &= \norm{\Delta_{X}Y^{\top}}_{\F}^2 + \norm{X\Delta_{Y}^{\top}}_{\F}^2 + \tr(\Delta_{X}Y^{\top}\Delta_{Y}X^{\top} + X\Delta_{Y}^{\top}Y\Delta_{X}^{\top}) \nonumber\\ 
    &= \norm{\Delta_{X}Y^{\top}}_{\F}^2 + \norm{X\Delta_{Y}^{\top}}_{\F}^2 + 2\tr(X\Delta_{Y}^{\top}Y\Delta_{X}^{\top})  \nonumber.
\end{align}
Consider 
$$
\norm{\Delta_{X}Y^{\top}}_{\F}^2 = \sum_{i=1}^{n} \norm{(\Delta_{X})_{i,\cdot} \cdot Y^{\top}}^2 \geq \sum_{i=1}^{n} \sigma^2_r(Y) \norm{(\Delta_{X})_{i,\cdot}}^2 = \sigma_{\min}/2 \norm{\Delta_{X}}_{\F}^2.
$$ 
Similarly, $\norm{X\Delta_{Y}^{\top}}_{\F}^2 \geq \sigma_{\min}/2 \norm{\Delta_{Y}}_{\F}^2.$ This implies that
\begin{align}
   \norm{\Delta_{X}Y^{\top} + X\Delta_{Y}^{\top}}_{\F}^2  
   &\geq \sigma_{\min}/2 \norm{\Delta_{X}}_{\F}^2 + \sigma_{\min}/2 \norm{\Delta_{Y}}_{\F}^2 + 2\tr(X\Delta_{Y}^{\top}Y\Delta_{X}^{\top}) \nonumber\\
    &= \sigma_{\min}/2 \norm{\Delta_{F}}_{\F}^2 + 2\tr(\Delta_{X}^{\top}X\Delta_{Y}^{\top}Y).\nonumber
\end{align}
Furthermore, note that we can do the following derivation to reduce the problem to control $\norm{\Delta_{X}^{\top}X-Y^{\top}\Delta_{Y}}_{\F}$.
\begin{align}
     &\norm{\Delta_{X}Y^{\top} + X\Delta_{Y}^{\top}}_{\F}^2 \nonumber\\ 
    &\geq \sigma_{\min}/2 \norm{\Delta_{F}}_{\F}^2 + 2\tr(\Delta_{X}^{\top}X\Delta_{Y}^{\top}Y)  \nonumber\\
    &= \sigma_{\min}/2 \norm{\Delta_{F}}_{\F}^2 + 2\tr(Y^{\top}\Delta_{Y}\Delta_{Y}^{\top}Y) + 2\tr((\Delta_{X}^{\top}X-Y^{\top}\Delta_{Y})\Delta_{Y}^{\top}Y) \nonumber\\
    &= \sigma_{\min}/2 \norm{\Delta_{F}}_{\F}^2 + 2\norm{Y^{\top}\Delta_{Y}}_{\F}^2 + 2\tr((\Delta_{X}^{\top}X-Y^{\top}\Delta_{Y})\Delta_{Y}^{\top}Y) \nonumber\\
    &\overset{(i)}{\geq} \sigma_{\min}/2 \norm{\Delta_{F}}_{\F}^2-2\norm{\Delta_{X}^{\top}X-Y^{\top}\Delta_{Y}}_{\F} \norm{Y^{\top}\Delta_{Y}}_{\F}. \label{eq:sigma-min-Delta-F}
\end{align}
Here, in (i) we use $\norm{Y^{\top}\Delta_{Y}}_{\F}^2 \geq 0$ and $\tr(AB^{\top}) \leq \norm{A}_{\F}\norm{B}_{\F}$ by Cauchy-Schwartz inequality. 

Then it boils down to control $\norm{\Delta_{X}^{\top}X-Y^{\top}\Delta_{Y}}_{\F}$. To begin, let us state the following technical lemma. 
\begin{lemma}[Orthogonal Procrustes Problem, Lemma 35 in \cite{ma2019implicit}]\label{lem:orthogonal-procrustes}
For $F_0, F_1 \in \R^{n \times m}$, $H$ is the minimizer of the following optimization problem if and only if $F_0^{\top}F_1H$ is symmetric and positive semidefinite:
\begin{align*}
    \minimize_{A \in \O^{r\times r}} \norm{F_1 A - F_0}_{\F}. 
\end{align*}
\end{lemma}
In order to show $\Delta_{X}^{\top}X \approx Y^{\top}\Delta_{Y}$, we tend to invoke \cref{lem:orthogonal-procrustes} and use the property that $X^{\top}X \approx Y^{\top}Y$. Note that by \cref{lem:orthogonal-procrustes}, we have $(F^{t}H^{t})^{\top}F^{*}$ is symmetric since $H^{t}$ is the optimal rotation to align $F^{t}$ and $F^{*}$. Recall that $X = X^{t}H^{t}, Y = Y^{t}H^{t}.$ This implies that $X^{\top}X^{*} + Y^{\top}Y^{*}$ is symmetric. Furthermore, since $X^{\top}X, Y^{\top}Y$ is symmetric, then 
$$
X^{\top}\Delta_{X} +Y^{\top} \Delta_{Y} = X^{\top}X + Y^{\top}Y - X^{\top}X^{*} - Y^{\top}Y^{*}
$$ 
is also symmetric. Therefore $X^{\top}\Delta_{X} +Y^{\top} \Delta_{Y} = (X^{\top}\Delta_{X} +Y^{\top} \Delta_{Y})^{\top} = \Delta_{X}^{\top}X + \Delta_{Y}^{\top}Y$, which further implies 
\begin{align}
   \Delta_{X}^{\top}X - Y^{\top}\Delta_{Y} = X^{\top}\Delta_{X}  - \Delta_{Y}^{\top}Y.\label{eq:DeltaXX-YDeltaY}
\end{align}
On the other hand, note that
\begin{align}
    X^{*\top}X^{*} - Y^{*\top}Y^{*}
    &=  (X-\Delta_{X})^{\top}(X-\Delta_{X}) - (Y-\Delta_{Y})^{\top}(Y-\Delta_{Y})\nonumber\\
    &= X^{\top}X - Y^{\top}Y + \Delta_{X}^{\top}\Delta_{X} - \Delta_{Y}^{\top}\Delta_Y\nonumber\\
    &\quad - (\Delta_{X}^{\top}X - Y^{\top}\Delta_{Y}) - (X^{\top}\Delta_{X}  - \Delta_{Y}^{\top}Y) \nonumber\\
    &\overset{(i)}{=} X^{\top}X - Y^{\top}Y + \Delta_{X}^{\top}\Delta_{X} - \Delta_{Y}^{\top}\Delta_Y - 2(\Delta_{X}^{\top}X - Y^{\top}\Delta_{Y}) \label{eq:XTX-YTY-DeltaXTX}
\end{align}
where (i) is due to \cref{eq:DeltaXX-YDeltaY}. Also note that $X^{*\top}X^{*} - Y^{*\top}Y^{*} = \Sigma^{*} - \Sigma^{*} = 0 $. Then \cref{eq:XTX-YTY-DeltaXTX} implies that
\begin{align}
    \Delta_{X}^{\top}X - Y^{\top}\Delta_{Y} = \frac{1}{2}\left( X^{\top}X - Y^{\top}Y +  \Delta_{X}^{\top}\Delta_{X} - \Delta_{Y}^{\top}\Delta_Y\right).\nonumber
\end{align}
Plug this equality into \cref{eq:sigma-min-Delta-F}, we obtain
\begin{align}
    &\norm{\Delta_{X}Y^{\top} + X\Delta_{Y}^{\top}}_{\F}^2  \nonumber\\
    &\geq \sigma_{\min}/2 \norm{\Delta_{F}}_{\F}^2-2\norm{\Delta_{X}^{\top}X-Y^{\top}\Delta_{Y}}_{\F} \norm{Y^{\top}\Delta_{Y}}_{\F}\nonumber\\
    &\geq \sigma_{\min}/2 \norm{\Delta_{F}}_{\F}^2 - \norm{X^{\top}X - Y^{\top}Y +  \Delta_{X}^{\top}\Delta_{X} - \Delta_{Y}^{\top}\Delta_Y}_{\F} \norm{Y}\norm{\Delta_{Y}}_{\F}. \nonumber
\end{align}
Recall that $\norm{X^{\top}X - Y^{\top}Y}_{\F} = \norm{H^{tT}(X^{tT}X^{t}-Y^{tT}Y^{t})H^{t}} \lesssim \frac{\sigma}{\kappa n^{15}}$ (by \cref{eq:XX-YY-Frobenious}) and $\norm{\Delta_{F}}_{\F} \lesssim \frac{\sigma \sqrt{n} \log^{2.5}(n)}{\sigma_{\min}} \sqrt{\sigma_{\max}r}$, then we have
\begin{align*}
    &\norm{X^{\top}X - Y^{\top}Y +  \Delta_{X}^{\top}\Delta_{X} - \Delta_{Y}^{\top}\Delta_Y}_{\F}\norm{Y}\norm{\Delta_{Y}}_{\F}\\
    &\leq \left(\norm{X^{\top}X - Y^{\top}Y}_{\F} + \norm{\Delta_{X}^{\top}\Delta_{X}}_{\F} + \norm{\Delta_{Y}^{\top}\Delta_{Y}}_{\F}\right)\sqrt{\sigma_{\max}}\frac{\sigma \sqrt{n} \log^{2.5}(n)}{\sigma_{\min}} \sqrt{\sigma_{\max}r} \\
    &\lesssim \left(\frac{\sigma}{\kappa n^{15}} + \norm{\Delta_{F}}_{\F}^2\right) \sigma \sqrt{n} \kappa r^{0.5}\log^{2.5}(n)\\
    &\overset{(i)}{\lesssim} \norm{\Delta_{F}}_{\F}^2 \frac{\sigma_{\min}}{\log^{0.5}(n)} + \frac{\sigma^2}{n^{13}}.
\end{align*}
Here, (i) is due to $\sSNR$. Then this implies, for large enough $n$, 
\begin{align*}
    \norm{\Delta_{X}Y^{\top} + X\Delta_{Y}^{\top}}_{\F}^2  
    &\geq \sigma_{\min}/2 \norm{\Delta_{F}}_{\F}^2 - \norm{\Delta_{F}}_{\F}^2 \frac{\sigma_{\min}}{4} - \frac{\sigma^2}{n^{13}}\\
    &\geq \frac{\sigma_{\min}}{4}\norm{\Delta_{F}}_{\F}^2  - \frac{\sigma^2}{n^{13}}
\end{align*}
which completes the proof.
\end{proof}

Combining \cref{eq:FH-Fstar-2}, \cref{eq:A3}, and \cref{eq:DeltaXY}, we arrive at
\begin{align}
    \norm{F^{t+1}H^{t}- F^{*}}_{\F}^2 
    &\overset{(i)}{=} \norm{\Delta_{F}}_{\F}^2 - 2 \eta A_3 + A_0 + A_1 + A_2 \nonumber\\
    &\overset{(ii)}{\leq} \norm{\Delta_{F}}_{\F}^2 - 2\eta  \left( \frac{C}{\log(n)}\norm{\Delta_{X}Y^{\top}+X^{*}\Delta_{Y}^{\top}}_{\F}^2 + B_2 - B_0 + B_1\right) \nonumber\\
    &\quad + A_0 + A_1 + A_2 \nonumber\\
    &\overset{(iii)}{\leq} \norm{\Delta_{F}}_{\F}^2 - 2\eta  \left( \frac{C}{\log(n)}\frac{\sigma_{\min}}{4} \norm{\Delta_{F}}_{\F}^2 - \underbrace{\frac{C\sigma^2}{\log(n)n^{13}}}_{B_3} +  B_2 - B_0 + B_1\right) \nonumber\\
    &\quad + A_0 + A_1 + A_2 \label{eq:2eta-A0-A1-A2-A4}
\end{align}
where (i) is due to \cref{eq:FH-Fstar-2}, (ii) is due to \cref{eq:A3}, and (iii) is due to \cref{eq:DeltaXY}. Next, due to the bound on $A_0 $ (\cref{eq:FH-Fstar-A0-bound}), the bound on $A_1, A_2$ (\cref{claim:FtHt-Fstar-Control-A}), and the bound on $B_0, B_1, B_2$ (\cref{claim:FtHt-Fstar-Control-B}), we have
\begin{align*}
     &|A_0| + |A_1| + |A_2| + \eta|B_0| + \eta|B_1| + \eta|B_2| + \eta|B_3|\\
     &\lesssim \eta \sigma \sqrt{n} \log^{1.5}(n) \norm{F^{*}}_{\F} \norm{\Delta_{F}}_{\F} + \eta \frac{\sigma^2}{n^{13}}.
\end{align*}

Continue the analysis for \cref{eq:2eta-A0-A1-A2-A4}, we then have (for some constant $C'$), 
\begin{align}
    &\norm{F^{t+1}H^{t+1}- F^{*}}_{\F}^2 \nonumber\\
    &\leq \norm{\Delta_{F}}_{\F}^2 - \eta \frac{C\sigma_{\min}}{\log n}\norm{\Delta_{F}}_{\F}^2 + C'\eta \sigma \sqrt{n} \log^{1.5}(n) \norm{\Delta_{F}}_{\F}  \norm{F^{*}}_{\F} + C' \eta \frac{\sigma^2}{n^{13}} \label{eq:delta-f-bound}.
\end{align}
Plug $\norm{\Delta_{F}} \leq C_{\F} \error \norm{F^{*}}_{\F}$ into above \cref{eq:delta-f-bound} with $C_{F} = \max(2C'/C, 1)$, one can verify that 
\begin{align*}
\norm{F^{t+1}H^{t+1}- F^{*}}_{\F}^2 \leq  \left(C_{\F} \error \norm{F^{*}}_{\F}\right)^2,
\end{align*}
which completes the proof for \cref{eq:FtHt-Fstar-Fnorm-Frobenious}.

\textit{Proof of Eq.~(\ref{eq:XX-YY-Frobenious}).}
Following the algebra of section D.8 in \cite{chen2019noisy}, one can verify the following equality
\begin{align}
    A^{t+1} = (1-\lambda \eta)^2 A^{t} + \eta^2 (Y^{tT}D^{tT}D^{t}Y^{t} - X^{tT}D^{t}D^{tT}X^{t})\nonumber
\end{align}
where 
\begin{align}
    A^{t} &:= X^{tT} X^{t} - Y^{tT} Y^{t}\nonumber\\
    A^{t+1} &:= X^{(t+1) \top}X^{t+1} - Y^{(t+1) \top} Y^{t+1}\nonumber\\
    D^{t} &:= X^{t}Y^{tT} - M^{*} - E + (\tau^{t} - \tau^{*}) Z. \nonumber
\end{align}

One can verify that $\max(\norm{D^{t}Y^{t}}_{\F}, \norm{D^{tT}X^{t}}_{\F}) \lesssim \sigma \sqrt{n} r \kappa \log^{2.5}(n) \sqrt{\sigma_{\max}}$ (by \cref{eq:D-bound}). Then there exists a constant $C$, where
\begin{align}
    \norm{A^{t+1}}_{\F} 
    &\leq (1-\lambda \eta)^2 \norm{A^{t}}_{\F} + \eta^2 \left(\norm{D^{t}Y^{t}}_{\F}^2+\norm{D^{tT}X^{t}}_{\F}^2\right) \nonumber\\
    &\leq (1-\lambda \eta) \norm{A^{t}}_{\F} +  C\eta^2(\sigma \sqrt{n} r \kappa \log^{2.5}(n) \sqrt{\sigma_{\max}})^2 \nonumber\\
    &\leq C_{B} \frac{\sigma}{\kappa n^{15}} - C_{B} \lambda \eta \frac{\sigma}{\kappa n^{15}} +  C\eta^2(\sigma \sqrt{n} r \kappa \log^{2.5}(n) \sqrt{\sigma_{\max}})^2 \nonumber.
\end{align}
In order to show the desired bound $\norm{A^{t+1}}_{\F} \leq  C_{B} \frac{\sigma}{\kappa n^{15}}$, it is sufficient to show that $C_{B} \lambda \eta \frac{\sigma}{\kappa n^{15}} \geq C\eta^2(\sigma \sqrt{n} r \kappa \log^{2.5}(n) \sqrt{\sigma_{\max}})^2.$ This can be obtained by noting that
\begin{align*}
\frac{\kappa n^{15}}{\sigma \lambda \eta} \eta^2(\sigma \sqrt{n} r \kappa \log^{2.5}(n) \sqrt{\sigma_{\max}})^2 
&= \frac{\kappa n^{15}}{\sigma \cdot \sigma \sqrt{n} \log^{1.5}(n) \kappa^3 n^{20} \sigma_{\max}} \sigma^2 n r^{2} \kappa^2 \log^{5}(n) \sigma_{\max}\\
&= \frac{\sqrt{n} r^{2} \log^{3.5}(n)}{n^5}\\
&\ll 1. 
\end{align*}

\textit{Proof of Eq.~(\ref{eq:f-decrease-Frobenious}).}
Note that
\begin{align}
    f(F^{t+1}; \tau^{t}) = \frac{1}{2}\norm{X^{t+1}Y^{t+1T} - M^{*} - E + (\tau^{t}-\tau^{*})Z}_{\F}^2 + \frac{\lambda}{2} \norm{X^{t+1}}_{\F}^2 + \frac{\lambda}{2} \norm{Y^{t+1}}_{\F}^2 \nonumber.
\end{align}

We write $X^{t}, Y^{t}, \tau^{t}$ as $X, Y, \tau$ if there is no ambiguity. Let $D := XY^{\top} - M^{*} - E + (\tau - \tau^{*})Z.$ Then 
\begin{align}
    X^{t+1} &= X - \eta (DY + \lambda X)\nonumber\\
    Y^{t+1} &= Y - \eta (D^{\top}X + \lambda Y).\nonumber
\end{align}
Furthermore, 
\begin{align}
    \norm{\nabla f(F^{t}; \tau^{t})}_{\F}^2 
    &= \norm{DY + \lambda X}_{\F}^2 + \norm{D^{\top}X + \lambda Y}_{\F}^2 \nonumber\\
    &= \tr(Y^{\top}D^{\top}DY) + 2\lambda \tr(X^{\top}DY) + \lambda^2 \tr(X^{\top}X) \nonumber\\
    &\quad + \tr(X^{\top}DD^{\top}X) + 2\lambda \tr(Y^{\top}D^{\top}X) + \lambda^2 \tr(Y^{\top}Y).\nonumber
\end{align}

Consider
\begin{align}
    &\norm{X^{t+1}Y^{t+1T} - M^{*} - E + (\tau^{t}-\tau^{*})Z}_{\F}^2 + \lambda \norm{X^{t+1}}_{\F}^2 + \lambda \norm{Y^{t+1}}_{\F}^2 \nonumber\\
    &=\norm{D - \eta \left((DY+\lambda X)Y^{\top}+ X(X^{\top}D+\lambda Y^{\top})\right) + \eta^2 (DY + \lambda X) (D^{\top}X + \lambda Y)^{\top}}_{\F}^2 \nonumber\\
    &\quad + \lambda \norm{X - \eta (DY + \lambda X)}_{\F}^2 + \lambda \norm{Y - \eta (D^{\top}X + \lambda Y)}_{\F}^2 \nonumber\\
    &= \norm{D}_{\F}^2 + \lambda \norm{X}_{\F}^2 + \lambda \norm{Y}_{\F}^2 + \eta^2 A_0\nonumber \\
    &\quad - 2\eta \left(\tr(D^{\top}DYY^{\top})+2\lambda\tr(D^{\top}XY^{\top}) + \tr(D^{\top}XX^{\top}D)\right) \nonumber\\
    &\quad - 2\eta \left(2\lambda \tr(DYX^{\top}) + \lambda^2 (\tr(X^{\top}X) + \tr(Y^{\top}Y))\right)\nonumber \\
    &= 2f(F^{t}; \tau^{t}) - 2\eta \norm{\nabla f(F^{t}; \tau^{t})}_{\F}^2 + \eta^2 A_0\nonumber
\end{align}
where $A_0$ is the term associated with the coefficient $\eta^2$ and one can verify that
\begin{align}
    \norm{A_0}_{\F} 
    &\leq 2\norm{D}_{\F} \norm{DY+\lambda X}_{\F} \norm{D^{\top}X+\lambda Y}_{\F} + \lambda \norm{DY+\lambda X}_{\F}^2 \nonumber\\
    &\quad + \lambda \norm{D^{\top}X + \lambda Y}_{\F}^2 + \left(\norm{DY+\lambda X}_{\F}\norm{Y^{\top}}_{\F}+ \norm{X}_{\F}\norm{X^{\top}D+\lambda Y^{\top}}_{\F}\right)^2 \nonumber\\
    &\quad + 2\eta \norm{DY+\lambda X}_{\F}\norm{Y^{\top}}_{\F} \left(\norm{DY+\lambda X}_{\F} \norm{D^{\top}X+\lambda Y}_{\F}\right)\nonumber\\
    &\quad + 2\eta \norm{X}_{\F}\norm{X^{\top}D+\lambda Y^{\top}}_{\F}  \left(\norm{DY+\lambda X}_{\F} \norm{D^{\top}X+\lambda Y}_{\F}\right)\nonumber\\
    &\quad + \eta^2 (\norm{DY+\lambda X}_{\F} \norm{D^{\top}X+\lambda Y}_{\F})^2.\nonumber
\end{align}
In order to show that $f(F^{t+1}; \tau^{t}) \leq f(F^{t}; \tau^{t}) - \frac{\eta}{2} \norm{\nabla f(F^{t}; \tau^{t})}_{\F}^2$, it is sufficient to show that $A_0 \eta \leq \norm{\nabla f(F^{t}; \tau^{t})}_{\F}^2.$

Note that $\max\left( \norm{DY+\lambda X}_{\F}, \norm{D^{\top}X+\lambda Y}_{\F}\right) \leq \norm{\nabla f(F^{t};\tau^{t})}_{\F}$, recall that 
$$
\norm{\nabla f(F^{t};\tau^{t})}_{\F} \lesssim \sigma \sqrt{n} r \kappa \log^{2.5}(n) \sqrt{\sigma_{\max}},
$$ then we have
\begin{align*}
     &\eta\norm{A_0}_{\F} \\
     &\lesssim \eta\norm{\nabla f(F^{t};\tau^{t})}_{\F}^2 (\norm{D}_{\F} + \lambda + \norm{F}_{\F}^2) + \eta^2 \norm{\nabla f(F^{t};\tau^{t})}_{\F}^3 \norm{F}_{\F} + \eta^3 \norm{\nabla f(F^{t};\tau^{t})}_{\F}^4\\
     &\overset{(i)}{\lesssim} \norm{\nabla f(F^{t};\tau^{t})}_{\F}^2 \eta( \sigma \sqrt{n} r \kappa \log^{2.5}(n) + \sigma_{\max}r) \\
     &\quad + \norm{\nabla f(F^{t};\tau^{t})}_{\F}^2 \eta^2  \sigma \sqrt{n} r \kappa \log^{2.5}(n) \sqrt{\sigma_{\max}} \sqrt{\sigma_{\max} r} \\
     &\quad + \norm{\nabla f(F^{t};\tau^{t})}_{\F}^2\eta^3 ( \sigma \sqrt{n} r \kappa \log^{2.5}(n) \sqrt{\sigma_{\max}})^2 \\
     &\lesssim \norm{\nabla f(F^{t};\tau^{t})}_{\F}^2 \frac{ \sigma \sqrt{n} r \kappa \log^{2.5}(n) + \sigma_{\max} r}{\kappa^{3}n^{20}\sigma_{\max}} \\
     &\quad +  \norm{\nabla f(F^{t};\tau^{t})}_{\F}^2 \frac{\sigma \sqrt{n} r \kappa \log^{2.5}(n) \sigma_{\max}}{\kappa^{6}n^{40}\sigma_{\max}^2} \\
     &\quad + \norm{\nabla f(F^{t};\tau^{t})}_{\F}^2 \frac{\sigma^2 n r^{2} \kappa^2 \log^{5}(n) \sigma_{\max}}{\kappa^{9} n^{60} \sigma_{\max}^{3}}\\
     &\overset{(ii)}{\lesssim} \frac{\norm{\nabla f(F^{t};\tau^{t})}_{\F}^2}{n^{10}}.
\end{align*}
In (i), we use that $\norm{D}_{\F} \lesssim \sigma \sqrt{n}r^{0.5} \kappa \log^{2.5}(n)$ (shown in \cref{eq:D-bound}). In (ii), we use that $\sSNR$. This implies $\eta A_0  \leq \norm{\nabla f(F^{t}; \tau^{t})}_{\F}^2.$ 

Therefore
\begin{align}
    f(F^{t+1}; \tau^{t}) \leq f(F^{t}; \tau^{t}) - \frac{\eta}{2} \norm{\nabla f(F^{t}; \tau^{t})}_{\F}^2.\nonumber
\end{align}
By our choice of $\tau^{t+1} := \min_{\tau} f(F^{t+1}; \tau)$, we have $f(F^{t+1}; \tau^{t+1}) \leq f(F^{t+1}; \tau^{t}).$ This completes the proof. 
\end{proof}

\section{Technical Lemmas for Theorem \ref{thm:error-rate-theorem}}\label{sec:tech-lemmas-main-theorem}

\begin{lemma}\label{lem:singular-values}
Consider $X, Y \in \R^{n\times r}.$ Suppose there exists two constant $C_{\F}, C_{2}$ such that
\begin{align*}
    \norm{X - X^{*}}_{\F} + \norm{Y - Y^{*}}_{\F} \leq C_{\F} \error \|X^{*}\|_{\F}
\end{align*}
and $\frac{\sigma \sqrt{n}}{\sigma_{\min}} \leq C_2 \frac{1}{\kappa^{0.5} r^{0.5} \log^{3}(n)}.$ Then there exists $N_0$, for $n \geq N_0$, one has for any $i \in [r]$, 
\begin{align*}
    \sigma_{i}(X) &\in \left[\sqrt{\frac{\sigma_{\min}}{2}}, \sqrt{2\sigma_{\max}}\right]\\
    \sigma_{i}(Y) &\in \left[\sqrt{\frac{\sigma_{\min}}{2}}, \sqrt{2\sigma_{\max}}\right]\\
    \sigma_{i}(XY^\top) &\in \left[\frac{\sigma_{\min}}{2}, 2\sigma_{\max}\right].
\end{align*}
\end{lemma}
\begin{proof}
Note that $\norm{X^{*}}_{\F} \leq \norm{X^{*}}\sqrt{r} = \sqrt{\sigma_{\max} r}$. Similarly, $\norm{Y^{*}}_{\F} \leq \sqrt{\sigma_{\max} r}$. Then one has 
\begin{align*}
    \norm{X - X^{*}}_{\F} + \norm{Y - Y^{*}}_{\F} 
    &\leq C_{\F}  \frac{\sigma \sqrt{n}}{ \sigma_{\min}} \log^{2.5}(n)\sqrt{\sigma_{\max} r} \\
    &\leq C_{\F} C_2 \frac{1}{\kappa^{0.5} r^{0.5} \log^{3}(n)}\log^{2.5}(n)\sqrt{\sigma_{\max} r} \\
    &= C_{\F}C_2 \frac{\sqrt{\sigma_{\min}}}{\log^{0.5}(n)}.
\end{align*}
When $n$ goes large enough, $\norm{X - X^{*}}_{\F} \leq \sqrt{\frac{\sigma_{\min}}{8}}, \norm{Y - Y^{*}}_{\F} \leq \sqrt{\frac{\sigma_{\min}}{8}}.$ Then, by Weyl's inequality,
\begin{align}
    \sigma_{i}(X) \leq \sigma_{i}(X^{*}) + \norm{X - X^{*}} \leq \sigma_1(X^{*}) + \norm{X - X^{*}}_{\F} \leq \sqrt{2\sigma_{\max}}. 
\end{align}
We also have $\sigma_{i}(X) \geq \sigma_{i}(X^{*}) - \norm{X - X^{*}} \geq \sigma_{r}(X^{*}) - \norm{X - X^{*}}_{\F} \geq \sqrt{\frac{\sigma_{\min}}{2}}.$ The similar results hold for $\sigma_{i}(Y):$ $\sigma_i(Y) \in \left[\sqrt{\frac{\sigma_{\min}}{2}}, \sqrt{2\sigma_{\max}}\right].$ 

Then consider the bounds for $\sigma_i(XY^{\top})$ where $i \in [r].$ Note that $\sigma_1(XY^{\top}) = \norm{XY^{\top}} \leq \norm{X}\norm{Y} = \sigma_1(X)\sigma_1(Y) \leq 2\sigma_{\max}.$ Furthermore, by the properties of SVD, for any $A \in \R^{n\times m}$ with rank $r$, we have
\begin{align*}
     \sigma_r(A) = \min_{\norm{u}=1, u \in \mathrm{rowspan}(A)}\norm{A u}
\end{align*}
Here, $\mathrm{rowspan}(A)$ is the row space of $A$. This implies
\begin{align*}
    \sigma_r(XY^{\top}) 
    &=  \min_{\norm{u}=1, u \in \mathrm{rowspan}(XY^{\top})} \norm{XY^Tu}\\
    &\overset{(i)}{\geq} \min_{\norm{u}=1, u \in \mathrm{rowspan}(XY^{\top})} \norm{Y^{\top}u} \\
    &\overset{(ii)}{\geq} \sigma_r(X) \cdot \min_{\norm{u}=1, u \in \mathrm{rowspan}(Y^{\top})} \norm{Y^{\top}u}\\
    &\geq \sigma_r(X)\sigma_r(Y)\\
    &\geq \frac{\sigma_{\min}}{2}
\end{align*}
where (i) is due to $Y^{\top}u \in \mathrm{rowspan}(X)$ since $X \in \R^{n\times r}$ has rank-$r$ and (ii) is due to $\mathrm{rowspan}(XY^{\top}) \subset \mathrm{rowspan}(Y^{\top}).$ This completes the proof.  
\end{proof}

\begin{lemma}[Generalization of \cref{lem:applying-the-main-lemma}]\label{lem:general-conditions-small-ball}
Suppose $X, Y \in \R^{n\times r}.$ Let $XY^{\top} = U\Sigma V^{\top}$ be the SVD of $XY^{\top}$. Let $T$ be the tangent space of $XY^{\top}.$ Suppose $\frac{\sigma\sqrt{n}}{\sigma_{\min}} \leq C_2 \frac{1}{\kappa^2 r^{2} \log^{5}(n)}$ and 
\begin{align*}
    \norm{X-X^{*}}_{\F} + \norm{Y-Y^{*}}_{\F} &\leq C_{\F} \frac{\sigma \sqrt{n} \log^{2.5}(n) \sqrt{\sigma_{\max}r}}{\sigma_{\min}}.
\end{align*}
 Assume \cref{cond:Z-condition-nonconvex} holds, then for large enough $n$, 
\begin{align}
    \norm{ZV}_{\F}^2 + \norm{Z^{\top}U}_{\F}^2 &\leq \left(1-\frac{C_{r_1}}{2\log(n)}\right) \norm{Z}_{\F}^2 \label{eq:generalized-condition-1}\\
    \norm{P_{T^{\perp}}(Z)}_{\F}^2 &\geq \frac{C_{r_1}}{2\log(n)} \norm{Z}_{\F}^2 \label{eq:generalized-PTP-Z-ratio}\\
     \norm{P_{T^{\perp}}(Z) - P_{T^{*\perp}}(Z)}_{\F} &\lesssim \frac{\sigma r\kappa^2\sqrt{n} \log^{2.5}(n)}{\sigma_{\min}} \norm{Z}_{\F} \label{eq:PTP-PTPstar-Z-F}\\
     \norm{P_{T^{\perp}}(Z) - P_{T^{*\perp}}(Z)}_{*} &\lesssim \frac{\sigma r^{1.5}\kappa^2\sqrt{n} \log^{2.5}(n)}{\sigma_{\min}} \norm{Z}_{\F}. \label{eq:PTP-PTPstar-Z}
\end{align}
In addition, assume \cref{cond:Z-condition-convex} holds, then
\begin{align}
    \left|\inner{Z}{UV^{\top}}\right| \norm{P_{T^{\perp}}(Z)} &\leq \left(1-\frac{C_{r_2}}{2\log n}\right)  \norm{P_{T^{\perp}}(Z)}_{\F}^2.\label{eq:generalized-condition-2}
\end{align}
\end{lemma}
\begin{proof}
To begin, we have
\begin{align*}
\|XY^{\top} - X^{*}Y^{*\top}\|_{\F} 
&\leq \|X-X^{*}\|_{\F} \|Y^{*}\| +  \|Y-Y^{*}\|_{\F} \|X\|\\
&\lesssim \sigma \sqrt{n} \log^{2.5}(n) \kappa \sqrt{r}
\end{align*}
where the bound on $\|X\|$ is due to \cref{lem:singular-values}. The by Lemma B.2 in \cite{chen2019noisy} (a variant of Davis-Kahan theorem) \footnote{It is straightforward to extend to rectangle matrices by ``symmetric dilation'' technique \cite{chen2019noisy}.}, there exists rotation matrices $R \in \R^{r\times r}$ such that
\begin{align}
\|UR - U^{*}\|_{\F} + \|VR - V^{*}\|_{\F} \lesssim \frac{\sigma \sqrt{n} \log^{2.5}(n) r \kappa}{\sigma_{\min}}.
\end{align}

This implies 
\begin{align*}
    \norm{UU^{\top} - U^{*}U^{*\top}}_{\F} 
    &= \norm{UR(UR)^{\top} - U^{*}U^{*\top}}_{\F}\\
    &\leq \norm{UR - U^{*}}_{\F}\norm{UR} + \norm{U^{*}}\norm{UR - U^{*}}_{\F}\\
    &\leq 2\norm{UR - U^{*}}_{\F} \\
    &\lesssim  \frac{\sigma \sqrt{n} \log^{2.5}(n) r \kappa}{\sigma_{\min}}.
\end{align*}
Similarly, we have 
\begin{align}
\norm{VV^{\top} - V^{*}V^{*\top}}_{\F} &\lesssim  \frac{\sigma \sqrt{n} \log^{2.5}(n) r \kappa}{\sigma_{\min}}.  \nonumber\\
\norm{UV^{\top} - U^{*}V^{*\top}}_{\F} &\lesssim  \frac{\sigma \sqrt{n} \log^{2.5}(n) r \kappa}{\sigma_{\min}}. \label{eq:UV-UVstar-bound}
\end{align}

Then, consider the proof of \cref{eq:generalized-condition-1}. We have
\begin{align*}
    &\norm{ZV}_{\F}^2 + \norm{Z^{\top}U}_{\F}^2\\
    &= \tr(Z^{\top}ZVV^{\top}) + \tr(ZZ^{\top}UU^{\top})\\
    &= \tr(Z^{\top}ZV^{*}V^{*\top}) + \tr(Z^{\top}Z(VV^{\top} - V^{*}V^{*\top})) \\
    &\quad + \tr(ZZ^{\top}U^{*}U^{*\top}) - \tr(ZZ^{\top}(UU^{\top}-U^{*}U^{*\top}))\\
    &\overset{(i)}{\leq} \norm{ZV^{*}}_{\F}^2 + \norm{Z^{\top}U^{*}}_{\F}^2 + \norm{Z^{\top}Z}_{*} \norm{VV^{\top}-V^{*}V^{*\top}} + \norm{ZZ^{\top}}_{*} \norm{UU^{\top} - U^{*}U^{*\top}}\\
    &\overset{(ii)}{\leq} \left(1-\frac{C_{r_1}}{\log(n)}\right) \norm{Z}_{\F}^2 + \norm{Z}_{\F}^2 \frac{\sigma \sqrt{n} \log^{2.5}(n) r \kappa}{\sigma_{\min}}\\
    &\leq \left(1-\frac{C_{r_1}}{2\log(n)}\right) \norm{Z}_{\F}^2
\end{align*}
where (i) is due to $\tr(AB) \leq \norm{A}_{*}\norm{B}$, (ii) is due to $\norm{ZV^{*}}_{\F}^2 + \norm{Z^{\top}U^{*}}_{\F}^2 \leq \left(1-\frac{C_{r_1}}{\log(n)}\right) \norm{Z}_{\F}^2$ and $\norm{Z^{\top}Z}_{*} = \norm{ZZ^{\top}}_{*} = \norm{Z}_{\F}^2$.

Next, consider the proof of \cref{eq:generalized-PTP-Z-ratio}. This is in fact the simple implication of \cref{eq:generalized-condition-1}. Note that
\begin{align}
\norm{P_{T}(Z)}_{\F}^2 
    &= \norm{Z - (I-UU^{\top})Z(I-VV^{\top})}_{\F}^2\nonumber\\
    &= \norm{UU^{\top}Z + (I - UU^{\top})Z(VV^{\top})}_{\F}^2\nonumber\\
    &= \tr\left(UU^{\top}ZZ^{\top}UU^{\top}\right) + \tr\left((I - UU^{\top})Z(VV^{\top})(VV^{\top})Z^{\top}(I - UU^{\top})\right)\nonumber\\
    &= \tr\left(UU^{\top}ZZ^{\top}\right) + \tr\left(Z^{\top}ZVV^{\top}\right) - \tr\left(UU^{\top}ZVV^{\top}Z^{\top}\right)\nonumber\\
    &\leq \norm{Z^{\top}U}_{\F}^2 + \norm{ZV}_{\F}^2 \leq \left(1-\frac{C_{r_1}}{2\log n}\right) \norm{Z}_{\F}^2. \label{eq:PTZ-F-bound}
\end{align}
This implies that $\frac{C_{r_1}}{2\log n}\norm{Z}_{\F}^2 \leq \norm{Z}_{\F}^2 - \norm{P_{T}(Z)}_{\F}^2 $. Note that $\norm{P_{T^{\perp}}(Z)}_{\F}^2 + \norm{P_{T}(Z)}_{\F}^2 = \norm{Z}_{\F}^2$, which completes the proof of \cref{eq:generalized-PTP-Z-ratio}.

Next, consider the proof of \cref{eq:PTP-PTPstar-Z}. Let $\Delta = \PTp{Z} - P_{T^{*\perp}}(Z)$. We have
\begin{align}
    \norm{\Delta}_{*} 
    &= \norm{(I-UU^{\top})Z(I-VV^{\top}) - (I - U^{*}U^{*\top})Z(I - V^{*}V^{*\top})}_{*}\nonumber\\
    &\leq \norm{(I-UU^{\top})Z(I-VV^{\top}) - (I-U^{*}U^{*\top})Z(I-VV^{\top})}_{*} \nonumber\\
    &\quad + \norm{(I-U^{*}U^{*\top})Z(I-VV^{\top}) - (I - U^{*}U^{*\top})Z(I - V^{*}V^{*\top})}_{*}\nonumber\\
    &\leq \norm{(U^{*}U^{*\top}-UU^{\top})Z(I-VV^{\top})}_{*} +  \norm{(I-U^{*}U^{*\top})Z(V^{*}V^{*\top}-VV^{\top})}_{*}\nonumber\\
    &\overset{(i)}{\leq} \norm{(U^{*}U^{*\top}-UU^{\top})Z(I-VV^{\top})}_{\F} \sqrt{\text{rank}(U^{*}U^{*\top}-UU^{\top})} \nonumber\\
    &\quad + \norm{(I-U^{*}U^{*\top})Z(V^{*}V^{*\top}-VV^{\top})}_{\F} \sqrt{\text{rank}(V^{*}V^{*\top}-VV^{\top})} \nonumber\\
    &\leq \norm{Z}_{\F}\left(\norm{(U^{*}U^{*\top}-UU^{\top})} + \norm{(V^{*}V^{*\top}-VV^{\top})}\right) \sqrt{2r}\nonumber\\
    &\lesssim    \frac{\sigma \sqrt{n} \log^{2.5}(n) r^{1.5} \kappa}{\sigma_{\min}} \norm{Z}_{\F}. \nonumber
\end{align}
Here, (i) is due to $\norm{A}_{*} \leq \norm{A}_{\F}\sqrt{\rank(A)}$ and $\rank(AB) \leq \rank(A)$ for any matrices $A$ and $B$. This completes the proof of \cref{eq:PTP-PTPstar-Z}. Next, we prove \cref{eq:PTP-PTPstar-Z-F} by the following
\begin{align}
    \norm{\Delta}_{\F} 
    &\leq \norm{(I-UU^{\top})Z(I-VV^{\top}) - (I - U^{*}U^{*\top})Z(I - V^{*}V^{*\top})}_{\F}\nonumber\\
    &\leq \norm{(U^{*}U^{*\top}-UU^{\top})Z(I-VV^{\top})}_{\F} +  \norm{(I-U^{*}U^{*\top})Z(V^{*}V^{*\top}-VV^{\top})}_{\F}\nonumber\\
    &\leq \norm{Z}_{\F} \left(\norm{(U^{*}U^{*\top}-UU^{\top})} + \norm{(V^{*}V^{*\top}-VV^{\top})}\right)\nonumber\\
    &\lesssim   \frac{\sigma \sqrt{n} \log^{2.5}(n) r \kappa}{\sigma_{\min}}\norm{Z}_{\F}. \label{eq:Delta-bound}
\end{align}

Finally, consider the proof of \cref{eq:generalized-condition-2}. Recall the \cref{cond:Z-condition-convex} provides the following.
\begin{align}
\frac{C_{r_2}}{\log(n)}\norm{P_{T^{*\perp}}(Z)}_{\F}^2 \leq \norm{P_{T^{*\perp}}(Z)}_{\F}^2 -  \left|\inner{Z}{U^{*}V^{*\top}}\right| \norm{P_{T^{*\perp}}(Z)} .\label{eq:technical-lemma-condition-2}
\end{align}
Consider substituting $\PTp{Z}$ with $\PTp{Z} = P_{T^{*\perp}}(Z) + \Delta$ in $\norm{\PTp{Z}}_{\F}^2$, when $\frac{C_{r_2}}{2\log(n)} \leq 1$, we have
\begin{align}
    &\left(1-\frac{C_{r_2}}{2\log(n)}\right)\norm{\PTp{Z}}_{\F}^2 \nonumber\\
    &= \left(1-\frac{C_{r_2}}{2\log(n)}\right)\norm{P_{T^{* \perp}}(Z)}_{\F}^2 + \left(1-\frac{C_{r_2}}{2\log(n)}\right)\norm{\Delta}_{\F}^2 + \left(1-\frac{C_{r_2}}{2\log(n)}\right)2\inner{P_{T^{* \perp}}(Z)}{\Delta} \nonumber\\
    &\overset{(i)}{\geq} \norm{P_{T^{* \perp}}(Z)}_{\F}^2 - \frac{C_{r_2}}{2\log(n)}\norm{P_{T^{* \perp}}(Z)}_{\F}^2 - 2\norm{P_{T^{* \perp}}(Z)}_{\F} \norm{\Delta}_{\F} \nonumber\\
    &\overset{(ii)}{\geq} \norm{P_{T^{* \perp}}(Z)}_{\F}^2 - \frac{C_{r_2}}{2\log(n)}\norm{P_{T^{* \perp}}(Z)}_{\F}^2 - \norm{Z}_{\F}^2 \frac{4C}{r\log^{2.5}(n)} \label{eq:1-Cr2-bound}
\end{align}
where (i) is due to $\norm{\Delta}_{\F}^2 \geq 0$ and $|\inner{A}{B}| \leq \norm{A}_{\F}\norm{B}_{\F}$, (ii) is due to \cref{eq:Delta-bound}.

Next, consider $\left|\inner{Z}{UV^{\top}}\right|\norm{\PTp{Z}}.$ Let $\Delta' = UV^{\top} - U^{*}V^{*\top}.$ Then, we have
\begin{align}
    &\left|\inner{Z}{UV^{\top}}\right|\norm{\PTp{Z}} \nonumber\\
    &= \left|\inner{Z}{U^{*}V^{*\top}+\Delta'}\right|\norm{P_{T^{* \perp}}(Z) + \Delta} \nonumber\\
    &\leq \left|\inner{Z}{U^{*}V^{*\top}}\right|\norm{P_{T^{*\perp}}(Z)} + \left|\inner{Z}{\Delta'}\right|\norm{P_{T^{*\perp}}(Z)} + \left|\inner{Z}{U^{*}V^{*\top}}\right|\norm{\Delta} + \left|\inner{Z}{\Delta'}\right|\norm{\Delta} \nonumber\\
    &\leq \left|\inner{Z}{U^{*}V^{*\top}}\right|\norm{P_{T^{*\perp}}(Z)} + \norm{Z}_{\F} \norm{\Delta'}_{\F} \norm{P_{T^{*\perp}}(Z)} \nonumber\\
    &\quad + \norm{Z}_{\F} \norm{U^{*}V^{*\top}}_{\F} \norm{\Delta} + \norm{Z}_{\F}\norm{\Delta'}_{\F}\norm{\Delta} \nonumber\\
    &\leq \left|\inner{Z}{U^{*}V^{*\top}}\right|\norm{P_{T^{*\perp}}(Z)} + \norm{Z}_{\F}^2 \norm{\Delta'}_{\F} + \norm{Z}_{\F} \sqrt{r} \norm{\Delta}_{\F} + \norm{Z}_{\F}\norm{\Delta'}_{\F}\norm{\Delta} \nonumber\\
    &\overset{(i)}{\leq} \left|\inner{Z}{U^{*}V^{*\top}}\right|\norm{P_{T^{*\perp}}(Z)} + \norm{Z}_{\F}^2 \left( \frac{C}{r\log^{2.5}(n)} + \frac{C}{r^{0.5}\log^{2.5}(n)} + \frac{C^2}{r^2\log^3(n)}\right) \nonumber.
\end{align}
where (i) is due to \cref{eq:Delta-bound} and \cref{eq:UV-UVstar-bound}. This implies
\begin{align}
\left|\inner{Z}{UV^{\top}}\right|\norm{\PTp{Z}} &\leq \left|\inner{Z}{U^{*}V^{*\top}}\right|\norm{P_{T^{*\perp}}(Z)} + \norm{Z}_{\F}^2 \frac{3C^2}{r^{0.5}\log^{2.5}(n)} \label{eq:ZUVT-bound}
\end{align}

Combining \cref{eq:1-Cr2-bound} and \cref{eq:ZUVT-bound}, we have
\begin{align*}
    &\left(1-\frac{C_{r_2}/2}{\log(n)}\right)\norm{\PTp{Z}}_{\F}^2 - \left|\inner{Z}{UV^{\top}}\right|\norm{\PTp{Z}} \nonumber\\
    &\geq \norm{P_{T^{* \perp}}(Z)}_{\F}^2 - \frac{C_{r_2}}{2\log(n)}\norm{P_{T^{* \perp}}(Z)}_{\F}^2 - \norm{Z}_{\F}^2 \frac{4C}{r\log^{2.5}(n)} \\
    &\quad - \left|\inner{Z}{U^{*}V^{*\top}}\right|\norm{P_{T^{*\perp}}(Z)} - \norm{Z}_{\F}^2 \frac{3C^2}{r^{0.5}\log^{2.5}(n)}\\
    &\overset{(i)}{\geq} \frac{C_{r_2}}{\log n}\norm{P_{T^{*\perp}}{Z}}_{\F}^2 - \frac{C_{r_2}}{2\log n}\norm{P_{T^{*\perp}}{Z}}_{\F}^2 - \norm{Z}_{\F}^2 \frac{4C}{r\log^{2.5}(n)}  - \norm{Z}_{\F}^2 \frac{3C^2}{r^{0.5}\log^{2.5}(n)}\nonumber\\
    &\overset{(ii)}{\geq} \frac{C_{r_2}C_{r_1}}{4\log^2 n} \norm{Z}_{\F}^2 -  \frac{3C^2+4C}{\log^{2.5}(n)}\norm{Z}_{\F}^2\nonumber\\
    &\overset{(iii)}{\geq} 0\nonumber
\end{align*}
where (i) is by \cref{eq:technical-lemma-condition-2}, (ii) is by $\norm{P_{T^{*\perp}}(Z)}_{\F}^2 \geq \frac{C_{r_1}}{2\log(n)} \norm{Z}_{\F}^2$ due to the analysis in \cref{eq:PTZ-F-bound}, and (iii) holds for large enough $n$. This implies $\left|\inner{Z}{UV^{\top}}\right|\norm{\PTp{Z}} \leq \left(1-\frac{C_{r_2}/2}{\log(n)}\right)\norm{\PTp{Z}}_{\F}^2$, which completes the proof. 

\end{proof}

\XYUVbalance*
\begin{proof}
Since $\rank(XY^{\top}) \geq \rank(X)+\rank(B) - r = r$ by rank inequality, we have $\Sigma \succ 0.$ Construct $Q := \Sigma^{-1/2}U^{\top}X$ (hence $Q \in \R^{r\times r}$). One can verify that
\begin{align}
U\Sigma^{1/2}Q = UU^{\top}X \overset{(i)}{=} X  \label{eq:USigmaQ}
\end{align}
where (i) is due to that $UU^{\top}$ is a projection matrix and the column space of $X$ is the same as the column space of $U$. From \cref{eq:USigmaQ}, one can also conclude that $\rank(Q) \geq \rank(X) = r$, i.e., $Q$ is invertible. Combining \cref{eq:USigmaQ} with $XY^{\top}=U\Sigma V^{\top}$, we have
\begin{align}
(U\Sigma^{1/2}Q)Y^{\top} = U\Sigma V^{\top} \implies Y = V\Sigma^{1/2}Q^{-\top}. \label{eq:VSigmaQ}
\end{align}
Using \cref{eq:USigmaQ,eq:VSigmaQ} to replace $X$ and $Y$ in $X^{\top}X = Y^{\top}Y$, one can obtain
\begin{align}
(U\Sigma^{1/2} Q)^{\top} (U\Sigma^{1/2} Q) &= (V\Sigma^{1/2} Q^{-\top})^{\top} (V\Sigma^{1/2} Q^{-\top}). \label{eq:UtopU}
\end{align}
Using $U^{\top}U=V^{\top}V=I_{r}$, \cref{eq:UtopU} can be simplified to 
\begin{align}
Q^{\top} \Sigma Q &= Q^{-1}\Sigma Q^{-\top}.\label{eq:QtopSigmaQ}
\end{align}
We hope from \cref{eq:QtopSigmaQ} establish that $Q = Q^{-\top}.$ To see this, suppose $Q = \sum_{i=1}^{r} \tilde{\sigma}_{i} \tilde{u}_i \tilde{v}_i^{\top}$ be the vector-form SVD of $Q$, where $\{\tilde{u}_{i}, i\in[r]\}$ ($\{\tilde{v}_{i}, i\in[r]\}$) are orthonormal singular vectors, and $\tilde{\sigma}_{1}\geq \tilde{\sigma}_2 \geq \dotsc > 0$ are singular values. We then have $Q^{-\top} = \sum_{i=1}^{r} \tilde{\sigma}^{-1}_{i} \tilde{u}_i \tilde{v}_i^{\top}$. Note that $Q^{\top} \Sigma Q = Q^{-1}\Sigma Q^{-\top} \implies \|\Sigma^{1/2}Qa\| = \|\Sigma^{1/2} Q^{-\top} a\|$ for any $a \in \R^{r}.$ Take $a = \tilde{v}_i, i \in [r]$, we then have
\begin{align*}
\tilde{\sigma}_i\|\Sigma^{1/2} \tilde{u}_i\| = \tilde{\sigma}_i^{-1} \|\Sigma^{1/2} \tilde{u}_i\|,
\end{align*}
from which one can establish that $\tilde{\sigma}_i = \tilde{\sigma}_i^{-1}$ since $\|\Sigma^{1/2} \tilde{u}_i\| > 0$ by $\Sigma \succ 0.$ This implies $Q=Q^{-\top}$, i.e., $Q$ is a rotation matrix. This completes the proof. 
\end{proof}

\section{Proof of \cref{lem:adaptive-noise}}
We restate the lemma by dividing it into the following two statements. 
\begin{lemma}\label{lem:generative-norm-E}
Under the model in \cref{def:model}, with probability $1-O(1/n^{C})$, 
\begin{align*}
\|E\| \lesssim \sigma \sqrt{n}. 
\end{align*} 
\end{lemma}

\begin{lemma}\label{lem:generative-avereage-E}
Under the model in \cref{def:model}, with probability $1-O(1/n^{C})$, 
\begin{align*}
\frac{\left|\inner{Z}{E}\right|}{\|Z\|_{\F}} \lesssim \sigma \log(n). 
\end{align*} 
\end{lemma}
To begin, we will state the following lemma borrowed from `self-normalized bound' in the theory of linear bandit \citep{abbasi2011improved}. 
\begin{lemma}\label{lem:one-D-bandit}
Let $\{F_{t}\}_{t=0}^{\infty}$ be a filtration. Let $\{e_{t}\}_{t=1}^{\infty}$ be a real-valued stochastic process such that $e_{t}$ is $F_{t}$-measurable and $e_{t}$ is conditionally $\sigma$-sub-Gaussian, i.e., $\E{\exp(\lambda e_{t}) ~|~F_{t-1}} \leq \exp(\lambda^2 \sigma^2 / 2), \forall \lambda.$ 

Let $\{z_{t}\}_{t=1}^{\infty}$ be a real-valued stochastic process such that $z_{t}$ is $F_{t-1}$-measurable. Assume that $V > 0$ is a positive number. Then for any $\delta > 0$, with probability at least $1-\delta$, for all $t \geq 0$, 
\begin{align*}
\left(\sum_{s=1}^{t}z_{t}e_{t}\right)^2 \leq 2\sigma^2 \left(V+\sum_{s=1}^{t}z_{t}^2\right) \log \left(\sqrt{V+\sum_{s=1}^{t}z_{t}^2} \sqrt{1/V} \frac{1}{\delta} \right).
\end{align*} 
\end{lemma}
\begin{proof}
This is a special case of Theorem 1 in \citep{abbasi2011improved}, restricted to the setting where features are one-dimensional.  
\end{proof}

\subsection{Proof of Lemma \ref{lem:generative-norm-E}}
We follow a similar $\epsilon$-net argument of Theorem 4.4.5 in \citep{vershynin2018high}, which proves a norm bound for random matrices with independent sub-Gaussian entries. 

The only thing different here is that $E_{it}$ can depend on the historical noises, this is where \cref{lem:one-D-bandit} is helpful. Consider two fixed unit vectors $x \in \R^{n}, y \in \R^{T}$ with $\|x\| = 1, \|y\|=1.$ We need to bound 
\begin{align*}
\sum_{it} E_{it} x_{i} y_{t}.
\end{align*}
To do so, let $e_{t} = \sum_{i} E_{it} x_{i}.$ Conditioned on the historical noises $\{E_{is}~|~i \in [n], s<t\}$, $E_{it}$ are independent from each other. Hence, $e_{t}$ is conditionally $\sigma$-sub-Gaussian. Then, we invoke \cref{lem:one-D-bandit} with $V=1$ and the following can be obtained: with probability $1-O(1/n^{C})$, 
\begin{align*}
\left|\sum_{t} e_{t} y_{t}\right|^2 \lesssim  \sigma \|y\|^2 \log(n) = \sigma \log(n)
\end{align*}
With this bound, we can then follow the same proof in Theorem 4.4.5 of \citep{vershynin2018high} and arrive at, with probability $1-O(1/n^{C})$, 
\begin{align*}
\|E\| \lesssim \sigma \sqrt{n}.
\end{align*} 

\subsection{Proof of Lemma \ref{lem:generative-avereage-E}}
Note that $\inner{Z}{E} = \sum_{t=1}^{T}  \sum_{i=1}^{n} Z_{it} E_{it}.$ Consider the sequence $\{Z_{it}, E_{it}\}$ ordered by pair indices (i.e., $(i_1,t_1) < (i_2, t_2)$ if $t_1 < t_2$ or $t_1 = t_2$ and $i_1 < i_2$).

We can then apply \cref{lem:one-D-bandit} directly to obtain, with probability $1-O(1/n^{C})$\footnote{We notice that if $Z=0$, the bound is trivial; when $Z \neq 0$, $\|Z\|_{\F} \geq 1$. Hence we take $V = 1$ when applying \cref{lem:one-D-bandit}.}, 
\begin{align*}
\left(\sum_{t=1}^{T}\sum_{i=1}^{n} Z_{it}E_{it} \right)^2 \lesssim \sigma^2 \left(\sum_{t=1}^{T} \sum_{i=1}^{n}Z_{it}^2\right) \log(n).
\end{align*}
This then implies
\begin{align*}
\frac{\left|\inner{Z}{E}\right|}{\|Z\|_{\F}} \lesssim \sigma \log^{0.5}(n) 
\end{align*}
which finishes the proof. 

\section{Proof of Lemma \ref{thm:non-convex-theorem}}\label{sec:proof-non-convex-iteration}\label{sec:full-non-convex-proof}
The proof of \cref{thm:non-convex-theorem} is similar to \cref{lem:induction-Frobenious-appendix}, but with significant refined analyses on controlling $l_{2,\infty}$-norm error.

Following the notations in the proof of \cref{lem:induction-Frobenious-appendix}, by the incoherence condition of $M^{*}$, we have the following facts for $F^{*}$.
\begin{align}
\norm{F^{*}}_{2,\infty} &= \max\{\norm{X^{*}}_{2,\infty}, \norm{Y^{*}}_{2,\infty}\} \leq \sqrt{\mu r\sigma_{\max}/n}.\label{eq:F-incoherence}
\end{align}

Before proceeding, inspired by the leave-one-out technique developed in \cite{ma2019implicit,abbe2017entrywise,chen2019noisy}, we introduce a set of auxiliary loss functions to facilitate the analysis of the gradient descent algorithm. In particular, for $1\leq l \leq 2n$, let
\begin{align}
    f^{(l)}(X, Y; \tau) := \frac{1}{2}\norm{XY^{\top} + \tau Z + P_{l}(E) - O}_{\F}^2+ \frac{\lambda}{2}\norm{X}_{\F}^2 + \frac{\lambda}{2} \norm{Y}_{\F}^2 \nonumber
\end{align}
where $P_{l}$ is defined as the following
\begin{align}
    [P_{l}(B)]_{ij} =\begin{cases} B_{ij} & i = l\\ 0 & i \neq l\end{cases} \quad 1\leq l\leq n \text{\quad and \quad } [P_{l}(B)]_{ij} =\begin{cases} B_{ij} & j = l - n\\ 0 & j \neq l - n\end{cases} \quad n+1 \leq l \leq 2n.\nonumber
\end{align}
Intuitively, $f^{(l)}$ removes the $l$-th row (or $(l-n)$-th column for $l > n$) of the noise matrix $E$. We then consider the following gradient descent algorithm on the loss function $f^{(l)}$ for $1\leq l \leq 2n.$

\begin{algorithm}[H]
\caption{Leave-one-out Gradient Descent of non-convex optimization} \label{alg:GD-leave-one-out}
{\bf Input:} the observation $O$ and $Z$
\begin{algorithmic}[1]

\State{\textbf{{Initialization}}: $X^{0, (l)} = X^{*}, Y^{0, (l)}=Y^{*}, \tau^{0, (l)} = \frac{\inner{Z}{O-P_{l}(E)-X^{*}Y^{*\top}}}{\norm{Z}_{\F}^2}$.}

\State{\textbf{{Gradient updates}}: \textbf{for }$t=0,1,\ldots,t_{\star}-1$
\textbf{do}
 \begin{subequations}
\begin{align}
X^{t+1, (l)}= & X^{t, (l)}-\eta[(X^{t, (l)}(Y^{t, (l)})^\top + \tau^{t, (l)} Z + P_{l}(E)- O)Y^{t,(l)}+\lambda X^{t, (l)}];\\
Y^{t+1, (l)}= & Y^{t, (l)}-\eta[(X^{t, (l)}(Y^{t,(l)})^{\top} + \tau^{t, (l)} Z + P_{l}(E)- O)^{\top}X^{t, (l)}+\lambda Y^{t, (l)}]\\
\tau^{t+1,(l)} = & \frac{\inner{Z}{O-P_{l}(E)-X^{t+1,(l)}Y^{(t+1, (l))\top}}}{\norm{Z}_{\F}^2}.
\end{align}
where $\eta$ determines the learning rate. 
\end{subequations}
}
\end{algorithmic}
\end{algorithm}

Let $F^{t,(l)} := \begin{bmatrix} X^{t,(l)} \\ Y^{t,(l)} \end{bmatrix} \in \R^{2n \times r}$. Similar to the definition of $H^{t}$ in \cref{eq:rotation-H}, we define the optimal rotation for aligning $F^{t, (l)}$ and $F^{*}$ by the following
\begin{align}
    H^{t,(l)} &:= \arg\min_{A \in \O^{r\times r}} \norm{F^{t,(l)}A - F^{*}}_{\F}. \nonumber
\end{align}
Furthermore, we are also interested in analyzing the relation between $F^{t, (l)}$ and $F^{t}$. Therefore, we introduce the following rotation to align $F^{t, (l)}$ and $F^{t}$. 
\begin{align}
    R^{t,(l)} &:= \arg\min_{A \in \O^{r\times r}} \norm{F^{t,(l)}A - F^{t}H^{t}}_{\F}.\nonumber
\end{align}

Next, consider $\tau^{t} = \frac{\inner{Z}{O-X^{t}Y^{tT}}}{\norm{Z}_{\F}^2}$. Use $O = M^{*} + \tau^{*}Z + E$, we have
\begin{align*}
\tau^{t}  = \frac{\inner{Z}{M^{*} - X^{t}Y^{tT}}}{\norm{Z}_{\F}^2} + \frac{\inner{Z}{E}}{\norm{Z}_{\F}^2} + \tau^{*}.
\end{align*}
To analyze $\tau^{t} - \tau^{*}$, the main difficulty is to bound the error $\frac{\inner{Z}{M^{*} - X^{t}Y^{tT}}}{\norm{Z}_{\F}^2}$ (note that $\inner{Z}{E}/\norm{Z}_{\F}^2$ keeps the same for each iteration, hence easier to analyze) To establish the result, we introduce the ``component separation'' idea by analyzing each one of $r$ components of the error separately. In particular, note that $M^{*} = X^{*}Y^{*\top}.$ Let $x^{*}_s, y^{*}_s$ be the $s$-column of $X^{*}$ and $Y^{*}$ respectively. Then
\begin{align*}
    M^{*} = \sum_{s=1}^{r} x^{*}_s y^{*\top}_s.
\end{align*}
In addition, note that $X^{t}Y^{tT} = (X^{t}H^{t})(Y^{t}H^{t})^{\top}.$ Let $(X^{t}H^{t})_{s}, (Y^{t}H^{t})_{s}$ be the $s$-th column of $X^{t}H^{t}$ and $Y^{t}H^{t}$ respectively. Then
\begin{align*}
    X^{t}Y^{tT} = \sum_{s=1}^{r} (X^{t}H^{t})_{s} (Y^{t}H^{t})_{s}^{\top}
\end{align*}
We introduce the ``component error'' $\Delta^{t}_s$ defined as the following 
\begin{align}
    \Delta^{t}_{s} &:= \frac{\inner{Z}{x_s^{*}y_s^{*\top} - (X^{t}H^{t})_s(Y^{t}H^{t})_s^{\top}}}{\norm{Z}_{\F}^2}. \label{eq:def-Deltats}
\end{align}
It is easy to see the connection between $\tau^{t} - \tau^{*}$ and $\Delta^{t}_s$ through the following. 
\begin{align}
    \tau^{t} - \tau^{*} - \frac{\inner{Z}{E}}{\norm{Z}_{\F}^2} = \frac{\inner{Z}{M^{*} - X^{t}Y^{tT}}}{\norm{Z}_{\F}^2} = \sum_{s=1}^{r} \Delta^{t}_s. \label{eq:delta-t-s-tau}
\end{align}
In order to analyze the bounds for $\Delta^{t}_s$, we introduce the ``component coefficient'' $b_s$. In particular, let $u^{*}_s, v^{*}_s$ be the $s$-column of $U^{*}$ and $V^{*}$
respectively. We denote 
\begin{align}
    b_s &:= \frac{\tr(Z^{\top}Zv_{s}^{*}v_{s}^{*\top})+\tr(ZZ^{\top}u_{s}^{*}u_s^{*\top})}{\norm{Z}_{\F}^2} \label{eq:def-bs}
\end{align} 
One shall see the connection between $b_s$ and \cref{cond:Z-condition-nonconvex}. In fact, \cref{cond:Z-condition-nonconvex} implies that, for some constant $C$, 
\begin{align*}
    \sum_{s=1}^{r} b_s \leq 1-\frac{C}{\log(n)}
\end{align*}
due to that $\sum_{s=1}^{r} v_{s}^{*}v_{s}^{*\top} = V^{*}V^{*\top}, \sum_{s=1}^{r} u_{s}^{*}u_{s}^{*\top} = U^{*}U^{*\top}.$ The usefulness of $b_s$ will be evident in bounding the component error $\Delta^{t}_s$ in \cref{sec:proof-taus}.


We use the mathematical induction to obtain the desired results. In particular, we aim to prove the following lemma.  
\begin{lemma}\label{lem:induction}
Suppose \cref{cond:Z-condition-nonconvex} hold. Suppose $O = M^{*} + \tau^{*}Z + E$ where $E_{ij}$ are independent sub-Gaussian random variables with $\|E_{ij}\|_{\psi_2} \leq \sigma.$ Suppose $\kappa^{4}\mu^2 r^2 \log^2(n) \lesssim n$, $\SNR$. Let $\lambda = C_{\lambda} \sigma \sqrt{n} \log^{1.5}(n), \eta = \frac{C_{\eta}}{\kappa^3 n^{20} \sigma_{\max}}$.  Suppose the following induction hypotheses \cref{eq:induction} for $q=t$ hold with (pre-determined) constants $C_{\tau}, C_{\F}, C_{l,1}, C_{l,2}, C_{\infty}$. Then, with probability at least $1-O(n^{-10^{6}})$, we have the \cref{eq:induction} for $q=t+1$ holds.
\begin{subequations}\label{eq:induction}
\begin{align}
    |\Delta_s^{q}| &\leq C_{\tau} \left(C_0 + b_s r\log(n)\right) \left(\frac{\sigma \mu r^{1.5}\kappa\log^{2.5}(n)}{\sqrt{n}} + \frac{ \sigma 
    \log^{0.5}(n)}{r\norm{Z}_{\F}}\right)\label{eq:taus-taustars}\\
    |\tau^{q} - \tau^{*}| &\leq C_{\tau} \left( \frac{\sigma \mu r^{2.5} \kappa \log^{3.5} n}{\sqrt{n}} + \frac{\sigma \log^{1.5} n}{\norm{Z}_{\F}}\right)\label{eq:tau-taustar}\\
    \norm{F^{q}H^{q} - F^{*}}_{\F} &\leq C_{\F} \error \norm{F^{*}}_{\F} \label{eq:FtHt-Fstar-Fnorm}\\
    \max_{1\leq l\leq 2n}\norm{F^{q}H^{q} - F^{q, (l)} R^{q, (l)}}_{\F} &\leq C_{l,1} \frac{\sigma \log^{1.5} n}{\sigma_{\min}} \norm{F^{*}}_{\F} \label{eq:FH-FlRl}\\
    \max_{1\leq l\leq 2n} \norm{(F^{q, (l)}H^{q, (l)} - F^{*})_{l, \cdot}}_{2} &\leq C_{l,2} \frac{\sigma \mu r^{2.5} \kappa \log^{3.5}(n)}{\sigma_{\min}}\norm{F^{*}}_{\F} \label{eq:FlHl-Fstar}\\
    \norm{F^{q}H^{q} - F^{*}}_{2,\infty} &\leq C_{\infty}\frac{\sigma \mu r^{2.5} \kappa \log^{3.5}(n)}{\sigma_{\min}}\norm{F^{*}}_{\F} \label{eq:FH-Fstar-rownorm}\\
    \norm{X^{q \top} X^{q} - Y^{q \top}Y^{q}}_{\F} &\leq C_{B} \frac{\sigma}{\kappa n^{15}} \label{eq:XX-YY}.
\end{align}
\end{subequations}
Furthermore, let $T_t$ be the tangent space of $X^{t}Y^{tT}.$ If \cref{eq:induction} holds for $q = t$, the following inequality holds with probability $1-O(n^{-10^{6}})$ for (pre-determined) constants $C_{T_1}, C_{T,2}, C_{T,3}$. 
\begin{subequations}\label{eq:PT-bound}
\begin{align}
    \norm{P_{T_t^{\perp}}(M^{*})}_{\infty} &\leq  C_{T,1}\frac{\sigma^2 \mu^2 r^{6} \kappa^{3} \log^{7}(n)}{\sigma_{\min}}\label{eq:PTM-bound}\\
    \norm{P_{T_t^{\perp}}(E) - P_{T^{*\perp}}(E)}_{\infty} &\leq C_{T,2}\frac{\sigma^2 r^{3.5} \mu^{1.5} \kappa^{2.5} \log^{4}(n)}{\sigma_{\min}}\label{eq:PTE-bound}\\
    \norm{P_{T_t^{\perp}}(Z)-P_{T^{*\perp}}(Z)}_{\F} &\leq C_{T,3}\frac{r\kappa \sigma \sqrt{n}\log^{2}(n)}{\sigma_{\min}} \norm{Z}\label{eq:PTZ-bound}\\
    f(X^{q+1}, Y^{q+1}; \tau^{q+1}) &\leq f(X^{q}, Y^{q}; \tau^{q}) - \frac{\eta}{2} \norm{\nabla f(X^{q}, Y^{q}; \tau^{q})}_{\F}^2. \label{eq:f-decrease}
\end{align}
\end{subequations}
\end{lemma}


Note that based on \cref{lem:induction}, we can prove \cref{thm:non-convex-theorem} in a straightforward way.
\begin{proof}[Proof of \cref{thm:non-convex-theorem}]
Note that $\norm{E} \leq C\sigma\sqrt{n}$ with probability $1-O(n^{-10^{6}})$ for some constant $C$. Note that when $q = 0$, $X^{0} = X^{*}, X^{0,(l)}=X^{*}$ and $Y^{0} = Y^{*}, Y^{0,(l)} = Y^{*}$ easily satisfy \cref{eq:FtHt-Fstar-Fnorm,eq:FH-FlRl,eq:FlHl-Fstar,eq:FH-Fstar-rownorm,eq:XX-YY} along with $\Delta_{s}^{0} = 0$ satisfying \cref{eq:taus-taustars}. For $\tau^{0}$, noting that $\tau^{0} - \tau^{*} = \inner{Z}{E}/\norm{Z}_{\F}^2 \lesssim \sigma\sqrt{\log(n)}/\norm{Z}_{\F}$ with probability $1-O(n^{-10^6}).$ This verifies the induction hypotheses for $q = 0.$ 

Take $t^{\star} = \kappa^{3}n^{10000}$, then by \cref{lem:induction} and the union bound, with probability at least $1-O(n^{-10^4})$, we have \cref{eq:induction} and \cref{eq:PT-bound} hold for $q = 0, 1, \dotsc, t^{\star}-1.$ 

The remainder is to show \cref{eq:gradient-small-main}. Telescoping the \cref{eq:f-decrease} implies the following
\begin{align*}
    f(X^{t_\star}, Y^{t_\star}; \tau^{t_\star}) \leq f(X^{0}, Y^{0}; \tau^{0}) - \frac{\eta}{2} \sum_{t=0}^{t_\star - 1} \norm{\nabla f(X^{t}, Y^{t}; \tau^{t})}_{\F}^2.
\end{align*}
Hence, 
\begin{align*}
    \min_{0 \leq t < t_\star} \norm{\nabla f(X^{t}, Y^{t}; \tau^{t})}_{\F}^2 
    &\leq \frac{1}{t_\star} \sum_{t=0}^{t_0-1} \norm{\nabla f(X^{t}, Y^{t}; \tau^{t})}_{\F}^2\\
    &\leq \frac{2}{\eta t_\star} \left(f(X^{0}, Y^{0}; \tau^{0}) - f(X^{t_{\star}}, Y^{t_\star}; \tau^{t_\star})\right).
\end{align*}
Since $\tau^{0} = \arg\min_{\tau} f(X^{*}, Y^{*}; \tau)$, hence $f(X^{*}, Y^{*}; \tau^{0}) \leq f(X^{*}, Y^{*}; \tau^{*})$. Then
\begin{align*}
    &f(X^{0}, Y^{0}; \tau^{0}) - f(X^{t_{\star}}, Y^{t_\star}; \tau^{t_\star}) \\
    &\leq f(X^{*}, Y^{*}; \tau^{*}) - f(X^{t_{\star}}, Y^{t_\star}; \tau^{t_\star}) \\
    &= \frac{1}{2}\left(\norm{E}_{\F}^2 + \lambda(\norm{X^{*}}_{\F}^2 + \norm{Y^{*}}_{\F}^2) \right) \\
    &\quad- \frac{1}{2}\left(\norm{O-X^{t_\star}Y^{t_\star \top} - \tau^{t_\star} Z}_{\F}^2 + \lambda\norm{X^{t_\star}}_{\F}^2 + \lambda \norm{Y^{t_\star}}_{\F}^2 \right)\\
    &\leq \norm{E}_{\F}^2 + \lambda \left|\norm{X^{*}}_{\F}^2 - \norm{X^{t_\star}}_{\F}^2\right| + \lambda \left|\norm{Y^{*}}_{\F}^2 - \norm{Y^{t_\star}}_{\F}^2\right|\\
    &= \norm{E}_{\F}^2 + \lambda \left|\norm{X^{*}}_{\F} - \norm{X^{t_\star}H^{t_\star}}_{\F}\right| \left|\norm{X^{*}}_{\F} + \norm{X^{t_\star}H^{t_\star}}_{\F}\right| \\
    &\quad + \lambda \left|\norm{Y^{*}}_{\F} - \norm{Y^{t_\star}H^{t_\star}}_{\F}\right| \left|\norm{Y^{*}}_{\F} + \norm{Y^{t_\star}H^{t_\star}}_{\F}\right|\\
    &\lesssim n\norm{E}^2 + \sigma \sqrt{n} \log^{1.5}(n) \norm{F^{t_\star}H^{t_\star} - F^{*}}_{\F} \norm{F^{*}}_{\F}\\
    &\overset{(i)}{\lesssim} \sigma^2 n^2 + \sigma \sqrt{n} \log^{1.5}(n) \frac{\sigma \sqrt{n} \log^{2.5}(n)}{\sigma_{\min}} \sigma_{\max} r\\
    &\lesssim \sigma^2 n^2 + \sigma^2 n \log^{4}(n) r \kappa. 
\end{align*}
Here (i) is due to $\norm{E} \lesssim \sigma \sqrt{n}$. This implies
\begin{align*}
    &\min_{0 \leq t < t_\star} \norm{\nabla f(X^{t}, Y^{t}; \tau^{t})}_{\F} \\
    & \leq \sqrt{\frac{2}{\eta t_\star} \left(f(X^{0}, Y^{0}; \tau^{0}) - f(X^{t_{\star}}, Y^{t_\star}; \tau^{t_\star})\right)}\\
    &\lesssim \sqrt{\frac{\sigma^2 n^2 + \sigma^2 n \log^{4}(n) r\kappa}{(\kappa^{3} n^{20}\sigma_{\max})^{-1} t^{\star}}} \\
    &\overset{(i)}{\lesssim} \frac{\lambda\sqrt{\sigma_{\min}}}{\kappa n^{10}}.
\end{align*}
Here (i) is obtained by taking $t^{\star} = n^{10000} \kappa^{3}$. This completes the proof.\footnote{Here, we consider all the bounds hold with probability $1-O(n^{-10^{4}})$. The generalization to $1-O(n^{-C})$, for some constant $C$, is straightforward.}  
\end{proof}

The remainder is to finish the proof for \cref{lem:induction}. Under the hypothesis for $q = t$, we prove \cref{eq:induction} and \cref{eq:PT-bound} for $q = t+1$ in the following subsections. Note that under the assumptions of \cref{lem:induction}, we can directly invoke  \cref{lem:induction-Frobenious}, and hence \cref{eq:FtHt-Fstar-Fnorm,eq:XX-YY,eq:f-decrease} directly hold. 

\subsection{Direct implication of Eq.~(\ref{eq:induction}) for \texorpdfstring{$q=t$}{q=t}}
Under the hypothesis $q=t$, we first present a few direct implications that are helpful in the proof for $q=t+1$. For large enough $n$, we claim that the following results hold for $F^{t}$.
\begin{align*}
    \norm{F^{t}}_{\F} &\leq 2\norm{F^{*}}_{\F} \lesssim \sqrt{\sigma_{\max} r} \\
    \norm{F^{t}} &\leq 2\norm{F^{*}} \lesssim \sqrt{\sigma_{\max}}\\
    \norm{F^{t}}_{2,\infty} &\leq 2\norm{F^{*}}_{2,\infty} \lesssim \sqrt{\frac{\sigma_{\max} \mu r}{n}}.
\end{align*}
Similarly, for $l \in [2n]$ and large enough $n$, we also have 
\begin{align*}
    \norm{F^{t,(l)}}_{\F} &\leq 2\norm{F^{*}}_{\F} \lesssim \sqrt{\sigma_{\max} r}\\
    \quad \norm{F^{t,(l)}} &\leq 2\norm{F^{*}} \lesssim \sqrt{\sigma_{\max}}\\
    \norm{F^{t,(l)}}_{2,\infty} &\leq 2\norm{F^{*}}_{2,\infty} \lesssim \sqrt{\frac{\sigma_{\max} \mu r}{n}}.
\end{align*}
In addition, we have
\begin{align}
    \norm{\nabla f(X^{t}, Y^{t}; \tau^{t})}_{\F} \lesssim \sigma r\kappa \log^{2.5}(n) \sqrt{n} \sqrt{\sigma_{\max}}.\label{eq:gradient-bound-easy}
\end{align}

The proof for $\norm{\nabla f(X^{t}, Y^{t}; \tau^{t})}_{\F}$ have been shown in \cref{eq:DX-F-bound}. We show the bound for $\norm{F^{t}H^{t}}_{2,\infty}$ below
\begin{align*}
    \norm{F^{t}H^{t}}_{2,\infty} 
    &\leq \norm{F^{t}H^{t}-F^{*}}_{2,\infty} + \norm{F^{*}}_{2,\infty}\\
    &\overset{(i)}{\leq} C_{\infty}\frac{\sigma \mu r^{2.5}\kappa \log^{3.5}(n)}{\sigma_{\min}} \norm{F^{*}}_{\F} +  \norm{F^{*}}_{2,\infty}\\
    &\overset{(ii)}{\leq} 2C_{\infty}\frac{\sigma \sqrt{n} \mu r^{2}\kappa \log^{3.5}(n)}{\sigma_{\min}} \cdot \sqrt{\frac{\sigma_{\max} r}{n}} + \norm{F^{*}}_{2,\infty}\\
    &\overset{(iii)}{\leq} \frac{\sqrt{\sigma_{\max}}}{2\sqrt{\kappa}\log(n)} \sqrt{\frac{r}{n}} + \norm{F^{*}}_{2,\infty}\\
    &\overset{(iv)}{\leq} 2\norm{F^{*}}_{2,\infty}.
\end{align*}
Here, (i) is due to \cref{eq:FH-Fstar-rownorm}, (ii) is due to $\norm{F^{*}}_{\F} \leq \sqrt{2\sigma_{\max} r}$, (iii) is due to $\SNR$ and holds with large enough $n$, and (iv) is due to $\norm{F^{*}}_{2,\infty} \geq \sqrt{\frac{\sigma_{\min} r}{n}}.$ This also implies $\norm{F^{t}}_{2,\infty} = \norm{F^{t}H^{t}H^{tT}}_{2,\infty} \leq \norm{F^{t}H^{t}}_{2,\infty} \norm{H^{t}} \leq 2\norm{F^{*}}_{2,\infty}.$

The similar bounds can be obtained based on the triangle inequality. We omit the proof for brevity. 

\subsection{Proof of Eqs.(\ref{eq:taus-taustars}) and (\ref{eq:tau-taustar})}\label{sec:proof-taus}
Recall that $\Delta^{t}_{s} = \frac{\inner{Z}{x_s^{*}y_s^{*\top} - (X^{t}H^{t})_s(Y^{t}H^{t})_s^{\top}}}{\norm{Z}_{\F}^2}.$ Then we have
\begin{align}
    \Delta_{s}^{t+1}
    &= \Delta_{s}^{t} + \frac{\inner{Z}{(X^{t}H^{t})_{s}(Y^{t}H^{t})^{\top}_{s} - (X^{t+1}H^{t+1})_{s}(Y^{t+1}H^{t+1})_s^{\top}}}{\norm{Z}_{\F}^2}\nonumber\\
    &=  \Delta_{s}^{t} + \frac{\inner{Z}{(X^{t}H^{t})_{s}(Y^{t}H^{t})^{\top}_{s} - (X^{t+1}H^{t})_{s}(Y^{t+1}H^{t})_s^{\top}}}{\norm{Z}_{\F}^2} \nonumber\nonumber\\
    &\quad + \underbrace{\frac{\inner{Z}{(X^{t+1}H^{t})_{s}(Y^{t+1}H^{t})_s^{\top} - (X^{t+1}H^{t+1})_{s}(Y^{t+1}H^{t+1})_s^{\top}}}{\norm{Z}_{\F}^2}}_{A_0} \label{eq:delta-s}.
\end{align}
To control $A_0$, we have the following claim to show that $A_0$ is negligible. 
\begin{claim}\label{claim:taus-taustars-A0}
\begin{align}
    \norm{H^{t} - H^{t+1}}_{\F} &\lesssim \eta \frac{\sigma \sqrt{n} \log^{2.5}(n)}{\sigma_{\min}} (\sigma \sqrt{n} \mu r^{3} \kappa^{2} \log^{3.5}(n)).\nonumber\\
    |A_0| &\lesssim \eta \sigma^2 \log^{6}(n) \kappa^3 \mu^2 r^4.\nonumber
\end{align}
\end{claim}

Next, by the gradient updating rule, 
$$
X^{t+1} = X^{t} - \eta \left((X^{t}Y^{tT} - M^{*} + (\tau^{t} - \tau^{*}) Z - E)Y^{t} +  \lambda X^{t}\right).
$$
Let $A_s$ be the $s$-column of $A$ for any matrix $A$. Suppose $x^{t}_s = (X^{t}H^{t})_{s}, y^{t}_s = (Y^{t}H^{t})_s$. Then, we have
\begin{align}
    (X^{t+1}H^{t})_s 
    &= (X^{t}H^{t})_s - \eta \left( (X^{t}Y^{tT}-M^{*} + (\tau^{t}-\tau^{*})Z - E)(Y^{t}H^{t})_{s} + \lambda (X^{t}H^{t})_s\right) \nonumber\\
    &= x^{t}_s - \eta (X^{t}Y^{tT}-M^{*} + (\tau^{t}-\tau^{*})Z - E)y^{t}_s - \eta \lambda x^{t}_s.\label{eq:XH-decompose}
\end{align}
Similarly,
\begin{align}
    (Y^{t+1}H^{t})_s 
    &= (Y^{t}H^{t})_s - \eta \left(X^{t}Y^{tT}-M^{*} + (\tau^{t}-\tau^{*})Z - E)^{\top}(X^{t}H^{t})_{s} + \lambda (Y^{t}H^{t})_s\right)\nonumber\\
    &= y^{t}_s - \eta (X^{t}Y^{tT}-M^{*} + (\tau^{t}-\tau^{*})Z - E)^{\top}x^{t}_s - \eta \lambda y^{t}_s.\label{eq:YH-decompose} 
\end{align}
To simplify the notation, we write $(XH)^{t}, (YH)^{t}, \tau^{t}, x^{t}_s, y^{t}_s$ by $X, Y, \tau, x_s, y_s$ if there is no ambiguity. 
Continuing the analysis for $\Delta_s^{t+1}$ in \cref{eq:delta-s}, we have
\begin{align}
    \Delta_{s}^{t+1} 
    &= \Delta_{s}^{t} + A_0 + \frac{\inner{Z}{x_{s}y^{\top}_{s} - (X^{t+1}H^{t})_{s}(Y^{t+1}H^{t})_s^{\top}}}{\norm{Z}_{\F}^2}  \nonumber\\
    &= \Delta_{s}^{t} + A_0 + \frac{\inner{Z}{(x_{s} - (X^{t+1}H^{t})_{s})y^{\top}_{s}}}{\norm{Z}_{\F}^2} + \frac{\inner{Z}{x_{s}(y_{s}-(Y^{t+1}H^{t})_s)^{\top}}}{\norm{Z}_{\F}^2} \nonumber\\
    &\quad - \frac{\inner{Z}{(x_{s} - (X^{t+1}H^{t})_{s})(y_{s}-(Y^{t+1}H^{t+1})_s)^{\top}}}{\norm{Z}_{\F}^2}. \nonumber
\end{align}
By expanding $x_s - (X^{t+1}H^{t})_{s}$ and $y_{s}-(Y^{t+1}H^{t})_s$ using \cref{eq:XH-decompose} and \cref{eq:YH-decompose}, we then have
\begin{align}
     \Delta_{s}^{t+1}  &= \Delta_{s}^{t} + A_0 + \frac{\eta}{\norm{Z}_{\F}^2} \underbrace{\inner{Z}{(XY^{\top}-M^{*} + (\tau - \tau^{*})Z - E) y_{s}y^{\top}_{s}+\lambda x_s y_{s}^{\top}}}_{A_1} \nonumber\\
    &\quad +  \frac{\eta}{\norm{Z}_{\F}^2}\underbrace{\inner{Z}{x_{s}x^{\top}_{s}(XY^{\top} - M^{*} + (\tau - \tau^{*})Z - E)+\lambda x_sy_s^{\top}}}_{A_2} + A_3
    \label{eq:Deltas-equation}
\end{align}
where $A_3$ includes the term with coefficient $\eta^2$:
\begin{align}
    A_3 := \frac{-\eta^2}{\norm{Z}_{\F}^2} \inner{Z}{D_1D_2^{\top}}\nonumber
\end{align}
with 
\begin{align*}
D_1 &:= (XY^{\top}-M^{*} + (\tau - \tau^{*})Z - E) y_{s}+\lambda x_s\\
D_2 &:= (XY^{\top}-M^{*} + (\tau - \tau^{*})Z - E) x_{s}+\lambda y_s.
\end{align*}

Note that $\eta$ has be chosen small enough, one can verify that 
\begin{align}
    |A_3| 
    &\leq \frac{\eta^2}{\norm{Z}_{\F}^2} \norm{Z}_{\F} \norm{D_1}_{\F} \norm{D_2}_{\F} \nonumber\\
    &\leq \eta^2 \norm{D_1}_{\F} \norm{D_2}_{\F} \nonumber\\
    &\overset{(i)}{\lesssim} \eta \frac{1}{n^{20}\sigma_{\max}} (\sigma \sqrt{n} \mu r^{3} \kappa \log^{3.5}(n) \sqrt{\sigma_{\max}})^2 \nonumber\\
    &= \eta \sigma^2 \frac{\mu^2 r^{6} \kappa^{2} \log^{7}(n)}{n^{19}}  \nonumber\\
    &\overset{(ii)}{=} \eta \frac{\sigma^2}{n^{15}}. \label{eq:A3-sigma-square}
\end{align}
In (i), we use the bound $\norm{D_1}_{\F}+\norm{D_2}_{\F} \lesssim  \sigma \sqrt{n} \mu r^{3} \kappa \log^{3.5}(n) \sqrt{\sigma_{\max}}$ (the proof is the same as showing $\norm{\nabla f(X, Y; \tau)}_{\F} \lesssim \sigma \sqrt{n} \mu r^{3} \kappa \log^{3.5}(n) \sqrt{\sigma_{\max}}$ in \cref{eq:DX-F-bound} and hence omitted here). In (ii), we use that $\kappa^4 \mu^2 r^2 \log^2(n) \lesssim n.$

Next, we analyze $A_1$ ($A_2$ is similar to $A_1$ by symmetry). For $A_1$, to facilitate the analysis, let 
\begin{align}
B_0 &:= \inner{Z}{(XY^{\top}-M^{*} + (\tau - \tau^{*})Z) \left(y_s y^{\top}_{s}-y^{*}_{s}y^{*\top}_{s}\right)}\nonumber\\
B_1 &:= \inner{Z}{\lambda x_s y_s^{\top}}.\nonumber
\end{align}
Then, we have
\begin{align}
A_1
    &= \inner{Z}{(XY^{\top}-M^{*} + (\tau - \tau^{*})Z) y^{*}_{s}y^{*\top}_{s}} + B_0 + B_1 + \underbrace{\inner{Z}{-E y_s y^{\top}_s}}_{B_2}\nonumber\\
    &= \inner{Z}{(XY^{\top}-M^{*})y^{*}_sy^{*\top}_s} + (\tau-\tau^{*}) \inner{Z}{Zy^{*}_sy^{*\top}_s} 
    + B_0+B_1 + B_2\nonumber.
\end{align}
Recall that $y^{*}_s = \sqrt{\sigma_s^{*}} v^{*}_s$ where $\sigma_s^{*}$ is the $s$-th largest singular value of $M^{*}$ and $v^{*}$ is the $s$-th corresponding right singular vector. Hence, for any $k\in [r]$, we have 
\begin{align*}
    y^{*\top}_{k} y^{*}_{s} = \sqrt{\sigma_s^{*}} \sqrt{\sigma_k^{*}} v^{*\top}_k v^{*}_s \overset{(i)}{=} \sigma_s^{*} \1{k=s}
\end{align*}
where (i) is due to orthogonality between different singular vectors. This will help us simplify the analysis for $A_1$. In particular,
\begin{align}
    A_1 
    &\overset{(i)}{=} \inner{Z}{(X-X^{*})Y^{*\top}y^{*}_sy^{*\top}_s + X(Y-Y^{*})^{\top}y^{*}_sy^{*\top}_s} + (\tau-\tau^{*}) \sigma_s^{*} \tr(Z^{\top}Zv^{*}_sv^{*\top}_s)\nonumber\\
    &\quad + B_2 + B_1 + B_0\nonumber\\
    &\overset{(ii)}{=} \sigma_s^{*} \inner{Z}{(x_s-x^{*}_s)y^{*\top}_s} + (\tau-\tau^{*}) \sigma_s^{*} \tr(Z^{\top}Zv^{*}_sv^{*\top}_s) + 
    \underbrace{\inner{Z}{X(Y-Y^{*})^{\top}y^{*}_sy^{*\top}_s}}_{B_3}\nonumber\\
    &\quad + B_2 + B_1 + B_0\nonumber
\end{align}
where in (i) we use that $XY^{\top} - M^{*} = (X-X^{*})Y^{*} + X(Y-Y^{*})^{\top}$ and $y^{*}_sy^{*\top}_s = \sigma_s^{*} v^{*}_s v^{*\top}_s$, in (ii) we use the fact that
$y^{*}_s$ is orthogonal to $y^{*}_k, k \neq s$, which implies 
\begin{align}
    (X-X^{*})Y^{*\top}y^{*}_s = \sum_{k=1}^r (X-X^{*})_{k} y^{*\top}_k y^{*}_s y^{*\top}_s = \sigma_s^{*} (x_s - x^{*}_s) y^{*\top}_s. \nonumber
\end{align}
In order to bound $B_0, B_1, B_2, B_3$, we have the following claim. 
\begin{claim}\label{claim:taus-control-B}
With probability $1-O(n^{-10^{7}})$, 
\begin{align}
    \frac{1}{\norm{Z}_{\F}^2}|B_0| &\lesssim \frac{\sigma \mu r^{1.5}\kappa \log^{2}(n)}{\sqrt{n}} \sigma_{\min} + \frac{\sigma \log^{0.5}(n)}{r\norm{Z}_{\F}} \sigma_{\min} \nonumber\\
    \frac{1}{\norm{Z}_{\F}^2}|B_1| &\lesssim \frac{\sigma \mu r^{1.5}\kappa \log^{2}(n)}{\sqrt{n}} \sigma_{\min}\nonumber\\
    \frac{1}{\norm{Z}_{\F}^2}|B_2| &\lesssim \frac{\sigma \mu r^{1.5}\kappa \log^{2}(n)}{\sqrt{n}} \sigma_{\min}\nonumber\\
    \frac{1}{\norm{Z}_{\F}^2}|B_3| &\lesssim \frac{\sigma \mu  r^{1.5}\kappa \log^{2.5}(n)}{\sqrt{n}} \sigma_{s}^{*}. \nonumber
\end{align}
\end{claim}

Let $\Delta = |B_0| + |B_1| + |B_2| + |B_3|$. Then we have
\begin{align}
    A_1 \leq  \sigma_s^{*} \inner{Z}{(x_s-x^{*}_s)y^{*\top}_s} + (\tau-\tau^{*}) \sigma_s^{*} \tr(Z^{\top}Zv^{*}_sv^{*\top}_s) + \Delta\nonumber.
\end{align}
Here, 
\begin{align}
    \frac{1}{\norm{Z}_{\F}^2}\Delta =  \frac{1}{\norm{Z}_{\F}^2} (|B_0| + |B_1| + |B_2| + |B_3|) \overset{(i)}{\lesssim} \frac{\sigma \mu r^{1.5}\kappa \log^{2.5}(n)}{\sqrt{n}} \sigma_{s}^{*} + \frac{\sigma \log^{0.5}(n)}{r\norm{Z}_{\F}} \sigma_{\min} \label{eq:bound-for-Delta}
\end{align}
where (i) uses \cref{claim:taus-control-B} and $\sigma_{\min} = \sigma_{r}^{*} \leq \sigma_s^{*}$ for $s \in [r].$

By symmetry, we can obtain the similar results for $A_2$ and in particular, we have
\begin{align}
    A_1 + A_2 
    &\leq \sigma_s^{*} \inner{Z}{(x_s-x^{*}_s)y^{*\top}} + \sigma_s^{*} \inner{Z}{x^{*}(y_s-y^{*})^{\top}} \nonumber\\
    &\quad + (\tau-\tau^{*}) \sigma_s^{*} \tr(Z^{\top}Zv^{*}_sv^{*\top}_s) + (\tau-\tau^{*}) \sigma_s^{*} \tr(ZZ^{\top}u^{*}_su^{*\top}_s)\nonumber\\
    &\quad + 2\Delta\nonumber\\
    &\overset{(i)}{\leq} \sigma_{s}^{*} \inner{Z}{x_sy_s^{\top}-x^{*}_sy^{*\top}_s} - \sigma_s^{*} \inner{Z}{(x_s-x_s^{*})(y_s-y_s^{*})^{\top}} + 
    |\tau-\tau^{*}| \sigma_s^{*} b_s \norm{Z}_{\F}^2 + 2\Delta \nonumber\\
    &\overset{(ii)}{\leq} -\Delta_{s}^{t} \sigma_s^{*} \norm{Z}_{\F}^2  + \sigma_s^{*} \norm{Z}_{\F}^2\norm{F-F^{*}}_{2,\infty}^2 + |\tau-\tau^{*}| \sigma_s^{*} b_s \norm{Z}_{\F}^2
    +2\Delta\nonumber
\end{align}
where in (i) we use that $x_sy_s^{\top}-x^{*}_sy^{*\top}_s = (x_s-x^{*}_s)y^{*\top}+x^{*}(y_s-y^{*})^{\top}+(x_s-x_s^{*})(y_s-y_s^{*})^{\top}$ and 
$b_s \norm{Z}_{\F}^2 = \tr(Z^{\top}Zv^{*}_sv^{*\top}_s)+\tr(ZZ^{\top}u^{*}_su^{*\top}_s)$, in (ii) we use that $\Delta_{s}^{t}\norm{Z}_{\F}^2 = 
\inner{Z}{x^{*}_sy^{*\top}_s - x_sy_s^{\top}}$ and $\left|\inner{Z}{AB^{\top}}\right|\leq \norm{Z}_{\F}^2 \norm{AB^{\top}}_{\infty} \leq \norm{Z}_{\F}^2\norm{A}_{2,\infty} \norm{B}_{2,\infty}$ for binary matrix $Z$. 

Next we can bound $\Delta_{s}^{t+1}$ by using the bounds all above for $A_0, A_1, A_2, A_3$ in \cref{eq:Deltas-equation}.
\begin{align}
    |\Delta_{s}^{t+1}| 
    &\leq \left|\Delta_{s}^{t} + \eta \frac{1}{\norm{Z}_{\F}^2} (A_1+A_2) + A_0 + A_3\right|\nonumber\\
    &\leq |\Delta_s^{t}|(1-\eta \sigma_{s}^{*}) + \eta \left(\sigma_s^{*} \norm{F-F^{*}}_{2,\infty}^2 + |\tau-\tau^{*}| \sigma_s^{*} b_s + \frac{1}{\norm{Z}_{\F}^2}2\Delta\right) + |A_0| + |A_3|\nonumber\\
    &\overset{(i)}{\leq}   C_{\tau} \left(C_0 + b_s r\log(n)\right) \left(\frac{\sigma \mu r^{1.5}\kappa\log^{2.5}(n)}{\sqrt{n}} + \frac{ \sigma 
    \log^{0.5}(n)}{r\norm{Z}_{\F}}\right) (1-\eta \sigma_s^{*}) \nonumber\\
    &\quad + \eta C_{\tau}\left( \frac{\sigma \mu r^{2.5} \kappa \log^{3.5} n}{\sqrt{n}}  + \frac{\sigma \log^{1.5} n}{\norm{Z}_{\F}}\right) \sigma_s^{*} b_s\nonumber\\
    &\quad + \eta \sigma_s^{*} \norm{F-F^{*}}_{2,\infty}^2 + \eta \frac{1}{\norm{Z}_{\F}^2}2\Delta + |A_0| + |A_3| \label{eq:etasigmas}
\end{align}
Here (i) is by the induction hypothesis \cref{eq:taus-taustars} and \cref{eq:tau-taustar}. Let 
$$
D :=  \left(\frac{\sigma \mu r^{1.5}\kappa\log^{2.5}(n)}{\sqrt{n}} + \frac{ \sigma 
    \log^{0.5}(n)}{r\norm{Z}_{\F}}\right).
$$ 
Then $D\log(n)r = \left( \frac{\sigma \mu r^{2.5} \kappa \log^{3.5} n}{\sqrt{n}}  + \frac{\sigma \log^{1.5} n}{\norm{Z}_{\F}}\right)$. To simplify \cref{eq:etasigmas}, by the direct algebra, we have 
\begin{align*}
    &C_{\tau} \left(C_0 + b_s r\log(n)\right) D (1-\eta \sigma_s^{*}) + \eta C_{\tau} D\log(n) r \sigma_s^{*}b_s \\
    &= C_{\tau} \left(C_0 + b_s r\log(n)\right) D -  C_{\tau} \left(C_0 + b_s r\log(n)\right) D \eta \sigma_s^{*} +  \eta C_{\tau} D\log(n) r \sigma_s^{*}b_s  \\
    &= C_{\tau} \left(C_0 + b_s r\log(n)\right) D - C_{\tau} C_0 D \eta \sigma_s^{*} - C_{\tau}b_s r\log(n)D \eta \sigma_s^{*} + \eta C_{\tau} D\log(n) r \sigma_s^{*}b_s\\
    &= C_{\tau} \left(C_0 + b_s r\log(n)\right) D - C_{\tau} C_0 D \eta \sigma_s^{*}.
\end{align*}
To control $\eta \sigma_s^{*} \norm{F-F^{*}}_{2,\infty}^2 + \eta \frac{1}{\norm{Z}_{\F}^2}2\Delta + |A_0| + |A_3|$, recall that $\norm{F-F^{*}}_{2,\infty} \lesssim \frac{\sigma \mu r^{2.5} \kappa \log^{3.5}(n) \sqrt{\sigma_{\max} r}}{\sigma_{\min}}$ by \cref{eq:FH-Fstar-rownorm}, $\frac{1}{\norm{Z}_{\F}^2}\Delta \lesssim \frac{\sigma \mu r^{1.5}\kappa \log^{2.5}(n)}{\sqrt{n}} \sigma_{s}^{*} + \frac{\sigma \log^{0.5}(n)}{r\norm{Z}_{\F}} \sigma_{\min}$ by \cref{eq:bound-for-Delta}, $|A_0| \lesssim \eta \sigma^2 \log^{6}(n) \kappa^3 \mu^2 r^4$ by \cref{claim:taus-taustars-A0}, and $|A_3|\lesssim \eta \frac{\sigma^2}{n^{15}}$. Then, we have, for some constant $C$, 
\begin{align*}
    &\eta \sigma_s^{*} \norm{F-F^{*}}_{2,\infty}^2 + \eta \frac{1}{\norm{Z}_{\F}^2}2\Delta + |A_0| + |A_3| \\
    &\lesssim \eta \left(\sigma^{*}_s \frac{\sigma^2 \mu^2 r^{6} \log^{7}(n) \kappa^3}{\sigma_{\min}} +  \frac{\sigma \mu r^{1.5}\kappa \log^{2.5}(n)}{\sqrt{n}} \sigma_{s}^{*}\right) \\
    &\quad +  \left(\frac{\sigma \log^{0.5}(n)}{r\norm{Z}_{\F}} \sigma_{\min} +  \sigma^2 \log^{6}(n) \kappa^3 \mu^2 r^4 +  \frac{\sigma^2}{n^{15}}\right) \\
    &\overset{(i)}{\lesssim}  \eta  \left(\frac{\sigma \mu r^{1.5}\kappa \log^{2.5}(n)}{\sqrt{n}} \sigma_{s}^{*} +  \frac{\sigma \log^{0.5}(n)}{r\norm{Z}_{\F}} \sigma_{\min}\right) \\
    &\overset{(ii)}{\leq} C \eta \sigma_s^{*} \left( \frac{\sigma \mu r^{1.5}\kappa \log^{2.5}(n)}{\sqrt{n}} +  \frac{\sigma \log^{0.5}(n)}{r\norm{Z}_{\F}}\right)
\end{align*}
where (i) is due to $\frac{\sigma}{\sigma_{\min}} \sqrt{n} \lesssim \frac{1}{\kappa^3 r^{4.5} \log^{5}(n) \mu}$ (hence some terms are negligible), and (ii) is due to $\sigma_{\min} \leq \sigma_{s}^{*}.$

This leads to the bound for $|\Delta_s^{t+1}|.$ 
\begin{align}
    |\Delta_s^{t+1}| 
    &\leq C_{\tau} \left(C_0 + b_s r\log(n)\right) D - C_{\tau} C_0 D \eta \sigma_s^{*} \nonumber\\
    &\quad + C \eta \sigma_s^{*} \left( \frac{\sigma \mu r^{1.5}\kappa \log^{2.5}(n)}{\sqrt{n}} +  \frac{\sigma \log^{0.5}(n)}{r\norm{Z}_{\F}}\right) \nonumber\\
    &\leq C_{\tau} \left(C_0 + b_s r\log(n)\right) D - C_{\tau} C_0 D \eta \sigma_s^{*} + C \eta \sigma_s^{*} D \nonumber\\
    &\overset{(i)}{\leq} C_{\tau} \left(C_0 + b_s r\log(n)\right) D \nonumber.
\end{align}
Here, (i) is obtained by choosing the constant $C_{\tau}$ large enough. This then completes the proof for the bound on $\Delta_{s}^{t+1}.$

The bound on $|\tau^{t+1}-\tau^{*}|$ is a direct consequence of the bound on $\Delta_{s}^{t+1}$ and \cref{cond:Z-condition-nonconvex}. Recall that $\tau^{t+1} - \tau^{*} - \frac{\inner{Z}{E}}{\norm{Z}_{\F}^2} = \sum_{s=1}^{r} \Delta_{s}^{t+1}$ by \cref{eq:delta-t-s-tau}. Then, we have
\begin{align}
    |\tau^{t+1} - \tau^{*}| 
    &\leq \left|\sum_{s=1}^{r} \Delta_{s}^{t+1}\right|+ \left|\frac{\inner{Z}{E}}{\norm{Z}_{\F}^2}\right| \nonumber\\
    &\overset{(i)}{\leq} \left|\sum_{s=1}^{r} \Delta_{s}^{t+1}\right| + C \frac{\sigma \sqrt{\log(n)}}{\norm{Z}_{\F}}\nonumber\\
    &\leq \sum_{s=1}^{r}  C_{\tau} \left(C_0 + b_s r\log(n)\right) \left(\frac{\sigma \mu r^{1.5}\kappa\log^{2.5}(n)}{\sqrt{n}} + \frac{ \sigma 
    \log^{0.5}(n)}{r\norm{Z}_{\F}}\right) + C \frac{\sigma \sqrt{\log(n)}}{\norm{Z}_{\F}}\nonumber\\
    &\overset{(ii)}{\leq}  C_{\tau} \left( rC_0 + \left(1 - \frac{C_{r_1}}{\log(n)}\right) r\log(n)\right) \left(\frac{\sigma \mu r^{1.5}\kappa\log^{2.5}(n)}{\sqrt{n}} + \frac{ \sigma 
    \log^{0.5}(n)}{r\norm{Z}_{\F}}\right) \nonumber\\
    &\quad + C \frac{\sigma \sqrt{\log(n)}}{\norm{Z}_{\F}} \nonumber
\end{align}
where (i) is due to with probability $1-O(n^{-10^{7}})$, 
$\left|\frac{\inner{Z}{E}}{\norm{Z}_{\F}^2}\right| \leq C\frac{\sigma \sqrt{\log(n)}}{\norm{Z}_{\F}}$ and (ii) is due to $\sum_{s=1}^{r} b_s \leq 1 - C_{r_1}/\log(n)$ by \cref{cond:Z-condition-nonconvex}. Note that $rC_0 + \left(1 - \frac{C_{r_1}}{\log(n)}\right) r\log(n) = rC_0 + r\log(n) - rC_{r_1}$ and choose $C_0 \leq C_{r_1} / 2$, we have 
\begin{align}
    |\tau^{t+1} - \tau^{*}| 
    &\leq C_{\tau} \left(rC_0 + r\log(n) - rC_{r_1}\right) \left(\frac{\sigma \mu r^{1.5}\kappa\log^{2.5}(n)}{\sqrt{n}} + \frac{ \sigma 
    \log^{0.5}(n)}{r\norm{Z}_{\F}}\right) \nonumber\\
    &\quad + C \frac{\sigma \sqrt{\log(n)}}{\norm{Z}_{\F}} \nonumber\\
    &\overset{(i)}{\leq} C_{\tau} r\log(n) \left(\frac{\sigma \mu r^{1.5}\kappa\log^{2.5}(n)}{\sqrt{n}} + \frac{ \sigma 
    \log^{0.5}(n)}{r\norm{Z}_{\F}}\right) \nonumber\\
    &\leq  C_{\tau} \frac{\sigma \mu r^{2.5} \kappa \log^{3.5}(n)}{\sqrt{n}} + C_{\tau} \frac{ \sigma 
    \log^{1.5} n}{\norm{Z}_{\F}}\nonumber
\end{align}
where (i) is obtained by choosing $C_0 \leq C_{r_1} / 2$ and $C_{\tau}$ large enough. This completes the proof for 
\cref{eq:tau-taustar}.

\begin{proof}[Proof of \cref{claim:taus-taustars-A0}]
Recall that we want to show
\begin{align*}
    \norm{H^{t} - H^{t+1}}_{\F} &\lesssim \eta \frac{\sigma \sqrt{n} \log^{2.5}(n)}{\sigma_{\min}} (\sigma \sqrt{n} \mu r^{3} \kappa^{2} \log^{3.5}(n))
\end{align*}
where $H^{t} = \arg\min_{A\in \O^{r\times r}} \norm{F^{t}A - F^{*}}_{\F}, H^{t+1} = \arg\min_{A\in \O^{r\times r}} \norm{F^{t+1}A - F^{*}}_{\F}$ with 
$F^{t+1} = [X^{t+1}; Y^{t+1}]$ and 
\begin{align*}
    X^{t+1} &= X^{t} - \eta ((X^{t}Y^{tT} - M^{*} - E + (\tau^{t} - \tau^{*})Z)Y^{t} +\lambda X^{t})\\
    Y^{t+1} &= Y^{t} - \eta ((X^{t}Y^{tT} - M^{*} - E + (\tau^{t} - \tau^{*})Z)^{\top}X^{t} +\lambda Y^{t}).
\end{align*}

The proof is similar to section D.4 in \cite{chen2019noisy}. We intend to invoke \cref{lem:F0-F1-F2} to control $\norm{H^{t} - H^{t+1}}_{\F}.$ However, $\norm{F^{t} - F^{t+1}}_{\F}$ is too large to provide the desired bounds. We tackle this by constructing an auxiliary point $\tilde{F}^{t+1}$ such that (i) $H^{t}$ is also the optimal rotation to align 
$\tilde{F}^{t+1}$ and $F^{*}$ and (ii) $\norm{\tilde{F}^{t+1} - F^{t+1}}_{\F} \ll \norm{F^{t} - F^{t+1}}_{\F}.$

In particular, we construct the auxiliary point $\tilde{F}^{t+1} := \begin{bmatrix} \tilde{X}^{t+1}\\ \tilde{Y}^{t+1} \end{bmatrix}$ by substituting some of $X^{t}, Y^{t}$ in $F^{t+1}$ with $X^{*}H^{t \top}, Y^{*}H^{t \top}$ in the following
\begin{align}
    \tilde{X}^{t+1} = X^{t} - \eta ((X^{t}Y^{tT} - M^{*} - E + (\tau^{t} - \tau^{*})Z)Y^{*}H^{t \top} +\lambda X^{*}H^{t \top})\nonumber\\
    \tilde{Y}^{t+1} = Y^{t} - \eta ((X^{t}Y^{tT} - M^{*} - E + (\tau^{t} - \tau^{*})Z)^{\top}X^{*}H^{t \top} +\lambda Y^{*}H^{t \top}).\nonumber
\end{align}
We then claim that $H^{t}$ is also the optimal rotation to align $\tilde{F}^{t+1}$ and $F^{*}$: 
$$
\norm{\tilde{F}^{t+1}H^{t} - F^{*}}_{\F} \leq \min_{A\in \O^{r\times r}} \norm{\tilde{F}^{t+1}A - F^{*}}_{\F}.
$$ 
To verify this, by the property of Orthogonal Procrustes problem (\cref{lem:orthogonal-procrustes}), it is sufficient to show that $F^{* \top} \tilde{F}^{t+1} H^{t}$ is symmetric and positive semi-definite. The symmetry can be verified by direct algebra:
\begin{align*}
    F^{* \top} \tilde{F}^{t+1} H^{t} 
    &= X^{*\top}\tilde{X}^{t+1}H^{t}+Y^{*\top}\tilde{Y}^{t+1}H^{t}\\
    &= X^{*\top}X^{t}H^{t} - \eta X^{*\top} (X^{t}Y^{tT} - M^{*} - E + (\tau^{t} - \tau^{*})Z)Y^{*} - \eta \lambda X^{*\top}X^{*} \\
    &\quad + Y^{*\top}Y^{t}H^{t} - \eta Y^{*\top}(X^{t}Y^{tT} - M^{*} - E + (\tau^{t} - \tau^{*})Z)^{\top}X^{*} - \eta \lambda Y^{*\top}Y^{*}\\
    &= F^{*\top}F^{t}H^{t} - \eta \lambda X^{*\top}X^{*} - \eta \lambda Y^{*\top}Y^{*} - \eta B_0
\end{align*}
where 
\begin{align*}
B_0 &:=  (X^{*\top} (X^{t}Y^{tT} - M^{*} - E + (\tau^{t} - \tau^{*})Z)Y^{*} \\
&\quad + Y^{*\top}(X^{t}Y^{tT} - M^{*} - E + 
    (\tau^{t} - \tau^{*})Z)^{\top}X^{*}).
\end{align*}
It is easy to see that $X^{*\top}X^{*}, Y^{*\top}Y^{*}$ are symmetric, together with $B_0 = B_0^{\top}$. Furthermore, note $H^{t}$ is the optimal rotation for aligning $F^{t}$ and 
$F^{*}$ by definition. Then, by \cref{lem:orthogonal-procrustes}, $F^{*\top}F^{t}H^{t}$ is also symmetric. This leads to the result that $F^{* \top} \tilde{F}^{t+1} H^{t}$ is symmetric.  

To verify the positive semi-definiteness, it is easy to check that 
\begin{align*}
\norm{F^{* \top}\tilde{F}^{t+1} H^{t} - F^{* \top} F^{*}} 
&\leq \norm{F^{* \top}} \norm{\tilde{F}^{t+1}H^{t} - F^{*}} \\
&\overset{(i)}{\leq} \sqrt{\sigma_{\max}} \norm{F^{t}H^{t} - F^{*} + \eta \delta} \\
&\overset{(ii)}{\leq} \sigma_{\min}.
\end{align*}
Here, $\delta$ in (i) represents the term with coefficient $\eta$; and (ii) is due to \cref{eq:FtHt-Fstar-Fnorm} and that we choose $\eta$ being sufficiently small. 

Also note that $\lambda_{\min}(F^{*\top}F^{*}) = 2\sigma_{\min}.$ By Weyl's inequality, this implies that the positive semi-definiteness of $F^{* \top}\tilde{F}^{t+1} H^{t}$.
\begin{align}
    \lambda_{\min}(F^{* \top}\tilde{F}^{t+1} H^{t}) \geq 2\sigma_{\min} - \norm{F^{* \top}\tilde{F}^{t+1} H^{t} - F^{* \top} F^{*}} > 0.\nonumber
\end{align}

Then, we can invoke \cref{lem:F0-F1-F2} to bound $\norm{H^{t} - H^{t+1}}.$ Let $F_0 = F^{*}, F_1 = \tilde{F}^{t+1}, F_2 = F^{t+1}.$ Recall that 
$\lambda \lesssim \sigma\sqrt{n} \log^{1.5}(n)$, $\norm{E} \lesssim \sigma \sqrt{n}$, $|\tau^{t} - \tau^{*}|\norm{Z} \lesssim \sigma \sqrt{n} \mu r^{3} \kappa^{2} \log^{3.5}(n)$ and 
$$
\norm{X^{t}Y^{tT} - M^{*}} \lesssim \norm{F^{*} - F^{t}H^{t}}_{\F} \norm{F^{*}} \lesssim \frac{\sigma \sqrt{n} \log^{2.5}(n)}{\sigma_{\min}} \sigma_{\max} \lesssim \sigma \sqrt{n} \log^{2.5}(n) \kappa.
$$

Then
\begin{align}
    \norm{F_2 - F_1}_{\F} \nonumber
    &\lesssim \eta \norm{F^{*} - F^{t}H^{t}}_{\F}\left(\lambda + \norm{X^{t}Y^{tT}-M^{*}-E+(\tau^{t} - \tau^{*})Z}\right)\nonumber\\
    &\lesssim \eta \norm{F^{*} - F^{t}H^{t}}_{\F} (\lambda + \norm{X^{t}Y^{tT} - M^{*}} + \norm{E} + |\tau^{t}-\tau^{*}|\norm{Z}) \nonumber\\
    &\lesssim \eta \frac{\sigma \sqrt{n} \log^{2.5}(n)}{\sigma_{\min}} \norm{F^{*}}_{\F} (\sigma \sqrt{n} \mu r^{2.5} \kappa \log^{3.5}(n)).\nonumber
\end{align}
Then by \cref{lem:F0-F1-F2}, we have
\begin{align}
\norm{H^{t} - H^{t+1}}_{\F} 
&\leq \frac{1}{\sigma_{\min}} \norm{F_2 - F_1}_{\F} \norm{F^{*}} \nonumber\\
&\lesssim \eta \frac{\sigma \sqrt{n} \log^{2.5}(n)}{\sigma_{\min}} (\sigma \sqrt{n} \mu r^{3} \kappa^{2} \log^{3.5}(n))\nonumber.
\end{align}
This completes the control for $H^{t} - H^{t+1}.$

Next, to control $A_0$, recall that $A_0 = \frac{\inner{Z}{(X^{t+1}H^{t})_{s}(Y^{t+1}H^{t})_s^{\top} - (X^{t+1}H^{t+1})_{s}(Y^{t+1}H^{t+1})_s^{\top}}}{\norm{Z}_{\F}^2}$. Then, we have
\begin{align}
    |A_0|
    &\leq \left|\frac{\inner{Z}{(X^{t+1}H^{t})_{s}(Y^{t+1}H^{t})_s^{\top} - (X^{t+1}H^{t+1})_{s}(Y^{t+1}H^{t+1})_s^{\top}}}{\norm{Z}_{\F}^2}\right| \nonumber\\
    &\leq \frac{1}{\norm{Z}_{\F}^2}|\inner{Z}{(X^{t+1}H^{t} - X^{t+1}H^{t+1})_{s}(Y^{t+1}H^{t})_s^{\top}}| \nonumber\\
    &\quad + \frac{1}{\norm{Z}_{\F}^2}|\inner{Z}{(X^{t+1}H^{t+1})_{s}(Y^{t+1}H^{t} - Y^{t+1}H^{t+1})_s^{\top}}| \nonumber\\
    &\overset{(i)}{\leq} \norm{(X^{t+1}H^{t} - X^{t+1}H^{t+1})_s}_{\infty}\norm{(Y^{t+1}H^{t})_s}_{\infty} \nonumber\\
    &\quad + \norm{(X^{t+1}H^{t+1})_{s}}_{\infty}\norm{(Y^{t+1}H^{t} - Y^{t+1}H^{t+1})_s}_{\infty} \nonumber\\
    &\overset{(ii)}{\leq} \norm{X^{t+1}(H^{t} - H^{t+1})}_{2,\infty} \norm{Y^{t+1}H^{t}}_{2,\infty} \nonumber\\
    &\quad + \norm{X^{t+1}H^{t+1}}_{2,\infty} \norm{Y^{t+1}(H^{t}-H^{t+1})}_{2,\infty} \nonumber\\
    &\overset{(iii)}{\lesssim} \norm{X^{t+1}}_{2,\infty} \norm{H^{t}-H^{t+1}} \norm{Y^{t+1}}_{2,\infty} \nonumber\\
    &\lesssim \eta \norm{F^{*}}_{2,\infty}^2 \frac{\sigma \sqrt{n} \log^{2.5}(n)}{\sigma_{\min}} (\sigma \sqrt{n} \mu r^{3} \kappa^{2} \log^{3.5}(n))\nonumber\\
    &\lesssim \eta \sigma^2 \log^{6}(n) \kappa^3 \mu^2 r^4.\nonumber
\end{align}
Here, (i) is due to $\frac{1}{\norm{Z}_{\F}^2} \left|\inner{Z}{ab^{\top}}\right| \leq \norm{ab^{\top}}_{\infty} \leq \norm{a}_{\infty}\norm{b}_{\infty}$ for any $a, b \in \R^{n}$, 
(ii) is due to $\norm{(A)_{s}}_{\infty} \leq \norm{A}_{2,\infty}$ for any $A \in \R^{n\times r}$, and (iii) is due to $\norm{AH}_{2,\infty} \leq \norm{A}_{2,\infty} \norm{H}$ for any matrices $A$ and $H$.  This completes the proof for $A_0$.
\end{proof}

\begin{proof}[Proof of \cref{claim:taus-control-B}]
Recall that we hope to show with probability $1-O(n^{-10^{7}})$, 
\begin{align}
    \frac{1}{\norm{Z}_{\F}^2}|B_0| &\lesssim \frac{\sigma \mu r^{1.5}\kappa \log^{2}(n)}{\sqrt{n}} \sigma_{\min} + \frac{\sigma \log^{0.5}(n)}{r\norm{Z}_{\F}} \sigma_{\min} \nonumber\\
    \frac{1}{\norm{Z}_{\F}^2}|B_1| &\lesssim \frac{\sigma \mu r^{1.5}\kappa \log^{2}(n)}{\sqrt{n}} \sigma_{\min}\nonumber\\
    \frac{1}{\norm{Z}_{\F}^2}|B_2| &\lesssim \frac{\sigma \mu r^{1.5}\kappa \log^{2}(n)}{\sqrt{n}} \sigma_{\min}\nonumber\\
    \frac{1}{\norm{Z}_{\F}^2}|B_3| &\lesssim \frac{\sigma \mu  r^{1.5}\kappa \log^{2.5}(n)}{\sqrt{n}} \sigma_{s}^{*} \nonumber
\end{align}
where
\begin{align*}
B_0 &= \inner{Z}{(XY^{\top}-M^{*} + (\tau - \tau^{*})Z) \left(y_s y^{\top}_{s}-y^{*}_{s}y^{*\top}_{s}\right)}\\
B_1 &= \inner{Z}{\lambda x_s y_s^{\top}}\\
B_2 &= \inner{Z}{-E y_s y^{\top}_s}\\
B_3 &= \inner{Z}{X(Y-Y^{*})^{\top}y^{*}_sy^{*\top}_s}.
\end{align*}

To control $B_0$, consider 
\begin{align}
    \frac{1}{\norm{Z}_{\F}^2}|B_0| 
    &\leq \underbrace{\left|\frac{\inner{Z}{(XY^{\top}-M^{*})(y_sy_s^{\top} - y_s^{*}y_s^{*\top})}}{\norm{Z}_{\F}^2}\right|}_{D_0} + \underbrace{\left|\frac{\inner{Z}{(\tau-\tau^{*})Z(y_sy_s^{\top} - y_s^{*}y_s^{*\top})}}{\norm{Z}_{\F}^2}\right|}_{D_1}.\nonumber
\end{align}

Note that $y_sy_s^{\top} - y_s^{*}y_s^{*\top} = y_s^{*}(y_s-y_s^{*})^{\top} + (y_s-y_s^{*})y_s^{\top}.$ Then 
\begin{align*}
\norm{y_sy_s^{\top} - y_s^{*}y_s^{*\top}}_{\F} 
&\leq (\norm{Y^{*}}+\norm{Y})\norm{Y-Y^{*}}_{\F} \\
&\lesssim \sqrt{\sigma_{\max}} \frac{\sigma\sqrt{n}\log^{2.5}(n)}{\sigma_{\min}} \norm{F^{*}}_{\F} \\
&\leq \sigma\sqrt{n}\log^{2.5}(n)\kappa r^{0.5}.
\end{align*}

This implies
\begin{align}
D_1 
&\leq |\tau-\tau^{*}| \frac{\norm{ZZ^{\top}}_{\F}}{\norm{Z}_{\F}^2} \norm{y_sy_s^{\top} - y_s^{*}y_s^{*\top}}_{\F}\nonumber\\
&\overset{(i)}{\leq} |\tau-\tau^{*}| \sigma\sqrt{n}\log^{2.5}(n)\kappa r^{0.5}\nonumber
\end{align}
where (i) is due to $\frac{\norm{ZZ^{\top}}_{\F}}{\norm{Z}_{\F}^2} \leq \frac{\norm{Z}_{\F}\norm{Z^{\top}}}{\norm{Z}_{\F}^2} \leq 1.$

Also, by $\frac{1}{\norm{Z}_{\F}^2}|\inner{Z}{A}| \leq \norm{A}_{\infty}$, one can verify that
\begin{align*}
    D_0 
    &\leq \norm{(XY^{\top}-M^{*})(y_sy_s^{\top} - y_s^{*}y_s^{*\top})}_{\infty} \\
    &\overset{(i)}{\leq} \norm{XY^{\top}-M^{*}}_{2,\infty} \norm{y_sy_s^{\top} - y_s^{*}y_s^{*\top}}_{2,\infty}\\
    &\leq \norm{(X-X^{*})Y^{*\top} + X(Y-Y^{*})^{\top}}_{2,\infty} \norm{y_s^{*}(y_s-y_s^{*})^{\top} + (y_s-y_s^{*})y_s^{\top}}_{2,\infty}\\
    &\overset{(ii)}{\leq} (\norm{X-X^{*}}_{2,\infty} \norm{Y^{*}} + \norm{X}_{2,\infty}\norm{Y-Y^{*}}) \\
    &\quad \cdot (\norm{Y^{*}}_{2,\infty} \norm{Y-Y^{*}} + \norm{Y-Y^{*}}_{2,\infty} \norm{Y}).
\end{align*}
Here, (i) is by $\norm{AB^{\top}}_{\infty} \leq \norm{A}_{2,\infty}\norm{B}_{2,\infty}$ for any matrices $A,B$, (ii) is by the triangle inequality and $\norm{AB}_{2,\infty} \leq \norm{A}_{2,\infty} \norm{B}.$ Recall that 
\begin{align*}
\max(\norm{X-X^{*}}_{2,\infty}, \norm{Y-Y^{*}}_{2,\infty}) &\leq \norm{F-F^{*}}_{2,\infty} \lesssim \frac{\sigma \mu r^{2.5} 
\kappa \log^{3.5}(n)}{\sigma_{\min}}\norm{F^{*}}_{\F}\\
\max(\norm{X}, \norm{Y}, \norm{X^{*}}, \norm{Y^{*}}) &\lesssim \sqrt{\sigma_{\max}}\\
\max(\norm{X-X^{*}}, \norm{Y-Y^{*}}) &\leq \norm{F-F^{*}}_{\F} \lesssim \frac{\sigma\sqrt{n}\log^{2.5}(n)}{\sigma_{\min}} \norm{F^{*}}_{\F}\\
\max(\norm{X}_{2,\infty}, \norm{Y}_{2,\infty}) &\lesssim \sqrt{\sigma_{\max}} \sqrt{\frac{\mu r}{n}}.
\end{align*}

We then have
\begin{align}
    D_0    
    &\lesssim \left(\frac{\sigma \mu r^{2.5} 
\kappa \log^{3.5}(n)}{\sigma_{\min}}\norm{F^{*}}_{\F} \sqrt{\sigma_{\max}} +  \frac{\sigma\sqrt{n}\log^{2.5}(n)}{\sigma_{\min}} \norm{F^{*}}_{\F} \sqrt{\sigma_{\max}} \sqrt{\frac{\mu r}{n}} \right)^2 \nonumber\\
    &\lesssim \left(\frac{\sigma \mu r^{2.5} 
\kappa \log^{3.5}(n)}{\sigma_{\min}}\norm{F^{*}}_{\F} \sqrt{\sigma_{\max}}\right)^2 \nonumber\\
    &\lesssim \left(\sigma \mu r^{3} \kappa^2 \log^{3.5}(n)\right)^2\nonumber\\
    &\lesssim \sigma^2 \mu^2 r^{6} \log^{7}(n) \kappa^4.\nonumber
\end{align}

This completes the bound for $B_0$ below.
\begin{align}
    \frac{1}{\norm{Z}_{\F}^2}|B_0| 
    &\leq |D_0| + |D_1| \nonumber\\
    &\lesssim  \sigma^2 \mu^2 r^{6} \log^{7}(n) \kappa^4 + |\tau-\tau^{*}| \sigma\sqrt{n}\log^{2.5}(n)\kappa r^{0.5}\nonumber\\
    &\lesssim \sigma^2 \mu^2 r^{6} \log^{7}(n) \kappa^4 + \left( \frac{\sigma \mu r^{2.5} \kappa \log^{3.5} n}{\sqrt{n}} + \frac{\sigma \log^{1.5} n}{\norm{Z}_{\F}}\right)\sigma\sqrt{n}\log^{2.5}(n)\kappa r^{0.5} \nonumber\\
    &\overset{(i)}{\lesssim} \sigma \mu r^{1.5} \log^{2}(n) \kappa \cdot (\sigma \kappa^3 r^{4.5} \log^{5}(n) \mu)  \nonumber\\
    &\quad + \left( \frac{\sigma \mu r^{2.5} \kappa \log^{3.5} n}{\sqrt{n}} + \frac{\sigma \log^{1.5} n}{\norm{Z}_{\F}}\right) \frac{\sigma_{\min}}{\kappa^2 r^{4}\log^{2.5}(n)} \nonumber\\
    &\overset{(ii)}{\lesssim} \sigma \mu r^{1.5}\kappa \log^{2}(n) \frac{\sigma_{\min}}{\sqrt{n}} + \frac{\sigma \mu \log(n)}{\kappa r^{1.5}} \frac{\sigma_{\min}}{\sqrt{n}} + \frac{\sigma}{r^2\kappa^4\log(n)\norm{Z}_{\F}} \sigma_{\min}\nonumber\\
    &\lesssim \sigma \mu r^{1.5}\kappa \log^{2}(n) \frac{\sigma_{\min}}{\sqrt{n}} + \frac{\sigma \log^{0.5}(n)}{r\norm{Z}_{\F}} \sigma_{\min} \nonumber
\end{align}
where (i) and (ii) are due to that $\frac{\sigma}{\sigma_{\min}} \sqrt{n} \lesssim \frac{1}{\kappa^3 r^{4.5} \log^{5}(n) \mu }$.

Next, consider $B_1$. Recall that $B_1 = \inner{Z}{\lambda x_s y_s^{\top}}$. We can directly obtain the desired bound $ \frac{1}{\norm{Z}_{\F}^2}|B_1| \lesssim \frac{\sigma \mu r^{1.5}\kappa \log^{2}(n)}{\sqrt{n}} \sigma_{\min}$ by the following.
\begin{align}
    \frac{1}{\norm{Z}_{\F}^2}|B_1| &\leq \norm{x_s}_{\infty} \norm{y_s}_{\infty} \lambda \nonumber\\
    &\lesssim \lambda \norm{F^{*}}_{2,\infty}^2\nonumber\\
    &\lesssim \sigma \sqrt{n} \log^{1.5}(n) \sigma_{\max} \frac{\mu r}{n}\nonumber\\
    &\lesssim \sigma \mu r \log^{1.5}(n) \frac{1}{\sqrt{n}} \sigma_{\max}\nonumber\\
    &\lesssim \frac{\sigma \mu r^{1.5}\kappa \log^{2}(n)}{\sqrt{n}} \sigma_{\min}.\nonumber
\end{align}

Next, consider $B_2$. Recall that $B_2 = \inner{Z}{-E y_s y^{\top}_s}$ and we aim to show $\frac{1}{\norm{Z}_{\F}^2}|B_2| \lesssim \frac{\sigma \mu r^{1.5}\kappa \log^{2}(n)}{\sqrt{n}} \sigma_{\min}$. Note that
\begin{align}
    \frac{1}{\norm{Z}_{\F}^2}|B_2| \leq \norm{Ey_s}_{\infty} \norm{y_s}_{\infty}. \nonumber
\end{align}
To bound $\norm{Ey_s}_{\infty}$, let $y_s^{(l)}$ be the $s$-th column of $F^{t,(l)}R^{t,(l)}$ for $l \in [n]$. Let $E_{l,\cdot} \in \R^{1\times n}$ be the $l$-th row of $E$. Then with probability $1-O(n^{-{10}^7})$
\begin{align*}
    |(Ey_s)_{l}| 
    &= |E_{l,\cdot} y_s|\\
    &\leq |E_{l,\cdot} y_s^{(l)}| + |E_{l,\cdot} (y_s^{(l)}-y_s)|\\
    &\overset{(i)}{\lesssim} \sigma\norm{y_s^{(l)}}\log^{0.5}(n) + \norm{E_{l,\cdot}} \norm{y_s^{(l)} - y_s}\\ 
    &\lesssim \sigma \norm{F^{t,(l)}R^{t,(l)}} \log^{0.5}(n) + \norm{E} \norm{F^{t,(l)}R^{t,(l)} - F^{t}H^{t}}_{\F}
\end{align*}
where (i) is due to independence between $E_{l,\cdot}$ and $y_s^{(l)}$ (Recall that the update of $F^{t,(l)}$ is independent from the $l$-th row of noise matrix $E$) and Hoeffding's inequality.

Recall that $\norm{F^{t,(l)}R^{t,(l)}} \lesssim \norm{F^{*}} = \sqrt{\sigma_{\max}}$, $\norm{E} \lesssim \sigma\sqrt{n}$, and $\norm{F^{t,(l)}R^{t,(l)} - F^{t}H^{t}}_{\F} \lesssim \frac{\sigma \log^{1.5} n}{\sigma_{\min}} \norm{F^{*}}_{\F}$, we then have
\begin{align*}        
    |(Ey_s)_{l}| 
    &\leq \sigma \sqrt{\sigma_{\max}} \log^{0.5}(n) + \sigma \sqrt{n}\frac{\sigma \log^{1.5} n}{\sigma_{\min}} \norm{F^{*}}_{\F} \\
    &\lesssim \sigma \sqrt{\sigma_{\max}} \log^{0.5}(n) + \sigma  \frac{\sigma \sqrt{n} \log^{1.5}(n) \sqrt{r}}{\sigma_{\min}} \sqrt{\sigma_{\max}}\\
    &\overset{(i)}{\lesssim} \sigma \sqrt{\sigma_{\max}} \log^{0.5}(n) + \sigma \sqrt{\sigma_{\max}}\\
    &\lesssim \sigma \sqrt{\sigma_{\max}} \log^{0.5}(n).
\end{align*}
Here, (i) is due to $\SNR.$ Then this directly implies 
\begin{align*}
    \norm{Ey_s}_{\infty} = \max_{l \in [n]} |(Ey_s)_{l}|  \lesssim \sigma \sqrt{\sigma_{\max}} \log^{0.5}(n).
\end{align*}
Therefore, we can obtain the desired bound for $B_2$ by the following. 
\begin{align}
    \frac{1}{\norm{Z}_{\F}^2}|B_2| 
    &\leq \norm{Ey_s}_{\infty} \norm{y_s}_{\infty} \nonumber\\
    &\lesssim   \sigma \sqrt{\sigma_{\max}} \log^{0.5}(n) \sqrt{\frac{\mu r \sigma_{\max}}{n}}\nonumber\\
    &\lesssim \frac{\sigma \mu r^{1.5}\kappa \log^{2}(n)}{\sqrt{n}} \sigma_{\min}.\nonumber
\end{align}

Finally, for $B_3$, recall that $B_3 = \inner{Z}{X(Y-Y^{*})^{\top}y^{*}_sy^{*\top}_s}.$ We can obtain the desired bound $\frac{1}{\norm{Z}_{\F}^2}|B_3| \lesssim \frac{\sigma \mu  r^{1.5}\kappa \log^{2.5}(n)}{\sqrt{n}} \sigma_{s}^{*}$ by the following. 
\begin{align}
\frac{1}{\norm{Z}_{\F}^2}|B_3| 
&\overset{(i)}{\leq} \norm{X(Y-Y^{*})^{\top}}_{2,\infty} \norm{y_s^{*}y_s^{*\top}}_{2,\infty}\nonumber \\
&\overset{(ii)}{\lesssim} \norm{X}_{2,\infty} \norm{Y-Y^{*}} \norm{y_s^{*}}_{2,\infty} \norm{y_s^{*}} \nonumber\\
&\lesssim \sqrt{\sigma_{\max}} \sqrt{\frac{\mu r}{n}} \frac{\sigma \sqrt{n} \log^{2.5}(n)}{\sigma_{\min}} \sqrt{\sigma_{\max} r} \sigma_{s}^{*}  \sqrt{\frac{\mu r}{n}} \nonumber\\ 
&\lesssim \sigma \kappa r^{1.5} \log^{2.5}(n) \mu \frac{\sigma_{s}^{*}}{\sqrt{n}}.\nonumber
\end{align}
Here, (i) is due to $\frac{1}{\norm{Z}_{\F}^2} |\inner{Z}{AB^{\top}}| \leq \norm{A}_{2,\infty}\norm{B}_{2,\infty}$ for any matrices $A, B$, and (ii) is due to $\norm{AB}_{2,\infty} \leq \norm{A}_{2,\infty}\norm{B}$ for any matrices $A, B$. This completes the proof of the bounds for $B_0, B_1, B_2, B_3$. 
\end{proof}

\subsection{Proof of Eq.~(\ref{eq:FH-FlRl})}
Recall that we aim to obtain $\norm{F^{t+1}H^{t+1} - F^{t+1, (l)} R^{t+1, (l)}}_{\F} \leq C_{l,1} \frac{\sigma \log^{1.5} n}{\sigma_{\min}} \norm{F^{*}}_{\F}.$ 
We will show the proof for $1\leq l \leq n$ below. The scenario that $n< l \leq 2n$ is similar and omitted for brevity. 

Note that $R^{t+1,(l)}$ is the optimal rotation to align $F^{t+1,(l)}$ and $F^{t+1}H^{t+1}$. In fact, we have 
$$
\norm{F^{t+1}H^{t+1} - F^{t+1,(l)}R^{t+1,(l)}}_{\F} \leq \norm{F^{t+1}H^{t} - F^{t+1,(l)}R^{t,(l)}}_{\F}.
$$
Otherwise, by contradiction, suppose we have
\begin{align*}
    \norm{F^{t+1}H^{t+1} - F^{t+1,(l)}R^{t+1,(l)}}_{\F} 
    &> \norm{F^{t+1}H^{t} - F^{t+1,(l)}R^{t,(l)}}_{\F}\\
    &\overset{(i)}{=} \norm{F^{t+1}H^{t}(H^{tT}H^{t+1}) - F^{t+1,(l)}R^{t,(l)}(H^{tT}H^{t+1})}_{\F}\\
    &= \norm{F^{t+1}H^{t+1} - F^{t+1,(l)}R^{t,(l)}H^{tT}H^{t+1}}_{\F}.
\end{align*}
Here, (i) is due to the unitary invariance of the $\norm{\cdot}_{\F}$. This contradicts with that $R^{t+1,(l)} = \arg\min_{A \in \O^{r\times r}} \norm{F^{t+1}H^{t+1} - F^{t+1,(l)}A}$. Hence we have
\begin{align*}
    \norm{F^{t+1}H^{t+1} - F^{t+1,(l)}R^{t+1,(l)}}_{\F} \leq \norm{F^{t+1}H^{t} - F^{t+1,(l)}R^{t,(l)}}_{\F}.
\end{align*}
Then it is sufficient to control $\norm{F^{t+1}H^{t} - F^{t+1,(l)}R^{t,(l)}}_{\F}.$

First consider $\norm{X^{t+1}H^{t} - X^{t+1,(l)}R^{t,(l)}}_{\F}.$ Recall the updating rule for $X^{t+1}$ and $X^{t+1, (l)}$ is
\begin{align*}
    X^{t+1} &= X^{t} - \eta \left((X^{t}Y^{tT} - M^{*} - E + (\tau^{t}-\tau^{*}) Z) Y^{t} + \lambda X^{t}\right) \\
    X^{t+1, (l)} &= X^{t,(l)} - \eta \left((X^{t,(l)}Y^{t,(l)\top} - M^{*} - E + P_{l}(E) + (\tau^{t,(l)}-\tau^{*}) Z) Y^{t,(l)} + \lambda X^{t,(l)}\right)
\end{align*}
We then have 
\begin{align}
    &X^{t+1,(l)}R^{t,(l)} - X^{t+1}H^{t}\nonumber\\
    &= \left( X^{t,(l)} - \eta \left((X^{t,(l)}Y^{t,(l)\top} - M^{*} - E + P_{l}(E) + (\tau^{t,(l)}-\tau^{*}) Z) Y^{t,(l)} + \lambda X^{t,(l)}\right) \right)R^{t,(l)}\nonumber\\
    &\quad - \left( X^{t} - \eta \left((X^{t}Y^{tT} - M^{*} - E + (\tau^{t}-\tau^{*}) Z) Y^{t} + \lambda X^{t}\right) \right) H^{t}\nonumber \\
    &= X^{t,(l)}R^{t,(l)} - \eta \left(X^{t,(l)}Y^{t,(l)\top} - M^{*} - E + P_{l}(E) + (\tau^{t,(l)} - \tau^{t})Z\right) Y^{t,(l)} R^{t,(l)} \nonumber\\
    &\quad - \eta \lambda X^{t,(l)} R^{t,(l)} - \left(X^{t}H^{t} - \eta (X^{t}Y^{tT} - M^{*} - E + (\tau^{t}-\tau^{*}) Z) Y^{t} H^{t} - \eta \lambda  X^{t}H^{t}\right) \nonumber\\
    &= (1-\eta \lambda) (X^{t,(l)}R^{t,(l)} - X^{t}H^{t}) \nonumber\\
    &\quad + \eta \underbrace{(X^{t,(l)}Y^{t,(l)\top} - M^{*} - E + P_{l}(E) + (\tau^{t,(l)}-\tau^{*}) Z)\left(Y^{t}H^{t} - 
    Y^{t,(l)}R^{t,(l)} \right)}_{A_0}\nonumber\\
    &\quad - \eta \left(X^{t,(l)}Y^{t,(l)\top} - X^{t}Y^{tT} + P_{l}(E) + (\tau^{t,(l)} - \tau^{t})Z\right) Y^{t} H^{t}\nonumber.
\end{align}

To simplify the rotation, we write $X^{t}H^{t}, Y^{t}H^{t}, F^{t}H^{t}, \tau^{t}$ as $X, Y, F, \tau$ if there is no ambiguity. Also, we write $X^{t,(l)}R^{t}, Y^{t,(l)}R^{t}, F^{t,(l)}R^{t}, \tau^{t,(l)}$ as $X^{(l)}, Y^{(l)}, F^{(l)}, \tau^{(l)}$ . Furthermore, we write $\Delta_{X} = X^{(l)} - X, \Delta_{Y} = Y^{(l)} - Y, \Delta_{F} = F^{(l)} - F.$  We have the following claim to control $A_0$.
\begin{claim}\label{claim:FH-FlRl-Control-A0}
With probability $1-O(n^{-10^{7}})$, 
\begin{align}
    \norm{A_0}_{\F} \lesssim \sigma \norm{F^{*}}_{\F}.\nonumber
\end{align}
\end{claim}

Then 
\begin{align}
    &\norm{X^{t+1,(l)}R^{t,(l)} - X^{t+1}H^{t}}_{\F}^2 \nonumber\\
    &= \norm{(1-\eta \lambda) \Delta_{X} - \eta  \left(\left(X^{(l)}Y^{(l)\top}-XY^{\top} + P_{l}(E) + (\tau^{(l)} - \tau)Z\right) Y-A_0\right)}_{\F}^2\nonumber\\
    &= (1-\eta \lambda)^2 \norm{\Delta_{X}}_{\F}^2 + A_1\nonumber\\
    &\quad - 2(1-\eta \lambda) \eta \underbrace{\inner{\Delta_{X}}{\left(\left(X^{(l)}Y^{(l)\top}-XY^{\top} + P_{l}(E) + (\tau^{(l)} - \tau)Z\right) Y-A_0\right)}}_{A_2}\nonumber
\end{align}
where $A_1$ includes the term with coefficient $\eta^2$:
\begin{align}
    A_1 := \eta^2 \norm{\left(X^{(l)}Y^{(l)\top}-XY^{\top} + P_{l}(E) + (\tau^{(l)} - \tau)Z\right) Y-A_0}_{\F}^2.\nonumber
\end{align}
Since we choose $\eta$ sufficiently small, one can verify that $A_1 \lesssim \eta \frac{\sigma^2}{n^{15}}$ (similar as \cref{eq:A3-sigma-square}).

Then, we proceed to analyze $A_2$. Note that
\begin{align}
    A_2 
    &= \inner{\Delta_{X}}{(X^{(l)}Y^{(l)\top} - XY^{\top})Y} + (\tau^{(l)} - \tau) \inner{\Delta_X}{ZY} \nonumber\\
    &\quad + \underbrace{\inner{\Delta_{X}}{P_{l}(E)Y}}_{B_0} - \inner{\Delta_X}{A_0}\nonumber\\
    &\overset{(i)}{=} \inner{X^{(l)}Y^{(l)\top} - XY^{\top}}{\Delta_{X}Y^{\top}} + \frac{\inner{Z}{XY^{\top} - X^{(l)}Y^{(l)\top} }\inner{Z}{\Delta_{X}Y^{\top}}}{\norm{Z}_{\F}^2}\nonumber \\
    &\quad + \underbrace{\frac{\inner{Z}{-P_{l}(E)}\inner{Z}{\Delta_{X}Y^{\top}}}{\norm{Z}_{\F}^2}}_{B_1} + B_0 - \inner{\Delta_{X}}{A_0}\nonumber\\
    &\overset{(ii)}{=} \inner{\Delta_{X}Y^{\top} + X\Delta_{Y}^{\top}}{\Delta_{X}Y^{\top}} + 
    \underbrace{\inner{\Delta_{X}\Delta_{Y}^{\top}}{\Delta_{X}Y^{\top}}}_{B_2} \nonumber\\
    &\quad - \frac{\inner{Z}{\Delta_{X}Y^{\top} + X\Delta_{Y}^{\top}} \inner{Z}{\Delta_{X}Y^{\top}}}{\norm{Z}_{\F}^2}\nonumber\\
    &\quad -\underbrace{\frac{\inner{Z}{\Delta_{X}\Delta_{Y}^{\top}} \inner{Z}{\Delta_{X}Y^{\top}}}{\norm{Z}_{\F}^2}}_{B_3} + B_1 + B_0 -\inner{\Delta_{X}}{A_0}\nonumber
\end{align}
where (i) is due to 
\begin{align*}
\tau^{(l)} - \tau &= \frac{\inner{Z}{O-X^{(l)}Y^{(l)\top}-P_{l}(E)}}{\norm{Z}_{\F}^2} -  \frac{\inner{Z}{O-XY^{\top}}}{\norm{Z}_{\F}^2} \\
&= \frac{\inner{Z}{XY^{\top}-X^{(l)}Y^{(l)\top}-P_{l}(E)}}{\norm{Z}_{\F}^2},
\end{align*} and (ii) is due to 
$X^{(l)}Y^{(l)\top} - XY^{\top} = \Delta_{X}Y^{\top} + X\Delta_{Y}^{\top} + \Delta_X \Delta_{Y}^{\top}.$

We have the following claim to control $B_0, B_1, B_2$ and $B_3$.
\begin{claim}\label{claim:FHl-Control-B}
With probability $1-O(n^{-10^{7}})$,
\begin{align}
    |B_0| &\lesssim \sigma \log^{0.5}(n)\norm{\Delta_{F}}_{\F}  \norm{F^{*}}_{\F} \nonumber\\
    |B_1| &\lesssim \sigma \log^{0.5}(n)\norm{\Delta_{F}}_{\F}  \norm{F^{*}}_{\F}\nonumber\\
    |B_2| &\lesssim \sigma \norm{\Delta_{F}}_{\F} \norm{F^{*}}_{\F} \nonumber\\
    |B_3| &\lesssim \sigma \norm{\Delta_{F}}_{\F} \norm{F^{*}}_{\F}.\nonumber
\end{align}
\end{claim}

Then, note that $\eta$ is choosing sufficiently small where $\eta \lambda \ll 1$, we have (for some constant $C$)
\begin{align}
    &\norm{X^{t+1,(l)}R^{t,(l)} - X^{t+1}H^{t}}_{\F}^2 \nonumber\\
    &= (1-\eta \lambda)^2 \norm{\Delta_{X}}_{\F}^2 + A_1 - 2(1-\eta \lambda)\eta A_2\nonumber\\
    &= (1-\eta \lambda)^2 \norm{\Delta_{X}}_{\F}^2 + A_1 - 2(1-\eta \lambda) \eta (B_0 + B_1 + B_2 - B_3 - \inner{\Delta_{X}}{A_0}) \nonumber\\
    &\quad - 2(1-\eta \lambda)\eta \left(\inner{\Delta_{X}Y^{\top} + X\Delta_{Y}^{\top}}{\Delta_{X}Y^{\top}} - \frac{\inner{Z}{\Delta_{X}Y^{\top} + X\Delta_{Y}^{\top}} \inner{Z}{\Delta_{X}Y^{\top}}}{\norm{Z}_{\F}^2}\right) \nonumber\\
    &\leq \norm{\Delta_{X}}_{\F}^2 + \eta  C \norm{\Delta_{F}}_{\F} \sigma \log^{0.5}(n) \norm{F^{*}}_{\F}\nonumber\\
    &\quad - 2(1-\eta \lambda)\eta \left(\inner{\Delta_{X}Y^{\top} + X\Delta_{Y}^{\top}}{\Delta_{X}Y^{\top}} - \frac{\inner{Z}{\Delta_{X}Y^{\top} + X\Delta_{Y}^{\top}} \inner{Z}{\Delta_{X}Y^{\top}}}{\norm{Z}_{\F}^2}\right).\nonumber
\end{align}

By the symmetry, we can similarly obtain the results for $Y^{t+1,(l)}R^{t,(l)} - Y^{t+1}H^{t}$
\begin{align}
    &\norm{Y^{t+1,(l)}R^{t,(l)} - Y^{t+1}H^{t}}_{\F}^2 \nonumber\\
    &\leq \norm{\Delta_{Y}}_{\F}^2 + \eta   C \norm{\Delta_{F}}_{\F} \sigma \log^{0.5}(n) \norm{F^{*}}_{\F} \nonumber\\
    &\quad - 2(1-\eta \lambda)\eta \left(\inner{\Delta_{X}Y^{\top} + X\Delta_{Y}^{\top}}{X\Delta_Y^{\top}} - \frac{\inner{Z}{\Delta_{X}Y^{\top} + X\Delta_{Y}^{\top}} \inner{Z}{X\Delta_Y^{\top}}}{\norm{Z}_{\F}^2}\right).\nonumber
\end{align}

Together implies that
\begin{align}
    &\norm{F^{t+1,(l)}R^{t,(l)} - F^{t+1}H^{t}}_{\F}^2 \nonumber\\
    &\leq \norm{\Delta_{F}}_{\F}^2 + 2\eta   C \norm{\Delta_{F}}_{\F} \sigma \log^{0.5}(n) \norm{F^{*}}_{\F} \nonumber\\
    &\quad - 2(1-\eta \lambda)\eta\left( \norm{\Delta_{X}Y^{\top} + X\Delta_{Y}^{\top}}_{\F}^2 - \frac{\left(\inner{Z}{\Delta_{X}Y^{\top} + X\Delta_{Y}^{\top}}\right)^2}{\norm{Z}_{\F}^2}\right) \nonumber. 
\end{align}
Let $T$ be the tangent space of $XY^{\top}$. We have (for some constant $C'$)
\begin{align*}
&\norm{\Delta_{X}Y^{\top} + X\Delta_{Y}^{\top}}_{\F}^2 - \frac{\left(\inner{Z}{\Delta_{X}Y^{\top} + X\Delta_{Y}^{\top}}\right)^2}{\norm{Z}_{\F}^2} \\
&=  \norm{\Delta_{X}Y^{\top} + X\Delta_{Y}^{\top}}_{\F}^2 - \frac{\inner{P_{T}(Z)}{\Delta_{X}Y^{\top} + X\Delta_{Y}^{\top}}^2}{\norm{Z}_{\F}^2}\\
&\geq \norm{\Delta_{X}Y^{\top} + X\Delta_{Y}^{\top}}_{\F}^2  \left(1 - \frac{\norm{P_{T}(Z)}_{\F}^2}{\norm{Z}_{\F}^2}\right)\\
&\overset{(i)}{\geq} \frac{C'}{\log(n)}\norm{\Delta_{X}Y^{\top} + X\Delta_{Y}^{\top}}_{\F}^2.
\end{align*}
Here (i) is due to the implication of \cref{cond:Z-condition-nonconvex} for $(X, Y)$ (by \cref{lem:general-conditions-small-ball}).

Then we have the following claim. The proof is similar to the analysis in \cref{claim:XDeltaY-lower-bound} by using the property that $\norm{X^{(l)\top}X - Y^{(l)\top}Y}_{\F}\lesssim \frac{\sigma}{n^{15}}$ (similar to \cref{eq:XX-YY}). We omit the proof for the brevity.
\begin{claim}
\begin{align}
    \norm{\Delta_{X}Y^{\top} + X\Delta_{Y}^{\top}}_{\F}^2 \geq \frac{\sigma_{\min}}{4}\norm{\Delta_{F}}_{\F}^2  - \frac{\sigma^2}{n^{13}}. \nonumber
\end{align}
\end{claim}

Then, we have
\begin{align}
    &\norm{F^{t+1,(l)}R^{t,(l)} - F^{t+1}H^{t}}_{\F}^2 \nonumber\\
    &\leq \norm{\Delta_{F}}_{\F}^2\left(1 - \eta \frac{C'\sigma_{\min}}{4\log(n)}\right)  +  2\eta \frac{C'\sigma^2}{\log(n)n^{13}} + 2\eta   C \norm{\Delta_{F}}_{\F} \sigma \log^{0.5}(n) \norm{F^{*}}_{\F} \nonumber\\
    &\leq \left(C_1 \frac{\sigma \log^{1.5} n}{\sigma_{\min}} \norm{F^{*}}_{\F}\right)^2 - \eta \underbrace{\frac{C'\sigma_{\min}}{4\log(n)}\left(C_1 \frac{\sigma \log^{1.5} n}{\sigma_{\min}} \norm{F^{*}}_{\F}\right)^2}_{K_1}\nonumber\\
    &\quad + \eta \underbrace{2\frac{C'\sigma^2}{\log(n)n^{13}}}_{K_2} + \eta \underbrace{2C\left(C_1 \frac{\sigma \log^{1.5} n}{\sigma_{\min}} \norm{F^{*}}_{\F}\right)\sigma \log^{0.5}(n) \norm{F^{*}}_{\F}}_{K_3}\nonumber.
\end{align}
In order to show $\norm{F^{t+1,(l)}R^{t,(l)} - F^{t+1}H^{t}}_{\F}^2 \leq \left(C_1 \frac{\sigma \log^{1.5} n}{\sigma_{\min}} \norm{F^{*}}_{\F}\right)^2$, it is sufficient to show that $K_1 \geq K_2 + K_3.$ Take $C_1\geq 4+8\frac{C}{C'}$, one can verify that
\begin{align*}
    K_1 - K_2 
    &= \left(\frac{C'C_1}{4} - 2C\right)\left(C_1 \frac{\sigma \log^{1.5} n}{\sigma_{\min}} \norm{F^{*}}_{\F}\right)  \sigma \log^{0.5}(n) \norm{F^{*}}_{\F}\\
    &\geq 2C' \frac{\sigma^2 \log^{2}(n)}{\sigma_{\min}} \norm{F^{*}}_{\F}^2\\
    &\geq 2C' \sigma^2 \log^2(n) \kappa\\
    &\geq K_3.
\end{align*}
This completes the proof for \cref{eq:FH-FlRl}.

\begin{proof}[Proof of \cref{claim:FH-FlRl-Control-A0}]
Recall that  we want to show $|A_0| \lesssim \sigma \norm{F^{*}}_{\F}$ where $A_0 = (X^{t,(l)}Y^{t,(l)\top} - M^{*} - E + P_{l}(E) + (\tau^{t,(l)}-\tau^{*}) Z)\left(Y^{t}H^{t} - Y^{t,(l)}R^{t,(l)}\right)$.

In fact, we have the following bounds for leave-one-out sequences similar to \cref{eq:tau-taustar} and \cref{eq:FtHt-Fstar-Fnorm}. The proof is exactly the same by viewing $E - P_{l}(E)$ as the noise. We omit it for brevity.  
\begin{align}
|\tau^{(l)} - \tau^{*}| &\lesssim \left( \frac{\sigma \mu r^{2.5} \kappa \log^{3.5} n}{\sqrt{n}} + \frac{\sigma \log^{1.5} n}{\norm{Z}_{\F}}\right)\nonumber\\
\norm{F^{(l)}H^{(l)} - F^{*}}_{\F} &\lesssim C_{\F} \error \norm{F^{*}}_{\F} .\nonumber
\end{align}
Then, with probability $1-O(n^{-10^{7}})$, 
\begin{align}
    |A_0| 
    &\leq \left(\norm{X^{(l)}Y^{(l)\top} - M^{*}} + \norm{E} + \norm{Z} |\tau^{(l)} - \tau^{*}|\right) \norm{\Delta_{F}}_{\F}\nonumber\\
    &\lesssim (\norm{F^{(l)}H^{(l)} - F^{*}}_{\F} \norm{F^{*}} + \sigma \sqrt{n} + \sigma \sqrt{n} \mu r^{2.5} \kappa \log^{3.5}(n))\norm{\Delta_{F}}_{\F}\nonumber\\
    &\lesssim (\sigma \sqrt{n}\log^{2.5}(n) \kappa r + \sigma \sqrt{n} \mu r^{2.5} \kappa \log^{3.5}(n)) \norm{\Delta_{F}}_{\F}\nonumber\\
    &\lesssim \sigma \norm{F^{*}}_{\F}\nonumber
\end{align}
providing that $\SNR$ and $\norm{\Delta_{F}}_{\F} \lesssim \frac{\sigma \log^{1.5}(n)}{\sigma_{\min}} \norm{F^{*}}_{\F}.$
\end{proof}

\begin{proof}[Proof of \cref{claim:FHl-Control-B}]


For $B_0$, recall that we want to show $|B_0| \lesssim \sigma\log^{0.5}(n)\norm{\Delta_{F}}_{\F}\norm{F^{*}}_{\F}$ where $B_0 = \inner{\Delta_{X}}{P_{l}(E)Y}$. With probability $1-O(n^{-10^{7}})$, we have
\begin{align}
    |B_0| 
    &\leq \norm{\Delta_{X}}_{\F} \norm{P_{l}(E) Y^{(l)}}_{\F} + \norm{\Delta_{X}}_{\F} \norm{P_{l}(E)} \norm{Y-Y^{(l)}}_{\F}\nonumber\\
    &\overset{(i)}{\lesssim} \norm{\Delta_{F}}_{\F} \sigma \log^{0.5}(n) \norm{F^{*}}_{\F} +  \norm{\Delta_{X}}_{\F} \norm{P_{l}(E)} \norm{Y-Y^{(l)}}_{\F}\nonumber\\
    &\lesssim \norm{\Delta_{F}}_{\F} \sigma \log^{0.5}(n) \norm{F^{*}}_{\F} + \norm{\Delta_{F}}_{\F} \sigma \sqrt{n} \frac{\sigma \log^{1.5}(n)}{\sigma_{\min}} \norm{F^{*}}_{\F}\nonumber\\
    &\lesssim \norm{\Delta_{F}}_{\F} \sigma \log^{0.5}(n) \norm{F^{*}}_{\F}\nonumber
\end{align}
where (i) is due to the independence between $P_{l}(E)$ and $Y^{(l)}$ and the Hanson-Wright inequality \cite{rudelson2013hanson}. 

For $B_1$, recall that we want to show $|B_1| \lesssim \sigma \log^{0.5}(n)\norm{\Delta_{F}}_{\F}  \norm{F^{*}}_{\F}$ where $B_1 = \frac{\inner{Z}{-P_{l}(E)}\inner{Z}{\Delta_{X}Y^{\top}}}{\norm{Z}_{\F}^2}$. With probability $1-O(n^{-10^{7}})$, we have 
$$
|\inner{Z}{P_{l}(E)}| \lesssim \sigma \sqrt{\log(n)} \norm{P_{l}(Z)}_{\F}
$$ by the Hoeffding's inequality, hence
\begin{align}
    |B_1| 
    &\lesssim \sigma \sqrt{\log(n)} \norm{P_{l}(Z)}_{\F} \norm{Z}_{\F} \norm{\Delta_{F}}_{\F} \norm{F^{*}}_{\F} \frac{1}{\norm{Z}_{\F}^2}\nonumber\\
    &\lesssim \sigma \sqrt{\log(n)}\norm{\Delta_{F}}_{\F} \norm{F^{*}}_{\F}.\nonumber
\end{align}

For $B_2$, recall that we want to show $|B_2| \lesssim \sigma \norm{\Delta_{F}}_{\F} \norm{F^{*}}_{\F}$ where $B_2 = \inner{\Delta_{X}\Delta_{Y}^{\top}}{\Delta_{X}Y^{\top}}$. We have
\begin{align}
    |B_2|
    &\lesssim \norm{\Delta_{F}}_{\F}^3 \norm{F^{*}}_{\F} \nonumber\\
    &\lesssim \norm{\Delta_{F}}_{\F} \norm{F^{*}}_{\F} \left(\frac{\sigma \log^{1.5}(n)}{\sigma_{\min}} \norm{F^{*}}_{\F}\right)^2\nonumber\\
    &\lesssim \norm{\Delta_{F}}_{\F} \norm{F^{*}}_{\F} \frac{\sigma^2 \log^{3}(n) r \kappa }{\sigma_{\min}} \nonumber\\
    &\lesssim \sigma \norm{\Delta_{F}}_{\F} \norm{F^{*}}_{\F}. \nonumber
\end{align}

For $B_3$, recall that we want to show $|B_3| \lesssim \sigma \norm{\Delta_{F}}_{\F} \norm{F^{*}}_{\F}$ where $B_3 = \frac{\inner{Z}{\Delta_{X}\Delta_{Y}^{\top}} \inner{Z}{\Delta_{X}Y^{\top}}}{\norm{Z}_{\F}^2}.$ Similar to $B_2$, we have
\begin{align}
    |B_3| 
    &\lesssim \norm{Z}_{\F}^2 \norm{\Delta_{F}}_{\F}^3 \norm{F^{*}}_{\F} \frac{1}{\norm{Z}_{\F}^2} \nonumber\\
    &\lesssim \sigma \norm{\Delta_{F}}_{\F} \norm{F^{*}}_{\F}. \nonumber
\end{align}
This completes the proof for $B_0, B_1, B_2, B_3.$
\end{proof}

\subsection{Proof of Eqs.~(\ref{eq:FlHl-Fstar}) and (\ref{eq:FH-Fstar-rownorm})}
Recall that we want to show (\cref{eq:FlHl-Fstar})
$$
\max_{1\leq l\leq 2n} \norm{(F^{t+1, (l)}H^{t+1, (l)} - F^{*})_{l, \cdot}}_{2} \leq C_{l,2} \frac{\sigma \mu r^{2.5} \kappa \log^{3.5}(n)}{\sigma_{\min}}\norm{F^{*}}_{\F}
$$  and $\norm{F^{t+1}H^{t+1} - F^{*}}_{2,\infty} \leq C_{\infty}\frac{\sigma \mu r^{2.5} \kappa \log^{3.5}(n)}{\sigma_{\min}}\norm{F^{*}}_{\F}$ (\cref{eq:FH-Fstar-rownorm}).

For \cref{eq:FlHl-Fstar}, we consider $1 \leq l \leq n$, where the scenario $n+1\leq  l \leq 2n$ is similar. By triangle inequality, 
\begin{align}
    \norm{(F^{t+1, (l)}H^{t+1, (l)} - F^{*})_{l, \cdot}}_{2} 
    &\leq \norm{(F^{t+1,(l)}H^{t,(l)} - F^{*})_{l,\cdot}}_{2} \nonumber\\
    &\quad + \underbrace{\norm{(F^{t+1,(l)}(H^{t,(l)} - H^{t+1,(l)}))_{l,\cdot}}_{2}}_{T_0}.\nonumber
\end{align}
We have the following claim.
\begin{claim}\label{claim:FlHl-Fstar-Control-T0}
\begin{align}
    |T_0| \lesssim \eta \sigma \mu^{0.5} \norm{F^{*}}_{\F}.\nonumber
\end{align}
\end{claim}

For $(F^{t+1,(l)}H^{t,(l)} - F^{*})_{l,\cdot}$, note that
\begin{align}
    &(F^{t+1,(l)}H^{t,(l)} - F^{*})_{l,\cdot} \nonumber\\
    &= (X^{t+1,(l)}H^{t,(l)} - X^{*})_{l,\cdot} \nonumber\\
    &= (X^{t, (l)}H^{t,(l)} - X^{*})_{l,\cdot}\nonumber \\
    &\quad - \eta ((X^{t,(l)}Y^{t,(l)\top} - M^{*} - E + P_{l}(E) + (\tau^{t,(l)} - \tau^{*})Z)Y^{t,(l)}H^{t,(l)} + \lambda X^{t,(l)}H^{t,(l)})_{l,\cdot}.\nonumber
\end{align}
For simplifying the notation, we write $X^{t,(l)}H^{t,(l)}, Y^{t,(l)}H^{t,(l)}$ as $X^{(l)}, Y^{(l)}$ and $F^{t,(l)}H^{t,(l)}, \tau^{t,(l)}$ as $F^{(l)}, \tau^{(l)}$. Furthermore, we write $\Delta_{X} = X^{(l)} - X^{*}, \Delta_{Y} = Y^{(l)} - Y^{*}, \Delta_{F} = F^{(l)} - F^{*}$. Then, we have
\begin{align}
    &\norm{(F^{t+1,(l)}H^{t,(l)} - F^{*})_{l,\cdot}}_{2}^2\nonumber\\
    &=\norm{(\Delta_X)_{l,\cdot}}_2^2 +A_0 \nonumber\\
    &\quad - 2\eta \underbrace{\inner{(\Delta_{X})_{l,\cdot}}{(X^{(l)}Y^{(l)\top} - M^{*} - E + P_{l}(E) + (\tau^{(l)} - \tau^{*})Z)_{l,\cdot}Y^{(l)} + \lambda X_{l,\cdot}}}_{A_1}\nonumber
\end{align}
where $A_0$ includes the term with $\eta^2$:
\begin{align}
    A_0 := \eta^2 \norm{(X^{(l)}Y^{(l)\top} - M^{*} - E + P_{l}(E) + (\tau^{(l)} - \tau^{*})Z)_{l,\cdot}Y^{(l)} + \lambda X_{l,\cdot}}_{2}^2.\nonumber
\end{align}
Since we choose $\eta$ sufficiently small, one can verify that $|A_0| \lesssim \eta \frac{\sigma^2}{n^{15}}$ (similar to \cref{eq:A3-sigma-square}). 

Next, consider $A_1$. Note that $(P_{l}(E) - E)_{l,\cdot} = 0$, then
\begin{align}
    A_1 
    &= \inner{(\Delta_{X})_{l,\cdot}}{(X^{(l)}Y^{(l)\top} - M^{*})_{l,\cdot}Y^{(l)}} \nonumber\\
    &\quad + \underbrace{\inner{(\Delta_{X})_{l,\cdot}}{(\tau^{(l)} - \tau^{*})Z_{l,\cdot} Y^{(l)}}}_{B_0}
    + \underbrace{\inner{(\Delta_{X})_{l,\cdot}}{\lambda X_{l,\cdot}}}_{B_1}\nonumber\\
    &= \underbrace{\inner{(\Delta_{X})_{l,\cdot}}{(\Delta_{X}Y^{(l)\top})_{l,\cdot}Y^{(l)}}}_{B_2} + \underbrace{\inner{(\Delta_{X})_{l,\cdot}}{(X^{*}\Delta_{Y}^{\top})_{l,\cdot}Y^{(l)}}}_{B_3}
    + B_0 + B_1.\nonumber
\end{align}
Note that
\begin{align}
    B_2 
    &= \tr((\Delta_{X})_{l,\cdot}^{\top}(\Delta_{X})_{l,\cdot} Y^{(l)\top} Y^{(l)})\nonumber\\
    &= \tr((\Delta_{X})_{l,\cdot}^{\top}(\Delta_{X})_{l,\cdot} (Y^{(l)\top} Y^{(l)}-\sigma_r(Y^{(l)\top} Y^{(l)}) I_{r})) + \sigma_{r}(Y^{(l)\top} Y^{(l)})\tr((\Delta_{X})_{l,\cdot}^{\top}(\Delta_{X})_{l,\cdot})\nonumber\\
    &\overset{(i)}{\geq} \sigma_{r}(Y^{(l)\top} Y^{(l)})\tr((\Delta_{X})_{l,\cdot}^{\top}(\Delta_{X})_{l,\cdot})\nonumber\\
    &\overset{(ii)}{\geq} \frac{\sigma_{\min}}{2} \norm{(\Delta_{X})_{l,\cdot}}_2^2 \nonumber
\end{align}
where (i) is due to $\tr(AB) \geq 0$ if $A,B$ are positive semi-definite matrices, (ii) is due to $\sigma_{r}(Y^{(l)}) \geq \sqrt{\frac{\sigma_{\min}}{2}}.$ We also have the following claim for controlling $B_0, B_1$ and $B_3$.
\begin{claim}\label{claim:FlHl-Fstar-Control-B}
\begin{align}
    |B_0| &\lesssim \sigma \mu r^{2.5} \log^{3.5}(n) \kappa \norm{\Delta_{F}}_{2,\infty} \norm{F^{*}}_{\F}\nonumber\\
    |B_1| &\lesssim \sigma r^{0.5} \mu^{0.5} \kappa^{0.5} \log^{1.5}(n)\norm{\Delta_{F}}_{2,\infty} \norm{F^{*}}_{\F} \nonumber\\
    |B_3| &\lesssim \sigma \kappa r^{1.5} \mu^{0.5} \log^{2.5}(n) \norm{\Delta_{X}}_{2,\infty} \norm{F^{*}}_{\F}.\nonumber
\end{align}
\end{claim}

Finally, take $C_2 > 4C$ large enough and note that $\norm{\Delta_{F}}_{2,\infty} \lesssim C_2 \frac{\sigma \mu r^{2.5} \kappa \log^{3.5}(n)}{\sigma_{\min}}\norm{F^{*}}_{\F}$, we arrive at
\begin{align}
     &\norm{(F^{t+1,(l)}H^{t,(l)} - F^{*})_{l,\cdot}}_{2}^2\nonumber\\
    &=\norm{(\Delta_X)_{l,\cdot}}_2^2 +A_0 - 2\eta A_1\nonumber\\
    &= \norm{(\Delta_X)_{l,\cdot}}_2^2 - 2\eta (B_2 + B_3 + B_0 + B_1) + A_0\nonumber\\
    &\leq \norm{(\Delta_X)_{l,\cdot}}_2^2 - 2\eta \frac{\sigma_{\min}}{2} \norm{(\Delta_X)_{l,\cdot}}_2^2 + C\eta \sigma \mu r^{2.5} \log^{3.5}(n) \kappa \norm{\Delta_{F}}_{2,\infty} \norm{F^{*}}_{\F}\nonumber\\
    &\leq \left(C_2 \frac{\sigma \mu r^{2.5} \kappa \log^{3.5}(n)}{\sigma_{\min}}\norm{F^{*}}_{\F}\right)^2 (1-\eta \sigma_{\min})\nonumber\\
    &\quad + C\eta \left(C_2 \frac{\sigma \mu r^{2.5} \kappa \log^{3.5}(n)}{\sigma_{\min}}\norm{F^{*}}_{\F}\right)\sigma \mu r^{2.5} \log^{3.5}(n) \norm{F^{*}}_{\F}\kappa\nonumber\\
    &\leq \left(C_2 \frac{\sigma \mu r^{2.5} \kappa \log^{3.5}(n)}{\sigma_{\min}}\norm{F^{*}}_{\F}\right)^2 \left(1 - \frac{\sigma_{\min}}{4}\eta \right)^2\nonumber.
\end{align}
Furthermore
\begin{align}
    &\norm{(F^{t+1, (l)}H^{t+1, (l)} - F^{*})_{l, \cdot}}_{2} \nonumber\\
    &\leq \norm{(F^{t+1,(l)}H^{t,(l)} - F^{*})_{l,\cdot}}_{2} + T_0\nonumber\\
    &\leq  \left(C_2 \frac{\sigma \mu r^{2.5} \kappa \log^{3.5}(n)}{\sigma_{\min}}\norm{F^{*}}_{\F}\right) \left(1 - \frac{\sigma_{\min}}{4}\eta \right) + C\eta \sigma \mu^{0.5} \norm{F^{*}}_{\F}\nonumber\\
    &\leq C_2 \frac{\sigma \mu r^{2.5} \kappa \log^{3.5}(n)}{\sigma_{\min}}\norm{F^{*}}_{\F}\nonumber
\end{align}
This completes the proof for \cref{eq:FlHl-Fstar}. 

Then, for \cref{eq:FH-Fstar-rownorm}, note that by triangle inequality, for each $l \in [n]$,
\begin{align*}
    \norm{F^{t+1}H^{t+1}-F^{*}}_{2,\infty} 
    &\leq \norm{F^{t+1,(l)}H^{t+1,(l)} - F^{*}}_{2,\infty} + \norm{F^{t+1,(l)}H^{t+1,(l)} - F^{t+1}H^{t+1}}_{\F}\\
    &\leq  C_2 \frac{\sigma \mu r^{2.5} \kappa \log^{3.5}(n)}{\sigma_{\min}}\norm{F^{*}}_{\F} + \norm{F^{t+1,(l)}H^{t+1,(l)} - F^{t+1}H^{t+1}}_{\F}.
\end{align*}
Then it is sufficient to consider $\norm{F^{t+1,(l)}H^{t+1,(l)} - F^{t+1}H^{t+1}}_{\F}$. Take $F_0 = F^{*}, F_1 = F^{t+1}H^{t+1}, F_2 = F^{t+1,(l)}R^{t+1,(l)}$. By \cref{lem:F0-F1-F2} and the definition of $H^{t+1}, H^{t+1,(l)}$ and $R^{t+1,(l)}$, we then have
\begin{align}
    \norm{F^{t+1,(l)}H^{t+1,(l)} - F^{t+1}H^{t+1}}_{\F} 
    &\leq \norm{F_1 - F_2}_{\F} \nonumber\\
    &\lesssim \kappa \frac{\sigma \log^{1.5}(n)}{\sigma_{\min}} \norm{F^{*}}_{\F}  \label{eq:FtlHtl-FtHt-Fnorm}.
\end{align}
This implies, for large enough $n$, 
\begin{align*}
     \norm{F^{t+1}H^{t+1}-F^{*}}_{2,\infty} 
     &\leq  C_2 \frac{\sigma \mu r^{2.5} \kappa \log^{3.5}(n)}{\sigma_{\min}}\norm{F^{*}}_{\F} + C\kappa \frac{\sigma \log^{1.5}(n)}{\sigma_{\min}} \norm{F^{*}}_{\F}\\
     &\leq 2C_2 \frac{\sigma \mu r^{2.5} \kappa \log^{3.5}(n)}{\sigma_{\min}}\norm{F^{*}}_{\F}
\end{align*}
which completes the proof. 

\begin{proof}[Proof of \cref{claim:FlHl-Fstar-Control-T0}]
Recall that we want to show $|T_0| \lesssim \eta \sigma \mu^{0.5} \norm{F^{*}}_{\F}$ where 
$$
T_0 = \norm{(F^{t+1,(l)}(H^{t,(l)} - H^{t+1,(l)}))_{l,\cdot}}_{2}.
$$ 
Note that
\begin{align}
    |T_0| &\lesssim \norm{F^{t+1,(l)}}_{2,\infty} \norm{H^{t,(l)}-H^{t+1,(l)}} \nonumber\\
    &\lesssim \sqrt{\sigma_{\max} r} \sqrt{\frac{\mu r}{n}}  \norm{H^{t,(l)}-H^{t+1,(l)}}\nonumber\\
    &\overset{(i)}{\lesssim} \sqrt{\sigma_{\max} r} \sqrt{\frac{\mu r}{n}}   \eta \frac{\sigma \sqrt{n} \log^{2.5}(n)}{\sigma_{\min}} (\sigma \sqrt{n} \mu r^{3} \kappa^{2} \log^{3.5}(n))\nonumber\\   
    &\lesssim \eta \sigma^2 \sqrt{n} \mu^{1.5} r^{4} \kappa^{2.5} \log^{5}(n) \frac{1}{\sqrt{\sigma_{\min}}}\nonumber\\
    &\lesssim \eta \frac{\sigma^2 \sqrt{n} \mu^{1.5} r^{4.5} \kappa^{2.5} \log^{5}(n)}{\sigma_{\min}}\norm{F^{*}}_{\F}\nonumber\\
    &\lesssim \eta \sigma \mu^{0.5} \norm{F^{*}}_{\F}\nonumber
\end{align}
where (i) is by the bound of 
\begin{align}
    \norm{H^{t,(l)}-H^{t+1,(l)}}_{\F} \lesssim \eta \frac{\sigma \sqrt{n} \log^{2.5}(n)}{\sigma_{\min}} (\sigma \sqrt{n} \mu r^{3} \kappa^{2} \log^{3.5}(n)).\nonumber
\end{align}
This proof of the bound of $\norm{H^{t,(l)}-H^{t+1,(l)}}_{\F}$ is exactly the same as the \cref{claim:taus-taustars-A0} by viewing $E - P_{l}(E)$ as the noise matrix, where we omit for the brevity. 
\end{proof}

\begin{proof}[Proof of \cref{claim:FlHl-Fstar-Control-B}]
For $B_0$, recall that we want to show 
$$
B_0 \lesssim \sigma \mu r^{2.5} \log^{3.5}(n) \kappa \norm{\Delta_{F}}_{2,\infty} \norm{F^{*}}_{\F}
$$ 
where $B_0 = \inner{(\Delta_{X})_{l,\cdot}}{(\tau^{(l)} - \tau^{*})Z_{l,\cdot} Y^{(l)}}.$ We have
\begin{align}
    |B_0| 
    &\leq |\tau^{(l)} - \tau^{*}| \norm{\Delta_{X}}_{2,\infty} \norm{ZY^{(l)}}_{2,\infty} \nonumber\\
    &\lesssim \left(\frac{\sigma \mu r^{2.5} \log^{3.5}(n) \kappa}{\sqrt{n}} + \frac{\sigma \log^{1.5}(n)}{\norm{Z}_{\F}}\right) \norm{\Delta_{F}}_{2,\infty}
    \norm{Z}_{2,\infty} \norm{Y^{*}}_{\F} \nonumber\\
    &\lesssim \left(\frac{\sigma \mu r^{2.5} \log^{3.5}(n) \kappa}{\sqrt{n}} + \frac{\sigma \log^{1.5}(n)}{\norm{Z}_{\F}}\right) \norm{\Delta_{F}}_{2,\infty}
    \min(\norm{Z}_{\F}, \sqrt{n}) \norm{Y^{*}}_{\F}\nonumber\\
    &\lesssim \sigma \mu r^{2.5} \log^{3.5}(n) \kappa \norm{\Delta_{F}}_{2,\infty} \norm{F^{*}}_{\F}. \nonumber
\end{align}

For $B_1$, recall that we want to show $B_1 \lesssim \sigma r^{0.5} \mu^{0.5} \kappa^{0.5} \log^{1.5}(n)\norm{\Delta_{F}}_{2,\infty} \norm{F^{*}}_{\F}$ where $B_1 = \inner{(\Delta_{X})_{l,\cdot}}{\lambda X_{l,\cdot}}.$ We have
\begin{align}
    |B_1| 
    &\leq \lambda \norm{\Delta_{X}}_{2,\infty} \norm{X}_{2,\infty} \nonumber\\
    &\lesssim \sigma \sqrt{n} \log^{1.5}(n) \norm{\Delta_{X}}_{2,\infty} \sqrt{\sigma_{\max} r} \sqrt{\frac{\mu r}{n}} \nonumber\\
    &\lesssim \sigma r^{0.5} \mu^{0.5} \kappa^{0.5} \log^{1.5}(n)\norm{\Delta_{X}}_{2,\infty} \sqrt{\sigma_{\min} r} \nonumber\\
    &\lesssim \sigma r^{0.5} \mu^{0.5} \kappa^{0.5} \log^{1.5}(n)\norm{\Delta_{F}}_{2,\infty} \norm{F^{*}}_{\F}. \nonumber
\end{align}

For $B_3$, recall that we want to show $B_3 \lesssim \sigma \kappa r^{1.5} \mu^{0.5} \log^{2.5}(n) \norm{\Delta_{X}}_{2,\infty} \norm{F^{*}}_{\F}$ where $B_3 = \inner{(\Delta_{X})_{l,\cdot}}{(X^{*}\Delta_{Y}^{\top})_{l,\cdot}Y^{(l)}}.$ We have
\begin{align}
    |B_3| 
    &\leq \norm{\Delta_{X}}_{2,\infty} \norm{X^{*} \Delta_{Y}^{\top} Y^{(l)}}_{2,\infty} \nonumber\\
    &\lesssim \norm{\Delta_{X}}_{2,\infty} \norm{X^{*}}_{2,\infty} \norm{\Delta_{Y}}_{\F} \norm{Y^{(l)}} \nonumber\\
    &\lesssim \sqrt{\sigma_{\max} r} \sqrt{\frac{\mu r}{n}} \frac{\sigma \sqrt{n} \log^{2.5}(n)}{\sigma_{\min}} \sqrt{\sigma_{\max} r} \norm{\Delta_{X}}_{2,\infty} \norm{F^{*}}_{\F} \nonumber\\
    &\lesssim \sigma \kappa r^{1.5} \mu^{0.5} \log^{2.5}(n) \norm{\Delta_{X}}_{2,\infty} \norm{F^{*}}_{\F}.\nonumber
\end{align}
This completes the proof for $B_0, B_1, B_3.$ 
\end{proof}

\subsection{Proof of Eq.~(\ref{eq:PTM-bound})}

Let $T$ be the tangent space of $X^{t}Y^{tT}$. Recall we want to show $\norm{P_{T^{\perp}}(M^{*})}_{\infty} \leq  C_{T,1}\frac{\sigma^2 \mu^2 r^{6} \kappa^{3} \log^{7}(n)}{\sigma_{\min}}$.

Let $X = X^{t}H^{t}, Y = Y^{t}H^{t}, \Delta_{X} = X - X^{*}, \Delta_{Y} = Y - Y^{*}$ if there is no ambiguity. Then $T = \text{span}\{XA^{\top} + BY^{\top}~|~A, B\in \R^{n\times r}\}.$

Note that we have
\begin{align*}
    P_{T^{\perp}}(M^{*}) 
    &= P_{T^{\perp}}(X^{*}Y^{*\top})\\
    &= P_{T^{\perp}}((X-\Delta_{X})(Y-\Delta_{Y})^{\top})\\
    &\overset{(i)}{=} P_{T^{\perp}}(\Delta_{X}\Delta_{Y}^{\top})
\end{align*}
where (i) is due to $P_{T^{\perp}}(XY^{\top}) = P_{T^{\perp}}(\Delta_{X}Y^{\top}) = P_{T^{\perp}}(X\Delta_{Y}^{\top}) = 0.$ 

Let $U\Sigma V^{\top}$ be the SVD of $XY^{\top}$. We also have that $U^{*}\Sigma^{*}V^{*\top}$ is the SVD of $X^{*}Y^{*\top}.$ By \cref{eq:FtHt-Fstar-Fnorm}, \cref{eq:FH-Fstar-rownorm}, and \cref{eq:XX-YY}, we can invoke \cref{lem:XY-UV} and obtain that there exists a rotation matrix $R \in \O^{r\times r}$ such that
\begin{align}
    \max(\norm{UR - U^{*}}_{\F}, \norm{VR - V^{*}}_{\F}) &\lesssim \frac{r\kappa \sigma \sqrt{n}\log^{2.5}(n)}{\sigma_{\min}} \label{eq:UR-Ustar-Fnorm}\\
    \max(\norm{UR - U^{*}}_{2,\infty}, \norm{VR - V^{*}}_{2,\infty}) &\lesssim  \frac{r^{3}\mu \kappa^{1.5}\sigma \log^{3.5}(n)}{\sigma_{\min}}. \label{eq:UR-2-infty-bound}
\end{align}

Then consider
\begin{align*}
    \norm{P_{T^{\perp}}(M^{*})}_{\infty} 
    &= \norm{P_{T^{\perp}}(\Delta_{X}\Delta_{Y}^{\top})}_{\infty} \\
    &= \norm{(I-UU^{\top})\Delta_{X}\Delta_{Y}^{\top}(I-VV^{\top})}_{\infty}\\
    &\leq \norm{(I-UU^{\top})\Delta_{X}}_{2,\infty} \norm{(I-VV^{\top})\Delta_{Y}}_{2,\infty}.
\end{align*}

From \cref{eq:UR-2-infty-bound}, we have $\norm{U}_{2,\infty} \lesssim \norm{U^{*}}_{2,\infty} + \frac{r^{3}\mu \kappa^{1.5}\sigma \log^{3.5}(n)}{\sigma_{\min}} \lesssim \sqrt{\frac{\mu r}{n}}$ due to the incoherence condition and $\SNR$. Then
\begin{align*}
    \norm{(I-UU^{\top})\Delta_{X}}_{2,\infty} 
    &\leq \norm{\Delta_{X}}_{2,\infty} + \norm{U}_{2,\infty} \norm{U^{\top}} \norm{\Delta_{X}}\\
    &\lesssim \frac{\sigma \mu r^{3}\kappa \log^{3.5}(n) \sqrt{\sigma_{\max}}}{\sigma_{\min}} + \sqrt{\frac{\mu r}{n}} \frac{\sigma \sqrt{n}\log^{2.5}(n) \sqrt{\sigma_{\max} r}}{\sigma_{\min}}\\
    &\lesssim \frac{\sigma \mu r^{3}\kappa^{1.5} \log^{3.5}(n) }{\sqrt{\sigma_{\min}}}.
\end{align*}
Similarly, one can also obtain that $\norm{(I-VV^{\top})\Delta_{Y}}_{2,\infty} \lesssim \frac{\sigma \mu r^{3}\kappa^{1.5} \log^{3.5}(n) }{\sqrt{\sigma_{\min}}}.$ This completes the proof for \cref{eq:PTM-bound}:
\begin{align*}
    \norm{P_{T^{\perp}}(M^{*})}_{\infty} 
    &\lesssim \frac{\sigma \mu r^{3}\kappa^{1.5} \log^{3.5}(n) }{\sqrt{\sigma_{\min}}}\cdot \frac{\sigma \mu r^{3}\kappa^{1.5} \log^{3.5}(n) }{\sqrt{\sigma_{\min}}}\\
    &\lesssim \frac{\sigma^2 \mu^2 r^{6} \kappa^{3} \log^{7}(n)}{\sigma_{\min}}.
\end{align*}

\subsection{Proof of Eq.~(\ref{eq:PTE-bound})}
Recall that we want to show $\norm{P_{T_t^{\perp}}(E) - P_{T^{*\perp}}(E)}_{\infty} \leq C_{T,2}\frac{\sigma^2 r^{3.5} \mu^{1.5} \kappa^{2} \log^{4}(n)}{\sigma_{\min}}.$

Let $U'\Sigma V^{'\top}$ be the SVD of $X^{t}Y^{tT}$, $U^{*}\Sigma^{*}V^{*\top}$ be the SVD of $X^{*}Y^{*\top}$. By \cref{eq:FtHt-Fstar-Fnorm},\cref{eq:FH-Fstar-rownorm}, and \cref{eq:XX-YY}, we can invoke \cref{lem:XY-UV} and obtain that there exists a rotation matrix $R \in \O^{r\times r}$ such that 
\begin{align*}
    \max(\norm{U'R - U^{*}}_{\F}, \norm{V'R - V^{*}}_{\F}) &\lesssim \frac{r\kappa \sigma \sqrt{n}\log^{2.5}(n)}{\sigma_{\min}}\\
    \max(\norm{U'R - U^{*}}_{2,\infty}, \norm{V'R - V^{*}}_{2,\infty}) &\lesssim  \frac{r^{3}\mu \kappa^{1.5}\sigma \log^{3.5}(n)}{\sigma_{\min}}. 
\end{align*}
One particular choice of $R$ is $R = U_{Q}V_Q^{\top}$ where $U_Q\Sigma_QV_Q^{\top}$ is the SVD of $\Sigma^{-1/2}U^{'\top}X^{t}H^{t}.$ 

Let $U = U'R, V = V'R$ for simplification. Note that
\begin{align*}
    &\norm{P_{T_t^{\perp}}(E) - P_{T^{*\perp}}(E)}_{\infty}\\
    &\leq \norm{(I-UU^{\top})E(I-VV^{\top}) - (I-U^{*}U^{*\top})E(I-V^{*}V^{*\top})}_{\infty}\\
    &\leq \norm{(U^{*}U^{*\top}-UU^{\top})E(I-VV^{\top})}_{\infty} + \norm{(I-U^{*}U^{*\top})E(V^{*}V^{*\top} - VV^{\top})}_{\infty}.
\end{align*}
We have
\begin{align*}
    &\norm{(U^{*}U^{*\top}-UU^{\top})E(I-VV^{\top})}_{\infty} \\
    &\leq  \norm{(U^{*}U^{*\top}-UU^{\top})E}_{\infty} + \norm{(U^{*}U^{*\top}-UU^{\top})EVV^{\top}}_{\infty} \\
    &\leq \norm{(U^{*}-U)U^{*\top}E}_{\infty} + \norm{U(U^{*}-U)^{\top}E}_{\infty} \\
    &\quad + \norm{(U^{*}-U)U^{*\top}EVV^{\top}}_{\infty} +  \norm{U(U^{*}-U)^{\top}EVV^{\top}}_{\infty}\\
    &\leq  \underbrace{\norm{(U^{*}-U)}_{2,\infty} \norm{E^{\top}U^{*}}_{2,\infty}}_{A_0} + \underbrace{\norm{U}_{2,\infty} \norm{E^{\top}(U^{*}-U)}_{2,\infty}}_{A_1} \\
    &\quad + \underbrace{\norm{(U^{*}-U)}_{2,\infty} \norm{U^{*}} \norm{E} \norm{V} \norm{V}_{2,\infty} + \norm{U}_{2,\infty} \norm{U-U^{*}} \norm{E} \norm{V} \norm{V}_{2,\infty}}_{A_2}. 
\end{align*}
For $A_2$, given that $\norm{E} \lesssim \sigma \sqrt{n}$, \cref{eq:UR-Ustar-Fnorm} and \cref{eq:UR-2-infty-bound}, it is easy to check that
\begin{align*}
    A_2 \lesssim \frac{\sigma^2 r^{3.5} \mu^{1.5} \kappa^{1.5} \log^{3.5}(n)}{\sigma_{\min}}. 
\end{align*}

For $A_0$, one can verify that $\norm{E^{\top}U^{*}}_{2,\infty} \lesssim \sigma \sqrt{r\log(n)}$ with probability $1-O(n^{-10^6})$ since $E$ and $U^{*}$ are independent. This provides
\begin{align*}
    A_0 \lesssim \frac{\sigma^2 r^{3.5} \mu \kappa^{1.5} \log^{4}(n)}{\sigma_{\min}}. 
\end{align*}

For $A_1$, we intend to use the leave-one-out technique. Let $U^{'(l)}\Sigma^{(l)}V^{'(l)}$ be the SVD of $X^{(l),t}(Y^{(l),t})^{\top}.$  By \cref{lem:XY-UV}, we can obtain that there exists a rotation matrix $R^{(l)}$ such that 
\begin{align*}
    \norm{U^{'(l)}R^{(l)} - U^{*}}_{\F} \lesssim \frac{r\kappa \sigma \sqrt{n}\log^{2.5}(n)}{\sigma_{\min}}.
\end{align*}
In particular, $R^{(l)} = U_{Q^{(l)}}V_{Q^{(l)}}^{\top}$ where $U_{Q^{(l)}}\Sigma_{Q^{(l)}}V_{Q^{(l)}}^{\top}$ is the SVD of $\Sigma^{(l) -1/2}U^{'(l)\top}X^{(l),t}H^{(l),t}.$

Let $U^{(l)} = U^{'(l)}R^{(l)}$. In fact, due to the construction of $R$ and $R^{(l)}$, by \cref{lem:XY-UV}, one can further verify that
\begin{align*}
    \norm{U^{'(l)}R^{(l)} - U'R}_{\F} \lesssim \frac{\sqrt{\kappa} \sqrt{r}}{\sqrt{\sigma_{\min}}}\norm{F^{t}H^{t}-F^{t,(l)}H^{t,(l)}}_{\F} \overset{(i)}{\lesssim} \frac{\sigma r\kappa^{2} \log^{1.5}(n)}{\sigma_{\min}}.
\end{align*}
Here, (i) is due to \cref{eq:FtlHtl-FtHt-Fnorm}. Then for $n< l \leq 2n$, we have with probability $1-O(n^{10^{-6}})$, 
\begin{align*}
    \norm{(E^{\top}(U^{*}-U))_{l-n,\cdot}} 
    &\leq \norm{(E^{\top}(U^{(l)}-U^{*}))_{l-n,\cdot}} + \norm{E^{\top}(U^{(l)}-U)}_{2,\infty} \\
    &\overset{(i)}{\lesssim} \sigma \norm{U^{(l)} - U^{*}}_{\F} \sqrt{\log(n)} + \sigma \sqrt{n} \norm{U^{(l)}-U}_{\F} \\
    &\lesssim \frac{\sigma^2 \sqrt{n} r \kappa^{2} \log^{3}(n)}{\sigma_{\min}}
\end{align*}
where (i) is due to the independence between $E^{\top}_{l-n}$ and $U^{(l)}, U^{*}.$ This implies that
\begin{align*}
    A_1 
    &\lesssim \sqrt{\frac{\mu r}{n}}\norm{E^{\top}(U^{*}-U)}_{2,\infty} \\
    &\lesssim \frac{\sigma^2 r^{1.5} \mu^{0.5} \kappa^{2} \log^{3.5}(n)}{\sigma_{\min}}.
\end{align*}
In conclusion,
\begin{align*}
    \norm{(U^{*}U^{*\top}-UU^{\top})E(I-VV^{\top})}_{\infty} 
    &\leq A_0 + A_1 + A_2\\
    &\lesssim \frac{\sigma^2 r^{3.5} \mu^{1.5} \kappa^{2} \log^{4}(n)}{\sigma_{\min}}.
\end{align*}
This similar results hold for $\norm{(I-U^{*}U^{*\top})E(V^{*}V^{*\top} - VV^{\top})}_{\infty}$. Therefore, 
\begin{align*}
    \norm{P_{T_t^{\perp}}(E) - P_{T^{*\perp}}(E)}_{\infty} \lesssim \frac{\sigma^2 r^{3.5} \mu^{1.5} \kappa^{2} \log^{4}(n)}{\sigma_{\min}},
\end{align*}
which completes the proof.

\subsection{Proof of Eq.~(\ref{eq:PTZ-bound})}
Follow the derivation for the proof of \cref{eq:PTE-bound}, we have
\begin{align*}
    &\norm{P_{T^{\perp}}(Z) - P_{T^{*\perp}}(Z)}_{\F} \\
    &\leq \norm{(U^{*}U^{*\top} - UU^{\top})Z(I-VV^{\top})}_{\F} + \norm{(I-U^{*}U^{*\top})Z(V^{*}V^{*\top}-VV^{\top})}_{\F} \\
    &\leq \norm{U^{*}U^{*\top} - UU^{\top}}_{\F} \norm{Z} + \norm{V^{*}V^{*\top}-VV^{\top}}_{\F} \norm{Z} \\
    &\lesssim (\norm{U-U^{*}}_{\F}+\norm{V-V^{*}}_{\F})\norm{Z}\\
    &\lesssim \frac{r\kappa \sigma \sqrt{n}\log^{2.5}(n)}{\sigma_{\min}} \norm{Z}. 
\end{align*}

For the entrywise norm bound, we also have
\begin{align*}
    &\norm{P_{T^{\perp}}(Z) - P_{T^{*\perp}}(Z)}_{\infty} \\
    &\leq \norm{(U^{*}U^{*\top} - UU^{\top})Z(I-VV^{\top})}_{\infty} + \norm{(I-U^{*}U^{*\top})Z(V^{*}V^{*\top}-VV^{\top})}_{\infty}.
\end{align*}
Similar to the bounds on \cref{eq:PTE-bound}, note that 
\begin{align*}
    &\norm{(U^{*}U^{*\top}-UU^{\top})Z(I-VV^{\top})}_{\infty} \\
    &\leq  \norm{(U^{*}U^{*\top}-UU^{\top})Z}_{\infty} + \norm{(U^{*}U^{*\top}-UU^{\top})ZVV^{\top}}_{\infty} \\
    &\leq \norm{(U^{*}-U)U^{*\top}Z}_{\infty} + \norm{U(U^{*}-U)^{\top}Z}_{\infty} \\
    &\quad + \norm{(U^{*}-U)U^{*\top}ZVV^{\top}}_{\infty} +  \norm{U(U^{*}-U)^{\top}ZVV^{\top}}_{\infty}\\
    &\leq  \norm{(U^{*}-U)}_{2,\infty} \norm{Z^{\top}}_{2,\infty} \norm{U^{*}} + \norm{U}_{2,\infty} \norm{Z^{\top}}_{2,\infty}\norm{(U^{*}-U)} \\
    &\quad + \norm{(U^{*}-U)}_{2,\infty} \norm{U^{*}} \norm{Z} \norm{V} \norm{V}_{2,\infty} + \norm{U}_{2,\infty} \norm{U-U^{*}} \norm{Z} \norm{V} \norm{V}_{2,\infty}.
\end{align*}
Based on 
\begin{align*}
\max(\norm{U^{*}-U}_{2,\infty}, \norm{V^{*}-V}_{2,\infty}) &\lesssim \frac{r^{3}\mu \kappa^{1.5}\sigma \log^{3.5}(n)}{\sigma_{\min}}\\\max(\norm{U^{*}-U}_{\F}, \norm{V^{*}-V}_{\F}) &\lesssim  \frac{r\kappa \sigma \sqrt{n}\log^{2.5}(n)}{\sigma_{\min}},
\end{align*}
we have
\begin{align*}
&\norm{(U^{*}U^{*\top}-UU^{\top})Z(I-VV^{\top})}_{\infty} \\
&\lesssim \norm{Z} \sqrt{\frac{\mu r}{n}}  \frac{r^{3}\mu \kappa^{1.5}\sigma \log^{3.5}(n)}{\sigma_{\min}}  + \norm{Z^{\top}}_{2,\infty}  \frac{r^{3}\mu \kappa^{1.5}\sigma \log^{3.5}(n)}{\sigma_{\min}}\\
&\lesssim   \norm{Z^{\top}}_{2,\infty}\frac{r^{3.5} \kappa^{1.5} \sigma \mu^{1.5} \log^{3.5}(n)}{\sigma_{\min}}.
\end{align*}
Similarly, one can obtain the bounds for $\norm{(I-U^{*}U^{*\top})Z(V^{*}V^{*\top}-VV^{\top})}_{\infty}$. This implies that
\begin{align*}
\norm{P_{T^{\perp}}(Z) - P_{T^{*\perp}}(Z)}_{\infty} \lesssim  \frac{r^{3.5} \kappa^{1.5} \sigma \mu^{1.5} \log^{3.5}(n)}{\sigma_{\min}} \left(\norm{Z}_{2,\infty}+\norm{Z^{\top}}_{2,\infty}\right).
\end{align*}

\section{Proof of Lemma \ref{lem:connection}}

The proof of \cref{lem:connection} can be established as a special case of the following two lemmas. 
\begin{restatable}{lemma}{LemRApproximateW}\label{lem:R-approximate-W}
Assume \cref{assum:conditions-Z} holds, and $\frac{\sigma}{\sigma_{\min}} \sqrt{n} \leq C_1\frac{1}{\kappa^2 r^2 \log^{5}(n)}.$ Suppose $O = M^{*} + \tau^{*} Z + E$ for a deterministic $E$. Assume $\|E\| \leq C_2\sigma \sqrt{n}$ with $|\inner{E}{P_{T^{*\perp}}(Z)}| \leq C_3 \sigma \sqrt{n/\log(n)} \|Z\|_{\F}.$

Let $\lambda = C_{\lambda} \sigma \sqrt{n} \log^{1.5}(n),  X, Y \in \R^{n\times r}$, $\tau = \frac{\inner{Z}{O-XY^{\top}}}{\|Z\|_{\F}^2}$, and $T$ is the tangent space of $XY^{\top}.$ 
\begin{subequations}\label{eq:R-W-conditions}
\begin{align}
    \norm{X - X^{*}}_{\F} + \norm{Y - Y^{*}}_{\F} &\leq C_{\F} \error \norm{X^{*}}_{\F}, \\
    \norm{\nabla f(X, Y; \tau)}_{\F} &\leq \frac{\lambda\sqrt{\sigma_{\min}}}{\kappa n}.
\end{align}
\end{subequations}
Let $R = \frac{1}{\lambda}(O - XY^{\top} - \tau Z - \lambda UV^{\top})$ where $U\Sigma V^{\top}$ is the SVD of $XY^{\top}.$ Then,
\begin{align}
    \norm{P_{T}(R)}_{\F} \leq \frac{72\kappa}{\lambda \sqrt{\sigma_{\min}}}  \norm{\nabla f(X, Y; \tau)}_{\F} \text{~\quad and \quad} \norm{P_{T^{\perp}}(R)} \leq \left(1-\frac{C}{\log n}\right). \label{eq:approximate-first-order-convex}
\end{align}
Here, $C$ is a constant depending (polynomially) on $C_1, C_2, C_3, C_{\lambda}, C_{\F}$.
\end{restatable}

\begin{restatable}{lemma}{LemApproximationNonconvexConvex}\label{lem:approximation-nonconvex-convex}
Assume the same setup from \cref{lem:R-approximate-W}. For any minimizer $(\hat{M}, \hat{\tau})$ of the convex program \cref{eq:convex-program}, we have 
\begin{align}
    \norm{XY^{\top} - \hat{M}}_{\F} &\leq C\frac{\kappa \log(n)}{\sqrt{\sigma_{\min}}}  \norm{\nabla f(X,Y; \tau)}_{\F} \nonumber\\
    |\tau - \hat{\tau}| &\leq C\frac{\kappa \log(n)}{\sqrt{\sigma_{\min}} \norm{Z}_{\F}}  \norm{\nabla f(X,Y; \tau)}_{\F}. \nonumber
\end{align}
Furthermore, let $\hat{T}$ be the tangent space of $\hat{M}$. For any matrix $A \in \R^{n\times n}$, we have
\begin{align*} 
    \norm{P_{\hat{T}^{\perp}}(A) - P_{T^{\perp}}(A)}_{\F} &\leq C \norm{A} \frac{\kappa \log(n)}{\sigma_{\min}^{1.5}} \norm{\nabla f(X, Y; \tau)}_{\F}\\
    \norm{P_{\hat{T}^{\perp}}(XY^{\top})}_{\F}  &\leq C \frac{\kappa^{3}\log^{2}(n)}{\sigma_{\min}^2} \norm{\nabla f}_{\F}^2.
\end{align*}
Here, $C$ is a constant depending (polynomially) on $C_1, C_2, C_3, C_{\lambda}, C_{\F}$.
\end{restatable}

\subsection{Proof of Lemma \ref{lem:R-approximate-W}}\label{appendix:lemma-R-approximate-W}
Let $XY^{\top} = U\Sigma V^{\top}$ be the singular value decomposition of $XY^{\top}.$ Recall that $T$ is the tangent space of $XY^{\top}$ and $P_{T}(A)$ is the projection of $A$ into the space $T$ for any $A \in \R^{n\times n}$:
\begin{align}
P_{T}(A) = A - (I_{n} - UU^{\top})A (I_n - VV^{\top}). \nonumber
\end{align}

By \cref{lem:singular-values}, for large enough $n$, the singular values of $X, Y$ are in the intervals $\left[\sqrt{\frac{\sigma_{\min}}{2}}, \sqrt{2\sigma_{\max}}\right]$ and
$$
\sigma_{\min}/2 \leq \sigma_{\min}(\Sigma) \leq \sigma_{\max}(\Sigma) \leq 2\sigma_{\max}.
$$

We then consider $P_{T}(R)$ and $P_{T^{\perp}}(R)$ separately. The proof of $P_{T}(R)$ is similar to the proof of Claim 2 in \cite{chen2019noisy}, while the proof of $P_{T^{\perp}}(R)$ is based on a careful analysis of the property of $\tau.$ 

\subsubsection{Control of \texorpdfstring{$P_{T}(R)$}{P\_T(R)}}
By the definition of $P_{T}$, we have
\begin{align}
    \norm{P_{T}(R)}_{\F}
    &= \norm{R - (I - UU^{\top}) R (I - VV^{\top})}_{\F}\nonumber\\
    &= \norm{UU^{\top} R (I - VV^{\top}) + R(VV^{\top})}_{\F}\nonumber\\
    &\overset{(i)}{\leq} \norm{U}\norm{U^{\top} R}_{\F} \norm{I-VV^{\top}} + \norm{RV}_{\F}\norm{V^{\top}} \nonumber\\
    &\overset{(ii)}{\leq} \norm{U^{\top}R}_{\F} + \norm{RV}_{\F} \label{eq:PTR-bound}
\end{align}
where (i) is due to $\norm{ABC}_{\F} \leq \norm{A}\norm{B}_{\F}\norm{C}$ and (ii) is due to $\norm{U}, \norm{I-VV^{\top}}, \norm{V} \leq 1$. 

Before showing bounds on $\norm{U^{\top}R}_{\F}$ and $\norm{RV}_{\F}$, consider an examination for the properties of $X$ and $Y$. By \cref{lem:SigmaQ}, we can write $X = U\Sigma^{1/2} Q, Y = V\Sigma^{1/2} Q^{-\top}$ where $Q \in \R^{r\times r}$ is an invertible matrix and 
\begin{align}
    \norm{\Sigma_{Q} - \Sigma_{Q}^{-1}}_{\F} \leq \frac{1}{\sigma_{\min}(\Sigma)} \norm{X^{\top}X - Y^{\top}Y}_{\F}, \label{eq:sigma-Q-sigma-Q-inverse}
\end{align}
where $U_{Q}\Sigma_{Q}V_{Q}^{\top}$ is the SVD of $Q$. In fact, we can show that $X^{\top}X \approx Y^{\top}Y$ due to that $\nabla f(X, Y; \tau) \approx 0.$ To see this, set $B_1, B_2$ for the gradient $\nabla_{X} f(X, Y; \tau)$ and $\nabla_{Y} f(X, Y; \tau)$:
\begin{align}
    B_1 &:= (XY^{\top} + \tau Z  - O)Y + \lambda X \label{eq:aX}\\
    B_2 &:= (XY^{\top} + \tau Z - O)^{\top}X + \lambda Y \label{eq:aY}.
\end{align}
Left-multiplying $X^{\top}$ on both sides of \cref{eq:aX} and left-multiplying $Y^{\top}$ on both sides of \cref{eq:aY} and taking the transpose, we can obtain
\begin{align*}
    \lambda X^{\top}X &= X^{\top}B_1 - X^{\top}(\tau Z + XY^{\top} - O)Y\\
    \lambda Y^{\top}Y &= B_2^{\top}Y - X^{\top}(\tau Z + XY^{\top} - O)Y.
\end{align*}
This implies $\norm{X^{\top}X - Y^{\top}Y}_{\F} = \frac{1}{\lambda}\norm{X^{\top}B_1 - B_2^{\top}Y}_{\F} \leq \frac{1}{\lambda}\max(\norm{B_1}_{\F},\norm{B_2}_{\F})(\norm{X}+\norm{Y}).$ Combining this with \cref{eq:sigma-Q-sigma-Q-inverse}, we obtain
\begin{align}
     \norm{\Sigma_{Q} - \Sigma_{Q}^{-1}}_{\F} 
     &\leq \frac{1}{\sigma_{\min}(\Sigma)} \norm{X^{\top}X - Y^{\top}Y}_{\F} \nonumber\\
     &\leq \frac{1}{\sigma_{\min}(\Sigma) \lambda} \max(\norm{B_1}_{\F},\norm{B_2}_{\F})(\norm{X}+\norm{Y}) \nonumber\\
     &\overset{(i)}{\leq} \frac{8\sqrt{\sigma_{\max}}}{\sigma_{\min} \lambda}\max(\norm{B_1}_{\F},\norm{B_2}_{\F}) \nonumber\\
     &\leq  \frac{8\sqrt{\sigma_{\max}}}{\sigma_{\min} \lambda} \norm{\nabla f(X, Y; \tau)}_{\F} \overset{(ii)}{\leq} 1 \label{eq:sigmaQ-bound}
\end{align}
where (i) is due to $\sigma_{\min}(\Sigma) \geq \sigma_{\min}/2, \sigma_{1}(X) \leq 2\sigma_{\max}, \sigma_1(Y) \leq 2\sigma_{\max}$ and (ii) holds for large enough $n$. This also implies $\norm{\Sigma_{Q}} \leq 2$. Intuitively speaking, $\nabla f(X, Y; \tau) \approx 0$ implies that $\Sigma_{Q} \approx I$ and hence $Q$ is similar to an orthogonal matrix.  

Next, we show how to control $\norm{RV}_{\F}$. Recall that 
\begin{align}
    O - XY^{\top} - \tau Z = \lambda UV^{\top} + \lambda R. \label{eq:defintion-R}
\end{align}

Combining \cref{eq:aX} and \cref{eq:defintion-R}, we have
\begin{align}
    (-\lambda UV^{\top} - \lambda R) Y + \lambda X = B_1.\nonumber
\end{align}
Therefore, $\lambda RY = \lambda X - B_1 - \lambda UV^{\top} Y.$ Substituting $X$ and $Y$ with $X = U\Sigma^{1/2} Q, Y = V\Sigma^{1/2} Q^{-\top}$, we obtain
\begin{align}
    \lambda R (V\Sigma^{1/2} Q^{-\top}) = \lambda (U\Sigma^{1/2} Q) - B_1 - \lambda UV^{\top} (V\Sigma^{1/2} Q^{-\top}).\nonumber
\end{align}
Right-multiplying $Q^{\top}\Sigma^{-1/2}$ on both sides and using $V^{\top}V = I$ leads to
\begin{align}
    \lambda R V = \lambda U\Sigma^{1/2} QQ^{\top} \Sigma^{-1/2} - B_1 Q^{\top}\Sigma^{-1/2} - \lambda U. \nonumber
\end{align}
Then this implies
\begin{align}
    \norm{RV}_{\F} 
    &= \norm{U\Sigma^{1/2}(QQ^{\top}-I)\Sigma^{-1/2} - \frac{1}{\lambda} B_1Q^{\top}\Sigma^{-1/2}}_{\F} \nonumber\\
    &\leq \frac{1}{\lambda} \norm{B_1}_{\F} \norm{Q} \norm{\Sigma^{-1/2}} +  \norm{U} \norm{\Sigma^{1/2}} \norm{\Sigma^{-1/2}} \norm{I - QQ^{\top}}_{\F} \nonumber\\
    &\overset{(i)}{\leq} \frac{1}{\lambda} \norm{B_1}_{\F} \norm{Q} \frac{\sqrt{2}}{\sqrt{\sigma_{\min}}} + \sqrt{2\sigma_{\max}} \frac{\sqrt{2}}{\sqrt{\sigma_{\min}}} \norm{I-QQ^{\top}}_{\F}\nonumber\\
    &\overset{(ii)}{\leq} \frac{2\sqrt{2}}{\lambda \sqrt{\sigma_{\min}}} \norm{B_1}_{\F}  + 2 \sqrt{\kappa} \norm{U_{Q}\Sigma_{Q}(\Sigma_{Q}^{-1} - \Sigma_{Q})U_{Q}^{\top}}_{\F} \nonumber\\
    &\overset{(iii)}{\leq} \frac{2\sqrt{2}}{\lambda \sqrt{\sigma_{\min}}} \norm{\nabla f(X, Y; \tau)}_{\F} + 4\sqrt{\kappa} \norm{\Sigma_Q^{-1} - \Sigma_{Q}}_{\F} \nonumber\\
    &\overset{(iv)}{\leq} \left(\frac{2\sqrt{2}}{\lambda \sqrt{\sigma_{\min}}} + \frac{32\kappa}{\lambda \sqrt{\sigma_{\min}}} \right) \norm{\nabla f(X, Y; \tau)}_{\F} \nonumber\\
    &\leq \frac{36\kappa}{\lambda \sqrt{\sigma_{\min}}} \norm{\nabla f(X, Y; \tau)}_{\F} \nonumber
\end{align}
where (i) is due to $\sigma_{\min}(\Sigma) \geq (1/2)\sigma_{\min}, \sigma_{\max}(\Sigma) \leq 2\sigma_{\max}$, (ii) is due to $\norm{Q}=\norm{\Sigma_Q} \leq 2$ and $Q = U_Q \Sigma_Q V_Q^{\top}$ is the SVD of $Q$, (iii) is due to $\norm{B_1}_{\F} \leq \norm{\nabla f(X, Y; \tau)}_{\F}$ and $\norm{U_Q} = 1, \norm{\Sigma_Q} \leq 2$, (iv) is due to the bound of $\|\Sigma_Q^{-1} - \Sigma_{Q}\|_{\F}$ (\cref{eq:sigmaQ-bound}).

Similarly, one can obtain $\norm{U^{\top}R}_{\F} \leq \frac{36\kappa}{\lambda \sqrt{\sigma_{\min}}} \norm{\nabla f(X, Y; \tau)}_{\F}.$ Combining these with \cref{eq:PTR-bound} leads to the result:
\begin{align}\label{eq:PTR-proof-bound}
    \norm{P_{T}(R)}_{\F} \leq \frac{72\kappa}{\lambda \sqrt{\sigma_{\min}}}  \norm{\nabla f(X, Y; \tau)}_{\F}.
\end{align}
This completes the proof for the bound of $P_{T}(R).$

\subsubsection{Control of \texorpdfstring{$\PTp{R}$}{PTperp(R)}}
Next, we control $\PTp{R}.$ Using $O = M^{*} + E + \tau^{*} Z$ and applying $P_{T^{\perp}}$ on both sides of $O - XY^{\top} - \tau Z = \lambda UV^{\top} + \lambda R$, we obtain
\begin{align}
    \lambda \PTp{R} 
    &\overset{(i)}{=} \PTp{E} - \PTp{XY^{\top} - M^{*}} - (\tau-\tau^{*})\PTp{Z} \nonumber\\
    &= \PTp{E} - \PTp{(X-X^{*})Y^{\top} + X(Y-Y^{*})^{\top}} + \PTp{(X-X^{*})(Y-Y^{*})^{\top}} \nonumber\\
    &\quad - (\tau-\tau^{*})\PTp{Z} \nonumber\\
    &\overset{(ii)}{=}  \PTp{E} + \PTp{(X-X^{*})(Y-Y^{*})^{\top}} - (\tau-\tau^{*})\PTp{Z} \label{eq:XXstar-YY}.
\end{align}
where (i) is due to $\PTp{UV^{\top}} = 0$, (ii) is due to $Y=V\Sigma^{1/2}Q$ and $X = U\Sigma^{1/2}Q^{-\top}$ by \cref{lem:SigmaQ} and $\PTp{UA^{\top}+BV^{\top}} = 0$ for any $A,B \in \R^{n \times r}.$   

This implies that 
\begin{align}
    \norm{\PTp{R}} 
    &\leq \frac{1}{\lambda} \norm{\PTp{E}} + \frac{1}{\lambda} \norm{\PTp{(X-X^{*})(Y-Y^{*})^{\top}}} + \frac{1}{\lambda}|\tau-\tau^{*}|\norm{\PTp{Z}}\nonumber\\
    &\overset{(i)}{\leq} \frac{1}{\lambda} \norm{E} + \frac{1}{\lambda} \norm{X-X^{*}}\norm{Y-Y^{*}} + \frac{1}{\lambda} |\tau-\tau^{*}|\norm{\PTp{Z}} \nonumber
\end{align}
where (i) is due to $\norm{\PTp{A}} \leq \norm{I-UU^{\top}}\norm{A}\norm{I-VV^{\top}} \leq \norm{A}$ for any $A \in \R^{n\times n}.$ Then it boils down to control $\frac{\norm{E}}{\lambda}, \frac{\norm{X-X^{*}}\norm{Y-Y^{*}}}{\lambda}$, and $\frac{|\tau-\tau^{*}|\norm{\PTp{Z}}}{\lambda}$ separately. 

Consider $\frac{\norm{E}}{\lambda}$. Since we have the condition $\norm{E} \lesssim \sigma \sqrt{n}$, then
\begin{align}
    \frac{\norm{E}}{\lambda} \leq \frac{c\sigma \sqrt{n}}{C_{\lambda} \sigma \sqrt{n} \log^{1.5}(n)} \lesssim \frac{1}{\log^{1.5}(n)}. \label{eq:E-bound}
\end{align}

Consider $\frac{\norm{X-X^{*}}\norm{Y-Y^{*}}}{\lambda}$. We have
\begin{align}
    \frac{\norm{X-X^{*}}\norm{Y-Y^{*}}}{\lambda} 
    &\leq \frac{\norm{X-X^{*}}_{\F} \norm{Y-Y^{*}}_{\F}}{\lambda} \nonumber\\
    &\overset{(i)}{\leq} C_{\F}^2 \frac{\sigma^2 n \log^{5}(n)}{\sigma_{\min}^2} \sigma_{\max}r \frac{1}{C_{\lambda} \sigma \sqrt{n}\log^{1.5}(n)} \nonumber\\
    &\leq \frac{C_{\F}^2}{C_{\lambda}} \frac{\sigma \sqrt{n} \log^{3.5}(n) \kappa r}{\sigma_{\min}} \nonumber\\
    &\overset{(ii)}{\lesssim} \frac{1}{\log^{1.5}(n)} \label{eq:XXstarYYstar-bound}
\end{align}
where (i) is due to 
$$
\norm{X-X^{*}}_{\F} + \norm{Y-Y^{*}}_{\F} \leq C_{\F} \frac{\sigma \sqrt{n} \log^{2.5}(n)}{\sigma_{\min}} \sqrt{\sigma_{\max} r}
$$ 
with $\lambda = C_{\lambda} \sigma \sqrt{n} \log^{1.5}(n)$, and (ii) is due to $\sSNR.$ 

Consider $\frac{1}{\lambda}|\tau-\tau^{*}|\norm{\PTp{Z}}$. We have the following claim.
\begin{claim}\label{claim:PTP}
Suppose \cref{cond:Z-condition-nonconvex} and \cref{cond:Z-condition-convex} hold. Suppose $\sSNR$. Then
\begin{align}
    \frac{1}{\lambda}|\tau-\tau^{*}|\norm{\PTp{Z}} \leq 1 - \frac{C}{\log(n)} \label{eq:tau-taustar-Z-bound}
\end{align}
for some constant $C$. 
\end{claim}

Combining \cref{eq:E-bound,eq:XXstarYYstar-bound,eq:tau-taustar-Z-bound}, we complete the proof for controlling $P_{T^{\perp}}(R)$. 

\begin{proof}[Proof of \cref{claim:PTP}]
Similar to the proof of \cref{lem:tau-decomposition}, we first consider the characterization for $\tau - \tau^{*}$. Since $\tau = \inner{Z}{O-XY^{\top}}/\norm{Z}_{\F}^2$, we have $\inner{O - XY^{\top} - \tau Z}{Z} = 0$. Combining this with $O - XY^{\top} - \tau Z = \lambda UV^{\top} + \lambda R$, we have 
\begin{align}
    &\inner{Z}{UV^{\top} + R} = 0\nonumber\\
    \implies& \inner{Z}{UV^{\top}} = -\inner{Z}{R} \nonumber\\
    \implies& \inner{Z}{UV^{\top}} + \inner{\PT{Z}}{\PT{R}} = -\inner{\PTp{Z}}{\PTp{R}} \label{eq:ZUVT}
\end{align}    
Substituting $\PTp{R}$ in \cref{eq:ZUVT} by \cref{eq:XXstar-YY}, we obtain
\begin{align*}
    &\inner{Z}{UV^{\top}} + \inner{\PT{Z}}{\PT{R}} \\
    &= \frac{1}{\lambda} (\tau - \tau^{*})\inner{\PTp{Z}}{\PTp{Z}} \\
    &\quad -\frac{1}{\lambda} \inner{\PTp{Z}}{\PTp{E}} - \frac{1}{\lambda} \inner{\PTp{Z}}{\PTp{(X-X^{*})(Y-Y^{*})^{\top}}}.
\end{align*}
Note that $\inner{\PTp{Z}}{\PTp{Z}} = \norm{\PTp{Z}}_{\F}^2$. This implies
\begin{align}
    \frac{(\tau - \tau^{*})}{\lambda}\norm{\PTp{Z}}_{\F}^2
    &= \inner{Z}{UV^{\top}} + \inner{\PT{Z}}{\PT{R}} + \frac{1}{\lambda}\inner{\PTp{Z}}{\PTp{E}} \nonumber\\
    &\quad + \frac{1}{\lambda}\inner{\PTp{Z}}{\PTp{(X-X^{*})(Y-Y^{*})^{\top}}}.\nonumber
\end{align}
This further implies
\begin{align}
    &\frac{|\tau - \tau^{*}|}{\lambda} \norm{\PTp{Z}} \nonumber\\
    &\leq \left|\inner{Z}{UV^{\top}}\right|\frac{\norm{\PTp{Z}}}{\norm{\PTp{Z}}_{\F}^2} + \left|\inner{\PT{Z}}{\PT{R}}\right|\frac{\norm{\PTp{Z}}}{\norm{\PTp{Z}}_{\F}^2} \nonumber\\
    &\quad + \frac{1}{\lambda}\left|\inner{\PTp{Z}}{\PTp{E}}\right|\frac{\norm{\PTp{Z}}}{\norm{\PTp{Z}}_{\F}^2} \nonumber\\
    &\quad +  \frac{1}{\lambda}\left|\inner{\PTp{Z}}{\PTp{(X-X^{*})(Y-Y^{*})^{\top}}}\right|\frac{\norm{\PTp{Z}}}{\norm{\PTp{Z}}_{\F}^2} \nonumber\\
    &\overset{(i)}{\leq} \underbrace{\left|\inner{Z}{UV^{\top}}\right|\frac{\norm{\PTp{Z}}}{\norm{\PTp{Z}}_{\F}^2}}_{A_0} + \underbrace{\frac{\norm{\PT{Z}}_{\F}}{\norm{\PTp{Z}}_{\F}} \norm{\PT{R}}_{\F}}_{A_1} + \underbrace{\frac{|\inner{\PTp{Z}}{E}|}{\lambda\norm{\PTp{Z}}_{\F}}}_{A_2}\nonumber\\
    &\quad + \underbrace{\frac{1}{\lambda}\norm{\PTp{(X-X^{*})(Y-Y^{*})^{\top}}}_{\F}}_{A_3}\nonumber
\end{align}
where (i) is due to $\norm{\PTp{Z}} \leq \norm{\PTp{Z}}_{\F}$ and $|\inner{A}{B}| \leq \norm{A}_{\F}\norm{B}_{\F}$ by Cauchy-Schwartz inequality.  
Then it boils down to control $A_0, A_1, A_2, A_3$.

By \cref{lem:general-conditions-small-ball}, we have that
\begin{align}
    A_0 &\leq 1 - \frac{C_{r_2}}{2\log(n)} \nonumber\\
    \norm{P_{T^{\perp}}(Z)}_{\F}^2 &\geq \frac{C_{r_1}}{2\log(n)} \norm{Z}_{\F}^2 \label{eq:PTperpZ-F-bound}\\
    \norm{P_{T^{\perp}}(Z) - P_{T^{*\perp}}(Z)}_{*} &\lesssim \frac{1}{r^{0.5}\log^{2.5}(n)}\norm{Z}_{\F}. \label{eq:PTPZ-PTstarZ-bound}
\end{align}
For $A_1$, this implies 
\begin{align*}
    A_1 
    &= \frac{\norm{\PT{Z}}_{\F}}{\norm{\PTp{Z}}_{\F}} \norm{\PT{R}}_{\F}\\
    &\leq  \frac{\norm{\PT{Z}}_{\F}}{\norm{Z}_{\F}} \frac{2\log(n)}{C_{r_1}} \norm{\PT{R}}_{\F}\\
    &\leq \frac{2\log(n)}{C_{r_1}} \norm{\PT{R}}_{\F}\\
    &\overset{(i)}{\lesssim} \frac{1}{\log^{2}(n)}
\end{align*}
where (i) is due to $ \norm{\PT{R}}_{\F} \lesssim \frac{\kappa}{\lambda \sqrt{\sigma_{\min}}} \norm{\nabla f(X, Y; \tau)}.$

For $A_2$, with probability $1-O(1/\text{poly}(n))$, we have
\begin{align*}
    |\inner{\PTp{Z}}{E}| 
    &\leq |\inner{\PTp{Z}-P_{T^{*\perp}}(Z)}{E}| + |\inner{P_{T^{*\perp}}(Z)}{E}|\\
    &\leq  \norm{P_{T^{\perp}}(Z) - P_{T^{*\perp}}(Z)}_{*} \norm{E} + |\inner{P_{T^{*\perp}}(Z)}{E}|\\
    &\overset{(i)}{\lesssim} \frac{\norm{Z}_{\F}}{r^{0.5}\log^{2.5}(n)} \sigma \sqrt{n} + |\inner{P_{T^{*\perp}}(Z)}{E}|\\
    &\overset{(ii)}{\lesssim}\frac{\norm{Z}_{\F}}{r^{0.5}\log^{2.5}(n)} \sigma \sqrt{n} + \sigma \sqrt{n} \norm{Z}_{\F}.
\end{align*}
Here, (i) is due to \cref{eq:PTPZ-PTstarZ-bound} and $\norm{E} \lesssim \sigma \sqrt{n}$, and (ii) is due to $|\inner{P_{T^{*\perp}}(Z)}{E}| \lesssim \sigma \sqrt{n/\log(n)} \norm{Z}_{\F}$. This implies
\begin{align*}
    A_2
    &= \frac{|\inner{\PTp{Z}}{E}|}{\lambda\norm{\PTp{Z}}_{\F}}\\
    &\lesssim  \frac{\sigma \sqrt{n} \norm{Z}_{\F}}{\sigma \sqrt{n}\log^{2}(n) \norm{\PTp{Z}}_{\F}}\\
    &\overset{(i)}{\lesssim} \frac{\norm{Z}_{\F}}{\log^{1.5}(n)\norm{Z}_{\F}}\\
    &\lesssim \frac{1}{\log^{1.5}(n)}.
\end{align*}
Here, (i) is due to \cref{eq:PTperpZ-F-bound}.

For $A_3$, we have
\begin{align*}
    A_3
    &= \frac{1}{\lambda}\norm{\PTp{(X-X^{*})(Y-Y^{*})^{\top}}}_{\F} \\
    &\leq \frac{1}{\lambda}\norm{X-X^{*}}_{\F}\norm{Y-Y^{*}}_{\F} \overset{(i)}{\lesssim} \frac{1}{\log^{1.5}(n)}.
\end{align*}
Here, (i) is due to \cref{eq:XXstarYYstar-bound}

Combining all the results for $A_0, A_1, A_2, A_3$, we have
\begin{align*}
    \frac{|\tau - \tau^{*}|}{\lambda} \norm{\PTp{Z}} 
    &\leq A_0 + A_1 + A_2 + A_3\\
    &\leq 1 - \frac{C_{r_2}}{2\log(n)} + \frac{c_1}{\log^2(n)} + \frac{c_2}{\log^{1.5}(n)} + \frac{c_3}{\log^{1.5}(n)}\\
    &\leq 1 - \frac{C}{\log(n)}
\end{align*}
for some constants $c_1, c_2, c_3, C$. This completes the proof of the claim. 
\end{proof}

\subsection{Proof of Lemma \ref{lem:approximation-nonconvex-convex}}\label{appendix:approximation-nonconvex-convex}

\begin{proof}[Proof of \cref{lem:approximation-nonconvex-convex}]
Based on \cref{lem:R-approximate-W}, recall that for $ R = \frac{1}{\lambda}(O - XY^{\top} - \tau Z - \lambda UV^{\top})$, we have $\norm{P_{T}(R)}_{\F} \lesssim \frac{\kappa}{\lambda \sqrt{\sigma_{\min}}} \norm{\nabla f}_{\F}$ and $\norm{\PTp{R}} \leq 1 - C/\log(n).$

Let $(M_{\cvx}, \tau_{\cvx})$ be one of the minimizers of $g$. Then we have
\begin{align}\label{eq:optimizer-property}
\frac{1}{2} \norm{O - M_{\cvx} - \tau(M_{\cvx}) Z}_{\F}^2 + \lambda \norm{M_{\cvx}}_{*} \leq \frac{1}{2} \norm{O - XY^{\top} - \tau(XY^{\top})Z}_{\F}^2 + \lambda \norm{XY^{\top}}_{*}
\end{align}
where $\tau(M) := \frac{\inner{Z}{O - M}}{\norm{Z}_{\F}^2}.$

Let $\Delta = M_{\cvx} - XY^{\top}.$ Let $\Delta' = \tau(M_{\cvx})Z - \tau(XY^{\top})Z.$ Then \cref{eq:optimizer-property} implies
\begin{align}
&\frac{1}{2} \norm{O - XY^{\top} - \tau(XY^{\top})Z - \Delta - \Delta'}_{\F}^2 + \lambda \norm{M_{\cvx}}_{*} \nonumber\\
&\leq \frac{1}{2} \norm{O - XY^{\top} - \tau(XY^{\top})Z}_{\F}^2 + \lambda \norm{XY^{\top}}_{*}. \nonumber
\end{align}
Hence,
\begin{align}
\frac{1}{2} \norm{\Delta + \Delta'}_{\F}^2 \leq \inner{O-XY^{\top}-\tau(XY^{\top})Z}{\Delta+\Delta'} - \lambda \norm{M_{\cvx}}_{*} + \lambda \norm{XY^{\top}}_{*}.\nonumber
\end{align}

Note that $\tau(XY^{\top}) = \inner{Z}{O-XY^{\top}}/\inner{Z}{Z}$, then $\inner{O-XY^{\top}-\tau(XY^{\top})Z}{Z} = 0.$ Hence
\begin{align}
\inner{O-XY^{\top}-\tau(XY^{\top})Z}{\Delta'} = 0\nonumber.
\end{align}
This implies that
\begin{align}
0 \leq \frac{1}{2} \norm{\Delta + \Delta'}_{\F}^2 \leq \inner{O-XY^{\top}-\tau(XY^{\top})Z}{\Delta} - \lambda \norm{M_{\cvx}}_{*} + \lambda \norm{XY^{\top}}_{*}.\nonumber
\end{align}

Note that $\norm{\cdot}_{*}$ is convex and $UV^{\top}+W$ is the subgradient of $\norm{XY^{\top}}_{*}$ given that $W \in T^{\perp}$ and $\norm{W} \leq 1.$
This implies that $\norm{M_{\cvx}}_{*} \geq \norm{XY^{\top}}_{*} + \inner{UV^{\top}+W}{\Delta}$ due to the property of the convex function. Choose $W$ such that $\inner{W}{\Delta} = \norm{P_{T^{\perp}}(\Delta)}_{*}$, we have
\begin{align}
 \frac{1}{2}\norm{\Delta + \Delta'}_{\F}^2  \nonumber
 &\leq \inner{O-XY^{\top}-\tau(XY^{\top})Z}{\Delta} - \lambda \inner{UV^{\top}}{\Delta} - \lambda \norm{\PTp{\Delta}}_{*} \nonumber\\
 &= \lambda \inner{R}{\Delta} -   \lambda \norm{\PTp{\Delta}}_{*}\nonumber\\
 &= \lambda \inner{\PT{R}}{\PT{\Delta}} + \lambda \inner{\PTp{R}}{\PTp{\Delta}} - \lambda \norm{\PTp{\Delta}}_{*}\nonumber\\
 &\leq \lambda \norm{\PT{R}}_{\F} \norm{\PT{\Delta}}_{\F} + \lambda \norm{\PTp{R}} \norm{\PTp{\Delta}}_{*} -  \lambda \norm{\PTp{\Delta}}_{*}\nonumber\\
 &\overset{(i)}{\leq}  \lambda \norm{\PT{R}}_{\F} \norm{\PT{\Delta}}_{\F} - \frac{C\lambda}{\log(n)} \norm{\PTp{\Delta}}_{*} \label{eq:add-inequality}.
\end{align}
Here, (i) is due to $\norm{\PTp{R}} \leq 1 - C/\log(n).$ This implies that
\begin{align}
\norm{\PTp{\Delta}}_{*} \leq \norm{\PT{R}}_{\F} \norm{\PT{\Delta}}_{\F} \frac{\log(n)}{C} \lesssim \frac{\kappa \log(n)}{\lambda \sqrt{\sigma_{\min}}} \norm{\nabla f}_{\F}\norm{\PT{\Delta}}_{\F}. \label{eq:PTP-Delta-bound}
\end{align}

Note that $\Delta' = - \inner{\Delta}{Z} / \inner{Z}{Z}\cdot Z$. Investigate further on $\norm{\Delta+\Delta'}_{\F}^2$:
\begin{align}
\norm{\Delta+\Delta'}_{\F}^2 
&= \inner{\Delta}{\Delta} - 2 \inner{\Delta}{Z}^2 / \inner{Z}{Z} + \inner{\Delta}{Z}^2 / \inner{Z}{Z}\nonumber \\
&= \inner{\Delta}{\Delta} - \inner{\Delta}{Z}^2 / \inner{Z}{Z}\nonumber\\
&= \norm{\Delta}_{\F}^2 - (\inner{\PT{\Delta}}{\PT{Z}}+  \inner{\PTp{\Delta}}{\PTp{Z}})^2 / \norm{Z}_{\F}^2\nonumber\\
&\geq \norm{\Delta}_{\F}^2 - (\norm{\PT{\Delta}}_{\F} \norm{\PT{Z}}_{\F} + \norm{\PTp{\Delta}}_{*} \norm{\PTp{Z}} )^2 / \norm{Z}_{\F}^2\nonumber\\
&\overset{(i)}{\geq}  \norm{\Delta}_{\F}^2 - \frac{\norm{\PT{\Delta}}_{\F}^2}{\norm{Z}_{\F}^2} \cdot \left(\norm{\PT{Z}}_{\F} + \frac{C'\kappa\log(n)}{\lambda \sqrt{\sigma_{\min}}}  \norm{\nabla f}_{\F}\norm{\PTp{Z}}\right)^2 \nonumber\\
&\overset{(ii)}{\geq} \norm{\Delta}_{\F}^2 \cdot \left(1 - \frac{1}{\norm{Z}_{\F}^2} \cdot \left(\norm{\PT{Z}}_{\F} + \frac{C'\kappa \log(n)}{\lambda \sqrt{\sigma_{\min}}}  \norm{\nabla f}_{\F}\norm{\PTp{Z}}_{\F}\right)^2\right)\nonumber
\end{align}
for some constant $C'$. Here, (i) is due to \cref{eq:PTP-Delta-bound} and (ii) is due to $\norm{\PT{\Delta}}_{\F} \leq \norm{\Delta}_{\F}.$

Note that $\norm{\PT{Z}}_{\F}^2 / \norm{Z}_{\F}^2 \leq 1-\frac{C''}{\log(n)}$ for some constant $C''$ (by \cref{lem:general-conditions-small-ball}), and $ \norm{\nabla f}_{\F}$ is small enough. This implies 
\begin{align}
\norm{\Delta+\Delta'}_{\F}^2  \gtrsim \norm{\Delta}_{\F}^2 \frac{1}{\log(n)}.\nonumber
\end{align}

Combining this with \cref{eq:add-inequality}, we have
\begin{align}
\frac{1}{\log(n)}\norm{\Delta}_{\F}^2 \lesssim \lambda \norm{\PT{R}}_{\F} \norm{\PT{\Delta}}_{\F}  \overset{(i)}{\lesssim} \frac{\kappa}{\sqrt{\sigma_{\min}}} \norm{\nabla f}_{\F}\norm{\Delta}_{\F}.\nonumber
\end{align}
Here, (i) is due to the bound for $\norm{\PT{R}}_{\F}$ and $\norm{\PT{\Delta}}_{\F} \leq \norm{\Delta}_{\F}.$ This implies 
\begin{align*}
\norm{\Delta}_{\F} \lesssim \frac{\kappa \log(n)}{\sqrt{\sigma_{\min}}} \norm{\nabla f}_{\F}.
\end{align*}
Furthermore, 
\begin{align*}
    |\tau(XY^{\top}) - \tau(M_{\cvx})| 
    &= \left|\inner{Z}{M_{\cvx} - XY^{\top}}\right| / \norm{Z}_{\F}^2\\
    &\leq \norm{Z}_{\F} \norm{\Delta}_{\F} / \norm{Z}_{\F}^2\\
    &\leq \norm{\Delta}_{\F} / \norm{Z}_{\F}\\
    &\leq \frac{\kappa \log(n)}{\sqrt{\sigma_{\min}}\norm{Z}_{\F}} \norm{\nabla f}_{\F}.
\end{align*}

Let $M_{\cvx} = U_{\cvx}\Sigma_{\cvx}V_{\cvx}^{\top}$ be SVD of $M_{\cvx}$. Let $T_{\cvx}$ be the tangent space of $M_{\cvx}.$ Next, we analyze the relation between $T_{\cvx}$ and $T.$
We aim to show (i) $\text{rank}(M_{\cvx}) = r$ based on the fact that $M_{\cvx}$ is obtained by soft-thresholding; (ii) establish the closeness between eigen-subspaces $U_{\cvx}$ ($V_{\cvx}$) and $U$ ($V$) by invoking the Davis-Kahan theorem; (iii) establish the closeness between $T_{\cvx}$ and $T$ based on the closeness between eigen-subspaces.

We first aim to show that $\text{rank}(M_{\cvx}) = r.$ Note that
\begin{align*}
    O - M_{\cvx} - \tau_{\cvx} Z = \lambda U_{\cvx} V_{\cvx}^{\top} + W_{\cvx} 
\end{align*}
where $\norm{W_{\cvx}} \leq 1, P_{T_{\cvx}}(W_{\cvx}) = 0.$ Let $\sigma_i(A)$ be the $i$-th largest singular values of $A.$  Note that there is $O - \tau Z = U(\lambda I_{r} + \Sigma) V^{\top} + \lambda R.$ By Weyl's inequality, for $1\leq i \leq r$
\begin{align*}
    \sigma_i(O - \tau Z) 
    &\geq \sigma_i(U(\lambda I_{r} + \Sigma) V^{\top}) - \lambda\norm{R}\\
    &\geq \frac{\sigma_{\min}}{2} + \lambda - \lambda (\norm{\PTp{R}} + \norm{\PT{R}})\\
    &\overset{(i)}{\geq} \frac{\sigma_{\min}}{2} \\
    &\overset{(ii)}{>} 2\lambda. 
\end{align*}
Here, (i) is due to the bound of $\norm{\PTp{R}}$ and $\norm{\PT{R}} \leq \norm{\PT{R}}_{\F}$, (ii) is due to $\sSNR.$ Then, for $1\leq i \leq r$, 
\begin{align*}
    \sigma_i(O - \tau_{\cvx} Z) 
    &\geq \sigma_i(O - \tau Z) - |\tau_{\cvx} - \tau|\norm{Z}\\
    &\geq \sigma_i(O - \tau Z) - \frac{C \kappa \log(n)}{\sqrt{\sigma_{\min}}\norm{Z}_{\F}} \norm{\nabla f}_{\F} \norm{Z}\\
    &\geq \sigma_i(O - \tau Z) - \frac{C \kappa \log(n)}{\sqrt{\sigma_{\min}}} \norm{\nabla f}_{\F}\\
    &\geq 2\lambda - \frac{\lambda \log(n)}{n^3}\\
    &>\lambda
\end{align*}
providing that $\norm{\nabla f} \leq \frac{\lambda \sqrt{\sigma_{\min}}}{\kappa n}$ and $ \log(n) \lesssim n.$

Meanwhile, for $r+1\leq i \leq n$, we have
\begin{align*}
    \sigma_i(O - \tau Z) 
    &\leq \sigma_i(U(\lambda I_{r} + \Sigma) V^{\top}) + \lambda \norm{R}\\
    &\leq \lambda \norm{R} \\
    &\leq (1 - C/\log(n)) \lambda. 
\end{align*}
Then for $r+1\leq i \leq n$, we have
\begin{align*}
    \sigma_i(O - \tau_{\cvx}Z) 
    &\leq \sigma_i(O - \tau Z) + |\tau_{\cvx} - \tau|\norm{Z}\\
    &\leq (1 - C/\log(n)) \lambda + \frac{\lambda \kappa \log(n)}{n^2}\\
    &< \lambda. 
\end{align*}
Note by soft-thresholding, ($M_{\cvx}$ is obtained by truncating the singular values of $\left(O - \tau_{\cvx} Z\right)$ by $\lambda$, see \cite{mazumder2010spectral}), we have that $\mathrm{\rank}(M_{\cvx}) = r.$ Next, by Davis-Kahan's theorem \cite{yu2015useful}, there exist rotation matrices $R_1, R_2 \in \R^{r\times r}$, where
\begin{align*}
    \max(\norm{UR_1 - U_{\cvx}}_{\F}, \norm{VR_2 - V_{\cvx}}_{\F}) 
    &\lesssim \frac{\norm{M_{\cvx} - M}_{\F}}{\sigma_r(M)}\\
    &\lesssim \frac{\kappa \log(n)}{\sigma_{\min}^{1.5}} \norm{\nabla f}_{\F}.
\end{align*}
This simply implies that $\norm{UU^{\top} - U_{\cvx}U_{\cvx}^{\top}}_{\F} \lesssim \norm{UR_1-U_{\cvx}}_{\F}\norm{U} \lesssim \frac{\kappa \log(n)}{\sigma_{\min}^{1.5}} \norm{\nabla f}_{\F}.$ Similarly, $\norm{VV^{\top} - V_{\cvx}V_{\cvx}^{\top}}_{\F} \lesssim \frac{\kappa \log(n)}{\sigma_{\min}^{1.5}} \norm{\nabla f}_{\F}.$ 

Then, for any matrix $A \in \R^{n\times n}$,
\begin{align*}
    &\norm{P_{T^{\perp}}(A) - P_{T^{\perp}_{\cvx}}(A)}_{\F}  \\
    &= \norm{(I-UU^{\top})A(I-VV^{\top}) - (I-U_{\cvx}U_{\cvx}^{\top})A(I-V_{\cvx}V_{\cvx}^{\top})}_{\F}\\
    &\leq \norm{(UU^{\top}-U_{\cvx}U_{\cvx}^{\top})A(I-VV^{\top})}_{\F} +  \norm{(I-U_{\cvx}U_{\cvx}^{\top})A(VV^{\top} - V_{\cvx}V_{\cvx}^{\top})}_{\F}\\
    &\lesssim \norm{A}\max\left(\norm{UU^{\top}-U_{\cvx}U_{\cvx}^{\top}}_{\F}, \norm{VV^{\top}-V_{\cvx}V_{\cvx}^{\top}}_{\F}\right)\\
    &\lesssim \norm{A} \frac{\kappa \log(n)}{\sigma_{\min}^{1.5}} \norm{\nabla f}_{\F}.
\end{align*}
Furthermore, let $\Delta_{U} = U - U_{\cvx}R_1^{\top}, \Delta_{V} = V - V_{\cvx}R_2^{\top}$. Then $U\Sigma V^{\top} = U\Sigma R_2V_{\cvx}^{\top} + U_{\cvx}R_1^{\top}\Sigma \Delta_{V}^{\top} + \Delta_{U}\Sigma\Delta_{V}^{\top}.$ We then have
\begin{align*}
&\norm{P_{T_{\cvx}^{\perp}}(XY^{\top})}_{\F} \\
&= \norm{(I-U_{\cvx}U_{\cvx}^{\top}) (U\Sigma V^{\top}) (I - V_{\cvx}V_{\cvx}^{\top})}_{\F}\\
&=  \norm{(I-U_{\cvx}U_{\cvx}^{\top}) (U\Sigma R_2V_{\cvx}^{\top} + U_{\cvx}R_1^{\top}\Sigma \Delta_{V}^{\top} + \Delta_{U}\Sigma\Delta_{V}^{\top}) (I - V_{\cvx}V_{\cvx}^{\top})}_{\F}\\
&=  \norm{(I-U_{\cvx}U_{\cvx}^{\top}) (\Delta_{U}\Sigma\Delta_{V}^{\top}) (I - V_{\cvx}V_{\cvx}^{\top})}_{\F}\\
&\leq \norm{\Delta_{U}\Sigma\Delta_{V}^{\top}}_{\F}\\
&\leq \norm{\Delta_{U}}_{\F} \norm{\Delta_{V}}_{\F} \norm{\Sigma}\\
&\lesssim \frac{\kappa^{3}\log^{2}(n)}{\sigma_{\min}^2} \norm{\nabla f}_{\F}^2.
\end{align*}

This finishes the proof.

\end{proof}

\section{Proof of Theorem \ref{thm:main-theorem}}\label{sec:full-details-proof-main}
In this section, we present the proof of an extended version of \cref{thm:main-theorem}, which is stated below.
\begin{theorem}\label{thm:extended-main-theorem}
Assume \cref{assum:conditions-Z} hold. Suppose $O = M^{*} + \tau^{*}Z + E$  where $E_{ij}$ are independent sub-Gaussian random variables with $\|E_{ij}\|_{\psi_2} \leq \sigma.$ Suppose \[\kappa^{4}\mu^2 r^2 \log^2(n) \leq C_1 n \;\;\;\;\; \text{and} \;\;\;\;\; \frac{\sigma}{\sigma_{\min}} \sqrt{n} \leq C_2\frac{1}{\kappa^3 \mu r^{4.5}\log^{5}(n)}.\] 
Let $(\hat{M}, \hat{\tau})$ be any minimizer of \cref{eq:convex-program} with $\lambda = C_{\lambda} \sigma \sqrt{n} \log^{1.5}(n)$, and let $\tau^{d}$ be defined as in \cref{eq:construct-taud} based on $\hat{M}, \hat{\tau}$. Then for any $C_3 > 0$, for sufficiently large $n$, with probability $1-O(n^{-C_3})$, we have 
\begin{align}
    \tau^{d} - \tau^{*} = \frac{\inner{P_{T^{*\perp}}(Z)}{E}}{\norm{P_{T^{*\perp}}(Z)}_{\F}^2} + \delta, \label{eq:taud-characterization}
\end{align}
where 
\[|\delta| \leq C_{\delta}\left(\frac{\sigma^2 \mu^2 r^{6} \kappa^{3} \log^{8}(n)}{\sigma_{\min}} + \frac{\sqrt{n}}{\norm{Z}_{\F}} \frac{r\kappa \sigma^2 \log^{4}(n)}{\sigma_{\min}}\right).\]
%
Here, $C_{\lambda}, C_{\delta}, C_{\tau}, C_{\infty}$ are constants depending (polynomially) on $C_1, C_2, C_3, C_{r_1}, C_{r_2}$ (where $C_{r_1}$ and $C_{r_2}$ are the constants in Assumptions \ref{cond:Z-condition-nonconvex} and \ref{cond:Z-condition-convex}).
\end{theorem}

Given \cref{thm:non-convex-theorem,lem:tau-decomposition,lem:connection}, it is easy to finish the proof of \cref{thm:extended-main-theorem}. Based on \cref{thm:non-convex-theorem}, there exists a $0\leq t < t^{\star}$, such that
\begin{align*}
\norm{F^{t}H^{t} - F^{*}}_{\F}&\leq C_{\F} \error \norm{F^{*}}_{\F} \\
   \norm{F^{t}H^{t} - F^{*}}_{2,\infty} &\leq C_{\infty}\frac{\sigma \mu r^{2.5} \kappa \log^{3.5}(n)}{\sigma_{\min}}\norm{F^{*}}_{\F} \\
    |\tau^{t} - \tau^{*}| &\leq C_{\tau} \left( \frac{\sigma \mu r^{2.5} \kappa \log^{3.5} n}{\sqrt{n}} + \frac{\sigma \log^{1.5} n}{\norm{Z}_{\F}}\right)\\
    \norm{\nabla f(X^{t}, Y^{t}; \tau^{t})}_{\F} &\leq \frac{\lambda \sqrt{\sigma_{\min}}}{n^{10}}.
\end{align*}
Furthermore, 
\begin{align*}
    \norm{P_{T_t^{\perp}}(M^{*})}_{\infty} &\leq  C_{T,1}\frac{\sigma^2 \mu^2 r^{6} \kappa^{3} \log^{7}(n)}{\sigma_{\min}}\\
    \norm{P_{T_t^{\perp}}(E) - P_{T^{*\perp}}(E)}_{\infty} &\leq C_{T,2}\frac{\sigma^2 r^{3.5} \mu^{1.5} \kappa^{2} \log^{4}(n)}{\sigma_{\min}}\\
    \norm{P_{T_t^{\perp}}(Z)-P_{T^{*\perp}}(Z)}_{\F} &\leq C_{T,3}\frac{r\kappa \sigma \sqrt{n}\log^{2.5}(n)}{\sigma_{\min}} \norm{Z}.
\end{align*}

Let $X = X^{t}H^{t}, Y = Y^{t}H^{t}, \tau = \tau^{t}, T = T^{t}.$ By \cref{lem:connection}, for any minimizer $(\hat{M}, \hat{\tau})$ of the convex program, we have
\begin{align*}
    \norm{XY^{\top} - \hat{M}}_{\F} &\lesssim  \frac{\lambda \kappa \log(n)}{n^{10}}\\
    |\tau - \hat{\tau}| &\lesssim \frac{\lambda \kappa \log(n)}{n^{10}\norm{Z}_{\F}}.
\end{align*}
Furthermore, let $\hat{T}$ be the tangent space of $\hat{M}$, for any $A \in \R^{n\times n}$, we have
\begin{align*}
        \norm{P_{\hat{T}^{\perp}}(A) - P_{T^{\perp}}(A)}_{\F} &\lesssim \norm{A} \frac{\lambda \kappa^2 \log(n)}{\sigma_{\min} n^{10}}\\
        \norm{P_{\hat{T}^{\perp}}(XY^{\top})}_{\F}  &\lesssim \frac{\sigma^2\kappa^{5}\log^{5}(n)}{\sigma_{\min}} \frac{1}{n^{19}}.
\end{align*}

Given all the information, we aim to provide a bound for $\tau^{d} := \hat{\tau} - \frac{\inner{Z}{\hat{U}\hat{V}^{\top}}}{\norm{P_{\hat{T}^{\perp}}(Z)}_{\F}^2}$ where $\hat{M} = \hat{U}\hat{\Sigma}\hat{V}^{\top}$ is the SVD of $\hat{M}.$ 

By the triangle inequality again, we can obtain the bounds for $\hat{T}$ from $T$. We have
\begin{align}
    &\norm{P_{\hat{T}^{\perp}}(M^{*})}_{\infty} \nonumber\\
    &\leq \norm{P_{\hat{T}^{\perp}}(M^{*})-P_{T^{\perp}}(M^{*})}_{\infty} + \norm{P_{T^{\perp}}(M^{*})}_{\infty} \nonumber\\
    &\overset{(i)}{\leq} \norm{P_{\hat{T}^{\perp}}(M^{*}-XY^{\top}+XY^{\top})-P_{T^{\perp}}(M^{*}-XY^{\top})}_{\infty}  + \norm{P_{T^{\perp}}(M^{*})}_{\infty} \nonumber\\
    &\leq  \norm{P_{\hat{T}^{\perp}}(M^{*}-XY^{\top}) -P_{T^{\perp}}(M^{*}-XY^{\top})}_{\F} + \norm{P_{\hat{T}^{\perp}}(XY^{\top})}_{\F} + \norm{P_{T^{\perp}}(M^{*})}_{\infty}\nonumber\\
    &\overset{(ii)}{\lesssim} \norm{M^{*}-XY^{\top}} \frac{\lambda \kappa^2 \log(n)}{\sigma_{\min} n^{10}} + \frac{\sigma^2\kappa^{5}\log^{5}(n)}{\sigma_{\min}} \frac{1}{n^{19}} + \frac{\sigma^2 \mu^2 r^{6} \kappa^{3} \log^{7}(n)}{\sigma_{\min}} \nonumber\\
    &\overset{(iii)}{\lesssim}  \sigma\kappa r^{0.5} \sqrt{n}\log^{2.5}(n) \frac{\sigma\sqrt{n} \kappa^2 \log^{2.5}(n)}{\sigma_{\min} n^{10}}  + \frac{\sigma^2 \mu^2 r^{6} \kappa^{3} \log^{7}(n)}{\sigma_{\min}} \nonumber\\
    &\lesssim \frac{\sigma^2 \mu^2 r^{6} \kappa^{3} \log^{7}(n)}{\sigma_{\min}}.\label{eq:delta-1}
\end{align}
Here, (i) is due to $P_{T}^{\perp}(XY^{\top}) = 0$, (ii) is due to $\norm{P_{\hat{T}^{\perp}}(A) - P_{T^{\perp}}(A)}_{\F} \lesssim \norm{A} \frac{\lambda \kappa^2 \log(n)}{\sigma_{\min} n^{10}}$ and $\norm{P_{\hat{T}^{\perp}}(XY^{\top})}_{\F} \lesssim \frac{\sigma^2\kappa^{5}\log^{5}(n)}{\sigma_{\min}} \frac{1}{n^{19}}$, and (iii) is due to $\kappa^{4}\mu^2 r^2 \log^2(n) \lesssim n.$ 

Similarly, one can obtain
\begin{align}
\norm{P_{\hat{T}^{\perp}}(E) - P_{T^{*\perp}}(E)}_{\infty} 
&\leq  \norm{P_{\hat{T}^{\perp}}(E) - P_{T^{\perp}}(E)}_{\infty} + \norm{P_{T^{\perp}}(E) - P_{T^{*\perp}}(E)}_{\infty} \nonumber\\
&\lesssim \norm{P_{\hat{T}^{\perp}}(E) - P_{T^{\perp}}(E)}_{\F} + \frac{\sigma^2 r^{3.5} \mu^{1.5} \kappa^{1.5} \log^{4}(n)}{\sigma_{\min}} \nonumber\\
&\lesssim \frac{\sigma^2 r^{3.5} \mu^{1.5} \kappa^{1.5} \log^{4}(n)}{\sigma_{\min}}.\label{eq:delta-2}
\end{align}
Also, 
\begin{align}
 \norm{P_{\hat{T}^{\perp}}(Z)-P_{T^{*\perp}}(Z)}_{\F}
 &\leq \norm{P_{\hat{T}^{\perp}}(Z)-P_{T^{\perp}}(Z)}_{\F} + \norm{P_{T^{\perp}}(Z)-P_{T^{*\perp}}(Z)}_{\F} \nonumber\\
 &\lesssim  \norm{P_{\hat{T}^{\perp}}(Z)-P_{T^{\perp}}(Z)}_{\F} + \frac{r\kappa \sigma \sqrt{n}\log^{2.5}(n)}{\sigma_{\min}} \norm{Z} \nonumber\\
 &\lesssim \frac{r\kappa \sigma \sqrt{n}\log^{2.5}(n)}{\sigma_{\min}} \norm{Z}_{\F}. \label{eq:delta-3}
\end{align}

Recall by \cref{lem:tau-decomposition}, we have
\begin{align*}
    \tau^{d} - \tau^{*} 
&    = \frac{\inner{P_{\hat{T}^{\perp}}(Z)}{E}}{\norm{P_{\hat{T}^{\perp}}(Z)}_{\F}^2} +\frac{\inner{Z}{P_{\hat{T}^{\perp}}(M^{*})}}{\norm{P_{\hat{T}^{\perp}}(Z)}_{\F}^2} \\
& \overset{(i)}{=} \frac{\inner{P_{T^{*\perp}}(Z)}{E}}{\norm{P_{\hat{T}^{\perp}}(Z)}_{\F}^2} + \frac{\inner{Z}{P_{\hat{T}^{\perp}}(E)-P_{T^{*\perp}}(E)} + \inner{Z}{P_{\hat{T}^{\perp}}(M^{*})}}{\norm{P_{\hat{T}^{\perp}}(Z)}_{\F}^2}
\end{align*}
where in (i) we use the fact that for any $A, B \in \R^{n\times n}$, 
$$\inner{A}{P_{T^{*\perp}}(B)} = \inner{P_{T^{*\perp}}(A)}{P_{T^{*\perp}}(B)} = \inner{P_{T^{*\perp}}(A)}{B}.
$$ 
Let 
$$
\delta = \tau^{d} - \tau^{*} - \frac{\inner{P_{T^{*\perp}}(Z)}{E}}{\norm{P_{T^{*}}(Z)}_{\F}^2}.
$$

The following claim provides the bounds for $\delta$, which completes the proof of \cref{thm:main-theorem}.
\begin{claim}\label{claim:delta-bound-claim}
Providing \cref{eq:delta-1,eq:delta-2,eq:delta-3}, we have
\begin{align*}
    |\delta| \lesssim \left(\frac{\sigma^2 \mu^2 r^{6} \kappa^{3} \log^{8}(n)}{\sigma_{\min}} + \frac{\sqrt{n}}{\norm{Z}_{\F}} \frac{r\kappa \sigma^2 \log^{4}(n)}{\sigma_{\min}}\right).
\end{align*}
\end{claim}

\subsection{Proof of Claim \ref{claim:delta-bound-claim}}
By the definition of $\delta$, we have
\begin{align*}
    |\delta| 
    &= \left|\frac{\inner{P_{T^{*\perp}}(Z)}{E}}{\norm{P_{\hat{T}^{\perp}}(Z)}_{\F}^2} - \frac{\inner{P_{T^{*\perp}}(Z)}{E}}{\norm{P_{T^{*}}(Z)}_{\F}^2} + \frac{\inner{Z}{P_{\hat{T}^{\perp}}(E)-P_{T^{*\perp}}(E)} + \inner{Z}{P_{\hat{T}^{\perp}}(M^{*})}}{\norm{P_{\hat{T}^{\perp}}(Z)}_{\F}^2}\right|\\
    &\leq \underbrace{|\inner{P_{T^{*\perp}}(Z)}{E}|\left|\frac{1}{\norm{P_{\hat{T}^{\perp}}(Z)}_{\F}^2} - \frac{1}{\norm{P_{T^{*\perp}}(Z)}_{\F}^2}\right|}_{A_0}
    + \underbrace{\frac{|\inner{Z}{P_{\hat{T}^{\perp}}(E)-P_{T^{*\perp}}(E)}|}{\norm{P_{\hat{T}^{\perp}}(Z)}_{\F}^2}}_{A_1} \\
    &\quad + \underbrace{\frac{|\inner{Z}{P_{\hat{T}^{\perp}}(M^{*})}|}{\norm{P_{\hat{T}^{\perp}}(Z)}_{\F}^2}}_{A_2}.
\end{align*}

For controlling $A_0$, note that $|\inner{P_{T^{*\perp}}(Z)}{E}| \lesssim \sigma\norm{P_{T^{*\perp}}(Z)}_{\F}\log^{0.5}(n)$ since the sum of independent sub-Gaussian random variables is still sub-Gaussian. We also have the facts that $\log(n)\norm{P_{T^{*\perp}}(Z)}_{\F}^2 \gtrsim \norm{Z}_{\F}^2$ by \cref{lem:general-conditions-small-ball}.
This implies 
$$\norm{P_{\hat{T}^{\perp}}(Z)-P_{T^{*\perp}}(Z)}_{\F} \lesssim \frac{r\kappa \sigma \sqrt{n}\log^{2.5}(n)}{\sigma_{\min}} \norm{Z}_{\F} \lesssim \frac{r\kappa \sigma \sqrt{n}\log^{3}(n)}{\sigma_{\min}} \norm{P_{T^{*\perp}}(Z)}_{\F}.
$$
Then due to $\SNR$, we have 
$$
\frac{1}{2}\norm{P_{T^{*\perp}}(Z)}_{\F} \leq \norm{P_{\hat{T}^{\perp}}(Z)}_{\F} \leq 2\norm{P_{T^{*\perp}}(Z)}_{\F}.
$$ Also, we have
\begin{align*}
    \left|\frac{1}{\norm{P_{\hat{T}^{\perp}}(Z)}_{\F}^2} - \frac{1}{\norm{P_{T^{*\perp}}(Z)}_{\F}^2}\right| 
    &\leq \frac{|\norm{P_{T^{*\perp}}(Z)}_{\F}^2 -\norm{P_{\hat{T}^{\perp}}(Z)}_{\F}^2|}{\norm{P_{T^{*\perp}}(Z)}_{\F}^2\norm{P_{\hat{T}^{\perp}}(Z)}_{\F}^2}\\
    &\lesssim \frac{\norm{P_{\hat{T}^{\perp}}(Z)-P_{T^{*\perp}}(Z)}_{\F}}{\norm{P_{T^{*\perp}}(Z)}_{\F}^3}\\
    &\lesssim \frac{r\kappa \sigma \sqrt{n}\log^{3}(n)}{\sigma_{\min}} \frac{1}{\norm{P_{T^{*\perp}}(Z)}_{\F}^2}.
\end{align*}
Then, we can control $A_0$.
\begin{align*}
    A_0 
    &\lesssim  \sigma\norm{P_{T^{*\perp}}(Z)}_{\F}\log^{0.5}(n) \cdot \frac{r\kappa \sigma \sqrt{n}\log^{3}(n)}{\sigma_{\min}} \frac{1}{\norm{P_{T^{*\perp}}(Z)}_{\F}^2}\\
    &\lesssim \frac{r\kappa \sigma^2 \sqrt{n}\log^{3.5}(n)}{\sigma_{\min}}\frac{1}{\norm{P_{T^{*\perp}}(Z)}_{\F}}\\
    &\lesssim \frac{r\kappa \sigma^2 \sqrt{n}\log^{4}(n)}{\sigma_{\min}}\frac{1}{\norm{Z}_{\F}}.
\end{align*}

Next, we control $A_1$. Note that $|\inner{Z}{A}| = |\sum_{ij}A_{ij}Z_{ij}| \leq \norm{A}_{\infty}\norm{Z}_{\F}^2$ for binary matrix $Z$. We have
\begin{align*}
    A_1 
    &= \frac{|\inner{Z}{P_{\hat{T}^{\perp}}(E)-P_{T^{*\perp}}(E)}|}{\norm{P_{\hat{T}^{\perp}}(Z)}_{\F}^2}\\
    &\leq \frac{\norm{P_{\hat{T}^{\perp}}(E)-P_{T^{*\perp}}(E)}_{\infty} \norm{Z}_{\F}^2}{\norm{P_{\hat{T}^{\perp}}(Z)}_{\F}^2}\\
    &\lesssim \frac{\sigma^2 r^{3.5} \mu^{1.5} \kappa^{1.5} \log^{4}(n)}{\sigma_{\min}} \frac{\norm{Z}_{\F}^2}{\norm{P_{\hat{T}^{*\perp}}(Z)}_{\F}^2}\\
    &\lesssim \frac{\sigma^2 r^{3.5} \mu^{1.5} \kappa^{1.5} \log^{5}(n)}{\sigma_{\min}}.
\end{align*}

Similarly, we can control $A_2$. 
\begin{align*}
    A_2 
    &= \frac{|\inner{Z}{P_{\hat{T}^{\perp}}(M - M^{*})}|}{\norm{P_{\hat{T}^{\perp}}(Z)}_{\F}^2}\\
    &\leq \frac{\norm{P_{\hat{T}^{\perp}}(M - M^{*})}_{\infty}\norm{Z}_{\F}^2}{\norm{P_{\hat{T}^{\perp}}(Z)}_{\F}^2}\\
    &\lesssim \frac{\sigma^2 \mu^2 r^{6} \kappa^{3} \log^{7}(n)}{\sigma_{\min}} \frac{\norm{Z}_{\F}^2}{\norm{P_{\hat{T}^{*\perp}}(Z)}_{\F}^2}\\
    &\lesssim \frac{\sigma^2 \mu^2 r^{6} \kappa^{3} \log^{8}(n)}{\sigma_{\min}}.
\end{align*}
Combining $A_0, A_1, A_2$ together, we have
\begin{align*}
    |\delta|  \lesssim \frac{r\kappa \sigma^2 \sqrt{n}\log^{4}(n)}{\sigma_{\min}}\frac{1}{\norm{Z}_{\F}} + \frac{\sigma^2 \mu^2 r^{6} \kappa^{3} \log^{8}(n)}{\sigma_{\min}}
\end{align*}
which completes the proof.

\section{Recovery of Counterfactuals}\label{section:proof-of-theorem-M}

We also studied the following de-biased estimator $M^{d}$ and derived the entry-wise characterization for $M^{d}-M^{*}$ (where the asymptotical normality follows directly).   
This particular way of de-biasing $\hat{M}$ can be seen from the following lemma. 
\begin{lemma}\label{lem:M-decomposition}
Suppose $(\hat{M}, \hat{\tau})$ is a minimizer of \eqref{eq:convex-program}. Let $\hat{M} = \hat{U}\hat{\Sigma} \hat{V}^{\top}$ be the SVD of $\hat{M}$, and let $\hat{T}$ denote the tangent space of $\hat{M}$. Then,
\begin{align*}
    \hat{M} - M^{*} 
    &= - \lambda \hat{U}\hat{V}^{\top} - \frac{\lambda \inner{Z}{\hat{U}\hat{V}^{\top}}}{\norm{P_{\hat{T}^{\perp}}(Z)}_{\F}^2} P_{\hat{T}}(Z)\\
    &\quad -\left(\frac{\inner{P_{\hat{T}^{\perp}}(Z)}{E}}{\norm{P_{\hat{T}^{\perp}}(Z)}_{\F}^2} 
    +\frac{\inner{Z}{P_{\hat{T}^{\perp}}(M^{*})}}{\norm{P_{\hat{T}^{\perp}}(Z)}_{\F}^2}\right) P_{\hat{T}}(Z)\\
    &\quad - P_{\hat{T}^{\perp}}(M^{*}) + P_{\hat{T}}(E).
\end{align*}
\end{lemma}

\begin{proof}[Proof Sketch of \cref{lem:M-decomposition}]
From the first-order conditions of $(\hat{M},\hat{\tau})$, we have $M^{*} - \hat{M} + (\tau^{*} - \hat{\tau})Z + E = \lambda \hat{U}\hat{V}^{\top} + \lambda W.$ Substituting $\lambda W$ with $P_{\hat{T}^{\perp}} (M^{*}) + P_{\hat{T}^{\perp}}(E) - (\hat{\tau} - \tau^{*})P_{\hat{T}^{\perp}}(Z) = \lambda W$ (\cref{eq:PTPerpO}), we obtain 
\begin{align*}
M^{*} - \hat{M} + (\tau^{*} - \hat{\tau})Z + E = \lambda \hat{U}\hat{V}^{\top} + P_{\hat{T}^{\perp}} (M^{*}) + P_{\hat{T}^{\perp}}(E) - (\hat{\tau} - \tau^{*})P_{\hat{T}^{\perp}}(Z).
\end{align*}
This implies 
\begin{align*}
M^{*} + P_{\hat{T}}(E) + (\tau^{*} - \hat{\tau}) P_{\hat{T}}(Z) = \hat{M} + \lambda \hat{U}\hat{V}^{\top} + P_{\hat{T}^{\perp}} (M^{*})
\end{align*}
Then substituting $\tau^{*} - \hat{\tau}$ with \cref{lem:tau-decomposition}, we can obtain the desired decomposition.
\end{proof}

\begin{theorem}\label{thm:M-de-bias}
Under the same setup in \cref{thm:main-theorem}, let $M^{d} := \hat{M} + \lambda \hat{U}\hat{V}^{\top} + P_{\hat{T}}(Z) \cdot \lambda \inner{Z}{\hat{U}\hat{V}^{\top}}\big/\norm{P_{\hat{T}^{\perp}}(Z)}_{\F}^2.$ We have, for all $i, j \in [n]$,  
\begin{align*}
M^{d}_{ij} - M^{*}_{ij} = \inner{P_{T^{*}}(e_ie_j^{\top}) - \frac{(P_{T^{*}}(Z))_{ij}}{\norm{P_{T^{*\perp}}(Z)}_{\F}^2}P_{T^{*\perp}}(Z)}{E+\delta\circ Z} + \tilde{O}\left(\frac{1}{n}\right)
\end{align*}
where $e_i$ ($e_j$) corresponds to the i-th (j-th) standard basis.
\end{theorem}

\begin{proof}
From \cref{lem:M-decomposition}, we have
\begin{align*}
&M^{d}_{ij} - M^{*}_{ij} \\
&= (P_{T^{*}}(E))_{ij} - \frac{\inner{P_{T^{*\perp}}(Z)}{E}}{\norm{P_{T^{*\perp}}(Z)}_{\F}^2} (P_{T^{*}}(Z))_{ij} \\
&\quad - \underbrace{(P_{\hat{T}^{\perp}}(M^{*}))_{ij}}_{A_0} + \underbrace{(P_{\hat{T}}(E)-P_{T^{*}}(E))_{ij}}_{A_1}\\
&\quad \underbrace{-\left(\frac{\inner{P_{\hat{T}^{\perp}}(Z)}{E}}{\norm{P_{\hat{T}^{\perp}}(Z)}_{\F}^2} +  \frac{\inner{Z}{P_{\hat{T}^{\perp}}(M^{*})}}{\norm{P_{\hat{T}^{\perp}}(Z)}_{\F}^2} \right) (P_{\hat{T}}(Z))_{ij} + \frac{\inner{P_{T^{*\perp}}(Z)}{E}}{\norm{P_{T^{*\perp}}(Z)}_{\F}^2} (P_{T^{*}}(Z))_{ij}}_{A_2}.
\end{align*}
Note that $(P_{T^{*}}(E))_{ij} = \inner{P_{T^{*}}(E)}{e_ie_j^{\top}} = \inner{E}{P_{T^{*}}(e_ie_j^{\top})}.$ Hence, $\delta_{ij} = -A_0 + A_1 + A_2$. Then, it is sufficient to control $A_0, A_1$ and $A_2$. 

For $A_0$, by \cref{eq:delta-1}, we have $|A_0| \leq \norm{P_{\hat{T}^{\perp}}(M^{*})}_{\infty} \lesssim \frac{\sigma^2 \mu^2 r^{6} \kappa^{3} \log^{7}(n)}{\sigma_{\min}}.$

For $A_1$, by \cref{eq:delta-2}, we have $|A_1| \leq \norm{P_{\hat{T}}(E) - P_{T^{*}}(E)}_{\infty} = \norm{P_{\hat{T}^{\perp}}(E) - P_{T^{*\perp}}(E)}_{\infty} \lesssim \frac{\sigma^2 r^{3.5} \mu^{1.5} \kappa^{1.5} \log^{4}(n)}{\sigma_{\min}}.$

In order to control $A_2$, we can further decompose $A_2$:
\begin{align*}
A_2 
&= -\underbrace{\left(\frac{\inner{P_{\hat{T}^{\perp}}(Z)}{E}}{\norm{P_{\hat{T}^{\perp}}(Z)}_{\F}^2} +  \frac{\inner{Z}{P_{\hat{T}^{\perp}}(M^{*})}}{\norm{P_{\hat{T}^{\perp}}(Z)}_{\F}^2} -  \frac{\inner{P_{T^{*\perp}}(Z)}{E}}{\norm{P_{T^{*\perp}}(Z)}_{\F}^2}\right) (P_{\hat{T}}(Z))_{ij}}_{B_0} \\
&\quad + \underbrace{\frac{\inner{P_{T^{*\perp}}(Z)}{E}}{\norm{P_{T^{*\perp}}(Z)}_{\F}^2} \left(P_{T^{*}}(Z) - P_{\hat{T}}(Z) \right)_{ij}}_{B_1}.
\end{align*}
Then it boils down to control $B_0$ and $B_1$. For $B_1$, note that we have the entrywise norm bounds for $P_{\hat{T}^{\perp}}(Z)-P_{T^{*\perp}}(Z)$ by 
\begin{align}
 &\norm{P_{\hat{T}^{\perp}}(Z)-P_{T^{*\perp}}(Z)}_{\infty} \nonumber\\
 &\leq \norm{P_{\hat{T}^{\perp}}(Z)-P_{T^{\perp}}(Z)}_{\F} + \norm{P_{T^{\perp}}(Z)-P_{T^{*\perp}}(Z)}_{\infty} \nonumber\\
 &\lesssim  \norm{P_{\hat{T}^{\perp}}(Z)-P_{T^{\perp}}(Z)}_{\F} +  C_{T,4}\frac{r^{3.5} \kappa^{1.5} \sigma \mu^{1.5} \log^{3.5}(n)}{\sigma_{\min}} \left(\norm{Z}_{2,\infty}+\norm{Z^{\top}}_{2,\infty}\right) \nonumber\\
 &\lesssim \norm{Z} \frac{\lambda \kappa^2 \log(n)}{\sigma_{\min} n^{10}} +   \frac{r^{3.5} \kappa^{1.5} \sigma \mu^{1.5} \log^{3.5}(n)}{\sigma_{\min}} \left(\norm{Z}_{2,\infty}+\norm{Z^{\top}}_{2,\infty}\right)\nonumber\\
 &\lesssim \frac{r^{3.5} \kappa^{1.5} \sigma \mu^{1.5} \log^{3.5}(n)}{\sigma_{\min}} \left(\norm{Z}_{2,\infty}+\norm{Z^{\top}}_{2,\infty}\right). \label{eq:delta-4}
\end{align}

 We also have $\frac{|\inner{P_{T^{*\perp}}(Z)}{E}|}{\norm{P_{T^{*\perp}}(Z)}_{\F}^2} \lesssim \frac{\sigma \log(n)}{\norm{Z}_{\F}^2}$. Therefore,
\begin{align*}
|B_1| 
&\lesssim  \frac{r^{3.5} \kappa^{1.5} \sigma \mu^{1.5} \log^{3.5}(n)}{\sigma_{\min}} \left(\norm{Z}_{2,\infty}+\norm{Z^{\top}}_{2,\infty}\right) \cdot \frac{\sigma \log(n)}{\norm{Z}_{\F}} \\
&\lesssim \frac{r^{3.5} \kappa^{1.5} \sigma^2 \mu^{1.5} \log^{4.5}(n)}{\sigma_{\min}}.
\end{align*}

For $B_0$, note that 
\begin{align}
\norm{P_{T^{*}}(Z)}_{\infty} &\leq \norm{U^{*}U^{*T}Z}_{\infty} + \norm{ZV^{*}V^{*T}}_{\infty} + \norm{U^{*}U^{*T}ZV^{*}V^{*T}}_{\infty} \nonumber\\
&\leq \norm{U^{*}}_{2,\infty} \norm{Z^{T}}_{2,\infty} \norm{U^{*}} + \norm{Z}_{2,\infty}\norm{V^{*}}_{2,\infty} \norm{V^{*}} \nonumber\\
&\quad + \norm{U^{*}}_{2,\infty}\norm{V^{*}}_{2,\infty} \norm{U^{*}} \norm{V^{*}} \norm{Z} \nonumber\\
&\lesssim  \sqrt{\frac{\mu r}{n}} \left(\norm{Z}_{2,\infty} + \norm{Z^{T}}_{2,\infty}\right) + \frac{\mu r}{n} \cdot \norm{Z}. \label{eq:PTstar-Z-infinity}
\end{align}
By the triangle inequality, one can obtain the similar bound for $\norm{P_{\hat{T}}(Z)}_{\infty}.$ Furthermore, let 
$$
\delta' = \left(\frac{\inner{P_{\hat{T}^{\perp}}(Z)}{E}}{\norm{P_{\hat{T}^{\perp}}(Z)}_{\F}^2} +  \frac{\inner{Z}{P_{\hat{T}^{\perp}}(M^{*})}}{\norm{P_{\hat{T}^{\perp}}(Z)}_{\F}^2} -  \frac{\inner{P_{T^{*\perp}}(Z)}{E}}{\norm{P_{T^{*\perp}}(Z)}_{\F}^2}\right).
$$
We can also obtain that $|\delta'| \lesssim \left(\frac{\sigma^2 \mu^2 r^{6} \kappa^{3} \log^{8}(n)}{\sigma_{\min}} + \frac{\sqrt{n}}{\norm{Z}_{\F}} \frac{r\kappa \sigma^2 \log^{4}(n)}{\sigma_{\min}}\right)$ by \cref{claim:delta-bound-claim}. This implies that
\begin{align*}
|B_0|
&\leq |\delta'| \norm{P_{\hat{T}}(Z)}_{\infty} \\
&\lesssim \left(\sqrt{\frac{\mu r}{n}} \left(\norm{Z}_{2,\infty} + \norm{Z^{T}}_{2,\infty}\right) + \frac{\mu r}{n} \cdot \norm{Z}\right) \\
&\quad \cdot\left(\frac{\sigma^2 \mu^2 r^{6} \kappa^{3} \log^{8}(n)}{\sigma_{\min}} + \frac{\sqrt{n}}{\norm{Z}_{\F}} \frac{r\kappa \sigma^2 \log^{4}(n)}{\sigma_{\min}}\right)\\
&\lesssim \frac{\sigma^2 \mu^2 r^{6} \kappa^{3} \log^{8}(n)}{\sigma_{\min}} \cdot \mu r + \frac{\mu r^2\kappa \sigma^2 \log^{4}(n)}{\sigma_{\min}} \sqrt{\mu r}\\
&\lesssim \frac{\sigma^2 \mu^3 r^{7} \kappa^{3} \log^{8}(n)}{\sigma_{\min}}.
\end{align*}
Combining the bounds for $A_0, A_1, B_0, B_1$ together, we arrive at
\begin{align*}
|\delta_{ij}| 
&\leq |A_0|+|A_1|+|B_0| + |B_1|\\
&\lesssim \frac{\sigma^2 \mu^2 r^{6} \kappa^{3} \log^{7}(n)}{\sigma_{\min}} +  \frac{\sigma^2 r^{3.5} \mu^{1.5} \kappa^{1.5} \log^{4}(n)}{\sigma_{\min}} + \frac{r^{3.5} \kappa^{1.5} \sigma^2 \mu^{1.5} \log^{4.5}(n)}{\sigma_{\min}}\\
&\quad + \frac{\sigma^2 \mu^3 r^{7} \kappa^{3} \log^{8}(n)}{\sigma_{\min}}\\
&\lesssim  \frac{\sigma^2 \mu^3 r^{7} \kappa^{3} \log^{8}(n)}{\sigma_{\min}}.
\end{align*}
This completes the proof. 
\end{proof}

\section{Technical Lemmas}\label{sec:technical-lemmas}
\begin{lemma}[Direct implication of \cref{cond:Z-condition-nonconvex}]\label{lem:condition-non-convex-implication}
Given $Z \in \R^{n\times n}, U \in \R^{n\times r}, V \in \R^{n\times r}$ where $U^{\top}U = V^{\top}V = I_{r}.$ Suppose 
\begin{align}
    \norm{ZV}_{\F}^2 + \norm{Z^{\top}U}_{\F}^2 &\leq \left(1-\eta\right) \norm{Z}_{\F}^2 \label{eq:condition-technical-lemma}
\end{align}
for some $0< \eta \leq 1.$ Let $P_{T^{\perp}}(Z) = (I - UU^{\top})Z(I - VV^{\top}).$ Then
\begin{align}\label{eq:implication-assumption3a}
    \norm{U^{\top}ZV}_{\F}^2 &\leq \norm{P_{T^{\perp}}(Z)}_{\F}^2 - \eta \norm{Z}_{\F}^2.
\end{align}
This also implies $\norm{P_{T^{\perp}}(Z)}_{\F}^2 \geq \eta\norm{Z}_{\F}^2$ and $\norm{U^{\top}ZV}_{\F} \leq \left(1-\frac{\eta}{2}\right)\norm{P_{T^{\perp}}(Z)}_{\F}.$
\end{lemma}

\begin{proof}
Note that $\norm{U^{\top}ZV}_{\F}^2 = \tr(UU^{\top}ZVV^{\top}Z^{\top}), \norm{P_{T^{\perp}}(Z)}_{\F}^2 = \tr((I-UU^{\top})Z(I-VV^{\top})Z^{\top})$ since $\norm{A}_{\F}^2 = \tr(AA^{\top})$ and $I-UU^{\top}, I-VV^{\top}$ are projection matrices. Then
\begin{align*}
    \norm{P_{T^{\perp}}(Z)}_{\F}^2 - \norm{U^{\top}ZV}_{\F}^2 
    &= \tr((I-UU^{\top})Z(I-VV^{\top})Z^{\top}) - \tr(UU^{\top}ZVV^{\top}Z^{\top})\\
    &= \tr(ZZ^{\top}) - \tr(UU^{\top}ZZ^{\top}) - \tr(ZVV^{\top}Z^{\top})\\
    &= \norm{Z}_{\F}^2 - \norm{Z^{\top}U}_{\F}^2 - \norm{ZV}_{\F}^2\\
    &\geq \eta \norm{Z}_{\F}^2
\end{align*}
where the last inequality is by \cref{eq:condition-technical-lemma}. To show $\norm{U^{\top}ZV}_{\F} \leq \left(1-\frac{\eta}{2}\right)\norm{P_{T^{\perp}}(Z)}_{\F}$, note that
\begin{align*}
    \left(1-\frac{\eta}{2}\right)^2\norm{P_{T^{\perp}}(Z)}_{\F}^2 
    &\geq (1 - \eta)\norm{P_{T^{\perp}}(Z)}_{\F}^2\\
    &\geq \norm{P_{T^{\perp}}(Z)}_{\F}^2 - \eta \norm{Z}_{\F}^2 \\
    &\geq \norm{U^{\top}ZV}_{\F}^2. 
\end{align*}
This finishes the proof. 

\end{proof}

\begin{lemma}\label{lem:F0-F1-F2}
Let $F_0, F_1, F_2 \in \R^{2n \times r}$ be three matrices. Suppose
\begin{align}
    \norm{F_1 - F_0} \norm{F_0} \leq \sigma_{r}^2(F_0) / 2 \text{\quad and \quad} \norm{F_1 - F_2} \norm{F_0} \leq \sigma_{r}^2(F_0) / 4,\nonumber 
\end{align}
where $\sigma_{i}(F)$ is the $i$-th largest singular value of $F$. Let
\begin{align}
    R_1 := \arg\min_{R \in O^{r\times r}} \norm{F_1 R - F_0}_{\F} \text{\quad and \quad} R_{2} := \arg\min_{R \in O^{r\times r}} \norm{F_2 R - F_0}_{\F}. \nonumber
\end{align}
Then, the followings hold
\begin{align}
    \norm{R_1 - R_2}_{\F} &\leq  \frac{2}{\sigma_{r}^2(F_0)} \norm{F_2 - F_1}_{\F} \norm{F_0} \label{eq:R1-R2-bound}\\
    \norm{F_1 R_1 - F_2 R_2} &\leq 5 \frac{\sigma_{1}^2(F_0)}{\sigma_{r}^2(F_0)} \norm{F_1 - F_2} \nonumber\\
    \norm{F_1 R_1 - F_2 R_2}_{\F} &\leq 5 \frac{\sigma_{1}^2(F_0)}{\sigma_{r}^2(F_0)} \norm{F_1 - F_2}_{\F}\nonumber
\end{align}
\end{lemma}

\begin{proof}
The same as the proof of Lemma 37 in \cite{ma2019implicit}.
\end{proof}

\begin{lemma}[Lemma 20 in \cite{chen2019noisy}]\label{lem:SigmaQ}
Let $U\Sigma V^{\top}$ be the SVD of a rank-$r$ matrix $XY^{\top}$ with $X, Y \in \R^{n\times r}.$ Then there exists an invertible matrix $Q \in \R^{r\times r}$ such that $X = U\Sigma^{1/2} Q$ and $Y = V\Sigma^{1/2} Q^{-\top}$. In addition, one has
\begin{align*}
    \norm{\Sigma_{Q} - \Sigma_{Q}^{-1}}_{\F} \leq \frac{1}{\sigma_{\min}(\Sigma)} \norm{X^{\top}X - Y^{\top}Y}_{\F},
\end{align*}
where $U_{Q}\Sigma_{Q}V_{Q}^{\top}$ is the SVD of $Q$. In particular, if $X^{\top}X - Y^{\top}Y = 0$, then $Q$ must be a rotation matrix. 
\end{lemma}

\begin{lemma}[Lemma 13 in \cite{CFMY:19}]\label{lem:square-root-bound}
Consider two symmetric matrices obeying $A_1 \succeq \mu_1 I$ and $A_2 \succeq \mu_2 I$ for some $\mu_1, \mu_2 > 0.$ Let $R_1 \succeq 0$ ($R_2 \succeq 0$ respectively) be the (principal) matrix square root of $A_1$ ($A_2$ respectively). Then one has  
\begin{align*}
    \norm{R_1 - R_2} \leq \frac{1}{\sqrt{\mu_1} + \sqrt{\mu_2}} \norm{A_1 - A_2}. 
\end{align*}
\end{lemma}

\begin{lemma}[Perturbation on the row space and column space]\label{lem:XY-UV}
Suppose $X_1, X_2, Y_1, Y_2 \in \R^{n\times r}$ with singular values in the interval $[\sigma_{a}, \sigma_{b}]$ where $0 < \sigma_a \leq \sigma_{b}.$ Suppose
\begin{align}
    \max (\norm{X_1 - X_2}_{\F}, \norm{Y_1 - Y_2}_{\F}) &\leq a\\
    \max \left(\norm{X_1^{\top}X_1 - Y_1^{\top}Y_1}_{\F}, \norm{X_2^{\top}X_2 - Y_2^{\top}Y_2}_{\F}\right) &\leq \epsilon. \label{eq:XX-YY-Technical}
\end{align}

Let $U_{1}\Sigma_1 V_1^{\top}$ be the SVD of the matrix $X_1 Y_1^{\top}$, $U_2 \Sigma_2 V_2^{\top}$ be the SVD of the matrix $X_2 Y_2^{\top}$ where $U_1, U_2, V_1, V_2 \in \R^{n\times r}, \Sigma_1, \Sigma_2 \in \R^{r\times r}.$ Then, there exists rotation matrices $R_1, R_2 \in \O^{r\times r}$ such that
\begin{align}\label{eq:U-F-V-F}
    \max (\norm{U_1R_1 - U_2R_2}_{\F}, \norm{V_1 R_1 - V_2R_2}_{\F}) \leq   2\sqrt{r}\frac{\sigma_b}{\sigma_a} \left(a + \frac{2\sigma_b \epsilon}{\sigma_a^2}\right) \frac{1}{\sigma_a}.
\end{align}

Furthermore, if one has $\max (\norm{X_1 - X_2}_{2,\infty}, \norm{Y_1 - Y_2}_{2,\infty}) \leq b$, then there exists $R_1, R_2 \in \O^{r\times r}$ such that \cref{eq:U-F-V-F} and the following hold
\begin{align}\label{eq:U1RU2-2-infty}
    &\max (\norm{U_1R_1 - U_2R_2}_{2,\infty}, \norm{V_1 R_1 - V_2R_2}_{2,\infty}) \nonumber\\
    &\leq \left(b + \norm{U_2}_{2,\infty} a \frac{\sigma_b}{\sigma_a} + \frac{2\sigma_b \epsilon}{\sigma_a^2} (1+\norm{U_2}_{2,\infty}\sigma_b/\sigma_a)\right) \frac{1}{\sigma_a}.
\end{align}
One particular choice of $R_1,R_2$ satisfying above conditions is $R_1 = U_{Q_1}V_{Q_1}^{\top}, R_2 = U_{Q_2}V_{Q_2}^{\top}$ where $U_{Q_1}\Sigma_{Q_1}V_{Q_1}^{\top}$ is the SVD of the matrix $\Sigma_{1}^{-1/2}U_1^{\top}X_1$ and $U_{Q_2}\Sigma_{Q_2}V_{Q_2}^{\top}$ is the SVD of the matrix $\Sigma_{2}^{-1/2}U_2^{\top}X_2.$
\end{lemma}

\begin{proof}
Note that we have the following facts.
\begin{align}
    \sigma_{\max}(\Sigma_1) &= \sigma_{\max}(X_1Y_1^{\top}) \leq \sigma_{\max}(X_1)\sigma_{\max}(Y_1) \leq \sigma_b^2 \label{eq:sigma-max-Sigma1}\\ 
    \sigma_{\min}(\Sigma_1) &= \sigma_{\min}(X_1Y_1^{\top}) \geq \sigma_{\min}(X_1)\sigma_{\min}(Y_1) \geq \sigma_a^2.
\end{align}
Similarly, $\sigma_{\max}(\Sigma_2) \leq \sigma_b^2, \sigma_{\min}(\Sigma_2) \geq \sigma_a^2.$

We first intend to invoke \cref{lem:SigmaQ} to establish the connection between $X_1, Y_1$ ($X_2, Y_2$) and $U_1, V_1$ ($U_2, V_2$).

By \cref{lem:SigmaQ}, there exists invertible matrix $Q_1$ such that $X_1 = U_1\Sigma_1^{1/2} Q_1, Y_1 = V_1\Sigma_1^{1/2} Q_1^{-\top}$ and 
\begin{align}
\norm{\Sigma_{Q_1} - \Sigma_{Q_1}^{-1}}_{\F} \leq \frac{\norm{X_1^{\top}X - Y_1^{\top}Y}_{\F}}{\sigma_{\min}(\Sigma_1)}\overset{(i)}{\leq} \frac{\epsilon}{\sigma_a^2}
\label{eq:SigmaQ1-Q1inverse}
\end{align}
where $U_{Q_1}\Sigma_{Q_1}V_{Q_1}^{\top}$ is the SVD of $Q_1$ and (i) is due to \cref{eq:XX-YY-Technical} and $\sigma_{\min}(\Sigma_1) \geq \sigma_a^2$.


Similarly, we can also obtain that there exists $Q_2 \in \R^{r\times r}$ such that $X_2 = U_2 \Sigma_2^{1/2} Q_2, Y_2 = V_2 \Sigma_2^{1/2} Q_2^{-\top}$ and
\begin{align*}
    \norm{\Sigma_{Q_2} - \Sigma_{Q_2}^{-1}}_{\F} \leq \frac{\epsilon}{\sigma_a^2} 
\end{align*}
where $U_{Q_2}\Sigma_{Q_2}V_{Q_2}^{\top}$ is the SVD of $Q_2$. 

We next intend to substitute $Q_1, Q_2$  by rotation matrices to simplify the analysis. 

Let $R_1 = U_{Q_1}V_{Q_1}^{\top}, R_2 = U_{Q_2}V_{Q_2}^{\top}$. Note that $R_1, R_2 \in O^{r\times r}$ are rotation matrices. Let $X_1' = X_1Q_1^{-1}R_1 = U_1 \Sigma_1^{1/2} R_1.$ In fact
\begin{align*}
    X_1' - X_1 &=  X_1 (Q_1^{-1}R_1 - I_{r}) \\
    &= X_1 (V_{Q_1} \Sigma_{Q_1}^{-1} U_{Q_1}^{\top} U_{Q_1} V_{Q_1}^{\top} - I_r)\\
    &\overset{(i)}{=} X_1 (V_{Q_1} \Sigma_{Q_1}^{-1} V_{Q_1}^{\top} - I_r)\\
    &\overset{(ii)}{=} X_1 (V_{Q_1} (\Sigma_{Q_1}^{-1} - I_r) V_{Q_1}^{\top})
\end{align*}
where (i) is due to $U_{Q_1}^{\top} U_{Q_1} = I_r$ and (ii) is due to $V_{Q_1} V_{Q_1}^{\top} = I_{r}.$ Then,
\begin{align*}
    \norm{X_1' - X_1}_{\F} 
    &\overset{(i)}{\leq} \norm{X_1} \norm{V_{Q_1}}\norm{\Sigma_{Q_1}^{-1} - I_r}_{\F} \norm{V_{Q_1}^{\top}}\\
    &\overset{(ii)}{\leq} \norm{X_1} \norm{\Sigma_{Q_1}^{-1} - I_r}_{\F}\\
    &\overset{(iii)}{\leq} \sigma_{b} \frac{\epsilon}{\sigma_a^2}
\end{align*}
where (i) is due to $\norm{ABC}_{\F} \leq \norm{A}\norm{B}_{\F}\norm{C}$, (ii) is due to $\norm{V_{Q_1}} = \norm{V_{Q_1}^{\top}} = 1$, and (iii) is due to $\norm{\Sigma_{Q_1}^{-1} - I_r}_{\F} \leq \norm{\Sigma_{Q_1}^{-1} - \Sigma_{Q_1}}_{\F}$ and \cref{eq:SigmaQ1-Q1inverse}.


Let $Y_1' = V_1 \Sigma_1^{1/2} R_1$, $X_2' = U_2 \Sigma_2^{1/2} R_2, Y_2' = V_2 \Sigma_2^{1/2} R_2$. Similarly, one can verify that
\begin{align}
    \max\left(\norm{Y_1' - Y_1}_{\F}, \norm{X_2' - X_2}_{\F}, \norm{Y_2' - Y_2}_{\F}\right) \leq \sigma_{b} \frac{\epsilon}{\sigma_a^2}. 
\end{align}

By triangle inequality, this guarantees that
\begin{align}
    \max(\norm{X_1' - X_2'}_{\F}, \norm{Y_1' - Y_2'}_{\F}) \leq a + \frac{2\sigma_b \epsilon}{\sigma_a^2}. \label{eq:X1prime-X2prime}
\end{align}

Note that we have the following decomposition for $X_1', X_2'$.
\begin{align}
    X_1' = U_1R_1R_1^{\top}\Sigma_1^{1/2} R_1 \label{eq:X1prime-decomposition}\\
    X_2' = U_2R_2R_2^{\top}\Sigma_2^{1/2} R_2. \label{eq:X2prime-decomposition}
\end{align}
In order to establish a bound for $U_1 R_1 - U_2 R_2$ based on $X_1' - X_2'$, we require an additional bound for $R_1^{\top}\Sigma_1^{1/2} R_1 - R_2^{\top}\Sigma_2^{1/2} R_2.$ Notice that the following provides a bound for $R_1^{\top}\Sigma_1 R_1 - R_2^{\top} \Sigma_2 R_2$.
\begin{align}
    \norm{R_1^{\top}\Sigma_1 R_1 - R_2^{\top} \Sigma_2 R_2}_{\F} \nonumber
    &\overset{(i)}{=} \norm{X_1^{'\top}X_1' - X_2^{'\top}X_2'}_{\F} \nonumber\\
    &\overset{(ii)}{\leq} \norm{(X_1^{'\top}-X_2^{'\top})X_1'}_{\F} + \norm{X_2^{'\top}(X_1' - X_2')}_{\F} \nonumber\\
    &\overset{(iii)}{\leq} \norm{X_1^{'\top}-X_2^{'\top}}_{\F} \norm{X_1'} + \norm{X_2^{'\top}} \norm{X_1' - X_2'}_{\F} \nonumber\\
    &\overset{(iv)}{\leq} (a + \frac{2\sigma_b \epsilon}{\sigma_a^2}) \norm{X_1'} + (a + \frac{2\sigma_b \epsilon}{\sigma_a^2})\norm{X_2'} \nonumber\\
    &\overset{(v)}{\leq} 2(a + \frac{2\sigma_b \epsilon}{\sigma_a^2}) \sigma_b.\label{eq:R1Sigma1R1-R2Sigma2R2}
\end{align}
where (i) is due to $X_1' = U_1 \Sigma_1^{1/2} R_1, X_2' = U_2 \Sigma_2^{1/2}R_2$, (ii) is due to the triangle inequality, (iii) is due to $\norm{ABC}_{\F} \leq \norm{A}\norm{B}_{\F}\norm{C}$, (iv) is due to \cref{eq:X1prime-X2prime}, and (v) is due to $\norm{X_1'}=\sigma_{\max}(\Sigma_1^{1/2}), \norm{X_2'} = \sigma_{\max}(\Sigma_2^{1/2})$ and $\sigma_{\max}(\Sigma_1) \leq \sigma_b^2, \sigma_{\max}(\Sigma_2) \leq \sigma_b^2$ (see \cref{eq:sigma-max-Sigma1}). 

By \cref{lem:square-root-bound}, this implies
\begin{align}
     \norm{R_1^{\top}\Sigma_1^{1/2} R_1 - R_2^{\top} \Sigma_2^{1/2} R_2} \nonumber
     &\leq \frac{1}{\sqrt{\sigma_{\min}(\Sigma_1)} + \sqrt{\sigma_{\min}(\Sigma_2)}} \norm{R_1^{\top}\Sigma_1 R_1 - R_2^{\top} \Sigma_2 R_2} \nonumber\\
     &\overset{(i)}{\leq}  \frac{1}{2\sigma_a} \norm{R_1^{\top}\Sigma_1 R_1 - R_2^{\top} \Sigma_2 R_2}_{\F} \nonumber\\
     &\overset{(ii)}{\leq} (a + \frac{2\sigma_b \epsilon}{\sigma_a^2}) \frac{\sigma_b}{\sigma_a} \label{eq:R1sqrtSigma1R1}
\end{align}
where (i) is due to $\sigma_{\min}(\Sigma_1) \geq \sigma_{a}^2, \sigma_{\min}(\Sigma_2) \geq \sigma_a^2$ and $\norm{A} \leq \norm{A}_{\F}$ for any matrix $A$, and (ii) is due to \cref{eq:R1Sigma1R1-R2Sigma2R2}.

Then we can establish the bound for $U_1R_1 - U_2R_2.$ Note that by \cref{eq:X1prime-decomposition,eq:X2prime-decomposition}, we have
\begin{align*}
X_1' - X_2' = (U_1R_1-U_2R_2)(R_1^{\top}\Sigma_1^{1/2}R_1) + U_2R_2(R_1^{\top}\Sigma_1^{1/2}R_1 - R_2^{\top}\Sigma_2^{1/2}R_2)
\end{align*}
which further implies
\begin{align}\label{eq:U1R1-U2R2}
 &U_1R_1-U_2R_2 \nonumber\\
 &= (X_1' - X_2') (R_1^{\top}\Sigma_1^{-1/2} R_1)+ U_2R_2(R_1^{\top}\Sigma_1^{1/2}R_1 - R_2^{\top}\Sigma_2^{1/2}R_2)(R_1^{\top}\Sigma_1^{-1/2} R_1).
\end{align}
Therefore, we can provide a bound for $\norm{U_1R_1-U_2R_2}_{\F}$ by the following.
\begin{align*}
    \norm{U_1R_1-U_2R_2}_{\F} 
    &\overset{(i)}{\leq} \norm{X_1'-X_2'}_{\F} \norm{R_1^{\top}\Sigma_1^{-1/2} R_1} \\
    &\quad + \norm{U_2R_2}_{\F}\norm{R_1^{\top}\Sigma_1^{1/2}R_1 - R_2^{\top}\Sigma_2^{1/2}R_2}\norm{R_1^{\top}\Sigma_1^{-1/2} R_1}\\
    &\overset{(ii)}{\leq} (a + \frac{2\sigma_b \epsilon}{\sigma_a^2}) \norm{R_1^{\top}\Sigma_1^{-1/2} R_1} +  \norm{U_2R_2}_{\F} (a + \frac{2\sigma_b \epsilon}{\sigma_a^2})  \frac{\sigma_b}{\sigma_a}\norm{R_1^{\top}\Sigma_1^{-1/2} R_1}\\
    &\overset{(iii)}{\leq}  (a + \frac{2\sigma_b \epsilon}{\sigma_a^2})  \sigma_{\max}(\Sigma_1^{-1/2}) + \sqrt{r}  (a + \frac{2\sigma_b \epsilon}{\sigma_a^2})  \frac{\sigma_b}{\sigma_a} \sigma_{\max}(\Sigma_1^{-1/2})\\
    &\overset{(iv)}\leq 2\sqrt{r}\frac{\sigma_b}{\sigma_a} (a + \frac{2\sigma_b \epsilon}{\sigma_a^2}) \frac{1}{\sigma_a}
\end{align*}
where (i) is due to the triangle inequality, (ii) is due to \cref{eq:X1prime-X2prime} and \cref{eq:R1sqrtSigma1R1}, (iii) is due to $\norm{R_1^{\top}\Sigma_1^{-1/2} R_1} = \sigma_{\max}(\Sigma_1^{-1/2})$ and $\norm{U_2R_2}_{\F} = \norm{U_2}_{\F} = \sqrt{r}$, (iv) is due to $\sigma_{\max}(\Sigma_1^{-1/2}) = 1/\sqrt{\sigma_{\min}(\Sigma_1)} \leq 1/\sigma_a.$ Similarly, one can obtain that $\norm{V_1R_1 - V_2R_2}_{\F} \leq 2\sqrt{r}\frac{\sigma_b}{\sigma_a} (a + \frac{2\sigma_b \epsilon}{\sigma_a^2}) \frac{1}{\sigma_a}.$ 

Note that $\norm{U_1R_1R_2^{\top} - U_2}_{\F} = \norm{U_1 R_1 - U_2 R_2}_{\F}$ and $\norm{V_1R_1R_2^{\top} - V_2}_{\F} = \norm{V_1 R_1 - V_2 R_2}_{\F}$, this completes the proof for \cref{eq:U-F-V-F}.

In addition, if $\norm{X_1 - X_2}_{2,\infty} \leq b$, we can obtain that
\begin{align*}
    \norm{X_1' - X_2'}_{2,\infty} 
    &\leq \norm{X_1 - X_2}_{2,\infty} + \norm{X_1' - X_1}_{\F} + \norm{X_2' - X_2}_{\F} \\
    &\leq b + \frac{2\sigma_b \epsilon}{\sigma_a^2}.
\end{align*}
Then based on \cref{eq:U1R1-U2R2}, we have
\begin{align*}
    \norm{U_1R_1 - U_2R_2}_{2,\infty} 
    &\overset{(i)}{\leq} \norm{(X_1'-X_2') R_1^{\top}\Sigma_1^{-1/2} R_1}_{2,\infty} \\
    &\quad +  \norm{U_2R_2(R_1^{\top}\Sigma_1^{1/2}R_1 - R_2^{\top}\Sigma_2^{1/2}R_2)R_1^{\top}\Sigma_1^{-1/2} R_1}_{2,\infty}\\
    &\overset{(ii)}{\leq} \norm{(X_1'-X_2')}_{2,\infty} \norm{R_1^{\top}\Sigma_1^{-1/2} R_1}\\
    &\quad + \norm{U_2}_{2,\infty}\norm{R_2}\norm{R_1^{\top}\Sigma_1^{1/2}R_1 - R_2^{\top}\Sigma_2^{1/2}R_2}\norm{R_1^{\top}\Sigma_1^{-1/2} R_1}\\
    &\overset{(iii)}{\leq} \left(b + \norm{U_2}_{2,\infty} a \frac{\sigma_b}{\sigma_a} + \frac{2\sigma_b \epsilon}{\sigma_a^2} (1+\norm{U_2}_{2,\infty}\sigma_b/\sigma_a)\right) \frac{1}{\sigma_a} 
\end{align*}
where (i) is due to triangle inequality, (ii) is due to $\norm{AB}_{2,\infty} \leq \norm{A}_{2,\infty} \norm{B}$, (iii) follows the same derivation for $\norm{U_1R_1-U_2R_2}_{\F}$. Notice that $\norm{U_1R_1R_2^{\top} - U_2}_{2,\infty} = \norm{U_1R_1 - U_2R_2}_{2,\infty}$. The similar results also can be obtained for $\norm{V_1R_1R_2^{\top} - V_2}_{2,\infty}.$ This completes the proof for \cref{eq:U1RU2-2-infty}.

\end{proof}